\makeatletter
\let\my@xfloat\@xfloat
\makeatother

\documentclass[11pt, oneside]{mnthesis}
\makeatletter
\def\@xfloat#1[#2]{
	\my@xfloat#1[#2]%
	\def\baselinestretch{1}%
	\@normalsize \normalsize
}
\makeatother
\usepackage{comment}
\usepackage{amsfonts}
\usepackage{mathtools}
\usepackage[mathscr]{euscript}
\usepackage{amsthm}
\usepackage{amsmath}
\usepackage{times}
\usepackage{graphicx} 
\usepackage{subfigure} 
\usepackage{algorithm}

\usepackage{algorithmic}
\usepackage[breaklinks]{hyperref}

\usepackage{lmodern}

\newcommand{\argmin}{\operatornamewithlimits{argmin}}
\newcommand{\cS}{\mathcal{S}}
\newcommand{\cR}{\mathcal{R}}
\newcommand{\cL}{\mathcal{L}}
\newcommand{\cH}{\mathcal{H}}
\newcommand{\bbR}{\mathbb{R}}
\newcommand{\Tr}{\mathrm{Tr}}
\newcommand{\bbE}{\mathbb{E}}
\newcommand{\diag}{\operatornamewithlimits{diag}}
\usepackage{tikz}
\newcommand*\circled[1]{\tikz[baseline=(char.base)]{
            \node[shape=circle,draw,inner sep=2pt] (char) {#1};}}

\newtheorem{theorem}{Theorem}[chapter]
\newtheorem{proposition}{Proposition}[chapter]
\newtheorem{corollary}{Corollary}[chapter] 
\newtheorem{lemma}{Lemma}[chapter]
\newenvironment{sproof}{%
  \proof}{\endproof}


\begin{document}


\title{\bf Online Convex Optimization in Changing Environments \\and its Application to Resource Allocation}
\author{Jianjun Yuan}
\campus{University of Minnesota} 
\program{Electrical and Computer Engineering} 
\degree{DOCTOR OF PHILOSOPHY}
\director{Andrew Lamperski, Advisor} 

\submissionmonth{December} 
\submissionyear{2019} 


\copyrightpage 
\acknowledgements{

First of all, I want to express my sincere gratitude to my smart advisor, Prof. Andrew Lamperski
for his unconditional support and guidance throughout my whole PhD study.
His enthusiasm towards the useful research, rigorous attitude, 
and good taste of the research topics
are infectious.
It would not be possible to complete this thesis 
without his invaluable encouragement 
as well as inspired interesting research advice.
As a mentor, he shared with me how to select good research topics,
how to tackle difficult/complex problems by getting ideas from easy ones,
how to write a good paper,
and
how to make a good presentation.
All these together shape me to become a qualified researcher,
who needs not only knowledge and solid understandings on his/her own area,
but also the capability to explore new directions and overcome new challenging problems. 
Moreover, I hope my PhD work can help push the research in my area a little bit forward
as promised to Andy at the beginning of my PhD study.
I will miss my every-Friday meeting
and the joyfulness of deriving various theoretical bounds on 
the white-board with him.

I want to thank Prof. Murti Salapaka, Prof. Mingyi Hong, and Prof. Steven Wu
for being in my thesis committee and providing feedback about my research.
I also want to thank Prof. Daniel Boley and Prof. Georgios B. Giannakis
for being in my preliminary exam committee and the useful suggestions for my future work.
I also want to thank all the professors that I have taken classes from
over my PhD study.

I am grateful to the University of Minnesota ECE department for the fellowship
in the academic year 2015-16, 
without which I would not be able to come to Minnesota to pursue my PhD degree.
I am indebted to University of Minnesota MnDRIVE program
for providing me with the Graduate Assistantship
in the academic year 2018-19.
I will not complete this thesis without its resource and financial support.
I want to express my gratitude to the ECE staff, 
especially Linda Jagerson, Jeanine Maiden, and Hallie White,
for their help and support.

I want to thank Prof. Mu Zhou, Prof. Weiwei Chen, Ruiliang Zhang, Prof. Ying Chen for their help 
during the start of my graduate study.
I would also like to thank my lab mates Venkat Ram Subramanian, Tyler Lekang,
Ran Tian, and Bolei Di for sharing with me useful experience and information.

I am grateful to my terrific friends Kun Xu, Zheran Li, Jiadong Chen,
Dr. Tengtao Li, Dr. Qianqian Fan, Dr. Bingzhe Li, Ruilin Dong, Chengyao Tan, Xiaonan Zhang,
Jiaji Qi, Dongsheng Ding, Xiangyi Chen
for their companion and support,
which make my PhD study full of happiness.

Finally, I want to thank my wonderful family.
I want to thank my parents, my sister, my brother-in-law,
and my parents-in-law for their constant encouragement and support.
Especially, I want to thank my cute wife, Dingyi Liu,
for her companion, encouragement, care, and inspiration.
I am also grateful to our baby to be born.


}
\dedication{To my parents, sister and dear wife. }
\abstract{
\noindent
In the era of the big data, we create and collect lots of data from all different kinds of sources:
the Internet, the sensors, the consumer market, and so on.
Many of the data are coming sequentially, and would like to be processed and understood 
quickly.
One classic way of analyzing data is based on batch processing,
in which the data is stored and analyzed in an offline fashion.
However, when the volume of the data is too large, it is much more difficult 
and time-consuming to do batch processing
than sequential processing.
What's more,
sequential data is usually changing dynamically, and 
needs to be understood on-the-fly
in order to capture the changes.
Online Convex Optimization (OCO) is a popular framework
that matches the above sequential data processing requirement.
Applications using OCO include online routing, online auctions, online classification and regression, 
as well as online resource allocation.
Due to the general applicability of OCO to the sequential data
and the rigorous theoretical guarantee,
it has attracted lots of researchers to 
develop useful algorithms to fulfill different needs.
In this thesis,
we show our contributions to OCO's development
by designing algorithms to adapt to changing environments.

In the first part of the thesis,
we propose algorithms to have better adaptivity 
by examining the notion of \emph{dynamic regret},
which compares the algorithm's cumulative loss against that incurred by a comparison sequence.
Dynamic regret extends a common performance measure known as static regret.
Since it may not be known whether the environment is dynamic or not,
it is desirable to take advantage of both regrets by having a trade-off between them.
To achieve that,
we discuss recursive least-squares algorithms
and show how forgetting factors can be used
to develop new OCO algorithms
that have such a regret trade-off.
More specifically,
we rigorously characterize the 
 effect of forgetting factors for a class of online Newton algorithms.
 For exp-concave or strongly convex objective, 
 the improved dynamic regret of $\max\{O(\log T),O(\sqrt{TV})\}$ is achieved,
 where $V$ is a bound on the path length of the comparison sequence.
 In particular, we show how classic
 recursive least-squares with a forgetting factor achieves this
 dynamic regret bound.
 By varying $V$, we obtain the regret trade-off.
 In order to obtain more computationally efficient algorithm, 
 we also propose a novel gradient descent step size rule for
 strongly convex functions,
 which recovers the dynamic regret bounds
 described above. 
 For smooth problems, we can obtain static regret of $O(T^{1-\beta})$ and
 dynamic regret of $O(T^\beta V^*)$, where $\beta \in (0,1)$ and $V^*$
 is the path length of the sequence of minimizers.  
 By varying $\beta$, we obtain the regret trade-off.

The second part of the thesis
describes how to design efficient algorithms to adapt to the changing environments.
Previous literature runs a pool of algorithms in parallel to gain better adaptivity,
which increases both the running time and the online implementation complexity.
Instead,
we propose a new algorithm requiring only one update per time step,
while with the same \emph{adaptive regret} performance guarantee as the current state-of-the-art result.
We then apply the algorithm to 
online Principal Component Analysis (online PCA)
and variance minimization under changing environments, since the
previous literature on online PCA has focused on performance guarantee under stationary environment.
We demonstrate both theoretically and experimentally that
the proposed algorithms can adapt to the changing environments.

The third part of the thesis starts from the observation that
the projection operator used in constrained OCO algorithms 
cannot really achieve true online implementation due to the high time-consumption.
To accelerate the OCO algorithms' update,
previous literature is proposed to approximate the true desired projection with a
simpler closed-form one at the cost of constraint violation ($g(\theta)> 0$) for some time steps.
Nevertheless, it can guarantee sub-linearity for both the static regret
and the long-term constraint, $\sum\limits_{t=1}^Tg(\theta_t)$,
having constraint satisfaction on average.
However, the sub-linear long-term constraint does not enforce small constraint violation for every time step,
because a strictly feasible solution can cancel out the effects of violated constraints.
To resolve it,
we propose algorithms to
have the cumulative constraint of the form $\sum\limits_{t=1}^T\big(\max\{g(\theta_t),0\}\big)^2$ upper bounded sub-linearly.
This new form heavily penalizes large constraint violations while the cancellation effects cannot occur. 
Furthermore, useful bounds on the single step constraint violation are derived.
For convex objectives, our result generalizes existing bounds, and for strongly convex objectives we give improved regret bounds.
In numerical experiments, we show that our algorithm closely follows the constraint boundary leading to low cumulative violation. 
Furthermore, we extend the proposed algorithms' idea to the more general time-dependent online resource allocation problems
with performance guarantee by a variant of \emph{dynamic regret}.

}
\words{771}    

\beforepreface 

\figurespage
\tablespage

\afterpreface            



\chapter{Introduction}
\label{chap:overall_intro}

\section{Why Online Convex Optimization?}
\label{sec:oco_gen_intro}

We live in an era that is full of accessible data:
the Internet, different sensors,
the consumer and financial markets, and so on.
Many people prefer to first store the entire data set and then process it together
like the classic machine learning algorithms do.
But when the volume of data is too large, 
such batch processing would fail or be computationally inefficient due to 
the large-scale dataset that needs to be loaded into the memory.
In order to understand and analyze the data quickly and efficiently,
we can treat the large-scale dataset 
as a stream,
so that it can be processed one data point-by-one data point or in a mini-batch fashion.
Online learning is such a popular framework
in doing so,
which is more computationally efficient because it does not require loading the whole dataset.
Moreover, it is theoretically guaranteed to be competitive with the
best fixed choice in hindsight (the batch processing solution).

Besides the advantages for offline data analysis,
online learning is also a natural choice for treating sequential data.
The reason is that it meets the requirement of processing and understanding
the data as quickly as possible.

One important framework of online learning is Online Convex Optimization (OCO),
which considers the case when the objective function is convex.
Featured by
the high computational efficiency and proven theoretical guarantees
against different notions of performance,
OCO algorithms have many applications in the areas of online routing \cite{hazan2016introduction}, 
online auctions \cite{blum2004online}, as well as online classification and regression \cite{crammer2006online}.
More specific applications are:
1. Online linear dynamical system identification \cite{hazan2018spectral}
that updates the identified system parameters on-the-fly as the sequential observation comes in;
2. Online expert selection \cite{herbster1998tracking,cesa2012new,cesa2012mirror}
which is about online decision making on the best expert;
3. Online Principal Component Analysis \cite{tsuda2005matrix,warmuth2006online,warmuth2008randomized,niew2016onlinepca}
in picking up the subspace for the sequential data/observation to be projected into;
4. Online resource allocation \cite{yu2017online,yuan2018online,liakopoulos2019cautious}
to sequentially allocate budgets or other resources.

Another property that makes the OCO framework unique from any other framework
is that there is no statistical assumptions on how the data/observation
is generated. It can be generated deterministically,
stochastically from a mixture of different fixed distributions,
or even adversarially.
This property is preferable when there is no clear conclusion on
what kind of data distribution we should use.
Also, it makes more sense
if the data is given by our opponent in an adversarial manner.
Such a game playing perspective \cite{abernethy2008optimal}
further enables the application
to the adversarial data processing like
online portfolio selection in the stock market \cite{helmbold1998line,das2013online}.

Influenced by the development of convex optimization tools,
there are many advances in the design of the OCO algorithms
trying to fulfill different needs under
different considerations and scenarios.
One of the most important considerations
is the guaranteed performance 
against different types of comparators.
For example,
for the best fixed comparator in hindsight (which is suitable in stationary environment),
the guaranteed performance
is called \emph{static regret} \cite{zinkevich2003online}.
For a changing environment, two other types of comparators are usually used.
One is the maximum static regret over any contiguous time interval,
which leads to the \emph{adaptive regret} performance guarantee \cite{hazan2009efficient}.
The other one is against all comparator sequences in a constrained set,
having performance guarantee named \emph{dynamic regret} \cite{besbes2015non,mokhtari2016online,zhang2018adaptive}.
In this thesis,
we show our contributions to the OCO's development
by designing algorithms to adapt to the changing environment.

\section{Motivation}
\label{sec:oco-gen-motivation}

The general motivation for this thesis
is to design OCO algorithms to enable the decision making on-the-fly 
with better adaptivity to the changing environments
and extend them to online resource allocation.

Tracking the changes of the environments is a key difference between
OCO algorithms and the batch processing based approaches,
because sequential data/observation tends to be shifting over time.
However, previous works on OCO problems are mainly focused on the static regret,
a performance metric well-suited for stationary environments.
Since the algorithms having sub-linear static regret 
will converge to the single
best fixed solution in hindsight,
their claimed performance somehow contradicts the original tracking goal.
In order to be aligned with the tracking objective,
dynamic regret is proposed
to let the cumulative loss of OCO algorithms compete
with any sequence of comparators within the constrained set.
Our first part of the thesis is motivated by designing algorithms to upper bound the dynamic regret 
for different types of problem setups.

Tracking the changing environments is usually achieved by running a pool of algorithms
with either different parameters (for upper bounding dynamic regret) 
or different starting points (for upper bounding adaptive regret).
Such complex online implementaion is very time-consuming and 
not appropriate in some problem setups. 
Proposing an efficient and easy-to-implement algorithm
is the goal of 
our second part of the thesis.
In particular, we show that such an algorithm
exists and can be applied to the online Principal Component Analysis (online PCA)
and the online variance minimization.
Compared to the mentioned general adaptive algorithms,
our proposed algorithm uses only one update per time step,
while maintains the same adaptive regret theoretical guarantee
as the general adaptive algorithms.

For constrained OCO algorithms, a projection operator is almost unavoidable.
When the constraint set is complex,
such operation is very time-consuming and prevents the algorithms from having a true online implementation.
Our third part of the thesis starts from the question of how to accelerate the computations.
Previous works propose to replace the true desired projection
with an approximate closed-form one,
since closed-form update eliminates any minimization-based computation.
The 'downside' is the possible constraint violation from time to time.
The remedy for it is the guarantee for the constraint satisfaction on average.
However, on-average constraint satisfaction does not lead to the desired small constraint violation for each time step.
To achieve that, we propose a new algorithm
to enforce small constraint violation not only on average but also for every time step.
The idea of the on-average constraint satisfaction
is also applied to online resource allocation by some previous works.
However,
such application is only limited to the budget type resource
because of the considered on-average constraint form.
Our second motivation in the third part of the thesis is to extend our proposed algorithms
to have time-dependent dynamic regret guarantee
in order to solve broader online resource allocation problems.

\section{Thesis Organization}
\label{sec:thesis-org}

The thesis is organized as follows:
\begin{enumerate}
	\item Chapter \ref{chap:related-work}
	discusses related work for Online Convex Optimization,
	that are relevant to the algorithms or problem setups we consider in the later chapters.

\item Chapter \ref{chap:trade-off}
is mainly concerned with the question of how to
enable the decision making on-the-fly with better adaptivity 
to the changing environments.
Algorithms equipped with static regret performance guarantee are
not appropriate
due to the fixed comparator they converge to. 
One way to better track the changes of the environments
is to use dynamic regret,
which compares the algorithm's cumulative loss 
against that incurred by a comparison sequence.
Inspired by the forgetting factor used in the Recursive Least Squares algorithms,
we propose a discounted Online Newton algorithm to
have improved dynamic regret guarantee
for both exp-concave and strongly convex objectives.
Moreover, the trade-off between static and dynamic regret
is analyzed for both Online Least-Squares and its generalization to
strongly convex and smooth objectives.
To obtain more computationally efficient algorithms, 
 we also propose a novel gradient descent step size rule for
 strongly convex functions,
 which recovers the dynamic regret bound
 described above. 

\item Chapter \ref{chap:adaptive-pca}
develops an online adaptive algorithm for Principal Component Analysis (PCA) and its extension
of variance minimization under changing environments.
The main idea is mixing the exponentiated gradient descent
with a fixed-share step.
Compared with the previous algorithms having adaptive
or dynamic regret guarantee,
our algorithm saves the need of running a pool of algorithms in parallel,
while achieves the same adaptive regret performance guarantee.

\item Chapter \ref{chap:oco-long-term}
contributes to the development of the OCO algorithms in achieving fast online computation
as well as online resource allocation.
The projection operator for the constrained OCO algorithms is the main bottleneck 
in preventing the algorithms from having a quick update.
We propose algorithms to approximate the true desired projection
with a simpler closed-form one
at the cost of the constraint violation for some time steps.
Nevertheless,
our proposed algorithms lead to
a sub-linear cumulative constraint violation
to ensure the constraint satisfaction on average.
It also has mild and bounded single step constraint violation.
For convex objectives, our results generalize existing ones, 
and for strongly convex objectives we give improved regret bounds.
Finally,
we extend our proposed algorithms' idea to solve the general time-dependent online resource allocation problems.

\item Chapter \ref{chap:conclusion}
draws some conclusions for the thesis.

\end{enumerate}

\section{Notation}
\label{sec:notation}

For the $n$ dimensional vector $\theta\in \mathbb{R}^n$, 
we use $\left\|\theta\right\|_1$ and $\left\|\theta\right\|$ to denote the $\ell_1$-norm and $\ell_2$-norm, respectively.
The gradient and Hessian of the function $f_t$ at time step $t$ in terms of the $\theta$
are denoted as $\nabla f_t(\theta)$ and $\nabla^2 f_t(\theta)$, respectively.
In order to differentiate between the vector at time step $i$ and the $i$-th element of it,
we sometimes use bold lower-case symbols to denote the vector.
The $i$-th element of a sequence of vectors at time step $t$, $\mathbf{x_t}$, 
is denoted by $x_{t,i}$.

For two probability vectors $\mathbf{q}, \mathbf{w} \in \mathbb{R}^n$, we use $d(\mathbf{q},\mathbf{w})$
to represent the relative entropy between them, which is defined as 
$\sum_{i=1}^n q_i\ln(\frac{q_i}{w_i})$. 
$\mathbf{q_{1:T}}$ is the sequence of vectors $\mathbf{q_1},\dots,\mathbf{q_T}$, 
and $m(\mathbf{q_{1:T}})$ is defined to be equal to $\sum\limits_{t=1}^{T-1}D_{TV}(\mathbf{q_{t+1}},\mathbf{q_t})$,
where $D_{TV}(\mathbf{q_t},\mathbf{q_{t-1}})$ is defined as
$\sum\limits_{i:q_{t,i}\ge q_{t-1,i}} (q_{t,i}-q_{t-1,i})$. The
expected value operator is denoted by $\mathbb{E}$.

For the matrix $A\in \mathbb{R}^{m\times n}$, its transpose is denoted by $A^\top$ and $A^\top A$ denotes the matrix multiplication.
The inverse of $A$ is denoted as $A^{-1}$. 
We use $\left\|A\right\|_2$ to represent the induced $2$ norm.
For the two square matrices $A\in\mathbb{R}^{n\times n}$ and $B\in\mathbb{R}^{n\times n}$,
$A\preceq B$ means $A-B$ is negative semi-definite, 
while $A\succeq B$ means $A-B$ is positive semi-definite.
For a positive definite matrix, $M$, let $\|x\|_M^2 = x^\top M x$. The
standard inner product between matrices is given by  $\langle A,B\rangle =
\Tr(A^\top B)$. The  determinant of a square matrix, $A$ is denoted by
$|A|$. 
We use $I$ to represent the identity matrix.

The quantum relative entropy between two 
density matrices\footnote{A density matrix is a symmetric positive
  semi-definite matrix with trace equal to 1. 
Thus, the eigenvalues of a density matrix form a probability vector.}
$P$ and $Q$ is defined as $\Delta(P,Q) = \Tr(P\ln P)- \Tr(P\ln Q)$, 
where $\ln P$ is the matrix logarithm for symmetric positive definite matrix $P$
(and $\exp(P)$ is the matrix exponential).

\chapter{Related Work}
\label{chap:related-work}

In this chapter,
we do a literature review for the works that are related to the contents of this thesis
as well as some necessary background and concepts.

As the previous chapter shows, online learning has
attracted lots of researchers to develop different algorithms
for many interesting settings and applications.
Some of them are concerned with more theoretical parts.
One particular aspect is deriving lower and upper bounds for the performance  
in various problem setups such as
the expert problem \cite{herbster1998tracking,cesa2007improved,cesa2012new,cesa2012mirror},
the general OCO setup \cite{zinkevich2003online,hazan2007logarithmic,abernethy2008optimal,abernethy2009stochastic},
online Reinforcement Learning \cite{fazel2018global},
online non-convex optimization \cite{hazan2017efficient,gao2018online},
the online bandit problem \cite{auer2002using,agarwal2011stochastic,bubeck2012regret},
and so on.
Other works apply or extend the existing algorithms to different scenarios.
Besides the ones mentioned in the previous chapter,
other scenarios include
online time-series prediction with ARMA/ARIMA \cite{anava2013online},
online controller design \cite{yuan2017online},
as well as
the well-known classification algorithm AdaBoost \cite{freund1997decision}.

Amongst all the techniques and applications mentioned above, 
Online Convex Optimization (OCO)
is one of the most important unified frameworks
that provides
efficient,
and theoretically guaranteed solutions
to many problems
and helps facilitate
the development of online learning's theoretical analysis.

This chapter is divided into three sections
with the literature review ranging from
classic OCO algorithms
to the applications related to this thesis.
More specifically, 
Section \ref{sec:oco-work-intro}
first discusses the basic concepts and definitions in the OCO framework.
It then covers popular OCO algorithms like Online Gradient Descent and
Online Newton's method,
with different performance guarantees.
Section \ref{sec:online-PCA-work}
focuses on a specific problem setup,
online Principal Component Analysis (online PCA).
It describes one classic online algorithm
as well as some extensions of it from the literature.
Section \ref{sec:oco-long-term-work} does the literature review
about how the previous works try to accelerate the online update in the OCO algorithms.
Two different kinds of algorithms are described and discussed 
for their pros and cons.
Furthermore, the extensions of them to handle
online resource allocation are also included.

\section{Online Convex Optimization (OCO)}
\label{sec:oco-work-intro}

The formula for Online Convex Optimization (OCO) is: 
at each time step $t$, before the true time-dependent convex objective function $f_t(\theta)$ is revealed,
we need to make a prediction $\theta_t$ from the convex set $\cS$, 
based on the history of the observations $f_i(\theta)$, $i<t$.
Then the value $f_t(\theta_t)$ is the loss suffered due to the lack
of the knowledge of the true objective function $f_t(\theta)$.  
Our prediction of $\theta$ is then updated to include the information of $f_t(\theta)$.
This whole process is repeated until termination. The convex function,
$f_t(\theta)$,
can be chosen from the convex function class in an arbitrary, possibly
adversarial manner.

To better understand different OCO algorithms, we first describe 
the basic definitions and concepts related to them.

\subsection{The Basics of OCO}

Since the key to the design of OCO algorithms is the convex optimization tools,
we would like to first discuss some important concepts about the convex optimization.

A set $\cS$ is a convex set if $\forall x,y\in \cS$, $\forall \lambda\ge 0,\mu \ge 0$ such that $\lambda+\mu =1$,
we have that
$\lambda x + \mu y\in \cS$.

A function $f:\cS\mapsto \mathbb{R}$ is convex if $\forall x,y\in\cS$ and $\forall \alpha\in[0,1]$,
we always have:
\begin{equation*}
f(\alpha x+(1-\alpha)y) \le \alpha f(x) + (1-\alpha)f(y)
\end{equation*}

If $f(x)$ is first-order differentiable,
then $f(x)$ is convex if and only if 
\begin{equation*}
f(x)\ge f(y)+ \nabla f(y)^\top (x-y), \forall x,y\in\cS
\end{equation*}

For second-order differentiable function $f(x)$,
it is convex if and only if $\nabla^2 f(x)\succeq 0$.

For the non-differentiable convex function $f(x)$,
the above inequality still holds when
we replace the gradient $\nabla f(y)$ with any element of
the sub-gradient, $\partial f(y)$, which is defined as 
the set of vectors satisfying the above inequality for all $x\in\cS$.

When a convex function $f(x)$ is $\ell$-strongly convex,
it means $\forall x,y\in\cS$, we have
\begin{equation*}
f(x) \ge f(y)+\nabla f(y)^\top (x-y) + \frac{\ell}{2}\|x-y\|^2
\end{equation*}

If $f(x)$ is second-order differentiable,
$\ell$-strong convexity is equivalent to
$\nabla^2 f(x)\succeq \ell I$.

Sometimes the convex function $f(x)$ is also $\mu$-smooth,
which means its gradient $\nabla f(x)$ satisfies the relation
\begin{equation*}
\|\nabla f(x)-\nabla f(y)\|\le \mu\|x-y\|,
\end{equation*}
which is also equivalent to
$f(x)\le f(y)+\nabla f(y)^\top (y-x) + \frac{\mu}{2}\|x-y\|^2$.

The projection operator $\Pi_\cS(y)$ is defined as
$\argmin_{x\in\cS}\|x-y\|$.
An important property of this operator that we use a lot in this thesis
is the Pythagorean theorem, which is listed below for completeness:
\begin{theorem}[Pythagoras, circa 500 BC]
Let $y\in\mathbb{R}^n$, $\cS\subseteq\mathbb{R}^n$ be a convex set,
and $x = \Pi_\cS(y)$. Then we have the following inequality
\begin{equation*}
\|y-z\| \ge \|x-z\|, \forall z\in\cS
\end{equation*}

\end{theorem}

Many OCO algorithms are designed by using the above convex optimization tools.
To measure the effectiveness of these OCO algorithms,
one commonly used metric is called \emph{regret}.
\emph{Static regret} $\cR_s$ is one type of the regret defined as
\begin{equation*}
\cR_s = \sum\limits_{t=1}^T f_t(\theta_t)- \sum\limits_{t=1}^T f_t(\theta^*)
\end{equation*}
where $\theta_1,\theta_2,\dots,\theta_T$ is the prediction sequence given by the OCO algorithm,
$\theta^* =\argmin_{\theta\in\cS}\sum\limits_{t=1}^T f_t(\theta)$ is a fixed comparator,
and $T$ is called time horizon.

According to \cite{cesa2006prediction}, the solution to the above static regret
is called Hannan consistent if $\cR_s$ is sub-linear in $T$,
which means the prediction sequence will converge to $\theta^*$, the best fixed solution in hindsight.
In order to achieve the useful regret bound,
the following assumptions are required:
1. the gradient $\nabla f_t(\theta)$ is upper bounded;
2. the convex constraint set $\cS$ is compact and bounded.

\subsection{Online Gradient Descent}

The most classic OCO algorithm designed for convex objective is called Online Gradient Descent (OGD) 
proposed by \cite{zinkevich2003online} in 2003.
The update rule after the observation $f_t(\theta)$ is
\begin{equation*}
\begin{array}{l}
\hat{\theta}_{t+1} = \theta_t - \eta_t\nabla f_t(\theta_t)\\
\theta_{t+1} = \Pi_\cS(\hat{\theta}_{t+1}) 
\end{array}
\end{equation*}
where $\eta_t$ is the step size at time step $t$ 
and we abuse the subgradient notation when $f_t$ is not differentiable
by denoting it as $\nabla f_t(\theta_t)$.

Although the above update rule is very simple, which is just doing gradient descent and then projecting back
to the feasible set,
it has an optimal static regret theoretical guarantee.
In other words, 
by setting $\eta_t$ to be equal to $O(1/\sqrt{T})$ or $O(1/\sqrt{t})$,
the $\cR_s$ can be upper bounded by $O(\sqrt{T})$,
which meets the lower bound shown in \cite{abernethy2008optimal}.
Note that when $f_t$ is $\ell$-strongly convex,
$\cR_s$ can be upper bounded by $O(\log T)$ by having $\eta_t = \frac{1}{\ell t}$,
which is also optimal \cite{hazan2007logarithmic}.

According to \cite{beck2003mirror}, OGD is a special case of Online Mirror Descent (OMD)
when the distance function $\psi$ is the squared Euclidean one ($\psi(x) = \frac{1}{2}\|x\|^2$).
For the OMD algorithm,
its prediction for the time step $t+1$ is updated as
\begin{equation*}
\theta_{t+1} = \argmin\limits_{\theta\in\cS} \eta_t\nabla f_t(\theta_t)^\top (\theta-\theta_t) + d_\psi(\theta,\theta_t)
\end{equation*}
where $d_\psi$ is the Bregman divergence
defined as 
$d_\psi(x,y) = \psi(x) - \psi(y) - \nabla \psi(y)^\top (x-y)$
with $\psi$ being the strongly convex differentiable function.

When the constraint set $\cS$ is specified to the unit simplex constraint
$\cS = \{\theta: \|\theta\|_1 =1, \theta\ge 0\}$,
the above OMD update rule has closed-form solution
if the Bregman divergence $d_\psi$ is replaced 
by the relative entropy.
The closed-form update is
\begin{equation*}
\begin{array}{l}
\theta_{t+1,i} = \frac{\theta_{t,i}\exp(-\eta_t\nabla f_t(\theta_t)_{i})}{\sum_{j=1}^n \theta_{t,j}\exp(-\eta_t\nabla f_t(\theta_t)_{i})}
\end{array}
\end{equation*}
which is called Exponentiated Gradient Descent \cite{shalev2012online}.

\subsection{Online Newton Step}

The OMD and OGD are designed for the general convex objective function.
When the objective function is $\alpha$-exp-concave,
we could use Online Newton Step (ONS) 
to further reduce its upper bound from $O(\sqrt{T})$ to $O(\log T)$.

The definition of being $\alpha$ exp-concave is that 
the function $e^{-\alpha f_t(\theta)}$ is concave.
If $f_t$ is twice differentiable, it can be shown that $f_t$ is 
$\alpha$-exp-concave if and only if
  \begin{equation*}
  \nabla^2 f_t(x) \succeq \alpha
  \nabla f_t(x) \nabla f_t(x)^\top
  \end{equation*}
  for all $x\in \cS$. 
What's more, class of being exp-concave functions is broader than the strongly convex class if the gradient is bounded,
as shown in \cite{hazan2007logarithmic,yuan2019trading}.

The update rule of ONS \cite{hazan2007logarithmic} is described below
\begin{equation*}
\begin{array}{l}
A_t = A_{t-1} + \nabla f_t(\theta_t)\nabla f_t(\theta_t)^\top \\
\hat{\theta}_{t+1} = \theta_t - \frac{1}{\gamma}A_t^{-1}\nabla f_t(\theta_t) \\
\theta_{t+1} = \Pi_\cS^{P_t}(\hat{\theta}_{t+1})
\end{array}
\end{equation*}
where $\Pi_\cS^{P_t}(y) = \argmin_{z\in \cS} \|z-y\|_{P_t}^2$
is the projection onto $\cS$ with respect to the norm induced by $P_t$. 

\subsection{Dynamic OCO}

When the goal of the OCO algorithm
is to track the changes of the underlying environments,
the classic static regret is not appropriate anymore.
This is because 
the algorithms achieving sub-linear static regret only guarantee
that the prediction will converge to the single best fixed solution in hindsight \cite{hall2013dynamical,besbes2015non,zhang2018adaptive}.

In order to better track the changes of the underlying environments, 
\emph{dynamic regret} is proposed to compare the cumulative loss against that
incurred by a comparison sequence, $z_1,\ldots,z_T\in \cS$:
\begin{equation*}
\mathcal{R}_d = \sum\limits_{t=1}^T f_t(\theta_t) 
- \sum\limits_{t=1}^T f_t(z_t)
\end{equation*}

The classic OGD \cite{zinkevich2003online} achieves dynamic regret
of order $O(\sqrt{T}(1+V))$, where $V$ is a bound on the
path length of the comparison sequence:
\begin{equation*}
  \sum_{t=2}^{T} \|z_{t}-z_{t-1}\| \le V.
\end{equation*}
This has been improved to $O(\sqrt{T(1+V)})$ in \cite{zhang2018adaptive} by applying a
meta-optimization over different step sizes. 

There are also other ways to bound the dynamic regret
including a variant of path-length \cite{hall2013dynamical}, functional variation \cite{besbes2015non},
as well as gradient variation \cite{chiang2012online}.

\subsection{Adaptive OCO}

Adaptive OCO algorithms are also concerned with 
how to enable the algorithms to better track changing environments.
Different from the Dynamic OCO setup,
the Adaptive OCO uses a 'different' performance metric
called \emph{adaptive regret}
defined as the maximum static regret over any contiguous time interval
\begin{equation*}
\cR_a = \max_{[r,s]\subset [1,T]}\Big\{
\sum\limits_{t=r}^s f_t(\theta_t) - \min\limits_{\theta\in \cS}\sum\limits_{t=r}^s f_t(\theta)\Big\}
\end{equation*}

To upper bound the adaptive regret $\cR_a$,
\cite{hazan2009efficient} proposed to run a pool of OGD or OMD with different step sizes
and different starting points.
Compared with the classic OGD or OMD,
its running time is increased by an order of $\log(T)$
due to the total number of the parallel running algorithms is $O(\log T)$.

Most recently, \cite{cesa2012new,cesa2012mirror} discovered that
for the specific online expert problem,
there is no need to run a pool of algorithms in order to bound the adaptive regret.
Instead, 
they showed that the same adaptive regret performance guarantee
can be obtained by incorporating the fixed-share step \cite{herbster1998tracking}
into the Exponentiated Gradient update,
which not only reduces the running time by $O(\log T)$,
but also makes the update easy to implement.

\section{Online Principal Component Analysis}
\label{sec:online-PCA-work}

The purpose of the online Principal Component Analysis (online PCA)
is to find the underlying subspace 
for the sequential data/observation to be projected to \cite{warmuth2008randomized,niew2016onlinepca}.

To achieve sub-linear static regret,
\cite{warmuth2008randomized} extended the idea of selecting the subset of experts
to the subset selection of the subspace.
Due to the eigendecomposition at every time step,
the online PCA's computational complexity is $O(n^3)$, where $n$ is the dimension of the data/observation.
This online PCA idea was used in the online variance minimization \cite{warmuth2006online}.

In order to reduce the running time,
we need to avoid the eigendecomposition step in \cite{warmuth2008randomized}.
\cite{kotlowski2015pca} proposed another algorithm replacing the full eigendecomposition
at each time step by the problem finding
$k$ principal components of the current covariance matrix that is perturbed by Gaussian
noise. In this way, the algorithm requires $O(kn^2)$ per time step 
with a worse static regret bound, which is off by a factor of $O(n^{1/4})$.

\section{OCO with Long-term Constraint}
\label{sec:oco-long-term-work}

Online Convex Optimization (OCO) with long-term constraint
is first proposed by \cite{mahdavi2012trading} in 2012,
aiming to accelerate the OCO algorithms to achieve real online computation.
The problem it tried to solve is the high computational complexity
of the projection operator step for constrained OCO algorithms. 


To do that, it used a closed-form update
to approximate the true desired projection step 
at the cost of the constraint violation for some time steps.
Its main goal is still keeping the static regret in a sub-linear order,
but it also aims to make sure that there is no constraint violation on average.
More specifically,
it can get $\cR_s \le O(\sqrt{T})$, 
while the sum of the constraint functions $\sum\limits_{t=1}^T g(\theta_t)$ is upper bounded by $O(T^{3/4})$.

The above result is later improved by \cite{jenatton2016adaptive}
via designing a version with time-dependent step size, 
which can have $\cR_s\le O(T^{\max\{\beta,1-\beta\}})$ and 
$\sum\limits_{t=1}^Tg(\theta_t)\le O(T^{1-\beta/2})$ with $\beta\in(0,1)$.

Later on, \cite{yu2017online} considered the stochastic version of the problem. 
Instead of following update idea in \cite{mahdavi2012trading,jenatton2016adaptive},
it used the idea in the stochastic network optimization to handle time-dependent constraints.
Although both the static regret and the long-term constraint $\sum\limits_{t=1}^Tg(\theta_t)$
can be upper bounded by $O(\sqrt{T})$,
it requires a very strong additional Slater condition,
which does not hold for many problems like equality constraint.

The long-term idea is also extended to do the online resource allocation.
\cite{yu2017online} applied it to the online job scheduling (although not appropriate as explained in Chapter \ref{chap:oco-long-term}).
\cite{liakopoulos2019cautious} used the long-term idea
in an online budget allocation problem.
Compared with \cite{yu2017online}, 
\cite{liakopoulos2019cautious} has a tighter regret guarantee
due to the increasing difficulty in finding a feasible comparator in \cite{yu2017online}.

\chapter{Trading-Off Static and Dynamic Regret in Online Least-Squares
and Beyond}
\label{chap:trade-off}

In this chapter, we are mainly concerned with online discounted recursive least-squares
and how the discounted factor idea can be used to derive 
improved dynamic regret as well as dynamic/static regret trade-off in different problem setups.

As discussed in the previous chapters, 
the general procedure for online learning algorithms
is as follows:
at each time $t$, before the true time-dependent objective function $f_t(\theta)$ is revealed,
we need to make the prediction, $\theta_t$, 
based on the history of the observations $f_i(\theta)$, $i<t$.
Then the value of $f_t(\theta_t)$ is the loss suffered due to the lack
of the knowledge for the true objective function $f_t(\theta)$.  
Our prediction is then updated to include the information of $f_t(\theta)$.
This whole process is repeated until termination. The functions,
$f_t(\theta)$,
can be chosen from a function class in an arbitrary, possibly
adversarial manner.

The performance of an online learning algorithm is typically assessed
using various notions of \emph{regret}.  
\emph{Static regret}, $\mathcal{R}_s$, measures the difference between the algorithm's cumulative loss 
and the cumulative loss of the best fixed decision in hindsight
\cite{cesa2006prediction}: 
\begin{equation*}
\mathcal{R}_s = \sum\limits_{t=1}^T f_t(\theta_t) 
- \min\limits_{\theta\in \cS}\sum\limits_{t=1}^T f_t(\theta),
\end{equation*}
where $\cS$ is a constraint set. For convex functions, variations of
gradient descent achieve static regret of $O(\sqrt{T})$, while for strongly
convex functions these can be improved to $O(\log T)$
\cite{hazan2016introduction}. 
However, when the underlying environment is changing, 
due to the fixed comparator \cite{herbster1998tracking} the algorithm converges to,
static regret is no longer appropriate.

In order to better track the changes of the underlying environments, 
\emph{dynamic regret} is proposed to compare the cumulative loss against that
incurred by a comparison sequence, $z_1,\ldots,z_T\in \cS$:
\begin{equation*}
\mathcal{R}_d = \sum\limits_{t=1}^T f_t(\theta_t) 
- \sum\limits_{t=1}^T f_t(z_t)
\end{equation*}
The classic work on online gradient descent \cite{zinkevich2003online} achieves dynamic regret
of order $O(\sqrt{T}(1+V))$, where $V$ is a bound on the
path length of the comparison sequence:
\begin{equation*}
  \sum_{t=2}^{T} \|z_{t}-z_{t-1}\| \le V.
\end{equation*}
This has been improved to $O(\sqrt{T(1+V)})$ in \cite{zhang2018adaptive} by applying a
meta-optimization over step sizes.

In works 
such as \cite{mokhtari2016online,yang2016tracking}, it is assumed that
$z_t = \theta_t^* = \argmin_{\theta \in \cS} f_t(\theta)$.
We denote that particular version of dynamic regret by:
\begin{equation*}
\mathcal{R}_d^* = \sum\limits_{t=1}^T f_t(\theta_t) 
- \sum\limits_{t=1}^T f_t(\theta^*_t)
\end{equation*}
In particular, if $V^*$ is the corresponding path length:
\begin{equation}
  \label{eq:optimizerLength}
 V^* = \sum\limits_{t=2}^T\left\|\theta_t^*-\theta_{t-1}^*\right\|,
\end{equation}
then \cite{mokhtari2016online} shows that for strongly convex functions, $\mathcal{R}_d^*$ of order
$O(V^*)$ is obtained by gradient descent.
However, as pointed out by \cite{zhang2018adaptive},
$V^*$ metric is too pessimistic and unsuitable for stationary problems,
which will result in poor generalization due to the
random perturbation caused by sampling from the \emph{same} distribution.
Thus, a trade-off between static regret $\cR_s$ and dynamic regret $\cR_d^*$
is desired to maintain the abilities of both generalization to stationary problem and tracking to the local changes.

\emph{Adaptive regret} \cite{hazan2009efficient} is another metric when dealing with changing environments,
which is defined as the maximum static regret over any contiguous time interval.
Although it shares the similar goal as the dynamic regret, their relationship is still an open question.

Closely related to the problem of online learning is adaptive
filtering, in which time series data is predicted using a filter that
is designed from past data \cite{sayed2011adaptive}. The performance
of adaptive filters is typically measured in an average case setting
under statistical assumptions.
One of the most famous adaptive filtering techniques is recursive least
squares, which bears strong resemblance to the online Newton method of
\cite{hazan2007logarithmic}. The work in \cite{hazan2007logarithmic}
proves a static regret bound of $O(\log T)$ for online Newton methods, but dynamic regret bounds
are not known.

In order to have an algorithm that adapts to non-stationary data, it
is common to use a forgetting factor in recursive least squares.
\cite{guo1993performance} analyzed the effect of the forgetting factor
in terms of the tracking error covariance matrix,
and \cite{zhao2019distribution} made the tracking error analysis
with the assumptions that the noise is sub-Gaussian and 
the parameter follows a drifting model.
However, none of the analysis mentioned is done in terms of the regret,
which eliminates any noise assumption.
For the online learning,
\cite{garivier2011upper} analyzed the discounted UCB,
which uses the discounted empirical average as the estimate for the upper confidence bound.
\cite{russac2019weighted} used the weighted least-squares to update the linear bandit's underlying parameter.

This chapter is adapted from the published work \cite{yuan2019trading}, 
and we are mainly concerned with exp-concave and strongly convex objectives.
The following is a summary of the main results:
\begin{enumerate}
\item For exp-concave and strongly convex problems, we propose a discounted
  Online Newton algorithm  which generalizes recursive least squares
  with forgetting factors and the original online Newton method of
  \cite{hazan2007logarithmic}. We show how tuning the forgetting
  factor can achieve a dynamic regret bound of $\cR_d \le \max\{O(\log
  T),O(\sqrt{TV})\}$. This gives a rigorous analysis of forgetting
  factors in recursive least squares and improves the bounds described in \cite{zhang2018adaptive}.
  However, this choice requires a bound on the path length, $V$.
  For an alternative choice of forgetting factors, which does not
  require path length knowledge, we can
  simultaneously bound static regret by
  $\cR_s\le O(T^{1-\beta})$ and dynamic regret by $\cR_d\le
  \max\{O(T^{1-\beta}),O(T^\beta V)\}$. Note that tuning $\beta$
  produces a trade-off between static and dynamic regret. 
\item Based on the analysis of discounted recursive least squares, we
  derive a novel step size rule for online gradient descent. 
  Using this step size rule for smooth, strongly convex functions we
  obtain a static regret bound of
$\cR_s\le O(T^{1-\beta})$ and  a dynamic regret bound against
$\theta_t^* = \argmin_{\theta \in \cS} f_t(\theta)$ of $\cR_d^*\le
O(T^{\beta}(1+V^*))$. This improves the trade-off obtained in the
exp-concave case, since static regret or dynamic regret can be made
small by appropriate choice of $\beta \in (0,1)$. 
\item We show how the step size rule can be modified further so that
  gradient descent recovers the $\max\{O(\log T),O(\sqrt{TV})\}$
  dynamic regret bounds obtained by discounted Online Newton
  methods. However, as above, these bounds require knowledge of the
  bound on the path length, $V$.
\item Finally, we describe a meta-algorithm, similar to that used in
  \cite{zhang2018adaptive}, which can recover the $\max\{O(\log
  T),O(\sqrt{TV})\}$ dynamic regret bounds without knowledge of
  $V$. These bounds are tighter than those in
  \cite{zhang2018adaptive}, since they exploit exp-concavity to reduce
  the loss incurred by running an experts algorithm. 
  Furthermore, we give a lower bound for the corresponding problems,
  which matches the obtained upper bound for certain range of $V$.
  
\end{enumerate}

\section{Discounted Online Newton Algorithm}
\label{sec:discounted-ONS}

As described above, the online Newton algorithm from
\cite{hazan2007logarithmic} strongly resembles the classic recursive
least squares algorithm from adaptive filtering
\cite{sayed2011adaptive}. Currently, only the static regret of
the online Newton method is studied. To obtain more adaptive
performance, forgetting factors are often used in recursive least
squares. However, the regret of forgetting factor algorithms has not
been analyzed. This section proposes a class of algorithms that
encompasses recursive least squares with forgetting factors and the
online Newton algorithm. We show how dynamic regret bounds for these
methods can be obtained by tuning the forgetting factor.

First we describe the problem assumptions. Throughout this chapter we
assume that $f_t : \cS \to \bbR$ are convex, differentiable functions, $\cS$ is a
compact convex set, $\|x\|\le D$ for all $x\in \cS$, and $\|\nabla
f_t(x)\| \le G$ for all $x\in \cS$.  
Without loss of generality, we assume throughout the chapter that $D\ge1$.

In this section we assume that all of the objective functions,
$f_t:\cS\to \bbR$
are $\alpha$-exp-concave for some $\alpha >0$. This means that $e^{-\alpha f_t(\theta)}$ is
concave.

If $f_t$ is twice differentiable, it can be shown that $f_t$ is 
$\alpha$-exp-concave if and only if
  \begin{equation}
    \label{expHess}
  \nabla^2 f_t(x) \succeq \alpha
  \nabla f_t(x) \nabla f_t(x)^\top
  \end{equation}
  for all $x\in \cS$.

For an $\alpha$-exp-concave function $f_t$, Lemma 4.2 of \cite{hazan2016introduction} implies that
the following
bound holds for all $x$ and $y$ in $\cS$
with
$\rho$ $\le$ $\frac{1}{2}\min\{\frac{1}{4GD},\alpha\}$:
\begin{subequations}
  \label{eq:functionBounds}
  \begin{align}
    \label{eq:expBound}
     f_t(y)\ge f_t(x)+\nabla f_t(x)^\top (y-x)+ 
\frac{\rho}{2}(x-y)^\top\nabla f_t(x)\nabla
    f_t(x)^\top(x-y).
  \end{align}

  In some variations on the algorithm, we will require extra
  conditions on the function, $f_t$. 
  In particular, in one variation
  we will require $\ell$-strong convexity. As shown in the previous chapter,
  this means that there is a number
  $\ell >0$ such that
  \begin{align}
    \label{eq:strongConvex}
 f_t(y)\ge f_t(x)+\nabla f_t(x)^\top (y-x) + \frac{\ell}{2}\|x-y\|^2 
  \end{align}
  for all $x$ and $y$ in $\cS$. 
  For twice-differentiable functions,  
  strong convexity implies $\alpha$-exp-concavity for $\alpha \le
  \ell /G^2$ on $\cS$.

  In another variant, we will require that the following bound holds
  for all $x$ and $y$ in $\cS$: 
  \begin{align}
    \label{eq:quadBound}
    f_t(y)\ge f_t(x) + \nabla f_t(x)^\top (y-x) + \frac{1}{2}
      \|x-y\|_{\nabla^2 f_t(x)}^2.
\end{align}
This bound does not correspond to a commonly used convexity class, but
it does hold for the important special case of quadratic functions:
$f_t(x) = \frac{1}{2}\|y_t - A_t x\|^2$. This fact will be important for
analyzing the classic discounted recursive least-squares
algorithm. Note that if $y_t$ and $A_t$ are restricted to compact
sets, $\alpha$ can be chosen so that $f_t$ is $\alpha$-exp-concave.

Additionally, the algorithms for strongly convex functions and those
satisfying (\ref{eq:quadBound}) will require that the gradients
$\nabla f_t(x)$ are $u$-Lipschitz for all $x\in \cS$ 
(equivalently, $f_t(x)$ is $u$-smooth as discussed in the previous chapter),
which means the gradient $\nabla f_t(x)$ satisfies the relation 
\begin{equation*}
\left\|\nabla f_t(x)-\nabla f_t(y)\right\| \le u\left\|x-y\right\|, \forall t.
\end{equation*}
This smoothness condition is  equivalent to 
$f_t(y)\le f_t(x)+\nabla
f_t(x)^\top(y-x)+\frac{u}{2}\left\|y-x\right\|^2$ and implies, in 
particular, that $\nabla^2 f_t(x) \preceq u I$.
\end{subequations}

\begin{algorithm}
\caption{Discounted Online Newton Step}
\label{alg:discountedNewton}
  \begin{algorithmic}
    \STATE{Given constants $\epsilon >0$, $\eta >0$, and $\gamma \in (0,1)$.}
    \STATE{Let $\theta_1 \in \cS$ and $P_0 = \epsilon I$.}
    \FOR{t=1,\ldots,T}
    \STATE{ Play $\theta_t$ and incur loss $f_t(\theta_t)$ }
    \STATE{ Observe $\nabla_t = \nabla f_t(\theta_t)$ and $H_t =
      \nabla^2 f_t(\theta_t)$ (if needed)}
    \STATE{ Update $P_t$:
      \begin{subequations}
        \begin{align}
          \label{eq:quasiP}
                            P_t &= \gamma P_{t-1} + \nabla_t \nabla_t^\top
                            && \textrm{(Quasi-Newton)}
                            \\
          \label{eq:fullP}
                            P_t &= \gamma P_{t-1} + H_t && \textrm{(Full-Newton)}
                          \end{align}
                          \end{subequations}
                        }
   \STATE{ Update $\theta_t$: $\theta_{t+1} = \Pi_{\cS}^{P_t}\left(\theta_t
       -\frac{1}{\eta} P_t^{-1}\nabla_t\right)$}
    \ENDFOR
  \end{algorithmic}
  
\end{algorithm}

To accommodate these three different cases, we propose Algorithm \ref{alg:discountedNewton},
in which $\Pi_\cS^{P_t}(y) = \argmin_{z\in \cS} \|z-y\|_{P_t}^2$
is the projection onto $\cS$ with respect to the norm induced by $P_t$. 

By using Algorithm \ref{alg:discountedNewton}, the following theorem can be obtained:
\begin{theorem}
  \label{thm:expConcaveThm}
  
  Consider the following three cases of Algorithm~\ref{alg:discountedNewton}:
  \begin{enumerate} 
  \item \label{it:exp} $f_t$ is $\alpha$-exp-concave. The algorithm
    uses $\eta \le
    \frac{1}{2}\min\{\frac{1}{4GD},\alpha\}$, $\epsilon = 1$
    \footnote{The value used here is only for proof simplicity, please see Section \ref{sec:meta-alg} for more discussion.}, 
    and \eqref{eq:quasiP}. 
  \item  \label{it:strong}
    $f_t$ is $\alpha$-exp-concave and $\ell$-strongly convex while
    $\nabla f_t(x)$ is $u$-Lipschitz. The algorithm uses $\eta \le
    \ell / u$, $\epsilon = 1$, and \eqref{eq:fullP}.
  \item $f_t$ is $\alpha$-exp-concave and satisfy
    (\ref{eq:quadBound}) while $\nabla f_t(x)$ is
    $u$-Lipschitz. The algorithm uses $\eta \le 1$, $\epsilon = 1$, and \eqref{eq:fullP}.
  \label{it:quad}
\end{enumerate}
For each of these cases, there are positive constants $a_1,\ldots a_4$ such that  
  \begin{equation*}
  \begin{array}{l}
    \sum_{t=1}^T (f_t(\theta_t)-f_t(z_t)) \le -a_1 T \log \gamma -a_2\log(1-\gamma)
     + \frac{a_3}{1-\gamma} V + a_4
   \end{array}
 \end{equation*}
 for all $z_1,\ldots,z_T\in \cS$ such that $\sum_{t=2}^T
 \|z_t-z_{t-1}\| \le V$.

\end{theorem}

Before proving the theorem, let us first describe
some consequences of it. 

\begin{corollary}
\label{cor::newton_static}
  Setting $\gamma = 1-T^{-\beta}$ with $\beta\in(0,1)$ leads to the following form:
  \begin{equation*}
  \begin{array}{ll}
   \sum_{t=1}^T (f_t(\theta_t)-f_t(z_t))
   \le O(T^{1-\beta} + \beta \log T+ T^{\beta} V)
   \end{array}
 \end{equation*}
\end{corollary}

\begin{proof}
  The first term is bounded as:
  \begin{align*}
    -T \log \gamma = -T \log(1-T^{-\beta}) 
                    \le \frac{T^{1-\beta}}{1-T^{-\beta}} = O(T^{1-\beta}),
  \end{align*}
  where the inequality follows from $-\log(1-x) \le
  \frac{x}{1-x}$ for $0 \le x < 1$. 

  The other terms follow by direct calculation.
\end{proof}

This corollary guarantees that the static regret is bounded in the
order of $O(T^{1-\beta})$ since $V=0$ in that case. The dynamic regret is of order
$O(T^{1-\beta}+T^{\beta} V)$. By choosing $\beta \in (0,1)$, we are
guaranteed that both the static and dynamic regrets are both sublinear
in $T$ as long as $V< O(T)$. Also, small static regret can be obtained by setting
$\beta$ near $1$.

In the setting of Corollary~\ref{cor::newton_static}, the algorithm
parameters do not depend on the path length $V$. Thus, the bounds hold
for any path length, whether or not it is known a priori.
The next corollary shows how tighter bounds could be obtained if
knowledge of $V$ were exploited in choosing the discount factor,
$\gamma$. 

\begin{corollary}
  \label{cor:logBounds}
Setting $\gamma = 1-\frac{1}{2}\sqrt{\frac{\max\{V,\log^2 T/T\}}{2DT}}$
leads to the form:
\begin{equation*}
\sum_{t=1}^T (f_t(\theta_t)-f_t(z_t)) \le \max\{O(\log T),O(\sqrt{TV})\}
\end{equation*}
\end{corollary}

The proof is similar to the proof of Corollary \ref{cor::newton_static}.

Note that Corollary~\ref{cor:logBounds} implies that the discounted
Newton method achieves logarithmic static regret by setting
$V=0$. This matches the bounds obtained in
\cite{hazan2007logarithmic}. For positive path lengths bounded by $V$,
we improve the $O(\sqrt{T(1+V)})$ dynamic bounds from
\cite{zhang2018adaptive}. However, the algorithm above current requires
knowing a bound on the path length,
whereas \cite{zhang2018adaptive} achieves its bound without knowing
the path length, a priori.  

If we view $V$ as the variation budget that $z_1^T = {z_1,\dots,z_T}$ can vary over $\cS$ like in \cite{besbes2015non},
and use this as a pre-fixed value to allow the comparator sequence to vary arbitrarily
over the set of admissible comparator sequence 
$\{z_1^T\in \cS:\sum\limits_{t=2}^T\left\|z_t-z_{t-1}\right\|\le V \}$,
we can tune $\gamma$ in terms of $V$.

In order to bound the dynamic regret
without knowing a bound on the path length, the method of
\cite{zhang2018adaptive} runs a collection of gradient descent algorithms
in parallel with different step sizes and then uses a meta-optimization
\cite{cesa2006prediction} to weight their solutions. In a later section,
we will show how a related meta-optimization over the discount factor
leads to 
$\max\{O(\log T),O(\sqrt{TV})\}$ dynamic regret bounds for unknown
$V$.

For the Algorithm \ref{alg:discountedNewton}, 
we need to invert $P_t$,
which can be achieved in time $O(n^2)$  for the Quasi-Newton case
in \eqref{eq:quasiP} by utilizing the matrix inversion lemma.
However, for the Full-Newton step \eqref{eq:fullP}, the inversion
requires  $O(n^3)$ time.

\paragraph{Proof of Theorem~\ref{thm:expConcaveThm}:}

Before proving the theorem, the following observation is
helpful. 

\begin{lemma}\label{lem:pBound}
  If $P_t$ is updated via \eqref{eq:quasiP} then $\|P_t\| \le \epsilon
  + \frac{G^2}{1-\gamma}$, while if $P_t$ is updated via
  $\eqref{eq:fullP}$, then $\|P_t\| \le \epsilon + \frac{u}{1-\gamma}$.
\end{lemma}
\begin{proof}
  First consider the quasi-Newton case. The bound holds at
  $P_0=\epsilon I$, so assume that it holds at time $t-1$ for $t\ge
  1$.  Then, by induction we have
  \begin{align*}
    \|P_t\| = \|\gamma P_{t-1} + \nabla_t \nabla_t \| 
            \le \gamma \|P_{t-1}\| + G^2 
            \le \gamma \epsilon + \frac{G^2}{1-\gamma} 
    \le \epsilon + \frac{G^2}{1-\gamma}.
  \end{align*}
  The full-Newton case is identical, except it uses the bound
  $\|H_t\|\le u$.
  
\end{proof}

  The generalized Pythagorean theorem implies that
  \begin{align*}
    \|\theta_{t+1}-z_t\|_{P_t}^2 & \le \left\|\theta_t -
                              \frac{1}{\eta}P_{t}^{-1} \nabla_t - z_t
                              \right\|_{P_t}^2\\
    &= \|\theta_t-z_t\|_{P_t}^2 + \frac{1}{\eta^2} \nabla_t^\top P_t^{-1}
      \nabla_t -\frac{2}{\eta} \nabla_t^\top (\theta_t-z_t).
  \end{align*}
  Re-arranging shows that
  \begin{equation}
    \label{eq:Pythagorean}
    \nabla_t^\top (\theta_t-z_t) \le \frac{1}{2\eta}\nabla_t^\top P_t^{-1} \nabla_t 
    + \frac{\eta}{2}\Big(\|\theta_t- z_t\|_{P_t}^2-\|\theta_{t+1}-z_t\|_{P_t}^2\Big)
  \end{equation}

  Let $c_1$ be the upper bound on $\|P_t\|$ from
  Lemma~\ref{lem:pBound}. Then we can lower bound
  $\|\theta_{t+1}-z_t\|_{P_t}^2$ by
  \begin{align}
    \nonumber
    \|\theta_{t+1}-z_t\|_{P_t}^2 & = \|\theta_{t+1}-z_{t+1}\|_{P_t}^2 +
                                   \|z_{t+1}-z_t\|_{P_t}^2
                              +2 (\theta_{t+1}-z_{t+1})^\top P_t
                              (z_{t+1}-z_t)\\
    \label{eq:csBound}
    & \ge \|\theta_{t+1}-z_{t+1}\|_{P_t}^2 - 4 D c_1 \|z_{t+1}-z_t\|
  \end{align}

  Combining \eqref{eq:Pythagorean} and \eqref{eq:csBound} gives
  \begin{multline}
    \nonumber
    \nabla_t^\top (\theta_t-z_t) \le \frac{1}{2\eta}\nabla_t^\top P_t^{-1}
    \nabla_t + 2 D c_1 \eta \|z_{t+1}-z_t\| +
    \frac{\eta}{2}\left(\|\theta_t- z_t\|_{P_t}^2-\|\theta_{t+1}-z_{t+1}\|_{P_t}^2 \right)
  \end{multline}
  Summing over $t$, dropping the term $-\|\theta_{T+1}-z_{T+1}\|_{P_T}^2$, setting $z_{T+1} = z_T$,  and re-arranging gives
  \begin{multline}
    \label{eq:sum}
    \sum_{t=1}^T \nabla_t^\top (\theta_t-z_t) \le \sum_{t=1}^T \frac{1}{2\eta}\nabla_t^\top P_t^{-1}
    \nabla_t  + 2 D c_1\eta V \\
    +\frac{\eta}{2} \epsilon \|\theta_1-z_1\|^2+
    \frac{\eta}{2}\sum_{t=1}^T(\theta_t-z_t)^\top (P_t-P_{t-1})(\theta_t-z_t)
  \end{multline}

  Now we will see how the choices of $\eta$ enable the final sum from
  \eqref{eq:sum} to cancel the terms from \eqref{eq:functionBounds}.
  In  Case~\ref{it:exp}, we have that $\eta(P_t-P_{t-1}) \preceq \eta
  \nabla_t \nabla_t^\top$ and the bound from \eqref{eq:expBound} holds
  for $\rho=\eta$. In Case~\ref{it:strong}, $\eta(P_t-P_{t-1}) \preceq
  \eta H_t \preceq \ell I$. In Case~\ref{it:quad}, $\eta(P_t-P_{t-1})
  \preceq \eta H_t \preceq H_t$. Thus in all cases, $\eta$ has been
  chosen so that combining the appropriate term of
  (\ref{eq:functionBounds}) with \eqref{eq:sum} gives  
  \begin{equation}
  \begin{array}{ll}
    \label{eq:cancelled}
    \sum_{t=1}^T (f_t(\theta_t)-f_t(z_t)) &\le  \sum_{t=1}^T \frac{1}{2\eta}\nabla_t^\top P_t^{-1}
    \nabla_t + 2 D c_1\eta V + 2 \eta \epsilon D^2
  \end{array}
  \end{equation}

  Now we will bound the first sum of \eqref{eq:cancelled}. Note that
  $\nabla_t^\top P_t^{-1}\nabla_t = \langle P_t^{-1}
  ,\nabla_t\nabla_t^\top \rangle$. 
  In Case~\ref{it:exp}, we have that $\nabla_t \nabla_t^\top =
  P_t-\gamma P_{t-1}$, while in Cases \ref{it:strong} and
  \ref{it:quad}, we have that $\nabla_t \nabla_t^\top \preceq
  \frac{1}{\alpha} H_t = \frac{1}{\alpha}(P_t-\gamma P_{t-1})$.
  So, in Case~\ref{it:exp}, let $c_2 = 1$ and in Cases~\ref{it:strong} and
  \ref{it:quad}, let $c_2 = 1/\alpha$. Then in all cases, we have that
  \begin{equation}
    \label{eq:traceBound}
  \nabla_t^\top P_{t}^{-1} \nabla_t  \le c_2 \langle P_t^{-1},P_t-\gamma P_{t-1}\rangle.
  \end{equation}

  Lemma 4.5 of \cite{hazan2016introduction} shows that
  \begin{equation}
    \label{eq:logBound}
     \langle P_t^{-1},P_t-\gamma P_{t-1}\rangle \le \log
    \frac{|P_t|}{|\gamma P_{t-1}|} =\log
    \frac{|P_t|}{|P_{t-1}|}-n\log \gamma, 
  \end{equation}
  where $n$ is the dimension of $x_t$. 

  Combining (\ref{eq:traceBound}) with (\ref{eq:logBound}), summing,
  and then using the bound that $\|P_T\| \le c_1$
  gives,
  \begin{align}
    \nonumber
    \sum_{t=1}^T \nabla_t^\top P_t^{-1} \nabla_t &\le c_2 \log |P_T| -
                                                   c_2 n \log \epsilon -n T \log \gamma \\
    \label{eq:telescope}
    & \le c_2 n \log \frac{c_1}{\epsilon} -c_2nT \log \gamma
  \end{align}

    Recall that $c_1 = \epsilon + \frac{c_3}{1-\gamma}$, where $c_3 =
    G^2$ or $c_3 = u$, depending on the case. Then a more explicit upper bound on
    \eqref{eq:telescope} is given by: 
    \begin{equation}
      \label{eq:telescopeCrude}
      \sum_{t=1}^t \nabla_t^\top P_t^{-1} \nabla_t \le
      c_2 n \log\left(
        1 + \frac{c_3}{\epsilon(1-\gamma)}
        \right)
      - c_2 n T \log \gamma. 
\end{equation}
    
  Combining (\ref{eq:cancelled}) and (\ref{eq:telescopeCrude}) gives the bound:
  \begin{equation*}
  \begin{array}{l}
    \sum_{t=1}^T (f_t(\theta_t)-f_t(z_t)) \le
    -\frac{c_2 nT}{2\eta} \log \gamma + \\
    \frac{c_2 n}{2\eta}\log\left(1+\frac{c_3}{\epsilon(1-\gamma)}\right)+
    2D\eta\left(\epsilon + \frac{c_3}{1-\gamma}\right)V + 
    2\eta \epsilon D^2
  \end{array}
  \end{equation*}
    The desired regret bound can now be found by simplifying the
    expression on the right, using the fact that $\frac{1}{1-\gamma} >
    1$.  
\hfill\qed

\section{From Forgetting Factors to a Step Size Rule}
\label{sec:discounted-RLS}

In the next few sections, we aim to derive gradient descent rules
that achieve similar static and regret bounds to the discounted Newton
algorithm, without the cost of inverting matrices. We begin by
analyzing the special
case of quadratic functions of the form:
\begin{equation}
\label{eq::quadratic_loss}
f_t(\theta) = \frac{1}{2}\left\|\theta - y_t\right\|^2,
\end{equation}
where $y_t \in \cS$.
In this case, we will see that discounted recursive least squares
can be interpreted as online gradient descent with a special
step size rule.
We will show how this step size rule achieves a trade-off between
static regret and dynamic regret with the specific comparison sequence $\theta_t^*
=y_t= \argmin_{\theta \in \cS} f_t(\theta)$.
For a related analysis of more general
quadratic functions, $f_t(\theta)=
\frac{1}{2}\|A_t \theta - y_t\|^2$, please see the appendix.

Note that the previous section focused on dynamic regret for arbitrary
comparison sequences, $z_1^T \in \cS$. The analysis techniques in this
and the next section are specialized to comparisons against
$\theta_t^* = \argmin_{\theta\in\cS} f_t(\theta)$, as studied in works
such as \cite{mokhtari2016online,yang2016tracking}.

Classic discounted recursive least squares corresponds to
Algorithm~\ref{alg:discountedNewton} running with full Newton steps, $\eta =
1$, and initial matrix $P_0=0$.  When $f_t$ is defined as in
\eqref{eq::quadratic_loss}, we have that $P_t = \sum_{k=0}^{t-1}
\gamma^k I$. Thus, the update rule can be expressed in the following
equivalent ways:
\begin{subequations}
\label{eq::quad_dis_update}
\begin{align}
\theta_{t+1} & = \argmin_{\theta\in\cS}\sum\limits_{i=1}^{t}\gamma^{i-1} f_{t+1-i}(\theta) \\
             & = \frac{\gamma-\gamma^t}{1-\gamma^t}\theta_t + \frac{1-\gamma}{1-\gamma^t}y_t\\
             &= \theta_t - P_t^{-1}\nabla f_t(\theta_t) \\
             & = \theta_t - \eta_t \nabla f_t(\theta_t),
\end{align}
\end{subequations}
where $\eta_t = \frac{1-\gamma}{1-\gamma^t}$. 
Note that since $y_t \in \cS$, no projection steps are needed.

The above update is the ubiquitous gradient descent with a changing step size.
The only difference between standard methods is the choice of $\eta_t$,
which will lead to the useful trade-off between dynamic and static regret.

By using the above update, we can get the relationship between 
$\theta_{t+1}-\theta_t^*$ and $\theta_t - \theta_t^*$ as the following result:
\begin{lemma}
  \label{lem:quadratic_var_path_length}
  {\it
  Let $\theta_t^* = \argmin_{\theta\in \cS} f_t(\theta)$ in Eq.\eqref{eq::quadratic_loss}. 
  When using the discounted recursive least-squares update in Eq.\eqref{eq::quad_dis_update},
  we have the following relation:
  \begin{equation*}
  \theta_{t+1}-\theta_t^* = \frac{\gamma-\gamma^t}{1-\gamma^t}(\theta_{t} - \theta_t^*)
  \end{equation*}

  }
\end{lemma}

\begin{proof}
Since $\theta_t^*$ $=$ $\argmin f_t(\theta)$ $=y_t$, for $\theta_{t+1}-\theta_t^*$, we have:
\begin{equation*}
\begin{array}{ll}
\theta_{t+1}-\theta_t^* = \theta_{t+1} - y_t
                        = \frac{\gamma-\gamma^t}{1-\gamma^t}\theta_t + \frac{1-\gamma}{1-\gamma^t}y_t - y_t
                        = \frac{\gamma-\gamma^t}{1-\gamma^t}(\theta_t - y_t)
                        = \frac{\gamma-\gamma^t}{1-\gamma^t}(\theta_t - \theta_t^*)
\end{array}
\end{equation*}

\end{proof}

Recall from \eqref{eq:optimizerLength} that the path length of optimizer sequence is denoted by
$V^*$. 
With the help of Lemma \ref{lem:quadratic_var_path_length}, 
we can upper bound the dynamic regret
in the next theorem:
\begin{theorem}
\label{thm::quad_dynamic_regret}
{\it Let $\theta_t^*$ be the solution to $f_t(\theta)$ in Eq.\eqref{eq::quadratic_loss}. 
  When using the discounted recursive least-squares update in Eq.\eqref{eq::quad_dis_update} 
  with $1-\gamma = 1/T^{\beta}, \beta\in (0,1)$,
  we can upper bound the dynamic regret as:
  \begin{equation*}
  \mathcal{R}_d^* \le 2DT^{\beta}\big(\left\|\theta_1-\theta_1^*\right\|+V^*\big)
  \end{equation*}

}
\end{theorem}

\begin{proof}
According to the Mean Value Theorem, 
there exists a vector $x\in \{v| v = \delta \theta_t + (1-\delta)\theta_t^*,\delta\in[0,1]\}$
such that $f_t(\theta_t)-f_t(\theta_t^*) = \nabla f_t(x)^T(\theta_t-\theta_t^*)\le \left\|\nabla f_t(x)\right\| \left\|\theta_t-\theta_t^*\right\|$. 
For our problem, $\left\|\nabla f_t(x)\right\| = \left\|x-y_t\right\|\le\left\|x\right\|+\left\|y_t\right\|$.
For $\left\|x\right\|$, we have:
\begin{equation*}
\begin{array}{ll}
\left\|x\right\| &= \left\|\delta \theta_t + (1-\delta)\theta_t^*\right\| \\
                 &\le \delta\left\|\theta_t\right\|+(1-\delta)\left\|y_t\right\| \\
                 & = \delta\left\|\frac{\sum\limits_{i=1}^{t-1}\gamma^{i-1}y_{t-i}}{\sum\limits_{i=1}^{t-1}\gamma^{i-1}}\right\|
                  + (1-\delta)\left\|y_t\right\| 
                \le D
\end{array}
\end{equation*}
where the second inequality is due to $\left\|y_i\right\|\le D, \forall i$.

As a result, the norm of the gradient can be upper bounded as $\left\|\nabla f_t(x)\right\| \le 2D$.
Then we have $\mathcal{R}_d^* = \sum\limits_{t=1}^T \Big(f_t(\theta_t)-f_t(\theta_t^*)\Big)\le 
2D\sum\limits_{t=1}^T\left\|\theta_t-\theta_t^*\right\|$. 
Now we could instead upper bound $\sum\limits_{t=1}^T\left\|\theta_t-\theta_t^*\right\|$,
which can be achieved as follows:
\begin{equation*}
\begin{array}{ll}
\sum\limits_{t=1}^T\left\|\theta_t-\theta_t^*\right\| 
&= \left\|\theta_1-\theta_1^*\right\| 
                                                           + \sum\limits_{t=2}^T\left\|\theta_t-\theta_{t-1}^*+\theta_{t-1}^*-\theta_t^*\right\| \\
&\le \left\|\theta_1-\theta_1^*\right\| + \sum\limits_{t=1}^{T-1}\left\|\theta_{t+1}-\theta_{t}^*\right\| + \sum\limits_{t=2}^T\left\|\theta_t^*-\theta_{t-1}^*\right\| \\
&= \left\|\theta_1-\theta_1^*\right\| + \sum\limits_{t=1}^{T-1}\frac{\gamma-\gamma^t}{1-\gamma^t}\left\|\theta_{t}-\theta_{t}^*\right\| + \sum\limits_{t=2}^T\left\|\theta_t^*-\theta_{t-1}^*\right\|\\
&\le \left\|\theta_1-\theta_1^*\right\| + \sum\limits_{t=1}^{T}\frac{\gamma-\gamma^t}{1-\gamma^t}\left\|\theta_{t}-\theta_{t}^*\right\| + \sum\limits_{t=2}^T\left\|\theta_t^*-\theta_{t-1}^*\right\|
\end{array}
\end{equation*}
where in the second equality, we substitute the result from Lemma \ref{lem:quadratic_var_path_length}.

From the above inequality, we get 
\begin{equation*}
\sum\limits_{t=1}^T\Big(1-\frac{\gamma-\gamma^t}{1-\gamma^t}\Big)\left\|\theta_t-\theta_t^*\right\|
\le \left\|\theta_1-\theta_1^*\right\| + \sum\limits_{t=2}^T\left\|\theta_t^*-\theta_{t-1}^*\right\|
\end{equation*}

Since $\Big(1-\frac{\gamma-\gamma^t}{1-\gamma^t}\Big) = \frac{1-\gamma}{1-\gamma^t}\ge 1-\gamma$, we get
\begin{equation*}
\begin{array}{ll}
\sum\limits_{t=1}^T\left\|\theta_t-\theta_t^*\right\|
&\le \frac{1}{1-\gamma}\left\|\theta_1-\theta_1^*\right\| + \frac{1}{1-\gamma}\sum\limits_{t=2}^T\left\|\theta_t^*-\theta_{t-1}^*\right\|\\
&= T^{\beta}(\left\|\theta_1-\theta_1^*\right\|+\sum\limits_{t=2}^T\left\|\theta_t^*-\theta_{t-1}^*\right\|)
\end{array}
\end{equation*}
Thus, $\mathcal{R}_d \le 2D\sum\limits_{t=1}^T\left\|\theta_t-\theta_t^*\right\|
\le 2DT^{\beta}(\left\|\theta_1-\theta_1^*\right\|+\sum\limits_{t=2}^T\left\|\theta_t^*-\theta_{t-1}^*\right\|)$.
\end{proof}

Theorem \ref{thm::quad_dynamic_regret} shows that if we choose the
discounted factor $\gamma = 1- T^{-\beta}$ we obtain a dynamic regret
of $O(T^{\beta}(1+V^*))$. This is a refinement of the Corollary
\ref{cor::newton_static} since the bound no longer
has the $T^{1-\beta}$ term.  Thus, the dynamic regret can be made
small by choosing a small $\beta$.

In the next theorem, we will show that this carefully chosen $\gamma$ can also lead to useful static regret,
which can give us a trade-off between them.  \begin{theorem}
\label{thm::quad_static_regret}
{\it Let $\theta^*$ be the solution to $\min\sum\limits_{t=1}^T f_t(\theta)$. 
  When using the discounted recursive least-squares update in Eq.\eqref{eq::quad_dis_update} 
  with $1-\gamma = 1/T^{\beta}, \beta\in (0,1)$,
  we can upper bound the static regret as:
  \begin{equation*}
  \mathcal{R}_s \le O(T^{1-\beta})
  \end{equation*}
}
\end{theorem}

Recall that the algorithm of this section can be interpreted both as a
discounted recursive least squares method, and as a gradient descent
method. As a result, this theorem is actually a direct consequence of
Corollary~\ref{cor::newton_static}, by setting $V=0$. However, we will
give a separate proof, since the techniques extend naturally to the
analysis of more general work on gradient descent methods of the next
section. 

Before presenting the proof, the following integral bound will be used in a few places.
\begin{lemma}
  \label{lem:integral}
  If $\gamma \in (0,1)$, then
  \begin{equation*}
    \sum_{t=1}^T \frac{1}{1-\gamma^t}  \le \frac{1}{1-\gamma}+T-1 + \frac{\log(1-\gamma)}{\log\gamma}
  \end{equation*}
\end{lemma}
\begin{proof}
  \begin{align*}
\sum\limits_{t=1}^T\frac{1}{1-\gamma^t} &\le 
\frac{1}{1-\gamma}+\int_1^T\frac{1}{1-\gamma^t}\mathrm{d}t \\
                                        &=\frac{1}{1-\gamma}+\Big(t-\frac{\ln(1-\gamma^t)}{\ln(\gamma)}\Big)\Big|_1^T
                                          \\
&= \frac{1}{1-\gamma}+T-1 + \frac{\ln(1-\gamma)}{\ln\gamma} -
  \frac{\ln(1-\gamma^T)}{\ln\gamma} \\
                                        &\le \frac{1}{1-\gamma}+ T-1 + \frac{\ln(1-\gamma)}{\ln\gamma}.
\end{align*}
\end{proof}

\paragraph{Proof of Theorem~\ref{thm::quad_static_regret}:}
\begin{proof}

To proceed, recall that the update in Eq.\eqref{eq::quad_dis_update} is
\begin{equation*}
\begin{array}{ll}
\theta_{t+1} & = \frac{\gamma-\gamma^t}{1-\gamma^t}\theta_t + \frac{1-\gamma}{1-\gamma^t}y_t 
             = \theta_t - \eta_t \nabla f_t(\theta_t)
\end{array}
\end{equation*}
where $\eta_t = \frac{1-\gamma}{1-\gamma^t}$.

Then we get the relationship between $\nabla f_t(\theta_t)^T(\theta_t -\theta^*)$ and 
$\left\|\theta_t - \theta^*\right\|^2-\left\|\theta_{t+1}-\theta^*\right\|^2$ as:
\begin{equation*}
\begin{array}{ll}
\left\|\theta_{t+1}-\theta^*\right\|^2 
&= \left\|\theta_t-\eta_t\nabla f_t(\theta_t) - \theta^*\right\|^2 \\
&= \left\|\theta_t - \theta^*\right\|^2 -2\eta_t\nabla f_t(\theta_t)^T(\theta_t-\theta^*)+ \eta_t^2\left\|\nabla f_t(\theta_t)\right\|^2
\end{array}
\end{equation*}
\begin{equation*}
\begin{array}{ll}
\nabla f_t(\theta_t)^T(\theta_t - \theta^*)
&= \frac{1}{2\eta_t}\big(\left\|\theta_t-\theta^*\right\|^2-\left\|\theta_{t+1}-\theta^*\right\|^2\big)
+\frac{\eta_t}{2}\left\|\nabla f_t(\theta_t)\right\|^2
\end{array}
\end{equation*}

Moreover, we write $f_t(\theta^*)$ as $f_t(\theta^*) = f_t(\theta_t)+\nabla f_t(\theta_t)^T(\theta^*-\theta_t)
+\frac{1}{2}\left\|\theta^*-\theta_t\right\|^2$,
which combined with the previous equation gives us the following equation:
\begin{equation*}
\begin{array}{ll}
f_t(\theta_t) - f_t(\theta^*) 
&= \frac{1}{2\eta_t}\big(\left\|\theta_t-\theta^*\right\|^2-\left\|\theta_{t+1}-\theta^*\right\|^2\big)
+\frac{\eta_t}{2}\left\|\nabla f_t(\theta_t)\right\|^2 - \frac{1}{2}\left\|\theta^*-\theta_t\right\|^2 \\
&\le 2D^2\eta_t + \frac{1}{2\eta_t}\big(\left\|\theta_t-\theta^*\right\|^2-
\left\|\theta_{t+1}-\theta^*\right\|^2\big)
- \frac{1}{2}\left\|\theta^*-\theta_t\right\|^2
\end{array}
\end{equation*}
where the inequality is due to 
$\left\|\nabla f_t(\theta_t)\right\|\le 2D$ as shown in Theorem \ref{thm::quad_dynamic_regret}.

Sum the above inequality from $t=1$ to $T$, we get:
\begin{equation*}
\begin{array}{ll}
\sum\limits_{t=1}^T\Big(f_t(\theta_t)-f_t(\theta^*)\Big)
&\le 2D^2\sum\limits_{t=1}^T\eta_t 
+ \frac{1/\eta_1-1}{2}\left\|\theta_1-\theta^*\right\|^2 \\
&\quad+ \frac{1}{2}\sum\limits_{t=2}^T\big[(\frac{1}{\eta_t}
-\frac{1}{\eta_{t-1}}-1)\left\|\theta^*-\theta_t\right\|^2\big] 
- \frac{1}{2\eta_T}\left\|\theta_{T+1}-\theta^*\right\|^2
\end{array}
\end{equation*}

Since $\eta_t = \frac{1-\gamma}{1-\gamma^t}$, $\eta_1 = 1$, $\frac{1}{\eta_t}-\frac{1}{\eta_{t-1}}-1<0$.
Then for the static regret, we have:
\begin{equation}
\begin{array}{ll}
\mathcal{R}_s = \sum\limits_{t=1}^T\Big(f_t(\theta_t)-f_t(\theta^*)\Big) 
\le 2D^2\sum\limits_{t=1}^T\eta_t 
= 2D^2(1-\gamma)\sum\limits_{t=1}^T\frac{1}{1-\gamma^t}
\end{array}
\end{equation}


Now we will use the integral bound from Lemma~\ref{lem:integral} to
bound the regret.
Since $1-\gamma = 1/T^{\beta}$, $\frac{\log(1-\gamma)}{\log\gamma} = \frac{\beta\log T}{\log(1+\frac{1}{T^{\beta}-1})}$.
Since $\log(1+x)\ge \frac{1}{2}x, x\in(0,1)$, $\log(1+\frac{1}{T^{\beta}-1}) \ge \frac{1}{2}\frac{1}{T^{\beta}-1}$.
Thus, we have $\frac{\log(1-\gamma)}{\log\gamma}\le 2\beta (T^{\beta}-1)\log T$.
Then $(1-\gamma)\sum\limits_{t=1}^T\frac{1}{1-\gamma^t} = O(T^{1-\beta})$, 
which results in
$\mathcal{R}_s \le O(T^{1-\beta})$.
\end{proof}

Theorems \ref{thm::quad_dynamic_regret} and \ref{thm::quad_static_regret} 
build a trade-off between dynamic and static regret
by the carefully chosen discounted factor $\gamma$.
Compared with the result from the last section, there are two improvements:
1. The two regrets are decoupled so that we could reduce the $\beta$
to make the dynamic regret result smaller than bound from Corollary~\ref{cor::newton_static};
2. The update is the first-order gradient descent, which is
computationally more efficient than second order methods.

In the next section, we will consider the strongly convex and smooth case,
whose result is inspired by this section's analysis.

\section{Online Gradient Descent for Smooth, Strongly Convex Problems}
\label{sec:smooth-strongly}

In this section, we generalize the results of the previous section
idea to functions which are
$\ell$-strongly convex and $u$-smooth. We will see that similar
bounds on $\cR_s$ and $\cR_d^*$ can be obtained.

Our proposed update rule for the prediction $\theta_{t+1}$ at time step $t+1$ is:
\begin{equation}
\label{eq::gen_prob_update}
\theta_{t+1} = \argmin\limits_{\theta\in\cS}\left\|\theta - (\theta_t-\eta_t \nabla f_t(\theta_t))\right\|^2
\end{equation}
where $\eta_t = \frac{1-\gamma}{\ell(\gamma-\gamma^t)+u(1-\gamma)}$ and $\gamma \in (0,1)$.

This update rule generalizes the step size rule  from the last section.

Before getting to the dynamic regret, we will first derive the relation between
$\left\|\theta_{t+1}-\theta_{t}^*\right\|$ and $\left\|\theta_t-\theta_t^*\right\|$
to try to mimic the result in Lemma \ref{lem:quadratic_var_path_length} of the quadratic case:
\begin{lemma}
  \label{lem:gen_prob_var_path}
  {\it
  Let $\theta_t^*\in\cS$ be the solution to $f_t(\theta)$ which is strongly convex and smooth.
  When we use the update in Eq.\eqref{eq::gen_prob_update},
  the following relation is obtained:
  \begin{equation*}
  \left\|\theta_{t+1} -\theta_t^*\right\| \le \sqrt{1-\frac{l(1-\gamma)}{u(1-\gamma)+l\gamma}} \left\|\theta_t-\theta_t^*\right\|
  \end{equation*}
  }
\end{lemma}
Since the idea is similar to the proof of Lemma \ref{lem:quadratic_var_path_length}, 
please refer to the appendix for the proof.

Following the idea of Theorem \ref{thm::quad_dynamic_regret}, 
now we are ready to present the dynamic regret result:
\begin{theorem}
\label{thm::gen_prob_dynamic_regret}
{\it 
Let $\theta_t^*$ be the solution to $f_t(\theta),\theta\in\cS$.
When using the update in Eq.\eqref{eq::gen_prob_update}
with $1-\gamma = 1/T^{\beta}, \beta \in (0,1)$,
we can upper bound the dynamic regret:
  \begin{equation*}
  \mathcal{R}_d^* \le G\big(2(T^{\beta}-1)+u/l\big)(\left\|\theta_1-\theta_1^*\right\|
+ V^*)
  \end{equation*}
}
\end{theorem}

Since the proof follows the similar steps in the proof of Theorem \ref{thm::quad_static_regret},
please refer to the appendix.

Theorem \ref{thm::gen_prob_dynamic_regret}'s result seems promising in achieving the trade-off,
since it has a similar form of the result from quadratic problems in Theorem \ref{thm::quad_dynamic_regret}.
Next, we will present the static regret result, which assures that the
desired trade-off can be obtained.
\begin{theorem}
\label{thm::gen_prob_static_regret}
{\it Let $\theta^*$ be the solution to $\min\limits_{\theta\in\cS}\sum\limits_{t=1}^T f_t(\theta)$. 
  When using the update in Eq.\eqref{eq::gen_prob_update} with $1-\gamma = 1/T^{\beta}, \beta\in (0,1)$,
  we can upper bound the static regret:
  \begin{equation*}
  \mathcal{R}_s \le O(T^{1-\beta})
  \end{equation*}
}
\end{theorem}

The proof follows the similar steps in the proof of Theorem \ref{thm::quad_static_regret}.
Please refer to the appendix.

The regret bounds of this section are similar to those obtained for
simple quadratics. Thus, this gradient descent rule maintains all of
the advantages over the discounted Newton method that were described
in the previous section.

\section{Online Gradient Descent for Strongly Convex Problems}
\label{sec:trade-off-strongly-only}

In this section, we extend step size idea from previous section
to problems which are $\ell$-strongly convex, but not necessarily
smooth. We obtain the same order dynamic regret as the discounted online Newton
method: $\cR_d$ $\le$ $\max\{O(\log
T),O(\sqrt{TV})\}$. However, our analysis does not lead to the clean trade-off of
$\cR_s \le O(T^{1-\beta})$ and $\cR_d^* \le O(T^{\beta}(1+V^*))$
obtained when smoothness is also used.

The update rule is online gradient descent:
\begin{equation}
\label{eq::strongly_update}
\theta_{t+1} = \argmin\limits_{\theta\in\cS}\left\|\theta - (\theta_t-\eta_t \nabla f_t(\theta_t))\right\|^2
\end{equation}
where $\eta_t = \frac{1-\gamma}{\ell(1-\gamma^t)}$,
and $\gamma \in (0,1)$.

We can see that the update rule is the same as the one in Eq.\eqref{eq::gen_prob_update}
while the step size $\eta_t$ is replaced with $\frac{1-\gamma}{\ell(1-\gamma^t)}$.

By using the new step size with the update rule in Eq.\eqref{eq::strongly_update},
we can obtain the following dynamic regret bound:
\begin{theorem}
\label{thm::strongly_regret}
{\it
If using the update rule in Eq.\eqref{eq::strongly_update} 
with $\eta_t = \frac{1-\gamma}{\ell(1-\gamma^t)}$ and $\gamma \in (0,1)$, 
the following dynamic regret can be obtained:
\begin{equation*}
\sum\limits_{t=1}^T \Big(f_t(\theta_t)-f_t(z_t)\Big) 
\le 2D\ell \frac{1}{1-\gamma}V + \frac{G^2}{2}\sum\limits_{t=1}^T\eta_t
\end{equation*}
  
}
\end{theorem}
\begin{proof}
According to the non-expansive property of the projection operator and the update rule in Eq.\eqref{eq::strongly_update},
we have
\begin{equation*}
\begin{array}{ll}
\left\|\theta_{t+1}-z_t\right\|^2 &\le \left\|\theta_t-\eta_t\nabla f_t(\theta_t)-z_t\right\|^2 \\
& = \left\|\theta_t-z_t\right\|^2 -2\eta_t\nabla f_t(\theta_t)^T(\theta_t-z_t)
+\eta_t^2\left\|\nabla f_t(\theta_t)\right\|^2
\end{array}
\end{equation*}
The reformulation gives us
\begin{equation} 
\label{eq::strongly_prob_contraction_ineq}
\begin{array}{ll}
\nabla f_t(\theta_t)^T(\theta_t-z_t)
&\le \frac{1}{2\eta_t}\big(\left\|\theta_t-z_t\right\|^2 - \left\|\theta_{t+1}-z_t\right\|^2\big)
+ \frac{\eta_t}{2}\left\|\nabla f_t(\theta_t)\right\|^2
\end{array}
\end{equation}

Moreover, from the strong convexity, we have 
$f_t(z_t)\ge f_t(\theta_t)+\nabla f_t(\theta_t)^T(z_t-\theta_t)+\frac{\ell}{2}\left\|z_t-\theta_t\right\|^2$,
which is equivalent to 
$\nabla f_t(\theta_t)^T(\theta_t-z_t)\ge f_t(\theta_t)-f_t(z_t)+\frac{\ell}{2}\left\|z_t-\theta_t\right\|^2$.
Combined with Eq.\eqref{eq::strongly_prob_contraction_ineq}, we have
\begin{equation}
\label{eq::strongly_ineq_with_l}
\begin{array}{ll}
f_t(\theta_t)-f_t(z_t) &\le \frac{1}{2\eta_t}\big(\left\|\theta_t-z_t\right\|^2 
- \left\|\theta_{t+1}-z_t\right\|^2\big)
+ \frac{\eta_t}{2}\left\|\nabla f_t(\theta_t)\right\|^2-\frac{\ell}{2}\left\|z_t-\theta_t\right\|^2
\end{array}
\end{equation}

Then we can lower bound $\|\theta_{t+1}-z_t\|^2$ by
\begin{equation}
\label{eq::strongly_lower_with_var}
\begin{array}{ll}
    \|\theta_{t+1}-z_t\|^2 & = \|\theta_{t+1}-z_{t+1}\|^2 +
                              \|z_{t+1}-z_t\|^2
                              +2 (\theta_{t+1}-z_{t+1})^\top
                              (z_{t+1}-z_t)\\
    
    & \ge \|\theta_{t+1}-z_{t+1}\|^2 - 4 D\|z_{t+1}-z_t\|
\end{array}
\end{equation}

Combining (\ref{eq::strongly_ineq_with_l}) and \eqref{eq::strongly_lower_with_var} gives
\begin{equation*}
\begin{array}{ll}
f_t(\theta_t)-f_t(z_t) 
&\le \frac{1}{2\eta_t}\big(\left\|\theta_t-z_t\right\|^2 
- \left\|\theta_{t+1}-z_{t+1}\right\|^2\big)+\frac{2D}{\eta_t}\|z_{t+1}-z_t\| \\
&\quad+ \frac{\eta_t}{2}\left\|\nabla f_t(\theta_t)\right\|^2-\frac{\ell}{2}\left\|z_t-\theta_t\right\|^2
\end{array}
\end{equation*}

Summing over $t$ from $1$ to $T$, dropping the term $-\frac{1}{2\eta_T}\|\theta_{T+1}-z_{T+1}\|^2$,
setting $z_{T+1}= z_T$, using the inequality $\|\nabla f_t(\theta_t)\|^2\le G^2$, and re-arranging gives
\begin{equation*}
\begin{array}{ll}
\sum\limits_{t=1}^T \Big(f_t(\theta_t)-f_t(z_t)\Big) 
&\le \frac{1}{2}(\frac{1}{\eta_1}-\ell)\|\theta_1-z_1\|^2
+\frac{1}{2}\sum\limits_{t=1}^T(\frac{1}{\eta_t}-\frac{1}{\eta_{t-1}}-\ell)\|\theta_t-z_t\|^2 \\
& \quad+2D\sum\limits_{t=1}^{T-1}\frac{1}{\eta_t}\|z_{t+1}-z_t\| + \frac{G^2}{2}\sum\limits_{t=1}^T\eta_t \\
& \le 2D\ell \frac{1}{1-\gamma}V + \frac{G^2}{2}\sum\limits_{t=1}^T\eta_t
\end{array}
\end{equation*}
where for the second inequality, we use the following results:
$\frac{1}{\eta_1}-\ell = 0$,
$\frac{1}{\eta_t}-\frac{1}{\eta_{t-1}}-\ell = \frac{\ell(1-\gamma)(\gamma^{t-1}-1)}{1-\gamma}\le 0$,
$\frac{1}{\eta_t} = \frac{\ell(1-\gamma^t)}{1-\gamma}\le \frac{\ell}{1-\gamma}$,
and the definition of $V$.

\end{proof}

Similar to the case of discounted online Newton methods, if a bound on
the path length, $V$, is known, the discount factor can be tuned to
achieve low dynamic regret: 
\begin{corollary}
\label{corr:strongly_dynamic_regret}
By setting $\gamma = 1-\frac{1}{2}\sqrt{\frac{\max\{V,\log^2 T/T\}}{2DT}}$,
the following bound can be obtained:
\begin{equation*}
\sum\limits_{t=1}^T \Big(f_t(\theta_t)-f_t(z_t)\Big)\le \max\{O(\log T),O(\sqrt{TV})\}.
\end{equation*}

\end{corollary} 

This result is tighter than the $O(\sqrt{T(1+V)})$ bound obtained
by \cite{zhang2018adaptive} on convex functions, but not directly
comparable to the $O(V^*)$ bounds obtained in
\cite{mokhtari2016online} for smooth, strongly convex functions.

Similar to the Corollary~\ref{cor:logBounds} on discounted online Newton methods, 
Corollary~\ref{corr:strongly_dynamic_regret} requires knowing $V$. In the next section, we will see how a
meta-algorithm can be used to obtain the same bounds without knowing $V$.

Please refer to the appendix for the proof of Corollary~\ref{corr:strongly_dynamic_regret}.

\section{Meta-algorithm}
\label{sec:meta-alg}

In previous sections, we discussed the results on dynamic regret for
both $\alpha$-exp-concave and $\ell$-strongly convex objectives.
The tightest regret bounds were obtained by choosing a discount factor
that depends on $V$, a bound on the path length. 
In this section, we solve this issue 
by running multiple algorithms in parallel with different discount factors.

For online convex optimization, a similar meta-algorithm has been used
by \cite{zhang2018adaptive} to search over step sizes. However, the
method of \cite{zhang2018adaptive} cannot be used directly in either the $\alpha$-exp-concave or $\ell$-strongly convex case
due to the added $O(\sqrt{T})$ regret from running multiple
algorithms. In order to remove this factor, we exploit the
exp-concavity in the experts algorithm, as in
Chapter 3 in \cite{cesa2006prediction}. 

In this section, we will show that by using appropriate parameters and analysis designed specifically for our cases,
the meta-algorithm can be used to solve our issues.

\begin{algorithm}
\caption{Meta-Algorithm}
\label{alg:meta}
  \begin{algorithmic}
    \STATE{Given step size $\lambda$, and a set $\cH$ containing
      discount factors for each algorithm.}
    \STATE{Activate a set of algorithms $\{A^\gamma|\gamma\in\cH\} $
    by calling Algorithm \ref{alg:discountedNewton} (exp-concave case) or the update in Eq.\eqref{eq::strongly_update} (strongly convex case)
    for each parameter $\gamma\in\cH$. }
    \STATE{Sort $\gamma$ in descending order $\gamma_1\ge\gamma_2\ge\dots\ge\gamma_N$, 
    and set $w_1^{\gamma_i} = \frac{C}{i(i+1)}$ with $C = 1+1/|\cH|$.}
    \FOR{t=1,\ldots,T}
    \STATE{Obtain $\theta_t^{\gamma}$ from each algorithm $A^{\gamma}$. }
    \STATE{Play $\theta_t = \sum\limits_{\gamma\in\cH}w_t^\gamma \theta_t^\gamma$,
    and incur loss $f_t(\theta_t^\gamma)$ for each $\theta_t^\gamma$. }
    \STATE{Update $w_t^\gamma$ by
            \begin{equation*}
               w_{t+1}^\gamma = \frac{w_t^\gamma\exp(-\lambda f_t(\theta_t^\gamma))}
               {\sum\limits_{\mu\in\cH}w_t^\mu \exp(-\lambda f_t(\theta_t^\mu))}.
            \end{equation*}
    \STATE{Send back the gradient $\nabla f_t(\theta_t^\gamma)$ for each algorithm $A^\gamma$.}
    }
    \ENDFOR
  \end{algorithmic}
\end{algorithm}

\subsection{Exp-concave Case}

Before showing the regret result, we first show that the cumulative loss of the meta-algorithm
is comparable to all $A^\gamma\in\cH$:
\begin{lemma}\label{lem:meta-expert-compare}
If $f_t$ is $\alpha$-exp-concave and $\lambda = \alpha$, 
the cumulative loss difference of Algorithm \ref{alg:meta} for any $\gamma\in\cH$ is bounded as:
\begin{equation*}
\sum\limits_{t=1}^T (f_t(\theta_t) - f_t(\theta_t^\gamma))\le \frac{1}{\alpha}\log\frac{1}{w_1^\gamma}
\end{equation*}

\end{lemma}

This result shows how $O(\sqrt{T})$ regret incurred by running an
experts algorithm is reduced in the $\alpha$-exp-concave case. %
The result is similar to Proposition 3.1 of
\cite{cesa2006prediction}.
We also provide a proof in the appendix.

Based on the above lemma, if we can show that there exists an algorithm $A^\gamma$, 
which can bound the regret $\sum_{t=1}^T (f_t(\theta_t^\gamma)-f_t(z_t)) \le \max\{O(\log T), O(\sqrt{TV})\}$,
then we can combine these two results and show that the regret holds for $\theta_t,t=1,\dots,T$ as well:
\begin{theorem}
\label{thm:mega_exp-concave}
For any comparator sequence $z_1,\dots,z_T\in\cS$, 
setting $\cH=\Big\{\gamma_i = 1-\eta_i\Big|i=1,\dots,N\Big\}$ with $T\ge 2$
where $\eta_i = \frac{1}{2}\frac{\log T}{T\sqrt{2D}}2^{i-1}$, 
$N = \lceil \frac{1}{2}\log_2 (\frac{2DT^2}{\log^2 T})\rceil+1$,
and $\lambda = \alpha$ leads to the result:
\begin{equation*}
\sum_{t=1}^T (f_t(\theta_t)-f_t(z_t))\le O(\max\{\log T, \sqrt{TV}\})
\end{equation*}

\end{theorem}

As described previously, the proof's main idea is 
to show that we could both find an algorithm $A^\gamma$
bounding the regret $\sum_{t=1}^T (f_t(\theta_t^\gamma)-f_t(z_t)) \le \max\{O(\log T), O(\sqrt{TV})\}$
and
cover the $V$ with $O(\log T)$ different $\gamma$ choices.
Please see the appendix for the formal proof.

In practice, we include the additional case when $\gamma = 1$ to make the overall algorithm 
explicitly balance the static regret.
Also, the free parameter $\epsilon$ used in Algorithm \ref{alg:discountedNewton} is important
for the actual performance. If it is too small,
the update will be easily effected by the gradient to have high generalization error.
In practice, it can be set to be equal to $1/(\rho^2 D^2)$ or $1/(\rho^2 D^2 N)$ 
with $\rho = \frac{1}{2}\min\{\frac{1}{4GD},\alpha\}$ like in \cite{hazan2016introduction}.

\subsection{Strongly Convex Case}

For the strongly convex problem, since the parameter $\gamma$ 
used in Corollary \ref{corr:strongly_dynamic_regret} is the same as the one in Corollary \ref{cor:logBounds},
it seems likely that the 
meta-algorithm should work with the same setup in as Theorem
\ref{thm:mega_exp-concave}. The only parameter that needs to be
changed is $\lambda$, which was set above to $\alpha$, the parameter
of $\alpha$-exp-concavity. 

To proceed, we first show that the $\ell$-strongly convex function 
with bounded gradient is also $\ell/G^2$-exp-concave (e.g.,$\left\|\nabla f_t\right\|$$\le$$ G$).
Previous works also pointed out this, 
but their statement only works when $f_t$ is second-order differentiable,
while our result is true when $f_t$ is first-order differentiable.

\begin{lemma}
\label{lem::strongly_is_exp}
For the $\ell$-strongly convex function $f_t$ with $\|\nabla f_t\|\le G$, 
it is also $\alpha$-exp-concave with $\alpha = \ell/G^2$.
\end{lemma}

Please refer to the appendix for the proof.

Lemma \ref{lem::strongly_is_exp} indicates that running Algorithm \ref{alg:meta} with strongly convex function
leads to the same result as in Lemma \ref{lem:meta-expert-compare}.
Thus, using the similar idea as discussed in the case of $\alpha$-exp-concavity and Algorithm \ref{alg:meta},
the theorem below can be obtained:
\begin{theorem}
\label{thm:mega_strongly}
For any comparator sequence $z_1,\dots,z_T\in\cS$, 
setting $\cH=\Big\{\gamma_i = 1-\eta_i\Big|i=1,\dots,N\Big\}$ with $T\ge 2$
where $\eta_i = \frac{1}{2}\frac{\log T}{T\sqrt{2D}}2^{i-1}$, 
$N = \lceil \frac{1}{2}\log_2 (\frac{2DT^2}{\log^2 T})\rceil+1$,
and $\lambda = \ell/G^2$ leads to the result:
\begin{equation*}
\sum_{t=1}^T (f_t(\theta_t)-f_t(z_t))\le O(\max\{\log T, \sqrt{TV}\})
\end{equation*}
\end{theorem}
Since the proof shares the same idea as Theorem \ref{thm:mega_exp-concave},
please refer to the appendix for the proof.

As discussed in the previous subsection,
in practice, we also include the case when $\gamma = 1$ to make the overall algorithm 
explicitly balance the static regret and set $\epsilon$ accordingly as in the exp-concave case.

\subsection{A Lower Bound}
In the previous subsections, we show how to achieve the improved dynamic regret
for both the exp-concave and strongly convex problems without knowing $V$.
In this subsection, we will give a lower bound,
which approaches the upper bound for large and small $V$.

\begin{proposition}
  \label{prop:lower-bound}
  For losses of the form $f_t(\theta) = (\theta - \epsilon_t)^2$, for
  all $\gamma_0 \in (0,1)$ and all $V=T^{\frac{2+\gamma_0}{4-\gamma_0}}$,
  there is a comparison sequence $z_1^T$
such that $\sum\limits_{t=2}^T\|z_t-z_{t-1}\|\le V$ and 
\begin{equation*}
\cR_d \ge \max\{O(\log T),O\big((VT)^{\frac{\gamma_0}{2}}\big)\}.
\end{equation*}

\end{proposition}

The above result has the following indications:
1. For $V = o(T)$ but approaching to $T$, it is impossible to achieve better bound of $\cR_d \ge O\Big((VT)^{\frac{\alpha_0}{2}}\Big)$
with $\alpha_0<1$.
2. For other ranges of $V$ like $V = O(\sqrt{T})$, its lower bound is not established and still an open question.

\paragraph{Proof of Proposition \ref{prop:lower-bound}:}

\begin{proof}

Since strongly convex problem with bounded gradient is also exp-concave 
due to Lemma \ref{lem::strongly_is_exp} shown in the next section,
we will only consider the strongly convex problem.

For the case when $V = 0$, $\cR_d$ reduces to the static regret $\cR_s$, which has the lower bound $O(\log T)$
as shown in \cite{abernethy2008optimal}.

Let us now consider the case when $V> 0$. The analysis is inspired by \cite{yang2016tracking}.
We will use $f_t(\theta) = (\theta-\epsilon_t)^2$ as the special case
to show the lower bound. Here $\epsilon_1^T$ is a sequence of independently generated
random variables from $\{-2\sigma,2\sigma\}$ with equal probabilities.
For the dynamic regret 
$\cR_d = \sum\limits_{t=1}^Tf_t(\theta_t) - \min\limits_{z_1^T\in \cS_V}\sum\limits_{t=1}^Tf_t(z_t)
\ge \sum\limits_{t=1}^Tf_t(\theta_t) - \sum\limits_{t=1}^Tf_t(z_t)$,
where $\cS_V = \{z_1^T:\sum\limits_{t=2}^T\|z_t-z_{t-1}\|\le V\}$, and
$z_t = \frac{1}{2}\epsilon_t$.
As a result, the expected value of
$\sum\limits_{t=1}^Tf_t(\theta_t) - \sum\limits_{t=1}^Tf_t(z_t)$
is $\bbE[\sum\limits_{t=1}^Tf_t(\theta_t) - \sum\limits_{t=1}^Tf_t(z_t)]$
$=$ $\bbE[\sum\limits_{t=1}^T(\theta_t^2 -2\theta_t\epsilon_t+\frac{3}{4}\epsilon_t^2)]$
$\ge$ $\sum\limits_{t=1}^T \bbE[-2\theta_t\epsilon_t +\frac{3}{4}\epsilon_t^2]$ $=$ $3\sigma^2 T$.
This implies that $\cR_d \ge 3\sigma^2 T$.
For the path length, $\sum\limits_{t=2}^T\|z_t-z_{t-1}\|\le 2\sigma T$.
Let us set $\sigma = T^{-\frac{2(1-\gamma_0)}{4-\gamma_0}}$ and $\gamma_0 \in (0,1)$.
Then $V = 2\sigma T = 2T^{\frac{2+\gamma_0}{4-\gamma_0}}$
and $(VT)^{\frac{\gamma_0}{2}}$ $=$ $2^{\frac{\gamma_0}{2}}T^{\frac{3\gamma_0}{4-\gamma_0}}$.
Then $\cR_d - \frac{3}{\sqrt{2}}(VT)^{\frac{\gamma_0}{2}}$ $\ge$ 
$3T^{\frac{3\gamma_0}{4-\gamma_0}} -  \frac{3}{\sqrt{2}}2^{\frac{\gamma_0}{2}}T^{\frac{3\gamma_0}{4-\gamma_0}}$
$\ge 0$. In other words, $\cR_d\ge O\Big((VT)^{\frac{\gamma_0}{2}}\Big)$, $\forall \gamma_0\in (0,1)$
with $V = 2T^{\frac{2+\gamma_0}{4-\gamma_0}}$.

In summary, we have
that
there always exist a 
exist a sequence of loss functions $f_1^T$
and a comparison sequence $z_1^T$
such that $\sum\limits_{t=2}^T\|z_t-z_{t-1}\|\le V = O(T^{\frac{2+\gamma_0}{4-\gamma_0}})$ and 
$\cR_d \ge \max\{O(\log T),O\big((VT)^{\frac{\gamma_0}{2}}\big)\}, \forall \gamma_0 \in (0,1)$.

\end{proof}

\section{Conclusion}
\label{sec:trade-off-conclusion}

In this chapter, 
we propose a discounted online Newton algorithm that
generalizes recursive least squares with forgetting factors and
existing online Newton methods. We prove a dynamic regret bound
$\max\{O(\log T),O(\sqrt{TV})\}$ which provides a rigorous
analysis of forgetting factor algorithms. 
In the special case of simple quadratic functions, we demonstrate that the
discounted Newton method reduces to a gradient descent algorithm with
a particular step size rule.
We show how this step size rule can be generalized to apply to
strongly convex functions, giving a substantially more
computationally efficient algorithm than the discounted online Newton
method, while recovering the dynamic regret guarantees.
The strongest regret guarantees depend on knowledge of the path
length, $V$. We show how to use a meta-algorithm that optimizes over
discount factors to obtain the same regret guarantees without
knowledge of $V$ as well as a lower bound which matches the obtained upper bound
for certain range of $V$.
Finally, when the functions are smooth we show how this
new gradient descent method enables a static regret of $\cR_s\le
O(T^{1-\beta})$ and $\cR_d^* \le O(T^{\beta} (1+V^*))$, where $\beta \in
(0,1)$ is a user-specified trade-off parameter.

\chapter{Online Adaptive Principal Component Analysis and Its extensions}
\label{chap:adaptive-pca}

In the previous chapter, 
we discussed several dynamic/static regret results under changing environments,
including the online least-squares and its extension to the 
exp-concave and strongly convex problem setups.
In this chapter, we are mainly concerned with the problem of online
Principal Component Analysis (online PCA) under changing environments.

As discussed in Chapter \ref{chap:related-work},
previous online PCA algorithms are based on either online gradient or matrix exponentiated gradient descent 
\cite{tsuda2005matrix,warmuth2006online,warmuth2008randomized,niew2016onlinepca}.
These works bound online PCA by the static regret,
which, as argued in previous chapters, is not appropriate for changing environments.
To have better adaptivity to the changing environments,
previous works proposed to run a pool of algorithms 
with either different parameters like in \cite{zhang2018adaptive,yuan2019trading} (for upper bounding dynamic regret) 
or different starting points like in \cite{hazan2009efficient} (for upper bounding adaptive regret).
This not only requires complex implementation,
but also increases the computational complexity per step by a factor of $O(\log T)$
due to the parallel running of different algorithms.

It is thus desired to have an efficient and easy-to-implement algorithm, 
which can eliminate the need of running multiple algorithms
while having the same theoretical guarantee.
This chapter introduces such an efficient algorithm for the specific online PCA problem 
under adaptive regret measure, which is adapted from the published work \cite{yuan2019online}. 
The proposed method mixes
the randomized algorithm from \cite{warmuth2008randomized} with a
fixed-share step \cite{herbster1998tracking}.
This is inspired by the work of \cite{cesa2012new,cesa2012mirror},
which shows that the Hedge algorithm \cite{freund1997decision} together
with a fixed-share step provides low regret under a variety of
measures, including adaptive regret.

Furthermore, we extend the idea of the additional fixed-share step 
to the online adaptive variance minimization 
in two different parameter spaces: the space of unit vectors and the
simplex.
In Section~\ref{sec:adaptive-pca-exp},
we also do the experiments to test our algorithm's effectiveness. 
In particular, we show
that our proposed algorithm can adapt to the changing environments
faster than the previous online PCA algorithm.

\section{Problem Formulation}
\label{sec:adaptive-pca-prob-form}

The goal of the PCA (uncentered) algorithm is to find a rank $k$
projection matrix $P$ 
that minimizes the compression loss: $\sum\limits_{t=1}^T
\left\|\mathbf{x_t}-P\mathbf{x_t}\right\|^2$. In this case,
 $P \in \mathbb{R}^{n\times n}$ 
 must be a symmetric positive semi-definite matrix
with only $k$ non-zero eigenvalues which are all equal to 1.

In online PCA,
the data points come in a stream. 
At each time $t$, the algorithm first chooses a projection matrix
$P_t$ with rank $k$, then the data point $\mathbf{x_t}$ is revealed,
and a compression loss of
$\left\|\mathbf{x_t}-P_t\mathbf{x_t}\right\|^2$ is incurred.

The online PCA algorithm \cite{warmuth2008randomized}
aims to minimize the static regret $\mathcal{R}_s$ ,which is the difference 
between the total expected compression loss
and the loss of the best projection matrix $P^*$ chosen in hindsight:
\begin{equation}
\label{eq::static_regret_pca}
\mathcal{R}_s = \sum\limits_{t=1}^T \mathbb{E}[\Tr((I-P_t)\mathbf{x_t}\mathbf{x_t}^\top)] - \sum\limits_{t=1}^T \Tr((I-P^*)\mathbf{x_t}\mathbf{x_t}^\top).
\end{equation}
The algorithm from \cite{warmuth2008randomized} is randomized and the
expectation is taken over the distribution of $P_t$ matrices. The
matrix 
$P^*$ is the solution to the following optimization problem with
$\mathcal{S}$ being the set of rank-$k$ projection matrices:
\begin{equation}
\label{eq::best_fixed_sol_PCA}
\min_{P\in \mathcal{S}} \sum\limits_{t=1}^T \Tr((I-P)\mathbf{x_t}\mathbf{x_t}^\top)
\end{equation}

Algorithms that minimize static regret will converge to $P^*$, which
is the best projection for the entire data set. However, in many 
scenarios the data generating process changes over time. 
In this case,
a solution that adapts to changes in the data set may be
desirable. To model environmental variation, 
several notions of dynamically varying regret have been
proposed \cite{herbster1998tracking,hazan2009efficient,cesa2012new}. 
In this chapter, we study adaptive regret $\mathcal{R}_a$ from \cite{hazan2009efficient}, 
which results in the following online adaptive PCA problem:
\begin{equation}
\label{eq::adaptive_regret_pca}
\begin{array}{l}
\mathcal{R}_a = \max\limits_{[r,s]\subset [1,T]}\Big\{
\sum\limits_{t=r}^s \mathbb{E}[\Tr((I-P_t)\mathbf{x_t}\mathbf{x_t}^\top)]
- \min\limits_{U\in \mathcal{S}}\sum\limits_{t=r}^s \Tr((I-U)\mathbf{x_t}\mathbf{x_t}^\top)  \Big\}
\end{array}
\end{equation}
In the next few sections,  
we will present an algorithm that achieves low adaptive regret.

\section{Learning the Adaptive Best Subset of Experts}
\label{sec:adaptive-subset-exp}

\begin{algorithm}[tb]
    \caption{Adaptive Best Subset of Experts}
    \label{alg::alg_best-sub-exp}
\begin{algorithmic}[1]
    \STATE {\bfseries Input:} $1\le k < n$ and an initial probability vector $\mathbf{w_1} \in \mathcal{B}_{\text{n-k}}^\text{n}$.
    \FOR{$t=1$ {\bfseries to} $T$}
    \STATE Use Algorithm \ref{alg::alg-mix-decomp} with input $d=n-k$ to decompose $\mathbf{w_t}$ into $\sum_j p_j\mathbf{r_j}$, which is a
    convex combination of at most $n$ corners of $\mathbf{r_j}$.
    \STATE Randomly select a corner $\mathbf{r}=\mathbf{r_j}$ with associated probability $p_j$.
    \STATE Use the k components with zero entries in the drawn corner $\mathbf{r}$ as the selected subset of experts.
    \STATE Receive loss vector $\mathbf{\ell_t}$.
    \STATE Update $\mathbf{w_{t+1}}$ as:
    \begin{subequations}
    \label{eq::our_expert_update}
    \begin{align}
    \label{eq::v_t+1}
    &v_{t+1,i} = \frac{w_{t,i}\exp(-\eta\ell_{t,i})}{\sum_{j=1}^n w_{t,j}\exp(-\eta\ell_{t,j})}\\
    \label{eq::fix_share_expert}
    &\hat{w}_{t+1,i} = \frac{\alpha}{n} + (1-\alpha)v_{t+1,i} \\
    \label{eq::w_t+1}
    &\mathbf{w_{t+1}} = \text{cap}_{\text{n-k}}(\mathbf{\hat{w}_{t+1}})
    \end{align}
    \end{subequations}
    where $\text{cap}_{\text{n-k}}()$ calls Algorithm \ref{alg::capping}.
    \ENDFOR
\end{algorithmic}
\end{algorithm}

\begin{algorithm}[tb]
    \caption{Mixture Decomposition \cite{warmuth2008randomized}}
    \label{alg::alg-mix-decomp}
\begin{algorithmic}[1]
    \STATE {\bfseries Input:} $1\le d < n$ and $\mathbf{w}\in \mathcal{B}_\text{d}^\text{n}$.
    \REPEAT 
    \STATE Let $\mathbf{r}$ be a corner for a subset of $d$ non-zero components of $\mathbf{w}$ 
    that includes all components of $\mathbf{w}$ equal to $\frac{\left|\mathbf{w}\right|}{d}$.
    \STATE Let $s$ be the smallest of the $d$ chosen components of $\mathbf{r}$ and $l$ be the largest value
    of the remaining $n-d$ components.
    \STATE update $\mathbf{w}$ as $\mathbf{w}-\min(ds,\left|\mathbf{w}\right|-dl)\mathbf{r}$ and {\bfseries Output} $p$ and $\mathbf{r}$.
    \UNTIL{$\mathbf{w}=0$}
\end{algorithmic}
\end{algorithm}

\begin{algorithm}[tb]
    \caption{Capping Algorithm \cite{warmuth2008randomized}}
    \label{alg::capping}
\begin{algorithmic}[1]
    \STATE {\bfseries Input:} probability vector $\mathbf{w}$ and set size $d$.
    \STATE Let $\mathbf{w}^{\downarrow}$ index the vector in decreasing order, that is, $\mathbf{w_1}^{\downarrow} = \max(\mathbf{w})$.
    \IF{$\max(\mathbf{w})\le 1/d$}
      \STATE {\bfseries return} $\mathbf{w}$.
    \ENDIF
    \STATE $i=1$.
    \REPEAT 
    \STATE (* Set first $i$ largest components to $1/d$ and normalize the rest to $(d-i)/d$ *)
    \STATE $\mathbf{\tilde{w}} = \mathbf{w}$, $\tilde{w}_j^{\downarrow} = 1/d$, for $j = 1,\dots,i$.
    \STATE $\tilde{w}_j^{\downarrow} = \frac{d-i}{d}\frac{\tilde{w}_j^{\downarrow}}{\sum_{l=j}^n \tilde{w}_l^{\downarrow}}$, for $j = i+1,\dots,n$.
    \STATE $i = i + 1$.
    \UNTIL{$\max(\mathbf{\tilde{w}})\le 1/d$}.
\end{algorithmic}
\end{algorithm}

In \cite{warmuth2008randomized},
it was shown that online PCA can be
viewed as an extension of a simpler problem known as the \emph{best subset of experts} problem. In
particular, they first propose an online algorithm to solve the best
subset of experts problem, and then they show how to modify the
algorithm to solve PCA problems. In this section, we show how the
addition of a fixed-share step \cite{herbster1998tracking,cesa2012new}
can lead to an algorithm for an adaptive variant of the best subset of
experts problem. Then we will show how to extend the resulting
algorithm to PCA problems. 

The \emph{adaptive best subset of experts}  problem can be described as follows:
we have $n$ experts making decisions at each time $t$.
Before revealing the loss vector $\mathbf{\ell_t}\in \mathbb{R}^n$ associated with the experts' decisions at time $t$,
we select a subset of experts of size $n-k$ (represented by vector $\mathbf{v_t}$) to try to minimize the adaptive regret defined as:
\begin{equation}
\label{eq::adaptive_regret_expert}
\mathcal{R}_a^{\text{subexp}} = \max_{[r,s]\subset [1,T]}\Big\{
\sum\limits_{t=r}^s \mathbb{E}[\mathbf{v_t}^\top\mathbf{\ell_t}] -
\min_{\mathbf{u}\in \mathcal{S}_{\text{vec}}}\sum\limits_{t=r}^s \mathbf{u}^\top\mathbf{\ell_t}  \Big\}.
\end{equation}
Here, the expectation is taken over the probability distribution of $\mathbf{v_t}$.
Both $\mathbf{v_t}$ and $\mathbf{u}$ are in $\mathcal{S}_{\text{vec}}$ which
denotes the vector set 
with only $n-k$ non-zero elements equal to 1.

Similar to the static regret case from \cite{warmuth2008randomized},
the problem in Eq.(\ref{eq::adaptive_regret_expert}) is equivalent to:
\begin{equation}
\label{eq::reform_adaptive_regret_expert}
\mathcal{R}_a^{\text{subexp}} = \max_{[r,s]\subset [1,T]}\Big\{
\sum\limits_{t=r}^s (n-k)\mathbf{w_t}^\top\mathbf{\ell_t} -
\min_{\mathbf{q}\in\mathcal{B}_{\text{n-k}}^\text{n} }\sum\limits_{t=r}^s (n-k)\mathbf{q}^\top\mathbf{\ell_t}  \Big\}
\end{equation}
where $\mathbf{w_t}\in\mathcal{B}_{\text{n-k}}^\text{n}$, and $\mathcal{B}_{\text{n-k}}^\text{n}$
represents the capped probability simplex
defined as $\sum_{i=1}^nw_{t,i} = 1$ and $0\le w_{t,i}\le 1/(n-k)$, $\forall i$.

Such equivalence is due to the Theorem 2 in \cite{warmuth2008randomized} ensuring that
any vector $\mathbf{q}\in\mathcal{B}_{\text{n-k}}^\text{n}$ can be decomposed as convex combination of 
at most $n$ corners of $\mathbf{r_j}$ by using Algorithm \ref{alg::alg-mix-decomp}, 
where the corner $\mathbf{r_j}$ is defined 
as having $n-k$ non-zero elements equal to $1/(n-k)$. 
As a result, the corner can be sampled by the associated probability obtained from the convex combination,
which is a valid subset selection vector $\mathbf{v_t}$ with the multiplication of $n-k$.

\textbf{Connection to the online adaptive PCA.} The problem from
Eq.(\ref{eq::adaptive_regret_expert}) can be viewed as restricted
version of 
the online adaptive PCA problem from
Eq.(\ref{eq::adaptive_regret_pca}).
In particular, say that $I-P_t = \mathrm{diag}(\mathbf{v_t})$. This corresponds
to restricting $P_t$ to be diagonal. If 
$\mathbf{\ell_t}$ is the diagonal of $\mathbf{x_t}\mathbf{x_t}^\top$,
then the objectives of Eq.(\ref{eq::adaptive_regret_expert}) and  
Eq.(\ref{eq::adaptive_regret_pca}) are equal.

We now return to the adaptive best subset of experts problem.
When $r=1$ and $s=T$, the problem reduces to the standard static regret minimization problem,
which is studied in \cite{warmuth2008randomized}. 
Their solution applies  the basic Hedge Algorithm to obtain a probability distribution for the experts,
and modifies the distribution to select a subset of the experts.

To deal with the adaptive regret considered in Eq.(\ref{eq::reform_adaptive_regret_expert}),
we propose the Algorithm \ref{alg::alg_best-sub-exp},
%
which is a simple modification to Algorithm 1 in \cite{warmuth2008randomized}.
More specifically, we add Eq.(\ref{eq::fix_share_expert}) when updating $\mathbf{w_{t+1}}$ in Step $7$,
%
which is called a \emph{fixed-share} step.
This is inspired by the analysis in \cite{cesa2012new}, 
which shows that the online adaptive best expert problem can be solved 
by simply adding this fixed-share step to the standard Hedge algorithm.

With the Algorithm \ref{alg::alg_best-sub-exp},
the following lemma can be obtained:
\begin{lemma}
\label{lem::adaptive_expert_step_ineq}
For all $t\ge 1$, all $\mathbf{\ell_t} \in [0,1]^n$, and for all $\mathbf{q_t}\in \mathcal{B}_{\text{n-k}}^\text{n}$, 
Algorithm \ref{alg::alg_best-sub-exp} satisfies 
\begin{equation*}
\mathbf{w_t}^\top\mathbf{\ell_t}(1-\exp(-\eta)) - \eta \mathbf{q_t}^\top\mathbf{\ell_t} \le
\sum_{i=1}^n q_{t,i}\ln(\frac{v_{t+1,i}}{\hat{w}_{t,i}})
\end{equation*}
\end{lemma}

\begin{proof}

With the update in Eq.(\ref{eq::our_expert_update}),
for any $\mathbf{q_t} \in\mathcal{B}_{\text{n-k}}^\text{n}$, we have
\begin{equation*}
d(\mathbf{q_t},\mathbf{w_t})-d(\mathbf{q_t},\mathbf{v_{t+1}}) = 
-\eta \mathbf{q_t}^\top\mathbf{\ell_t}-\ln(\sum_{j=1}^n w_{t,j}\exp(-\eta\ell_{t,j}))
\end{equation*}

Also,
$-\ln(\sum_{j=1}^n w_{t,j}\exp(-\eta\ell_{t,j})) \ge \mathbf{w_t}^\top\mathbf{\ell_t}(1-\exp(-\eta))$
based on the proof of Theorem 1 in \cite{warmuth2008randomized}.
Thus, we get 
\begin{equation}
\label{eq::online_pca_thm1_ineq}
d(\mathbf{q_t},\mathbf{w_t})-d(\mathbf{q_t},\mathbf{v_{t+1}}) \ge -\eta \mathbf{q_t}^\top\mathbf{\ell_t} + \mathbf{w_t}^\top\mathbf{\ell_t}(1-\exp(-\eta))
\end{equation}

Moreover, Eq.(\ref{eq::w_t+1}) is the solution to the following projection problem as shown in \cite{warmuth2008randomized}:
\begin{equation*}
\mathbf{w_t} = \argmin\limits_{\mathbf{w}\in\mathcal{B}_{\text{n-k}}^\text{n}} d(\mathbf{w},\mathbf{\hat{w}_t})
\end{equation*}

Since the relative entropy is one kind of Bregman divergence \cite{bregman1967relaxation,censor1981iterative},
the Generalized Pythagorean Theorem holds \cite{herbster2001tracking}:
\begin{equation}
\label{eq::general_pythagorean}
d(\mathbf{q_t},\mathbf{\hat{w}_t}) - d(\mathbf{q_t},\mathbf{w_t}) \ge d(\mathbf{w_t},\mathbf{\hat{w}_t})\ge 0
\end{equation}
where the last inequality is due to the non-negativity of Bregman divergence.

Combining Eq.(\ref{eq::online_pca_thm1_ineq}) with Eq.(\ref{eq::general_pythagorean})
and expanding the left part of $d(\mathbf{q_t},\mathbf{\hat{w}_t})-d(\mathbf{q_t},\mathbf{v_{t+1}})$, we arrive at Lemma \ref{lem::adaptive_expert_step_ineq}.
\end{proof}

Now we are ready to state the following theorem 
to upper bound the adaptive regret $\mathcal{R}_a^{\text{subexp}}$:

\begin{theorem}
\label{thm::adaptive_subset_expert}
{\it If we run the Algorithm \ref{alg::alg_best-sub-exp} to select a subset of $n-k$ experts,
then for any sequence of loss vectors $\mathbf{\ell_1}$, $\dots$, $\mathbf{\ell_T}$ $\in$ $[0,1]^n$ with $T\ge 1$,
$\min_{\mathbf{q}\in\mathcal{B}_{\text{n-k}}^\text{n} }\sum\limits_{t=r}^s (n-k)\mathbf{q}^\top\mathbf{\ell_t} \le L$,
$\alpha = 1/(T(n-k)+1)$, 
$D = (n-k)\ln(n(1+(n-k)T))+1$,
and $\eta = \ln(1+\sqrt{2D/L})$, we have 
\begin{equation*}
\mathcal{R}_a^{\text{subexp}} \le O(\sqrt{2LD}+D) 
\end{equation*}
}
\end{theorem}

\begin{sproof}
After showing the inequality from Lemma
\ref{lem::adaptive_expert_step_ineq}, the main work that remains is to
sum the right side from $t=1$ to $T$ and provide an upper bound.
This is achieved by following the proof of the Proposition 2 in \cite{cesa2012new}.
The main idea is to expand the term $\sum_{i=1}^n q_{t,i}\ln(\frac{v_{t+1,i}}{\hat{w}_{t,i}})$ as follows:
\begin{equation*}
\begin{array}{ll}
\sum_{i=1}^n q_{t,i}\ln(\frac{v_{t+1,i}}{\hat{w}_{t,i}})
= \underbrace{\sum_{i=1}^n\Big(q_{t,i}\ln\frac{1}{\hat{w}_{t,i}} - q_{t-1,i}\ln\frac{1}{v_{t,i}}\Big)}_A
 + \underbrace{\sum_{i=1}^n\Big(q_{t-1,i}\ln\frac{1}{v_{t,i}} - q_{t,i}\ln\frac{1}{v_{t+1,i}}\Big)}_B
\end{array}
\end{equation*}

Then we can upper bound the expression of $A$ with the \emph{fixed-share} step,
since $\hat{w}_{t,i}$ is lower bounded by $\frac{\alpha}{n}$.
We can telescope the expression of $B$.
Then our desired upper bound can be obtained with the help of Lemma 4 from \cite{freund1997decision}.
\end{sproof}

Please refer to the Appendix for all the omitted/sketched
proofs in this chapter.

\section{Online Adaptive PCA}
\label{sec:adaptive-pca}

Recall that the online adaptive PCA problem is below:
\begin{equation}
\label{eq::adaptive_regret_pca_alg_sec}
\begin{array}{l}
\mathcal{R}_a = \max\limits_{[r,s]\subset [1,T]}\Big\{
\sum\limits_{t=r}^s \mathbb{E}[\Tr((I-P_t)\mathbf{x_t}\mathbf{x_t}^\top)]
-
\min\limits_{U\in \mathcal{S}}\sum\limits_{t=r}^s \Tr((I-U)\mathbf{x_t}\mathbf{x_t}^\top)  \Big\}
\end{array}
\end{equation}
where $\mathcal{S}$ is the rank $k$ projection matrix set.

Again, inspired by \cite{warmuth2008randomized}, 
we first reformulate the above problem into the following 'capped probability simplex' form:
\begin{equation}
\label{eq::ref_adaptive_pca}
\begin{array}{ll}
\mathcal{R}_a =
 \max\limits_{[r,s]\subset [1,T]}\Big\{ \sum\limits_{t=r}^s (n-k)\Tr(W_t\mathbf{x_t}\mathbf{x_t}^\top) 
 - \min\limits_{Q \in \mathscr{B}_{\text{n-k}}^\text{n}}\sum\limits_{t=r}^s (n-k)\Tr(Q\mathbf{x_t}\mathbf{x_t}^\top) \Big\}
\end{array}
\end{equation}
where $W_t \in \mathscr{B}_{\text{n-k}}^\text{n}$, 
and $\mathscr{B}_{\text{n-k}}^\text{n}$ is the set of all density
matrices with eigenvalues bounded by $1/(n-k)$. Note that
$\mathscr{B}_{\text{n-k}}^\text{n}$ can be expressed as the convex set
$\{W: W\succeq 0, \left\|W\right\|_2 \le 1/(n-k), \Tr(W) = 1\}$. 

\begin{algorithm}[tb]
    \caption{Uncentered online adaptive PCA}
    \label{alg::adaptive_pca}
\begin{algorithmic}[1]
    \STATE {\bfseries Input:} $1\le k < n$ and an initial density matrix $W_1 \in \mathscr{B}_{n-k}^n$.
    \FOR{$t=1$ {\bfseries to} $T$}
    \STATE Apply eigendecomposition to $W_t$ as $W_t = \bar{D}\diag(\mathbf{w_t})\bar{D}^\top$.
    \STATE Apply Algorithm \ref{alg::alg-mix-decomp} with $d=n-k$ to the vector $\mathbf{w_t}$ to decompose it into 
    a convex combination $\sum_j p_j\mathbf{r_j}$ of at most $n$ corners $\mathbf{r_j}$.
    \STATE Randomly select a corner $\mathbf{r}=\mathbf{r_j}$ with the associated probability $p_j$.
    \STATE Form a density matrix $R = (n-k)\bar{D}\diag(\mathbf{r})\bar{D}^\top$
    \STATE Form a rank $k$ projection matrix $P_t = I - R$
    \STATE Obtain the data point $\mathbf{x_t}$, which incurs the compression loss $\left\|\mathbf{x_t}-P_t\mathbf{x_t}\right\|^2$ 
    and expected compression loss $(n-k)\Tr(W_t\mathbf{x_t}\mathbf{x_t}^\top)$.
    \STATE Update $W_{t+1}$ as:
    \begin{subequations}
    \label{eq::our_pca_update}
    \begin{align}
    \label{eq::V_t+1}
    &V_{t+1} = \frac{\exp(\ln W_t - \eta \mathbf{x_t}\mathbf{x_t}^\top)}{\Tr(\exp(\ln W_t - \eta \mathbf{x_t}\mathbf{x_t}^\top))}\\
    \label{eq::fix_share_pca}
    &\hat{w}_{t+1,i} = \frac{\alpha}{n} + (1-\alpha)v_{t+1,i}, 
    \widehat{W}_{t+1} = U\diag(\mathbf{\hat{w}_{t+1}})U^\top \\
    \label{eq::W_t+1}
    &W_{t+1} = \text{cap}_{n-k}(\widehat{W}_{t+1})
    \end{align}
    \end{subequations}
    where we apply eigendecomposition to $V_{t+1}$ as $V_{t+1} = U\diag(\mathbf{v_{t+1}})U^\top$,
    and $\text{cap}_{n-k}(W)$ invokes Algorithm \ref{alg::capping} with input being the eigenvalues of $W$.
    \ENDFOR
\end{algorithmic}
\end{algorithm}

The static regret online PCA is a special case of the above problem with $r = 1$ and $s = T$,
and is solved by Algorithm 5 in \cite{warmuth2008randomized}.

Follow the idea in the last section, 
we propose the Algorithm \ref{alg::adaptive_pca}.
Compared with the Algorithm 5 in \cite{warmuth2008randomized},
we have added the fixed-share step in the update of $W_{t+1}$ at step $9$,
which will be shown to be the key in upper bounding the adaptive regret of the online PCA.

In order to analyze Algorithm \ref{alg::adaptive_pca}, we need a few
supporting results. The first result comes from \cite{warmuth2006online}:
\begin{theorem}\cite{warmuth2006online}
\label{thm::quantum_ineq}
{\it
For any sequence of data points $\mathbf{x_1}$, $\dots$, $\mathbf{x_T}$ 
with $\mathbf{x_t}\mathbf{x_t}^\top \preceq I$ and for any learning rate $\eta$,
the following bound holds for any matrix $Q_t \in \mathscr{B}_{\text{n-k}}^\text{n}$ 
with the update in Eq.(\ref{eq::V_t+1}):
\small
\begin{equation*}
\Tr(W_t\mathbf{x_t}\mathbf{x_t}^\top) \le \frac{\Delta(Q_t,W_t) - \Delta(Q_t,V_{t+1}) 
+ \eta\Tr(Q_t\mathbf{x_t}\mathbf{x_t}^\top)}{1-\exp(-\eta)}
\end{equation*}
    
}
\end{theorem}

Based on the above theorem's result, we have the following lemma:

\begin{lemma}
\label{lem::adaptive_pca_step_ineq}
For all $t\ge 1$, all $\mathbf{x_t}$ with $\left\|\mathbf{x_t}\right\| \le 1$,
and for all $Q_t\in\mathscr{B}_{n-k}^n$, Algorithm \ref{alg::adaptive_pca} satisfies:
\begin{equation*}
\begin{array}{l}
\Tr(W_t\mathbf{x_t}\mathbf{x_t}^\top)(1-\exp(-\eta)) -\eta\Tr(Q_t\mathbf{x_t}\mathbf{x_t}^\top)
\le -\Tr(Q_t\ln \widehat{W}_t) + \Tr(Q_t\ln V_{t+1})
\end{array}
\end{equation*}
\end{lemma}

\begin{proof}

First, we need to 
reformulate the above inequality in Theorem \ref{thm::quantum_ineq}, we have:
\begin{equation}
\label{eq::V_update_ineq}
\begin{array}{l}
\Delta(Q_t,W_t) - \Delta(Q_t,V_{t+1})
\ge -\eta\Tr(Q_t\mathbf{x_t}\mathbf{x_t}^\top) + \Tr(W_t\mathbf{x_t}\mathbf{x_t}^\top)(1-\exp(-\eta))
\end{array}
\end{equation}
which is very similar to the Eq.(\ref{eq::online_pca_thm1_ineq}).

As is shown in \cite{warmuth2008randomized}, the Eq.(\ref{eq::W_t+1}) is the solution to
the following optimization problem:
\begin{equation*}
W_t = \argmin\limits_{W\in\mathscr{B}_{\text{n-k}}^\text{n}}\Delta(W,\widehat{W}_t)
\end{equation*}

As a result, the Generalized Pythagorean Theorem holds \cite{herbster2001tracking} 
for any $Q_t\in\mathscr{B}_{\text{n-k}}^\text{n}$:
\begin{equation*}
\Delta(Q_t,\widehat{W}_t) - \Delta(Q_t,W_t) \ge \Delta(W_t,\widehat{W}_t) \ge 0
\end{equation*}

Combining the above inequality with Eq.(\ref{eq::V_update_ineq}) and expanding the left part, we have
\begin{equation*}
\begin{array}{l}
\Tr(W_t\mathbf{x_t}\mathbf{x_t}^\top)(1-\exp(-\eta)) -\eta\Tr(Q_t\mathbf{x_t}\mathbf{x_t}^\top)
\le -\Tr(Q_t\ln \widehat{W}_t) + \Tr(Q_t\ln V_{t+1})
\end{array}
\end{equation*}
which proves the result.
\end{proof}

In the next theorem, we show that 
with the addition of the fixed-share step in Eq.(\ref{eq::fix_share_pca}),
we can solve the online adaptive PCA problem in Eq.(\ref{eq::adaptive_regret_pca_alg_sec}).

\begin{theorem}
\label{thm::adaptive_pca}
{\it
For any sequence of data points $\mathbf{x_1}$, $\dots$, $\mathbf{x_T}$
with $\left\|\mathbf{x_t}\right\| \le 1$, and 
$\min\limits_{Q\in\mathscr{B}_{\text{n-k}}^\text{n}}\sum\limits_{t=r}^s (n-k)\Tr(Q\mathbf{x_t}\mathbf{x_t}^\top) \le L$,
if we run Algorithm \ref{alg::adaptive_pca} with
$\alpha = 1/(T(n-k)+1)$, 
$D = (n-k)\ln(n(1+(n-k)T))+1$,
and $\eta = \ln(1+\sqrt{2D/L})$, for any $T\ge 1$ we have:
\begin{equation*}
\mathcal{R}_a \le O(\sqrt{2LD}+D) 
\end{equation*}    
}
\end{theorem}

\begin{sproof}
The proof idea is the same as in the proof of Theorem \ref{thm::adaptive_subset_expert}.
After getting the inequality relationship in Lemma \ref{lem::adaptive_pca_step_ineq}
which has a similar form as in Lemma \ref{lem::adaptive_expert_step_ineq},
we need to upper bound sum over $t$ of the right side. 
To achieve this, we first reformulate it as two parts below:
\begin{equation}
\label{eq::quantum_split_two_parts-main}
\begin{array}{l}
-\Tr(Q_t\ln \widehat{W}_t) + \Tr(Q_t\ln V_{t+1}) = \bar{A} + \bar{B}
\end{array}
\end{equation}
where $\bar{A} = -\Tr(Q_t\ln \widehat{W}_t) + \Tr(Q_{t-1}\ln V_t)$, 
and $\bar{B} = - \Tr(Q_{t-1}\ln V_t) + \Tr(Q_t\ln V_{t+1})$.

The first part can be upper bounded with the help of the fixed-share step
in lower bounding the singular value of $\hat{w}_{t,i}$.
After telescoping the second part, we can get the desired upper bound 
with the help of Lemma 4 from \cite{freund1997decision}.
\end{sproof}

\section{Extension to Online Adaptive Variance Minimization}
\label{sec:adaptive-var-min}

In this section, we study the closely related problem of online
adaptive variance minimization. The problem is defined as follows:
At each time $t$, we first select a vector $\mathbf{y_t}\in \Omega$,
and then a covariance matrix $C_t\in \mathbb{R}^{n\times n}$ such that
$0 \preceq C_t\preceq I$ is revealed. The goal is to
minimize the adaptive regret defined as:
\begin{equation}
\label{eq::general_var}
\mathcal{R}_{a}^{\text{var}} = \max_{[r,s]\subset [1,T]}\Big\{
\sum\limits_{t=r}^s \mathbb{E}[\mathbf{y_t}^\top C_t\mathbf{y_t}] -
\min_{\mathbf{u}\in\Omega}\sum\limits_{t=r}^s \mathbf{u}^\top C_t\mathbf{u}  \Big\}
\end{equation}
where the expectation is taken over the probability distribution of $\mathbf{y_t}$.

This problem has two different situations corresponding to different parameter space $\Omega$ of $\mathbf{y_t}$ and $\mathbf{u}$.

\textbf{Situation 1:} 
When $\Omega$ is the set of $\{\mathbf{x} | \left\|\mathbf{x}\right\| = 1 \}$ (e.g., the unit vector space),
the solution to $\min_{\mathbf{u}\in\Omega}\sum_{t=r}^s \mathbf{u}^\top C_t\mathbf{u}$ is the minimum eigenvector
of the matrix $\sum_{t=r}^sC_t$.

\textbf{Situation 2:} When $\Omega$ is the probability simplex (e.g., $\Omega$ is equal to $\mathcal{B}_{\text{1}}^\text{n}$),
it corresponds to the risk minimization in stock portfolios \cite{markowitz1952portfolio}. 

We will start with \textbf{Situation 1} since it is highly related to the previous section. 

\subsection{Online Adaptive Variance Minimization over the Unit vector space}

We begin with the observation of the following equivalence \cite{warmuth2006online}:
\begin{equation*}
\min_{\left\|\mathbf{u}\right\| = 1} \mathbf{u}^\top C\mathbf{u} = \min_{U\in \mathscr{B}_{\text{1}}^\text{n}}\Tr(UC) 
\end{equation*}
where $C$ is any covariance matrix, 
and $\mathscr{B}_{\text{1}}^\text{n}$ is the set of all density matrices.

Thus, the problem in (\ref{eq::general_var}) can be reformulated as:
\begin{equation}
\label{eq::var_unit_form}
\mathcal{R}_{a}^{\text{var-unit}} = \max_{[r,s]\subset [1,T]}\Big\{
\sum\limits_{t=r}^s \Tr(Y_tC_t) -
\min_{U\in\mathscr{B}_{\text{1}}^\text{n}}\sum\limits_{t=r}^s \Tr(UC_t)  \Big\}
\end{equation}
where $Y_t\in\mathscr{B}_{\text{1}}^\text{n}$.

To see the equivalence between $\mathbb{E}[\mathbf{y_t}^\top C_t\mathbf{y_t}]$ in Eq.(\ref{eq::general_var})
and $\Tr(Y_tC_t)$, 
we do the eigendecomposition of $Y_t = \sum_{i=1}^n\sigma_i\mathbf{y_i}\mathbf{y_i}^\top$.
Then $\Tr(Y_tC_t)$ is equal to $\sum_{i=1}^n\sigma_i\Tr(\mathbf{y_i}\mathbf{y_i}^\top C_t)$
$=$ $\sum_{i=1}^n\sigma_i\mathbf{y_i}^\top C_t\mathbf{y_i}$. 
Since $Y_t\in\mathscr{B}_{\text{1}}^\text{n}$, 
the vector $\mathbf{\sigma}$ is a simplex vector, 
and $\sum_{i=1}^n\sigma_i\mathbf{y_i}^\top C_t\mathbf{y_i}$ is equal to $\mathbb{E}[\mathbf{y_i}^\top C_t\mathbf{y_i}]$ 
with probability distribution defined by the vector $\mathbf{\sigma}$.

If we examine Eq.(\ref{eq::var_unit_form}) and (\ref{eq::ref_adaptive_pca}) together,
we will see that they share some similarities:
First, they are almost the same if we set $n-k =1$ in Eq.(\ref{eq::ref_adaptive_pca}).
Also, $\mathbf{x_t}\mathbf{x_t}^\top $ in Eq.(\ref{eq::ref_adaptive_pca}) is a special case of $C_t$ in Eq.(\ref{eq::var_unit_form}). 

Thus, it is possible to apply Algorithm \ref{alg::adaptive_pca} 
to solving the problem (\ref{eq::var_unit_form}) by setting $n-k =
1$. In this case, Algorithms \ref{alg::alg-mix-decomp} and \ref{alg::capping} are
not needed.
This is summarized in Algorithm \ref{alg::adaptive_var_unit}.

\begin{algorithm}[tb]
    \caption{Online adaptive variance minimization over unit sphere}
    \label{alg::adaptive_var_unit}
\begin{algorithmic}[1]
    \STATE {\bfseries Input:} an initial density matrix $Y_1 \in \mathscr{B}_{1}^n$.
    \FOR{$t=1$ {\bfseries to} $T$}
    \STATE Perform eigendecomposition $Y_t = \widehat{D}\diag(\sigma_t)\widehat{D}^\top$.
    \STATE Use the vector $\mathbf{y_t}=\widehat{D}[:,j]$ with probability $\sigma_{t,j}$.
    \STATE Receive covariance matrix $C_t$, which incurs the loss $\mathbf{y_t}^\top C_t\mathbf{y_t}$ and expected loss $\Tr(Y_tC_t)$.
    \STATE Update $Y_{t+1}$ as:
    \begin{subequations}
    \label{eq::var_unit_update}
    \begin{align}
    \label{eq::var_unit_v_t+1}
    &V_{t+1} = \frac{\exp(\ln Y_t - \eta C_t)}{\Tr(\exp(\ln Y_t - \eta C_t))}\\
    \label{eq::var_unit_fix_share_pca}
    &\sigma_{t+1,i} = \frac{\alpha}{n} + (1-\alpha)v_{t+1,i}, 
    Y_{t+1} = \widehat{U}\diag(\sigma_{t+1})\widehat{U}^\top
    \end{align}
    \end{subequations}
    where we apply eigendecomposition to $V_{t+1}$ as $V_{t+1} = \widehat{U}\diag(\mathbf{v_{t+1}})\widehat{U}^\top$.
    \ENDFOR
\end{algorithmic}
\end{algorithm}

The theorem below is analogous to Theorem \ref{thm::adaptive_pca} in
the case that $n-k = 1$.

\begin{theorem}
\label{thm::adaptive_var_unit}
{\it
For any sequence of covariance matrices $C_1$, $\dots$, $C_T$
with $0\preceq C_t \preceq I$, and for 
$\min\limits_{U\in\mathscr{B}_{\text{1}}^\text{n}}\sum\limits_{t=r}^s \Tr(UC_t) \le L$,
if we run Algorithm \ref{alg::adaptive_var_unit} with
$\alpha = 1/(T+1)$, 
$D = \ln(n(1+T))+1$,
and $\eta = \ln(1+\sqrt{2D/L})$, for any $T\ge 1$ we have:
\begin{equation*}
\mathcal{R}_a^{\text{var-unit}} \le O(\sqrt{2LD}+D) 
\end{equation*}    
}
\end{theorem}

\begin{sproof}
Similar inequality can be obtained as in Lemma \ref{lem::adaptive_pca_step_ineq}
by using the Theorem 2 in \cite{warmuth2006online}.
The rest follows the proof of Theorem \ref{thm::adaptive_pca}.
\end{sproof}

In order to apply the above theorem, we need to either estimate the step size $\eta$ heuristically
or estimate the upper bound $L$,
which may not be easily done.

In the next theorem, we show that 
we can still upper bound the $\mathcal{R}_a^{\text{var-unit}}$ without knowing $L$,
but the upper bound is a function of time horizon $T$ instead of the upper bound $L$.

Before we get to the theorem, we need the following lemma which lifts the
vector case of Lemma 1 in \cite{cesa2012new} to the density matrix case:
\begin{lemma}
\label{lem::mat_lem1_bianchi}
For any $\eta \ge 0$, $t\ge 1$, any covariance matrix $C_t$ with $0\preceq C_t\preceq I$,
and for any $Q_t\in\mathscr{B}_{1}^n$, Algorithm \ref{alg::adaptive_var_unit} satisfies:
\begin{equation*}
\begin{array}{l}
\Tr(Y_tC_t) - \Tr(Q_tC_t)
\le \frac{1}{\eta}\Big( \Tr(Q_t\ln V_{t+1}) - \Tr(Q_t\ln Y_t) \Big) + \frac{\eta}{2}
\end{array}
\end{equation*}
\end{lemma}

Now we are ready to present the upper bound on the regret for
Algorithm~\ref{alg::adaptive_var_unit}. 
\begin{theorem}
\label{thm::T_depend_adaptive_var_unit}
{\it
For any sequence of covariance matrices $C_1$, $\dots$, $C_T$
with $0\preceq C_t \preceq I$,
if we run Algorithm \ref{alg::adaptive_var_unit} with
$\alpha = 1/(T+1)$
and $\eta = \frac{\sqrt{\ln(n(1+T))}}{\sqrt{T}}$, for any $T\ge 1$ we have:
\begin{equation*}
\mathcal{R}_a^{\text{var-unit}} \le O\Big(\sqrt{T\ln\big(n(1+T)\big)}\Big) 
\end{equation*}    
}
\end{theorem}

\begin{proof}

In the proof, we will use two cases of $Q_t$: $Q_t\in\mathscr{B}_{1}^n$, and $Q_t = 0$.

From Lemma \ref{lem::mat_lem1_bianchi}, the following inequality is valid for both cases of $Q_t$:
\begin{equation*}
\begin{array}{l}
\Tr(Y_tC_t) - \Tr(Q_tC_t) 
\le \frac{1}{\eta}\Big( \Tr(Q_t\ln V_{t+1}) - \Tr(Q_t\ln Y_t) \Big) + \frac{\eta}{2}
\end{array}
\end{equation*}

Follow the same analysis as in the proof of Theorem \ref{thm::adaptive_pca},
we first do the eigendecomposition to $Q_t$ as $Q_t = \widetilde{D}\diag(q_t)\widetilde{D}^\top$.
Since $\left\|q_t\right\|_1$ is either $1$ or $0$, we will re-write the above inequality as:
\begin{equation}
\begin{array}{l}
\left\|q_t\right\|_1\Tr(Y_tC_t) - \Tr(Q_tC_t)
\le \frac{1}{\eta}\Big( \Tr(Q_t\ln V_{t+1}) - \Tr(Q_t\ln Y_t) \Big) + \frac{\eta}{2}\left\|q_t\right\|_1
\end{array}
\end{equation}

Analyzing the term $\Tr(Q_t\ln V_{t+1}) - \Tr(Q_t\ln Y_t)$ in the above inequality 
is the same as the analysis of the Eq.(\ref{eq::quantum_split_two_parts-main}) in the appendix.

Thus, summing over $t=1$ to $T$ to the above inequality, and setting $Q_t = Q\in\mathscr{B}_1^n$ for $t=r,\dots,s$
and $0$ elsewhere,
we have 
\begin{equation*}
\begin{array}{l}
\sum\limits_{t=r}^s\Tr(Y_tC_t) - \min\limits_{U\in\mathscr{B}_1^n}\sum\limits_{t=r}^s\Tr(UC_t)
\le \frac{1}{\eta}\Big(\ln\frac{n}{\alpha}+T\ln\frac{1}{1-\alpha}\Big) + \frac{\eta}{2}T,
\end{array}
\end{equation*} 
since it holds for any $Q\in\mathscr{B}_1^n$.

After plugging in the expression of $\eta$ and $\alpha$, we have
\begin{equation*}
\begin{array}{l}
\sum\limits_{t=r}^s\Tr(Y_tC_t) - \min\limits_{U\in\mathscr{B}_1^n}\sum\limits_{t=r}^s\Tr(UC_t) 
\le O\Big(\sqrt{T\ln\big(n(1+T)\big)}\Big) 
\end{array}
\end{equation*}

Since the above inequality holds for any $1\le r\le s\le T$, 
we put a $\max\limits_{[r,s]\subset [1,T]}$ in the left part,
which proves the result.
\end{proof}

\subsection{Online Adaptive Variance Minimization over the Simplex space}

We first re-write the problem in Eq.(\ref{eq::general_var}) when
$\Omega$ is the simplex below:
\small
\begin{equation}
\label{eq::var_simplex_form}
\mathcal{R}_{a}^{\text{var-sim}} = \max_{[r,s]\subset [1,T]}\Big\{
\sum\limits_{t=r}^s \mathbb{E}[\mathbf{y_t}^\top C_t\mathbf{y_t}] -
\min_{\mathbf{u}\in\mathcal{B}_1^n}\sum\limits_{t=r}^s \mathbf{u}^\top C_t\mathbf{u}  \Big\}
\end{equation}
\normalsize
where $\mathbf{y_t}\in\mathcal{B}_1^n$, and $\mathcal{B}_1^n$ is the simplex set.

When $r = 1$ and $s = T$, the problem reduces to the static regret problem,
which is solved in \cite{warmuth2006online} by the exponentiated gradient algorithm as below:
\begin{equation*}
y_{t+1,i} = \frac{y_{t,i}\exp\big(-\eta(C_t\mathbf{y_t})_i\big)}{\sum_i y_{t,i}\exp\big(-\eta(C_t\mathbf{y_t})_i\big)}
\end{equation*}

As is done in the previous sections, we add the fixed-share step after the above update,
which is summarized in Algorithm \ref{alg::adaptive_var_simplex}.

\begin{algorithm}[tb]
    \caption{Online adaptive variance minimization over simplex}
    \label{alg::adaptive_var_simplex}
\begin{algorithmic}[1]
    \STATE {\bfseries Input:} an initial vector $\mathbf{y_1} \in \mathcal{B}_{1}^n$.
    \FOR{$t=1$ {\bfseries to} $T$}
    \STATE Receive covariance matrix $C_t$.
    \STATE Incur the loss $\mathbf{y_t}^\top C_t\mathbf{y_t}$.
    \STATE Update $\mathbf{y_{t+1}}$ as:
    \begin{subequations}
    \begin{align}
    &v_{t+1,i} = \frac{y_{t,i}\exp\big(-\eta(C_t\mathbf{y_t})_i\big)}{\sum_i y_{t,i}\exp\big(-\eta(C_t\mathbf{y_t})_i\big)},\\
    &y_{t+1,i} = \frac{\alpha}{n} + (1-\alpha)v_{t+1,i}.
    \end{align}
    \end{subequations}
    \ENDFOR
\end{algorithmic}
\end{algorithm}

With the update of $y_t$ in the Algorithm \ref{alg::adaptive_var_simplex},
we have the following theorem:
\begin{theorem}
\label{thm::adaptive_var_simplex}
{\it
For any sequence of covariance matrices $C_1$, $\dots$, $C_T$
with $0\preceq C_t \preceq I$, and for 
$\min\limits_{\mathbf{u}\in\mathcal{B}_{\text{1}}^\text{n}}\sum\limits_{t=r}^s \mathbf{u}^\top C_t\mathbf{u} \le L$,
if we run Algorithm \ref{alg::adaptive_var_simplex} with
$\alpha = 1/(T+1)$, $c = \frac{\sqrt{2\ln\big((1+T)n\big)+2}}{\sqrt{L}}$,
$b = \frac{c}{2}$, $a = \frac{b}{2b+1}$, 
and $\eta = 2a$, for any $T\ge 1$ we have:
\begin{equation*}
\mathcal{R}_a^{\text{var-sim}} \le 2\sqrt{2L\Big(\ln\big((1+T)n\big)+1\Big)} + 2\ln\big((1+T)n\big)
\end{equation*}    
}
\end{theorem}

\section{Experiments}
\label{sec:adaptive-pca-exp}

\begin{figure}
\vskip 0.0in
  \centering
  \subfigure[]{
    \includegraphics[height=5.cm]{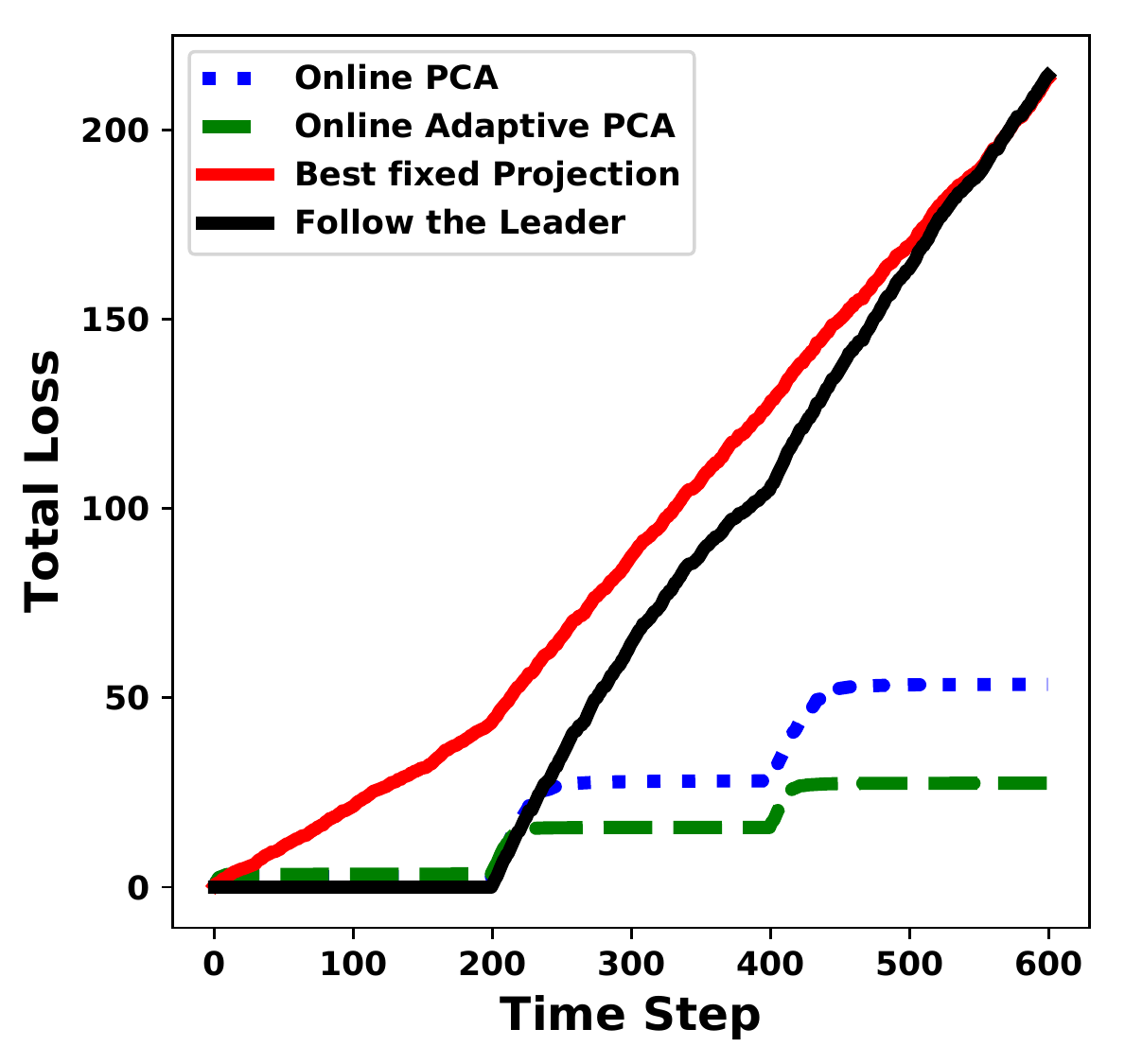}
    }
  \hspace{.1in}
  \subfigure[]{
    \includegraphics[height=5.cm]{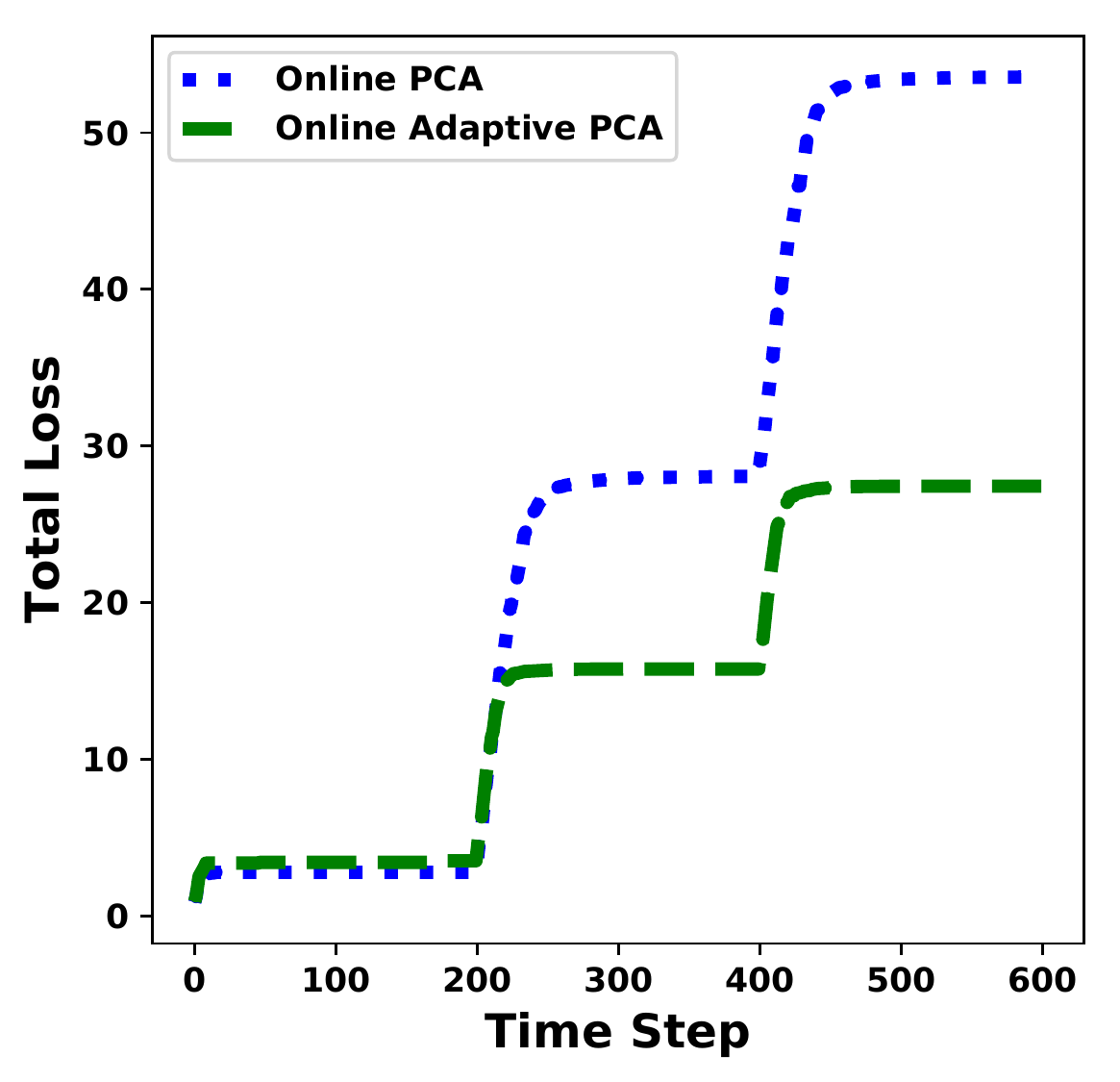}
    }
  \caption{
           (a): 
           The cumulative loss of the toy example with data samples coming from three different subspaces.
           (b):
           The detailed comparison for the two online algorithms. }
  \label{fig::synthetic_subspace} 
\vskip 0.in
\end{figure}

In this section, we use two examples to illustrate the effectiveness of our proposed online adaptive PCA algorithm.
The first example is synthetic, which shows that our proposed algorithm (denoted as Online Adaptive PCA) 
can adapt to the changing subspace faster than the method of
\cite{warmuth2008randomized}.
The second example uses the practical dataset Yale-B to demonstrate that the proposed algorithm
can have lower cumulative loss in practice when the data/face samples are coming from different persons.

The other algorithms that are used as comparators are:
1. Follow the Leader algorithm (denoted as Follow the Leader) \cite{kalai2005efficient}, 
which only minimizes the loss on the past history;
2. The best fixed solution in hindsight (denoted as Best fixed Projection),
which is the solution to the Problem described in Eq.(\ref{eq::best_fixed_sol_PCA});
3. The online static PCA (denoted as Online PCA) \cite{warmuth2008randomized}.
Other PCA algorithms are not included, since they are not designed for regret minimization.

\subsection{A Toy Example}

\begin{figure}
\vskip 0.0in
\centering
\includegraphics[height=6cm]{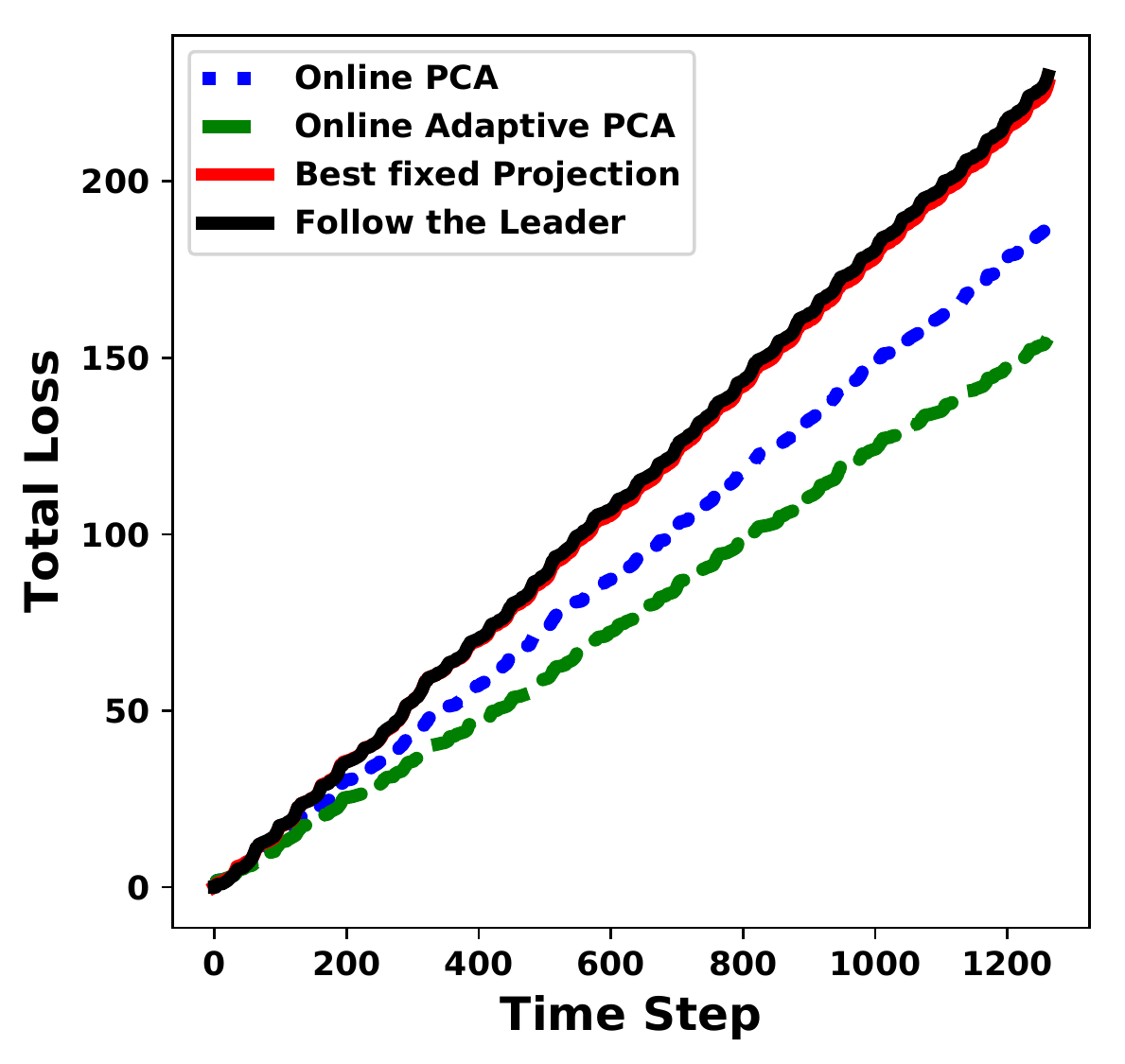}
\caption{The cumulative loss for the face example with data samples coming from 20 different persons}
\label{fig::yale_face}
\vskip 0.in
\end{figure}

In this toy example, we create the synthetic data samples coming from changing subspace,
which is a similar setup as in \cite{warmuth2008randomized}.
The data samples are divided into three equal time intervals, 
and each interval has 200 data samples.
The 200 data samples within same interval is randomly generated by a Gaussian distribution 
with zero mean and data dimension equal to 20, and
the covariance matrix is randomly generated with rank equal to 2.
In this way, the data samples are from some unknown 2-dimensional subspace,
and any data sample with $\ell_2$-norm greater than 1 is normalized to 1.
Since the stepsize used in the two online algorithms is determined by 
the upper bound of the batch solution, we first find the upper bound and 
plug into the stepsize function, which gives $\eta = 0.19$.
We can tune the stepsize heuristically in practice 
and in this example we just use $\eta = 1$ and $\alpha =1\mathrm{e}{-5}$.

After all data samples are generated, we apply the previously mentioned algorithms
with $k=2$ and obtain the cumulative loss as a function of time steps,
which is shown in Fig.\ref{fig::synthetic_subspace}.
From this figure we can see that:
1. Follow the Leader algorithm is not appropriate in the setting 
where the sequential data is shifting over time.
2. The static regret is not a good metric under this setting,
since the best fixed solution in hindsight is suboptimal.
3. Compared with Static PCA, the proposed Adaptive PCA can
adapt to the changing environment faster,
which results in lower cumulative loss and
is more appropriate when the data is shifting over time.

\subsection{Face data Compression Example}

In this example, we use the Yale-B dataset which is a collection of
face images. The data is split into 
20 time intervals corresponding to 20 different people. Within each interval,
there are 64 face image samples.
Like the previous example, we first normalize the data
to ensure its $\ell_2$-norm not greater than 1.
We use $k = 2$, which is the same as the previous example.
The stepsize $\eta$ is also tuned heuristically like the previous example,
which is equal to $5$ and $\alpha = 1\mathrm{e}{-4}$.

We apply the previously mentioned algorithms 
and again obtain the cumulative loss as the function of time steps,
which is displayed in Fig.\ref{fig::yale_face}.
From this figure we can see that
although there is no clear bumps indicating the shift from one subspace to another
as the Fig.\ref{fig::synthetic_subspace} of the toy example,
our proposed algorithm still has the lowest cumulative loss,
which indicates that upper bounding the adaptive regret
is still effective when the compressed faces are coming from different persons.

\section{Conclusion}
\label{sec:adaptive-pca-conclusion}

In this chapter, we propose an online adaptive PCA algorithm, which
augments the previous online static PCA algorithm with a fixed-share
step. 
However,
different from the previous online PCA algorithm which is designed
to minimize the static regret,
the proposed online adaptive PCA algorithm aims to
minimize the adaptive regret
which is more appropriate when the underlying environment is changing
or the sequential data is shifting over time.
We demonstrate theoretically and experimentally that our algorithm can adapt to the
changing environments. 
Furthermore, we extend the online adaptive PCA algorithm 
to online adaptive variance minimization problems.

One may note that the proposed algorithms suffer from the per-iteration computation complexity of $O(n^3)$
due to the eigendecomposition step, although some tricks mentioned in \cite{arora2012stochastic} 
could be used to make it comparable with
incremental PCA of $O(k^2n)$.
For the future work,
one possible direction is to investigate algorithms with slightly worse adaptive regret bound
but with better per-iteration computation complexity.

\chapter{Online Convex Optimization for Cumulative Constraints}
\label{chap:oco-long-term}

Previous chapter focuses on how to enable the online PCA algorithm
to have better adaptivity to the changing environments.

In this chapter, we come back to the general online convex optimization (OCO) problem.
For online convex optimization with constraints, a projection operator
is typically applied
in order to make the updated prediction feasible for each time step \cite{zinkevich2003online,duchi2008efficient,duchi2010composite}.
However, when the constraints are complex, 
the computational burden of the projection may be too high
to have online computation.

To circumvent this
dilemma, \cite{mahdavi2012trading,jenatton2016adaptive,yu2017online} proposed
algorithms which approximates the true desired projection with a
simpler closed-form projection. 
The algorithm can still upper bound the static
regret $\cR_s$ by $\sqrt{T}$ as the optimal result in \cite{zinkevich2003online},
but the constraint $\cS = \{\theta: g_i(\theta)\le 0, i = 1,\dots,m\}$ may not be satisfied in every time
step.
Instead, the long-term constraint violation $\sum\limits_{t=1}^Tg_i(\theta_t), \forall i$
can be upper bounded in a sub-linear order $o(T)$,
which is useful when we only require the constraint violation to be
non-positive on average: $\lim_{T\to \infty}\sum\limits_{t=1}^Tg_i(\theta_t)/T \le 0,\forall i$.
However,
this bound does not enforce that the violation of the constraint gets small,
which is originally desired.
A situation can arise in which strictly satisfied
constraints at one time step can cancel out violations of the
constraints at other time steps.

Along the line of the long-term constraint work, there are some variations,
which make the long-term constraint idea apply to the online resource allocation.
This is achieved by regarding the total resource consumption constraint for different time steps as the
long-term time-dependent constraint.

For the online job scheduling, 
\cite{yu2017online} considered the stochastic long-term constraint case.
It achieves $O(\sqrt{T})$ bound for both $\cR_s$ 
and the expected long-term constraint $\bbE[\sum\limits_{t=1}^T g_t(\theta_t)]$
with Slater condition assumption.
However, such Slater condition assumes that the compared static action needs to be strictly feasible,
which means there exists at least one point 
lying in the intersection of constraints $\bbE[g_t(\theta)]< 0, t=1,\dots,T$.
This limits the claimed regret performance due to the increasing difficulty in satisfying all the constraints,
resulting in \emph{loose regret}.

To solve this loose regret problem,
\cite{liakopoulos2019cautious} came up with the idea that the fixed comparator only needs to satisfy
part of the time-dependent (possibly adversarial) constraints.
That is, it used a different fixed comparator 
$\hat{\theta}=\argmin_{\theta\in \Theta_K}\sum\limits_{t=1}^T f_t(\theta)$,
where $\Theta_K = \{\theta\in\cS_0: \sum\limits_{i=t}^{t+K-1}g_i(\theta)\le 0, 1\le t\le T-K+1\}$,
$\cS_0$ is the fixed convex set, and $K$ is a user-determined parameter.

Although $\hat{\theta}\in\Theta_K$ can be used in the long-term budget constraint
when $g_t$ represents budget at each time step,
it is sometimes not a reasonable choice
in that many other resource allocation problems' constraints 
cannot simply be added together due to causality.
For example, in the online job scheduling,
previous time step's vacancy of the server cannot be carried over to the next time step,
while the unfinished jobs can.
Or we want to ensure that the rate of failures (the constraint violation itself) is upper bounded.

In this chapter, we show how our proposed algorithms can be used to tackle the two previously mentioned problems:
not enforcing low constraint violation and limited application in resource allocation.

In the first part of the chapter which is adapted from the published work \cite{yuan2018online}, 
we will show how the proposed algorithms can enforce low constraint violation
for the following two different problem setups:

\textbf{Convex Case:} The first algorithm is for the convex case,
which also has the user-determined trade-off
as in \cite{jenatton2016adaptive}, while the constraint violation is more strict.
Specifically, we have $\cR_s\le O(T^{\max\{\beta,1-\beta\}})$ 
and $\sum\limits_{t=1}^T\big([g_i(\theta_t)]_+\big)^2 \le O(T^{1-\beta}),\forall i$
where $[g_i(\theta_t)]_+ =\max\{0,g_i(\theta_t)\}$ and  
$\beta\in(0,1)$.
Note the square term heavily penalizes large constraint violations and
constraint violations from one step cannot be canceled out by strictly
feasible steps. 
Additionally, we give a bound on the cumulative
constraint violation 
$\sum\limits_{t=1}^T[g_i(\theta_t)]_+ \le O(T^{1-\beta/2})$, which
generalizes the bounds from \cite{mahdavi2012trading,jenatton2016adaptive}. 

In the case of $\beta = 0.5$, which we call "balanced", 
both $\cR_s$ and 
$\sum\limits_{t=1}^T([g_i(\theta_t)]_+)^2$ have the same upper bound of $O(\sqrt{T})$.
More importantly,
our algorithm guarantees that at each time step, 
the clipped constraint term $[g_i(\theta_t)]_+$ is upper bounded by $O(\frac{1}{T^{1/6}})$,
which does not follow from the results of \cite{mahdavi2012trading,jenatton2016adaptive}.
However, our results currently cannot generalize those of
\cite{yu2017online}, 
which has $\sum\limits_{t=1}^Tg_i(\theta_t)\le O(\sqrt{T})$. 
It is unclear how to extend the work of \cite{yu2017online} to
the clipped constraints, $[g_i(\theta_t)]_+$.

\textbf{Strongly Convex Case:} Our second algorithm for strongly convex function $f_t(\theta)$
gives us the improved upper bounds compared with the previous work in \cite{jenatton2016adaptive}.
Specifically, we have $\cR_s\le O(\log(T))$, and
$\sum\limits_{t=1}^T[g_i(\theta_t)]_+ \le O(\sqrt{\log(T)T}),\forall i$. 
The improved bounds match the regret order of standard OCO from
\cite{hazan2007logarithmic}, while maintaining a constraint violation of reasonable order.

We show numerical experiments on three problems. A toy example is used
to compare trajectories of our algorithm with those of
\cite{jenatton2016adaptive,mahdavi2012trading}, and we see that our
algorithm tightly follows the constraints. The algorithms are
also compared on a doubly-stochastic matrix approximation problem
\cite{jenatton2016adaptive} and an economic dispatch problem from
power systems. In these, our algorithms lead to reasonable objective
regret and low cumulative constraint violation.

In the second part of the chapter, 
we will discuss how to apply the proposed algorithms
to the general resource allocation problems
with \emph{tight regret} guarantee
by using a variant of \emph{dynamic regret}.

\section{Problem Formulation}
\label{sec:oco-long-term-prob-form}

The basic projected gradient algorithm achieving $\cR_s\le O(\sqrt{T})$
for convex problem was defined in \cite{zinkevich2003online}.
Specifically, at each iteration $t$, the update rule is:
\begin{equation}
\label{eq::ogd_2003}
\begin{array}{lll}
\theta_{t+1} &= \Pi_S(\theta_t-\eta \nabla f_t(\theta_t)) 
& = \arg \min\limits_{y\in S}\left\|y-(\theta_t-\eta \nabla f_t(\theta_t))\right\|^2
\end{array}
\end{equation}
where $\Pi_S$ is the projection operation to the set $S$.

Although the algorithm is simple,  
it needs to solve
a constrained optimization problem at every time step, which might be too time-consuming
for online implementation
when the constraints are complex.

In order to lower the computational complexity and accelerate the online processing speed,
the work of \cite{mahdavi2012trading} avoids the convex optimization by projecting the variable to a fixed ball $\cS\subseteq\mathcal{B}$,
which always has a closed-form solution.
That paper gives an online solution for the following problem:
\begin{equation}
\label{eq::original long term version}
\begin{array}{llll}
\underset{\theta_1,\ldots,\theta_T\in \mathcal{B}}{\min} & \sum\limits_{t=1}^T f_t(\theta_t) -
                                    \min\limits_{\theta\in
                                    \cS}\sum\limits_{t=1}^T f_t(\theta)
&s.t. & \sum\limits_{t=1}^T g_i(\theta_t)\le 0, i = 1,2,...,m 
\end{array}
\end{equation}
where $\cS = \{\theta: g_i(\theta)\le 0, i=1,2,...,m \} \subseteq \mathcal{B}$. It
is assumed that there exist constants $R>0$ and $r<1$ such that
$r\mathbb{K}\subseteq \cS \subseteq R\mathbb{K}$ with $\mathbb{K}$ being
the unit $\ell_2$ ball centered at the origin and $\mathcal{B} = R\mathbb{K}$.

Compared to the update in Eq.~\eqref{eq::ogd_2003}, which requires
$\theta_t \in \cS$ for all $t$, \eqref{eq::original long term version} implies that 
 only the sum of constraints is
 required. This sum of constraints is known as the \emph{long-term
   constraint}.

To solve this new problem, \cite{mahdavi2012trading} considers the following augmented Lagrangian function at each iteration $t$:
\begin{equation}
\label{eq::pre_l_t}
\mathcal{L}_t(\theta,\lambda) = f_t(\theta) + \sum\limits_{i=1}^m \Big\{ \lambda_ig_i(\theta) - \frac{\sigma \eta}{2}\lambda_i^2 \Big\}
\end{equation}

The update rule is as follows:
\begin{equation}
\label{eq::original update rule}
\begin{array}{ll}
\theta_{t+1} = \Pi_{\mathcal{B}}(\theta_t-\eta\nabla_\theta \mathcal{L}_t(\theta_t,\lambda_t) ), &
\lambda_{t+1} = \Pi_{[0,+\infty)^m}(\lambda_t + \eta\nabla_{\lambda} \mathcal{L}_t(\theta_t,\lambda_t) )
\end{array}
\end{equation}
where $\eta$ and $\sigma$ are the pre-determined step size and some constant, respectively. 

More recently, an adaptive version was developed in \cite{jenatton2016adaptive},
which has a user-defined trade-off
parameter. The algorithm proposed by \cite{jenatton2016adaptive}
utilizes two different step size sequences
to update $\theta$ and $\lambda$, respectively,  
instead of using a single step size $\eta$.

In both algorithms of \cite{mahdavi2012trading} and
\cite{jenatton2016adaptive}, the bound for the violation of the
long-term constraint is that  
$\forall i$, $\sum\limits_{t=1}^T g_i(\theta_t)\le O(T^{\gamma})$ for some $\gamma \in (0,1)$. 
However, as argued in the last section,
this bound does not enforce that the violation of the constraint $\theta_t
\in \cS$ gets small
due to the possible cancellation from strictly feasible steps.
This problem can be rectified by considering clipped constraint, $[g_i(\theta_t)]_+$, in place
of $g_i(\theta_t)$.

For convex problems,  
our goal is to bound the term $\sum\limits_{t=1}^T \big([g_i(\theta_t)]_+\big)^2$,
which, as discussed in the previous section, is more useful for
enforcing small constraint violations,
and also recovers the existing bounds for both $\sum\limits_{t=1}^T [g_i(\theta_t)]_+$ and $\sum\limits_{t=1}^T g_i(\theta_t)$.
For strongly convex problems, we also show the improvement on the upper bounds
compared to the results in \cite{jenatton2016adaptive}.

In sum,
in this chapter, our first goal is to solve the following problem for the general convex condition:
\begin{equation}
\label{eq::new long term problem}
\begin{array}{llll}
\min\limits_{\theta_1,\theta_2,...,\theta_T\in\mathcal{B}} & \sum\limits_{t=1}^T
                                                f_t(\theta_t) -
                                                \min\limits_{\theta\in
                                                S}\sum\limits_{t=1}^T
                                                f_t(\theta) 
&\quad \quad s.t. & \sum\limits_{t=1}^T \big([g_i(\theta_t)]_+\big)^2\le O(T^\gamma),
       \forall i 
\end{array}
\end{equation}
where $\gamma \in (0,1)$. The new constraint from
\eqref{eq::new long term problem} is called the \emph{square-clipped
 long-term constraint} (since it is a square-clipped version of the
long-term constraint) or \emph{square-cumulative constraint} (since it
encodes the square-cumulative violation of the constraints).

To solve Problem (\ref{eq::new long term problem}), we change the augmented Lagrangian function $\mathcal{L}_t$ as follows:
\begin{equation}
  \label{eq::new long term lagrangian}
\mathcal{L}_t(\theta,\lambda) = f_t(\theta) + \sum\limits_{i=1}^m \Big\{ \lambda_i[g_i(\theta)]_+ - \frac{\phi_t}{2}\lambda_i^2 \Big\}
\end{equation}

We will also see in a later section how this new function $\mathcal{L}_t$ can be used to 
get general time-dependent resource allocation problem
under a variant of \emph{dynamic regret}.

Throughout this chapter, 
we will use the following assumptions as in \cite{mahdavi2012trading}:
1. The convex set $\cS$ is non-empty, closed, bounded, and can be described by $m$ convex functions 
as $\cS = \{\theta: g_i(\theta)\le 0, i =1,2,...,m \}$.
2. Both the loss functions $f_t(\theta)$, $\forall t$ and constraint functions $g_i(\theta)$, $\forall i$ 
are Lipschitz continuous in the set $\mathcal{B}$.
That is, $\left\|f_t(x) - f_t(y)\right\| \le L_f\left\|x-y\right\|$, $\left\|g_i(x) - g_i(y)\right\| \le L_g\left\|x-y\right\|$,
$\forall x,y\in \mathcal{B}$ and $\forall t,i$. $G=\max\{L_f,L_g\}$, and 
\begin{equation*}
\begin{array}{ll}
F = \max\limits_{t=1,2,...,T}\max\limits_{x,y\in\mathcal{B}} f_t(x) -f_t(y)\le 2L_fR,  &
D = \max\limits_{i=1,2,...,m}\max\limits_{x\in\mathcal{B}}g_i(x)\le L_gR
\end{array}
\end{equation*}

\section{Algorithm}
\label{sec:static-long-term-algorithm}

\subsection{Convex Case}

\begin{algorithm}[tb]
    \caption{Generalized Online Convex Optimization with Long-term Constraint}
    \label{alg::convex-long-term}
\begin{algorithmic}[1]
    \STATE {\bfseries Input:} constraints $g_i(\theta)\le 0,i=1,2,...,m$, stepsize $\eta$, time horizon T, and constant $\sigma>0$.
    \STATE {\bfseries Initialization:} $\theta_1$ is in the center of the $\mathcal{B}$ .
    \FOR{$t=1$ {\bfseries to} $T$}
    \STATE Input the prediction result $\theta_t$.
    \STATE Obtain the convex loss function $f_t(\theta)$ and the loss value $f_t(\theta_t)$.
    \STATE Calculate a subgradient $\partial_\theta
    \mathcal{L}_t(\theta_t,\lambda_t)$, where:
    \small
    \begin{equation*}
    \begin{array}{ll}
     \partial_\theta \mathcal{L}_t(\theta_t,\lambda_t) = \partial_\theta f_t(\theta_t) + \sum\limits_{i=1}^m \lambda_t^i\partial_\theta ([g_i(\theta_t)]_+),
     &\partial_\theta ([g_i(\theta_t)]_+) =
     \begin{cases}
     0, \quad\mbox{$g_i(\theta_t) \le0$}\\
     \partial_\theta g_i(\theta_t), \mbox{otherwise}\\
     \end{cases}
    \end{array}
    \end{equation*}
    \normalsize
    \STATE Update $\theta_t$ and $\lambda_t$ as below:
    \begin{equation*}
    \begin{array}{ll}
    \theta_{t+1} = \Pi_{\mathcal{B}}(\theta_t-\eta \partial_\theta \mathcal{L}_t(\theta_t,\lambda_t)),
    \lambda_{t+1} = \frac{[g(\theta_{t+1})]_+}{\sigma \eta}
    \end{array}
    \end{equation*}
    \ENDFOR
\end{algorithmic}
\end{algorithm}

The main algorithm for this chapter is shown in Algorithm \ref{alg::convex-long-term}.
For
simplicity, we abuse the subgradient notation, denoting a single
element of the subgradient by $\partial_x \mathcal{L}_t(x_t,\lambda_t)$.
We also replace the $\phi_t$ in Eq.~\eqref{eq::new long term lagrangian}
with $\sigma\eta$.
Comparing our algorithm with Eq.(\ref{eq::original update rule}), we
can see that the gradient projection step for $\theta_{t+1}$ is similar, while the update
rule for $\lambda_{t+1}$ is different.
Instead of a projected gradient step, we explicitly maximize
$\mathcal{L}_{t+1}(\theta_{t+1},\lambda)$ over $\lambda$.
This explicit projection-free update for $\lambda_{t+1}$ is possible
because the constraint clipping guarantees that the maximizer is non-negative. 
Furthermore, this constraint-violation-dependent update helps to
enforce small cumulative and individual constraint
violations. Specific bounds on constraint violation are given in
Theorem \ref{thm::sumOfSquareLongterm} and Lemma \ref{lem:bound_step} below.

Based on the update rule in Algorithm \ref{alg::convex-long-term},
the following theorem gives the upper bounds for both the regret on the loss
and the squared-cumulative constraint violation, $\sum\limits_{t=1}^T\Big([g_i(x_t)]_+\Big)^2$ in Problem \ref{eq::new long term problem}. 

\begin{theorem}
\label{thm::sumOfSquareLongterm}
{\it
Set $\sigma = \frac{(m+1)G^2}{2(1-\alpha)}$, 
$\eta = \frac{1}{G\sqrt{(m+1)RT}}$. 
If we follow the update rule in Algorithm \ref{alg::convex-long-term} with $\alpha\in(0,1)$
and $\theta^*$ being the optimal solution for $\min\limits_{\theta\in \cS}\sum\limits_{t=1}^Tf_t(\theta)$, 
we have
\begin{equation*}
\begin{array}{ll}
\sum\limits_{t=1}^T\Big( f_t(\theta_t) - f_t(\theta^*)\Big)\le O(\sqrt{T}),& 
\sum\limits_{t=1}^T\Big([g_i(\theta_t)]_+\Big)^2 \le O(\sqrt{T}), \forall i \in \{1,2,...,m\}
\end{array}
\end{equation*}
}
\end{theorem}

Before proving Theorem~\ref{thm::sumOfSquareLongterm}, we need the
following preliminary result. 

\begin{lemma}
  \label{lem:sumOfLagfunction}
  {\it
    For the sequence of $\theta_t$, $\lambda_t$ obtained from Algorithm \ref{alg::convex-long-term} and $\forall \theta\in\mathcal{B}$,
    we have the following inequality:
    \begin{equation*}
    \begin{array}{l}
    \sum\limits_{t=1}^T[\mathcal{L}_t(\theta_t,\lambda_t)-\mathcal{L}_t(\theta,\lambda_t)]\le 
    \frac{R^2}{2\eta}+\frac{\eta T}{2}(m+1)G^2 
    +\frac{\eta}{2}(m+1)G^2\sum\limits_{t=1}^T\left\|\lambda_t\right\|^2
    \end{array}
    \end{equation*}

  }
\end{lemma}

\begin{proof}
First, $\mathcal{L}_t(\theta,\lambda)$ is convex in $\theta$. Then for any $\theta\in\mathcal{B}$,
we have the following inequality:
\begin{equation*}
\mathcal{L}_t(\theta_t,\lambda_t)-\mathcal{L}_t(\theta,\lambda_t) \le (\theta_t-\theta)^\top\partial_\theta\mathcal{L}_t(\theta_t,\lambda_t)
\end{equation*}
Using the non-expansive property of the projection operator and the update rule for $\theta_{t+1}$ in Algorithm \ref{alg::convex-long-term}, 
we have
\begin{equation}
\label{eq::x_projection_inequality}
\begin{array}{ll}
\left\|\theta-\theta_{t+1}\right\|^2& \le \left\|\theta-(\theta_t-\eta\partial_\theta \mathcal{L}_t(\theta_t,\lambda_t))\right\|^2 \\
                          & = \left\|\theta-\theta_t\right\|^2-2\eta(\theta_t-\theta)^\top\partial_\theta\mathcal{L}_t(\theta_t,\lambda_t)
                          +\eta^2\left\|\partial_\theta\mathcal{L}_t(\theta_t,\lambda_t)\right\|^2  
\end{array}
\end{equation}
Then we have
\begin{equation}
\label{eq::L_t_diff_inequality}
\begin{array}{rl}
\mathcal{L}_t(\theta_t,\lambda_t)-\mathcal{L}_t(\theta,\lambda_t) & \le \frac{1}{2\eta}\Big(\left\|\theta-\theta_t\right\|^2
                                                                           -\left\|\theta-\theta_{t+1}\right\|^2\Big)
                                                         +\frac{\eta}{2}\left\|\partial_\theta\mathcal{L}_t(\theta_t,\lambda_t)\right\|^2
\end{array}
\end{equation}

Furthermore, for $\left\|\partial_\theta\mathcal{L}_t(\theta_t,\lambda_t)\right\|^2$, we have
\begin{equation}
\label{eq::grad_lag_inequality}
\begin{array}{ll}
\left\|\partial_\theta\mathcal{L}_t(\theta_t,\lambda_t)\right\|^2 = \left\|\partial_\theta f_t(\theta_t) 
+ \sum\limits_{i=1}^m\lambda_t^i\partial_\theta([g_i(\theta_t)]_+)\right\|^2 
                                                        \le (m+1)G^2(1+\left\|\lambda_t\right\|^2)

\end{array}
\end{equation}
where the last inequality is from the inequality that $(y_1+y_2+...+y_n)^2\le n(y_1^2+y_2^2+...+y_n^2)$, 
and both $\left\|\partial_\theta f_t(\theta_t)\right\|$ and $\left\|\partial_\theta([g_i(\theta_t)]_+)\right\|$ 
are less than or equal to $G$ by the definition.

Then we have 
\begin{equation*}
\begin{array}{l}
\mathcal{L}_t(\theta_t,\lambda_t)-\mathcal{L}_t(\theta,\lambda_t)
 \le \frac{1}{2\eta}\Big(\left\|\theta-\theta_t\right\|^2-\left\|\theta-\theta_{t+1}\right\|^2\Big)
                                                        +\frac{\eta}{2}(m+1)G^2(1+\left\|\lambda_t\right\|^2)

\end{array}
\end{equation*}

Since $\theta_1$ is in the center of $\mathcal{B}$, we can assume $\theta_1 = 0$ without loss of generality. 
If we sum the $\mathcal{L}_t(\theta_t,\lambda_t)-\mathcal{L}_t(\theta,\lambda_t)$ from 1 to $T$, we have
\scriptsize
\begin{equation*}
\begin{array}{ll}
\sum\limits_{t=1}^T[\mathcal{L}_t(\theta_t,\lambda_t)-\mathcal{L}_t(\theta,\lambda_t)] 
                                                &\le \frac{1}{2\eta}\Big(\left\|\theta-\theta_1\right\|^2-\left\|\theta-\theta_{T+1}\right\|^2\Big) 
                                                 +\frac{\eta T}{2}(m+1)G^2 
                                                 +\frac{\eta}{2}(m+1)G^2\sum\limits_{t=1}^T\left\|\lambda_t\right\|^2 \\
                                                & \le \frac{R^2}{2\eta}+\frac{\eta T}{2}(m+1)G^2
                                                 +\frac{\eta}{2}(m+1)G^2\sum\limits_{t=1}^T\left\|\lambda_t\right\|^2

\end{array}
\end{equation*}
\normalsize
where the last inequality follows from the fact that $\theta_1 = 0$ and $\left\|\theta\right\|^2\le R^2$.
\end{proof}

Now we are ready to prove the main theorem.

\begin{proof}[Proof of Theorem~\ref{thm::sumOfSquareLongterm}]
From Lemma \ref{lem:sumOfLagfunction}, we have 
\begin{equation*}
\begin{array}{l}
  \sum\limits_{t=1}^T[\mathcal{L}_t(\theta_t,\lambda_t)-\mathcal{L}_t(\theta,\lambda_t)]\le 
  \frac{R^2}{2\eta}+\frac{\eta T}{2}(m+1)G^2
  +\frac{\eta}{2}(m+1)G^2\sum\limits_{t=1}^T\left\|\lambda_t\right\|^2
\end{array}
\end{equation*}

If we expand the terms in the LHS and move the last term in RHS to the left,
we have
\scriptsize
\begin{equation*}
\begin{array}{l}
\sum\limits_{t=1}^T\Big(f_t(\theta_t) - f_t(\theta)\Big) +
\sum\limits_{t=1}^T\sum\limits_{i=1}^m\Big(\lambda_t^i[g_i(\theta_t)]_+ -
\lambda_t^i[g_i(\theta)]_+\Big) 
-\frac{\eta}{2}(m+1)G^2\sum\limits_{t=1}^T\left\|\lambda_t\right\|^2
\le \frac{R^2}{2\eta}+\frac{\eta T}{2}(m+1)G^2
\end{array}
\end{equation*}
\normalsize

We can set $\theta = \theta^*$ to have $[g_i(\theta^*)]_+ = 0$ 
and plug in the expression $\lambda_t = \frac{[g(\theta_t)]_+}{\sigma\eta}$
to have
\begin{equation}
\label{eq::UpperBoundOfsumOfObjAndLongTermConstrain}
\begin{array}{l}
\sum\limits_{t=1}^T\Big(f_t(\theta_t) - f_t(\theta^*)\Big) +
\sum\limits_{i=1}^m\sum\limits_{t=1}^T\frac{([g_i(\theta_t)]_+)^2}{\sigma\eta}\Big(1-\frac{(m+1)G^2}{2\sigma}\Big)
\le \frac{R^2}{2\eta}+\frac{\eta T}{2}(m+1)G^2
\end{array}
\end{equation}

Plugging in the expression for $\sigma$ and $\eta$, we have 
\begin{equation*}
\begin{array}{l}
\sum\limits_{t=1}^T\Big(f_t(\theta_t) - f_t(\theta^*)\Big) +
\sum\limits_{i=1}^m\sum\limits_{t=1}^T\frac{([g_i(\theta_t)]_+)^2}{\sigma\eta}\alpha 
\le O(\sqrt{T})
\end{array}
\end{equation*}

Because $\frac{([g_i(\theta_t)]_+)^2}{\sigma\eta}\alpha\ge 0$, we have
\begin{equation*}
\sum\limits_{t=1}^T\Big(f_t(\theta_t) - f_t(\theta^*)\Big) \le O(\sqrt{T})
\end{equation*} 

Furthermore, we have $\sum\limits_{t=1}^T\Big(f_t(\theta_t) - f_t(\theta^*)\Big)\ge -FT$
according to the assumption.
Then we have
\begin{equation*}
\begin{array}{l}
\sum\limits_{i=1}^m\sum\limits_{t=1}^T\Big([g_i(\theta_t)]_+\Big)^2  \le \frac{\sigma\eta}{\alpha}(O(\sqrt{T})+FT)
    = \frac{\sigma}{\alpha}(O(\sqrt{T})+FT)O(\frac{1}{\sqrt{T}}) = O(\sqrt{T})
\end{array}
\end{equation*}

Because $\Big([g_i(\theta_t)]_+\Big)^2 \ge 0$, we have 
\begin{equation*}
\sum\limits_{t=1}^T\Big([g_i(\theta_t)]_+\Big)^2 \le O(\sqrt{T}), \forall i \in \{1,2,...,m\}
\end{equation*}

\end{proof}

From Theorem \ref{thm::sumOfSquareLongterm}, we can see that by
setting appropriate step size, $\eta$, and constant, $\sigma$,
we can obtain the upper bound for the regret of the loss function being less than or equal to $O(\sqrt{T})$,
which is also shown in \cite{mahdavi2012trading} \cite{jenatton2016adaptive}.
The main difference of the Theorem \ref{thm::sumOfSquareLongterm} is that 
previous results of \cite{mahdavi2012trading} \cite{jenatton2016adaptive}
all obtain the upper bound for the 
long-term constraint $\sum\limits_{t=1}^T g_i(\theta_t)$, 
while here the upper bound for 
the constraint violation of the form $\sum\limits_{t=1}^T\Big([g_i(\theta_t)]_+\Big)^2$ is achieved. 
Also note that the step size depends on $T$, which may not be available. 
In this case, we can use the 'doubling trick' described in the book \cite{cesa2006prediction}
to transfer our $T$-dependent algorithm into $T$-free one with a worsening factor of $\sqrt{2}/(\sqrt{2}-1)$. 

The proposed algorithm and the resulting bound are useful for two reasons:
1. The square-cumulative constraint implies a bound on the
cumulative constraint violation,
$\sum\limits_{t=1}^T[g_i(\theta_t)]_+$, while enforcing larger penalties
for large violations. 
%
2. The proposed algorithm can also upper bound the constraint violation for each single step $[g_i(\theta_t)]_+$, 
which is not bounded in the previous literature.

The next results show how to bound constraint violations at each
step. Please refer to the Appendix for the proof.

\begin{lemma}
  \label{lem:bound_step}
  {\it
     If there is only one differentiable constraint function $g(\theta)$ with Lipschitz continuous gradient parameter $L$, 
     and we run the Algorithm \ref{alg::convex-long-term}
     with the parameters in Theorem \ref{thm::sumOfSquareLongterm} and large enough $T$, we have
     \begin{equation*}
     \begin{array}{lll}
     [g(\theta_t)]_+ \le O(\frac{1}{T^{1/6}}),& \forall t \in \{1,2,...,T\},
     &if \quad [g(\theta_1)]_+ \le O(\frac{1}{T^{1/6}}).
     \end{array}
     \end{equation*}

  }
\end{lemma}

Lemma \ref{lem:bound_step} only considers single constraint case. 
For case of multiple differentiable constraints, 
we have the following:
\begin{proposition}
\label{prop::bound_step_max}
  {\it 
     For multiple differentiable constraint functions $g_i(\theta)$, $i\in\{1,2,...,m\}$ with Lipschitz continuous gradient parameters $L_i$,
     if we use $\bar{g}(\theta) = \log\Big(\sum\limits_{i=1}^m \exp{g_i(\theta)}\Big)$
     as the constraint function in Algorithm \ref{alg::convex-long-term}, then for
     large enough $T$, we have 
     \begin{equation*}
     \begin{array}{lll}
     [g_i(\theta_t)]_+ \le O(\frac{1}{T^{1/6}}),& \forall i, t,
     &if \quad [\bar{g}(\theta_1)]_+ \le O(\frac{1}{T^{1/6}}).
     \end{array}
     \end{equation*}

  }
\end{proposition}

Clearly, both Lemma \ref{lem:bound_step} and Proposition \ref{prop::bound_step_max} 
only deal with differentiable functions.
For a non-differentiable function $g(\theta)$,
we can first use a differentiable function $\bar{g}(\theta)$ to approximate the $g(\theta)$
with $\bar{g}(\theta)\ge g(\theta)$, and then apply the previous Lemma \ref{lem:bound_step}
and Proposition \ref{prop::bound_step_max} to upper bound
each individual $g_i(\theta_t)$. 
Many non-smooth convex functions can be approximated in this way as
shown in \cite{nesterov2005smooth}.

\subsection{Strongly Convex Case}

For $f_t(\theta)$ to be strongly convex, the Algorithm \ref{alg::convex-long-term} is still valid. 
But in order to reduce the upper bounds for both objective regret and the clipped long-term constraint
 $\sum\limits_{t=1}^T [g_i(\theta_t)]_+$ compared with Proposition \ref{prop::tradeOffLossAndConstraint} in next section, 
we need to use time-varying step size as 
the one used in \cite{hazan2007logarithmic}. Thus, we modify the update rule 
of $\theta_t$, $\lambda_t$ to have time-varying step size 
as below:
\begin{equation}
\label{eq::update_strongly_convex}
\begin{array}{ll}
\theta_{t+1} = \Pi_{\mathcal{B}}(\theta_t-\eta_t \partial_\theta \mathcal{L}_t(\theta_t,\lambda_t)),&
\lambda_{t+1} = \frac{[g(\theta_{t+1})]_+}{\phi_{t+1}}.
\end{array}
\end{equation}

If we replace the update rule in Algorithm \ref{alg::convex-long-term} with Eq.(\ref{eq::update_strongly_convex}),
we can obtain the following theorem: 

\begin{theorem}
\label{thm::stronglyconvex}
{\it
Assume $f_t(\theta)$ has strong convexity parameter $H_1$.
If we set $\eta_t = \frac{1}{H_1t}$, $\phi_t=\eta_t(m+1)G^2$,
follow the new update rule in Eq.(\ref{eq::update_strongly_convex}),
and $\theta^*$ being the optimal solution for $\min\limits_{\theta\in \cS}\sum\limits_{t=1}^Tf_t(\theta)$, 
for $\forall i \in \{1,2,...,m\}$,
we have
\begin{equation*}
\begin{array}{ll}
\sum\limits_{t=1}^T\Big( f_t(\theta_t) - f_t(\theta^*)\Big)\le O(\log(T)),
&\sum\limits_{t=1}^Tg_i(\theta_t)\le\sum\limits_{t=1}^T[g_i(\theta_t)]_+ \le O(\sqrt{\log(T) T}).
\end{array}
\end{equation*}
}
\end{theorem}

The paper \cite{jenatton2016adaptive} also has a discussion of
strongly convex functions, 
but only provides a bound similar to the convex one.
Theorem \ref{thm::stronglyconvex} shows the improved bounds for both objective regret and
the constraint violation.
On one hand the objective regret is consistent with the standard OCO result in \cite{hazan2007logarithmic}, 
and on the other the constraint violation is further reduced compared with the result in \cite{jenatton2016adaptive}.

\paragraph{Proof of Theorem \ref{thm::stronglyconvex}:}
\begin{proof}
For the strongly convex case of $f_t(\theta)$ with strong convexity parameter equal to $H_1$,
we can also conclude that the modified augmented Lagrangian function in Eq.(\ref{eq::update_strongly_convex})
is also strongly convex w.r.t. $\theta$ with the strong convexity parameter $H\ge H_1$.
Then we have 
\begin{equation}
\label{inequal::strong_lag}
\begin{array}{l}
\mathcal{L}_t(\theta^*,\lambda_t)-\mathcal{L}_t(\theta_t,\lambda_t) \ge 
\partial_\theta\mathcal{L}_t(\theta_t)^\top(\theta^*-\theta_t)
 + \frac{H_1}{2}\left\|\theta^*-\theta_t\right\|^2
\end{array}
\end{equation}

From concavity of $\mathcal{L}$ in terms of $\lambda$, we can have
\begin{equation}
\label{inequal::strong_lambda}
\mathcal{L}_t(\theta_t,\lambda) - \mathcal{L}_t(\theta_t,\lambda_t)\le (\lambda-\lambda_t)^\top\nabla_{\lambda}\mathcal{L}_t(\theta_t,\lambda_t)
\end{equation}
Since $\lambda_t$ maximizes the augmented Lagrangian, we can see that the RHS is $0$.

From Eq.(\ref{eq::x_projection_inequality}), we have
\begin{equation}
\label{inequal::strong_proj}
\begin{array}{l}
\partial_\theta\mathcal{L}_t(\theta_t)^\top(\theta_t-x^*) 
\le \frac{1}{2\eta_t}\Big(\left\|\theta^*-\theta_t\right\|^2-\left\|\theta^*-\theta_{t+1}\right\|^2\Big)
 +\frac{\eta_t}{2}(m+1)G^2(1+\left\|\lambda_t\right\|^2)

\end{array}
\end{equation}

Multiply Eq.(\ref{inequal::strong_lag}) by $-1$ and add Eq.(\ref{inequal::strong_lambda}) together with Eq.(\ref{inequal::strong_proj}) plugging in:
\scriptsize
\begin{equation*}
\begin{array}{l}
\mathcal{L}_t(\theta_t,\lambda) - \mathcal{L}_t(\theta^*,\lambda_t) 
\le \frac{1}{2\eta_t}\Big(\left\|\theta^*-\theta_t\right\|^2-\left\|\theta^*-\theta_{t+1}\right\|^2\Big)
 +\frac{\eta_t}{2}(m+1)G^2(1+\left\|\lambda_t\right\|^2)-\frac{H_1}{2}\left\|\theta^*-\theta_t\right\|^2
\end{array}
\end{equation*}
\normalsize

Let $b_t=\left\|\theta^*-\theta_t\right\|^2$, and plug in the expression for $\mathcal{L}_t$, we can get:
\begin{equation*}
\begin{array}{l}
f_t(\theta_t) - f_t(\theta^*) +\lambda^\top[g(\theta_t)]_+-\frac{\phi_t}{2}\left\|\lambda\right\|^2 \le \frac{1}{2\eta_t}(b_t-b_{t+1})\\
\quad\quad\quad -\frac{H_1}{2}b_t +\frac{(m+1)G^2}{2}\eta_t+\frac{(m+1)G^2}{2}\left\|\lambda_t\right\|^2(\eta_t-\frac{\phi_t}{(m+1)G^2})
\end{array}
\end{equation*}

Plug in the expressions $\eta_t = \frac{1}{H_1t}$, $\phi_t = (m+1)G^2\eta_t$, and sum over $t=1$ to $T$:
\begin{equation*}
\begin{array}{l}
\sum\limits_{t=1}^T\Big(f_t(\theta_t)-f_t(\theta^*)\Big) + \lambda^\top\Big(\sum\limits_{t=1}^T[g(\theta_t)]_+\Big)
-\frac{\left\|\lambda\right\|^2}{2}\sum\limits_{t=1}^T\phi_t \\
\le \underbrace{\frac{1}{2}\sum\limits_{t=1}^T\Big(\frac{b_t-b_{t+1}}{\eta_t}-\frac{H_1}{2}b_t\Big)}_A 
+ \underbrace{\frac{(m+1)G^2}{2}\sum\limits_{t=1}^T\eta_t}_B
\end{array}
\end{equation*}

For the expression of $A$, we have:
\begin{equation*}
\begin{array}{l}
A = \frac{1}{2}\Big[\frac{b_1}{\eta_1}+\sum\limits_{t=2}^Tb_t(\frac{1}{\eta_t}-\frac{1}{\eta_{t-1}}-H_1)-\frac{b_{T+1}}{\eta_T}-H_1b_1\Big]
 \le 0
\end{array}
\end{equation*}

For the expression of $B$, with the expression of $\eta_t$ and the inequality relation between sum and integral, we have:
\begin{equation*}
B\le O(\log(T))
\end{equation*}

Thus, we have:
\begin{equation*}
\begin{array}{l}
\sum\limits_{t=1}^T\Big(f_t(\theta_t)-f_t(\theta^*)\Big) + \lambda^\top\Big(\sum\limits_{t=1}^T[g(\theta_t)]_+\Big)
-\frac{\left\|\lambda\right\|^2}{2}\sum\limits_{t=1}^T\phi_t 
\le O(\log(T))
\end{array}
\end{equation*}

If we set $\lambda = \frac{\sum\limits_{t=1}^T[g(x_t)]_+}{\sum\limits_{t=1}^T\phi_t}$, 
and due to non-negativity of $\frac{\Big\|\sum\limits_{t=1}^T[g(\theta_t)]_+\Big\|^2}{2\sum\limits_{t=1}^T\phi_t}$,
we have
\begin{equation*}
\sum\limits_{t=1}^T\Big(f_t(\theta_t)-f_t(\theta^*)\Big) \le O(\log(T))
\end{equation*}

Furthermore, we have $\sum\limits_{t=1}^T\Big(f_t(\theta_t) - f_t(\theta^*)\Big)\ge -FT$
according to the assumption.
Then we have
\begin{equation*}
\frac{\Big\|\sum\limits_{t=1}^T[g(\theta_t)]_+\Big\|^2}{2\sum\limits_{t=1}^T\phi_t} \le O(\log(T)) + FT
\end{equation*}

Because $\sum\limits_{t=1}^T\phi_t\le O(\log(T))$, we have:
\begin{equation*}
\sum\limits_{t=1}^T[g(x_t)]_+\le O(\sqrt{\log(T)T})
\end{equation*}

\end{proof}

\section{Relation with Previous Results}
\label{sec:long-term-relation-pre}

In this section, we extend Theorem
\ref{thm::sumOfSquareLongterm} to enable direct
comparison with the results from  \cite{mahdavi2012trading}
\cite{jenatton2016adaptive}. In particular, it is shown how
Algorithm~\ref{alg::convex-long-term} recovers the existing regret bounds, while
the use of the new augmented Lagrangian \eqref{eq::new long term
  lagrangian} in the previous algorithms also provides regret bounds
for the clipped constraint case. 

The first result puts a bound on the clipped long-term constraint, rather
than the sum-of-squares that appears in
Theorem~\ref{thm::sumOfSquareLongterm}. This will allow more direct
comparisons with the existing results. 

\begin{proposition}
\label{prop::similarResultTo2012}
  {\it 
      If $\sigma = \frac{(m+1)G^2}{2(1-\alpha)}$, 
      $\eta = O(\frac{1}{\sqrt{T}})$,
      $\alpha\in (0,1)$, and
      $\theta^* = \underset{\theta\in \cS}\argmin\sum\limits_{t=1}^Tf_t(\theta)$, 
      then the result of Algorithm \ref{alg::convex-long-term} satisfies
      \begin{equation*}
      \begin{array}{ll}
      \sum\limits_{t=1}^T\Big( f_t(\theta_t) - f_t(\theta^*)\Big)\le O(\sqrt{T}),
      &\sum\limits_{t=1}^Tg_i(\theta_t)\le \sum\limits_{t=1}^T[g_i(\theta_t)]_+ \le O(T^{3/4}), \forall i \in \{1,2,...,m\}
      \end{array}
      \end{equation*}
  }
\end{proposition}

This result shows that our algorithm generalizes the
regret and long-term constraint bounds of \cite{mahdavi2012trading}.
Please refer to the Appendix for this section's proofs.

The next result shows that by changing our constant stepsize accordingly,
with the Algorithm \ref{alg::convex-long-term}, we can 
achieve the user-defined trade-off from \cite{jenatton2016adaptive}.
Furthermore, we also include the squared version and clipped
constraint violations.

\begin{proposition}
\label{prop::tradeOffLossAndConstraint}
  {\it 
      If $\sigma = \frac{(m+1)G^2}{2(1-\alpha)}$, $\eta =
      O(\frac{1}{T^{\beta}})$, $\alpha\in(0,1)$, $\beta \in(0,1)$,
      and $\theta^* = \underset{\theta\in \cS}\argmin\sum\limits_{t=1}^Tf_t(\theta)$, 
      then the result of  Algorithm~\ref{alg::convex-long-term} satisfies
      \begin{equation*}
      \begin{array}{lll}
      \sum\limits_{t=1}^T\Big( f_t(\theta_t) - f_t(\theta^*)\Big)\le O(T^{max\{\beta,1-\beta\}}),\\
      \sum\limits_{t=1}^Tg_i(\theta_t)\le \sum\limits_{t=1}^T[g_i(\theta_t)]_+ \le O(T^{1-\beta/2}),
      &\sum\limits_{t=1}^T([g_i(\theta_t)]_+)^2\le O(T^{1-\beta}), \forall i \in \{1,2,...,m\}
      \end{array}
      \end{equation*}
  }
\end{proposition}

Proposition \ref{prop::tradeOffLossAndConstraint} 
 provides a systematic way 
to balance the regret of the objective and the constraint violation.
Next, we will show that previous algorithms can use our proposed augmented Lagrangian function to 
have their own clipped long-term constraint bound.
\begin{proposition}
\label{prop::true_violation_bound_2011}
{\it
       If we run Algorithm 1 in \cite{mahdavi2012trading} with the augmented Lagrangian formula 
       defined in Eq.(\ref{eq::new long term lagrangian}), the result satisfies
      \begin{equation*}
      \begin{array}{ll}
      \sum\limits_{t=1}^T\Big( f_t(\theta_t) - f_t(\theta^*)\Big)\le O(\sqrt{T}),
      &\sum\limits_{t=1}^Tg_i(\theta_t)\le \sum\limits_{t=1}^T[g_i(\theta_t)]_+ \le O(T^{3/4}), \forall i \in \{1,2,...,m\}.
      \end{array}
      \end{equation*}
  }
\end{proposition}

For the update rule proposed in \cite{jenatton2016adaptive},
we need to change the 
$\mathcal{L}_t(\theta,\lambda)$ to the following one:
\begin{equation}
\label{eq::new_lag_2016}
  \mathcal{L}_t(\theta,\lambda) = f_t(\theta)+\lambda [g(\theta)]_+ - \frac{\phi_t}{2}\lambda^2
\end{equation}
where $g(\theta) = \max\limits_{i\in \{1,\dots,m\}} g_i(\theta)$.

\begin{proposition}
\label{prop::true_violation_bound_2016}
  {\it 
      If we use the update rule and the parameter choices in \cite{jenatton2016adaptive} with
      the augmented Lagrangian in Eq.(\ref{eq::new_lag_2016}),
      then $\forall i \in \{1,...,m\}$, we have
      \begin{equation*}
      \begin{array}{ll}
      \sum\limits_{t=1}^T\Big( f_t(\theta_t) - f_t(\theta^*)\Big)\le O(T^{max\{\beta,1-\beta\}}),
      &\sum\limits_{t=1}^Tg_i(\theta_t)\le \sum\limits_{t=1}^T[g_i(\theta_t)]_+ \le O(T^{1-\beta/2}).
      \end{array}
      \end{equation*}
  }
\end{proposition}

Propositions \ref{prop::true_violation_bound_2011} and
\ref{prop::true_violation_bound_2016} show that clipped long-term
constraints can be bounded by combining the algorithms of
\cite{mahdavi2012trading,jenatton2016adaptive} with our augmented
Lagrangian. 
Although these results are similar in part to our Propositions \ref{prop::similarResultTo2012} and \ref{prop::tradeOffLossAndConstraint},
they do not imply the results in Theorems \ref{thm::sumOfSquareLongterm} and \ref{thm::stronglyconvex} 
as well as the new single step constraint violation bound in Lemma \ref{lem:bound_step}, which are our key contributions.
Based on Propositions \ref{prop::true_violation_bound_2011} and \ref{prop::true_violation_bound_2016},
it is natural to ask whether we could apply our new augmented Lagrangian formula (\ref{eq::new long term lagrangian})
to the recent work in \cite{yu2017online} .
%
Unfortunately, we have not found a way to do so.

Furthermore, since $\Big([g_i(\theta_t)]_+\Big)^2$ is also convex, 
we could define $\tilde{g}_i(\theta_t) = \Big([g_i(\theta_t)]_+\Big)^2$ 
and apply the previous algorithms \cite{mahdavi2012trading} \cite{jenatton2016adaptive} and \cite{yu2017online}.
This will result in the upper bounds of $O(T^{3/4})$ \cite{mahdavi2012trading} and $O(T^{1-\beta/2})$ \cite{jenatton2016adaptive},
which are worse than our upper bounds of $O(T^{1/2})$ (Theorem \ref{thm::sumOfSquareLongterm})
and $O(T^{1-\beta})$ ( Proposition \ref{prop::tradeOffLossAndConstraint}).
Note that the algorithm in \cite{yu2017online} cannot be applied 
since the clipped constraints do not satisfy the required Slater condition.

\section{Experiment}
\label{sec:oco-long-term-exp}

\begin{figure}
  \centering
  \subfigure[]{
    \includegraphics[height=3.6cm]{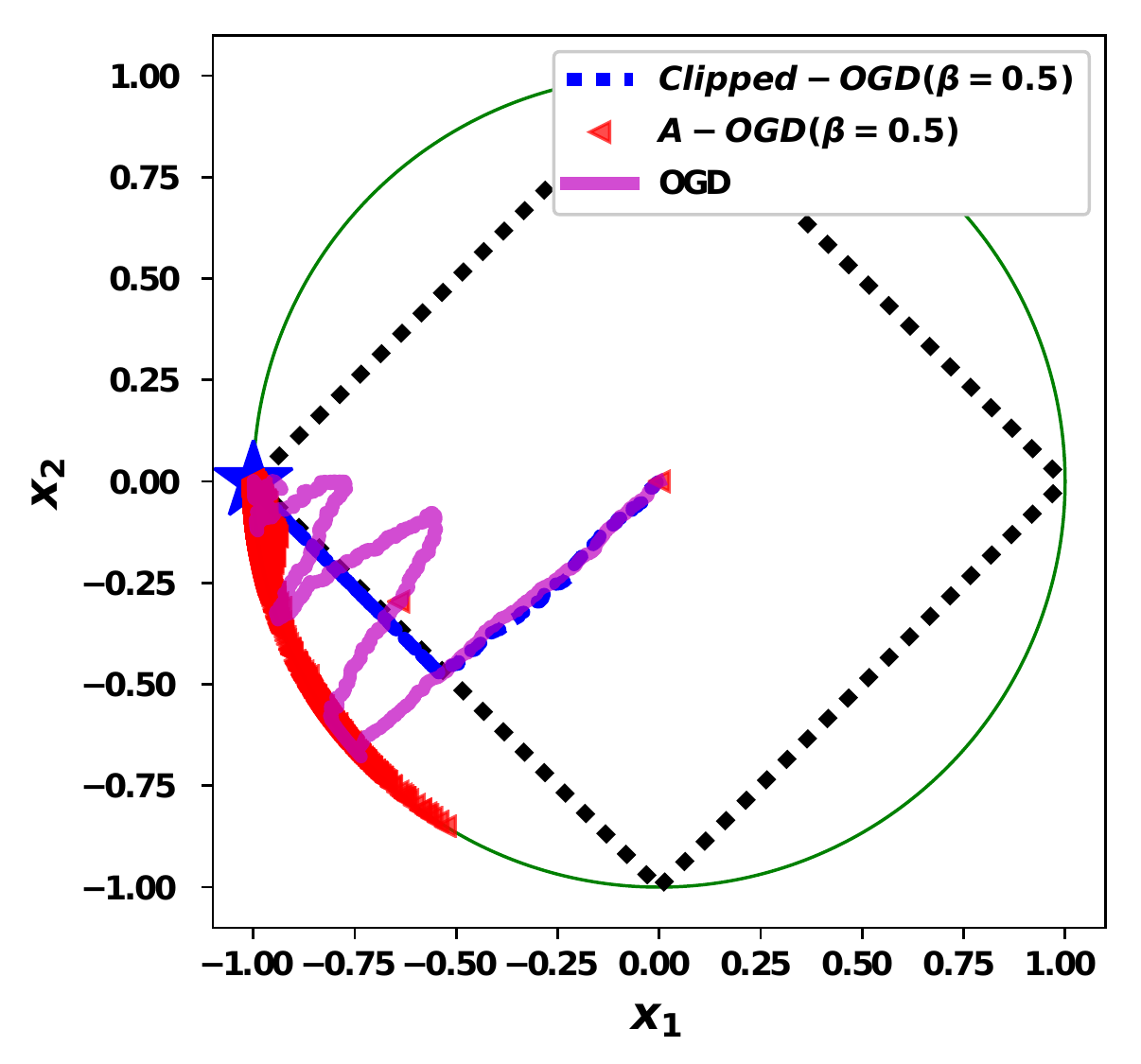}}
  \hspace{.1in}
  \subfigure[]{
    \includegraphics[height=3.6cm]{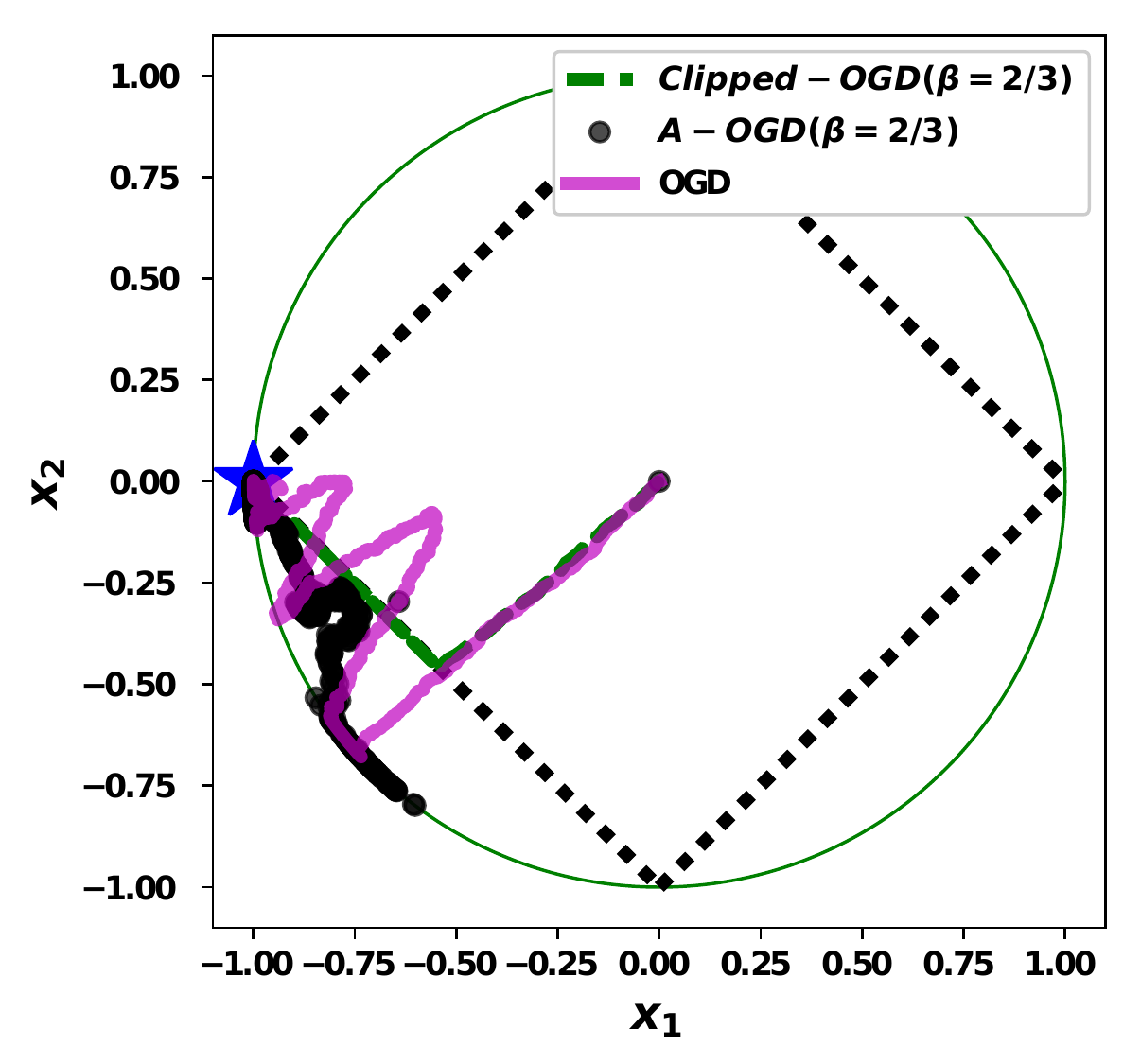}}
  \caption{Toy Example Results: Trajectories generated by different algorithms. 
    Note how trajectories generated by Clipped-OGD follow the
    desired constraints tightly. In contrast, OGD oscillates around
    the true constraints, and A-OGD closely follows the outer ball's boundary.}
  \label{fig::toy_traj} 
\end{figure}

\begin{figure}
\vskip 0.0in
  \centering
  \subfigure[]{
    \includegraphics[height=3.cm]{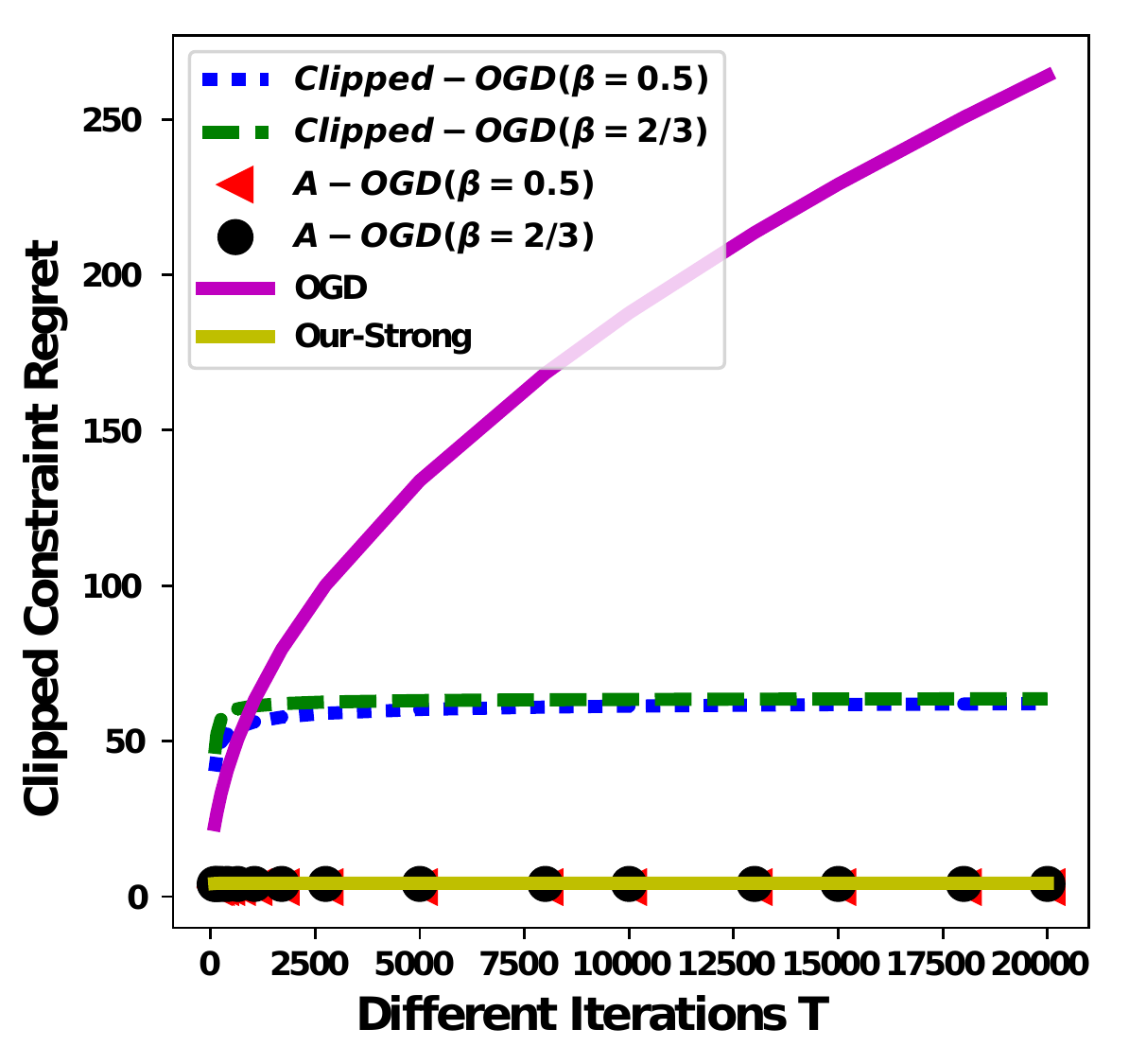}}
  \hspace{.3in}
  \subfigure[]{
    \includegraphics[height=3.cm]{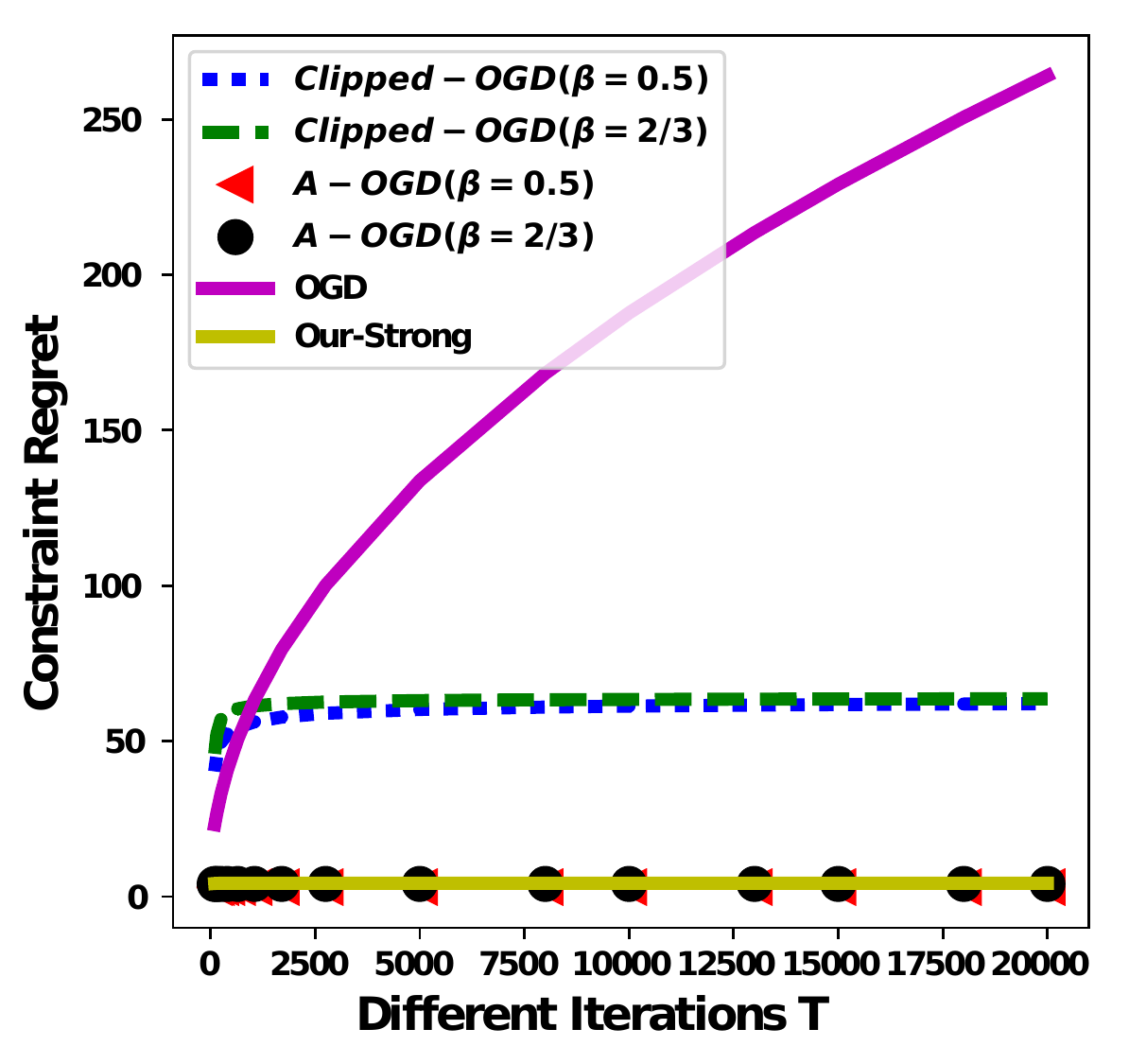}}
  \hspace{.3in}
  \subfigure[]{
    \includegraphics[height=3.cm]{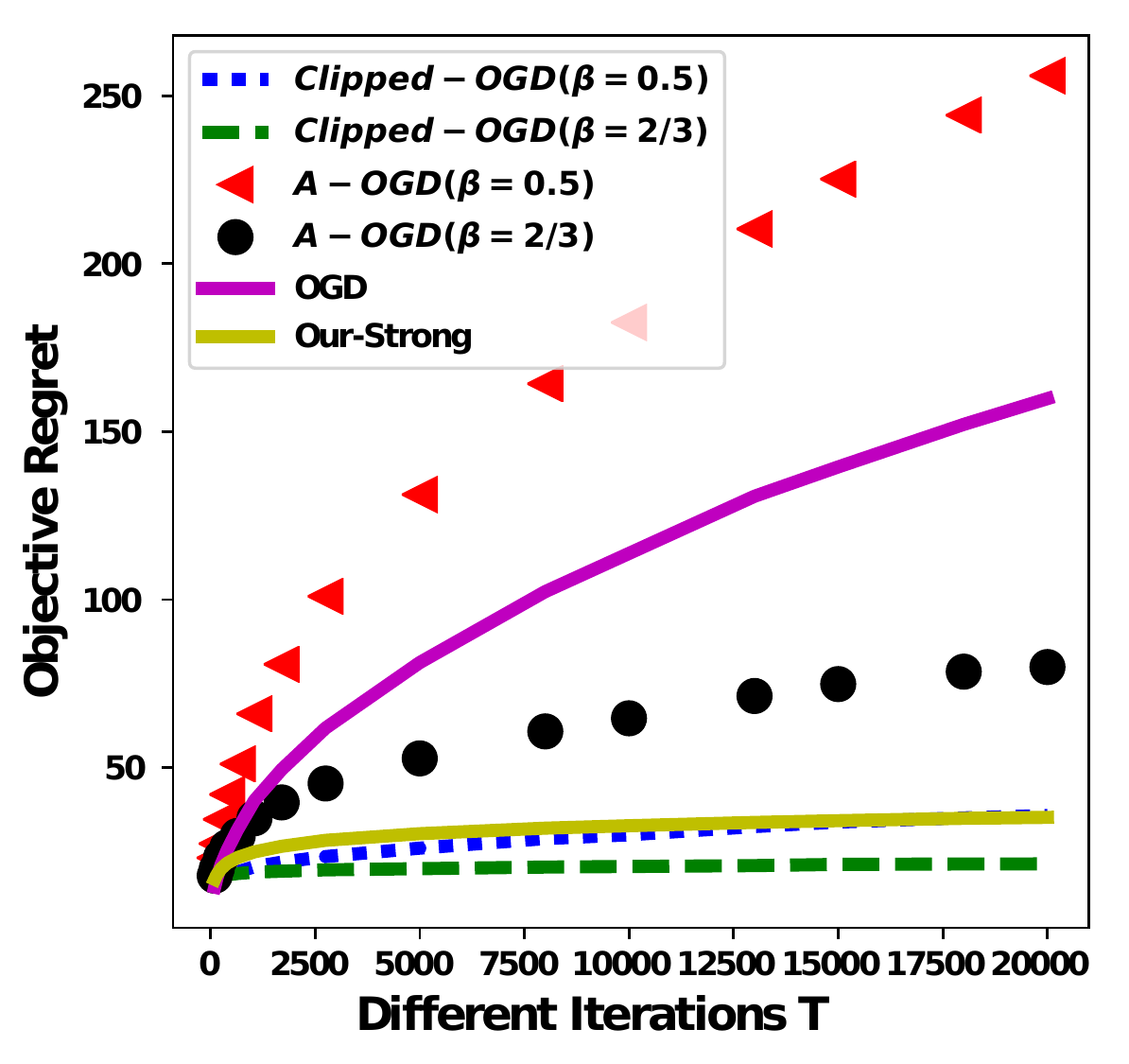}}
  \caption{Doubly-Stochastic Matrices. (a): Clipped Long-term Constraint Violation. 
           (b): Long-term Constraint Violation.
           (c): Cumulative Regret of the Loss function}
  \label{fig::doubly_obj_con} 
\end{figure}

\begin{figure}
  \centering
  \subfigure[]{
    \includegraphics[height=2.4cm]{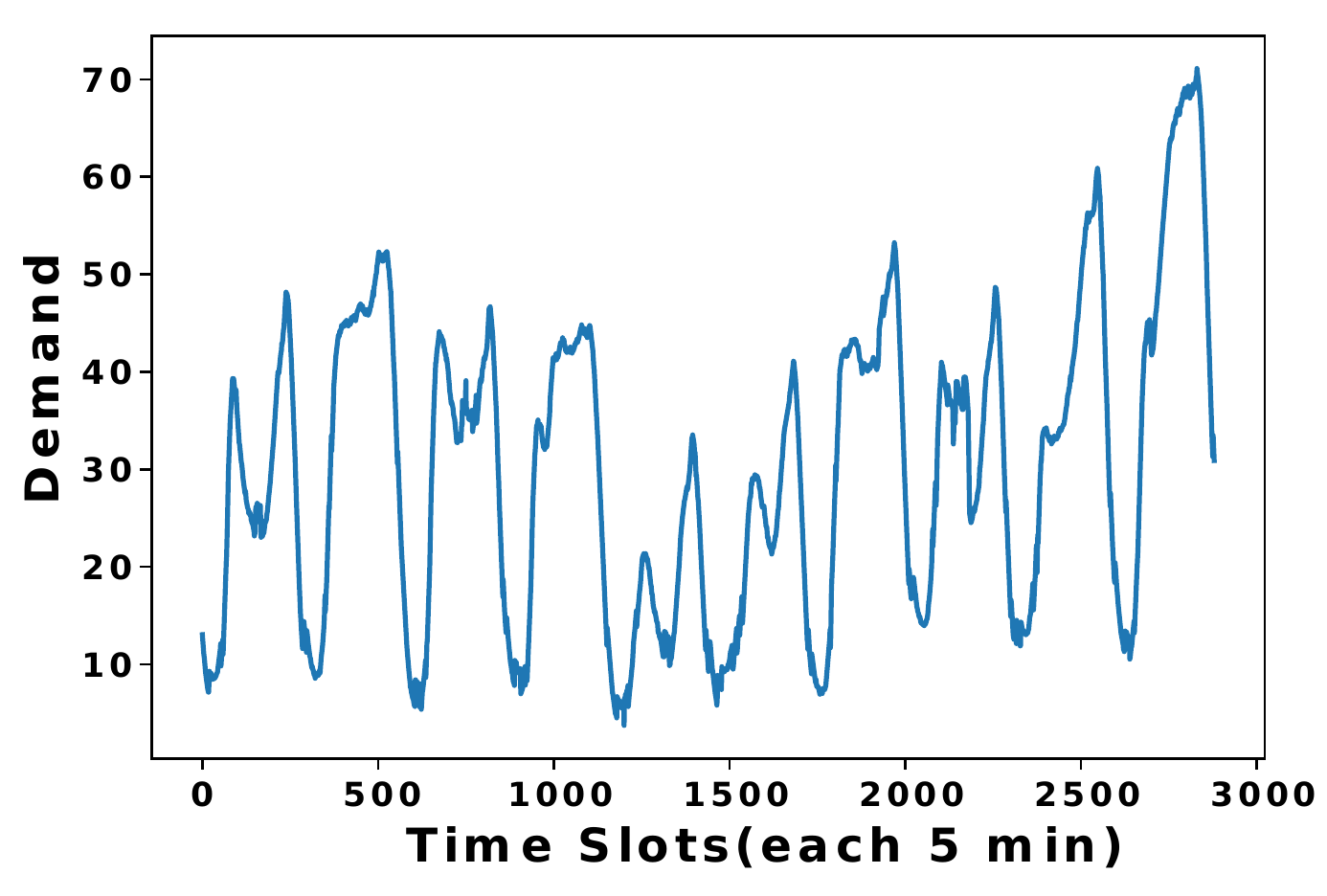}}
  \hspace{.1in}
  \subfigure[]{
    \includegraphics[height=3.1cm]{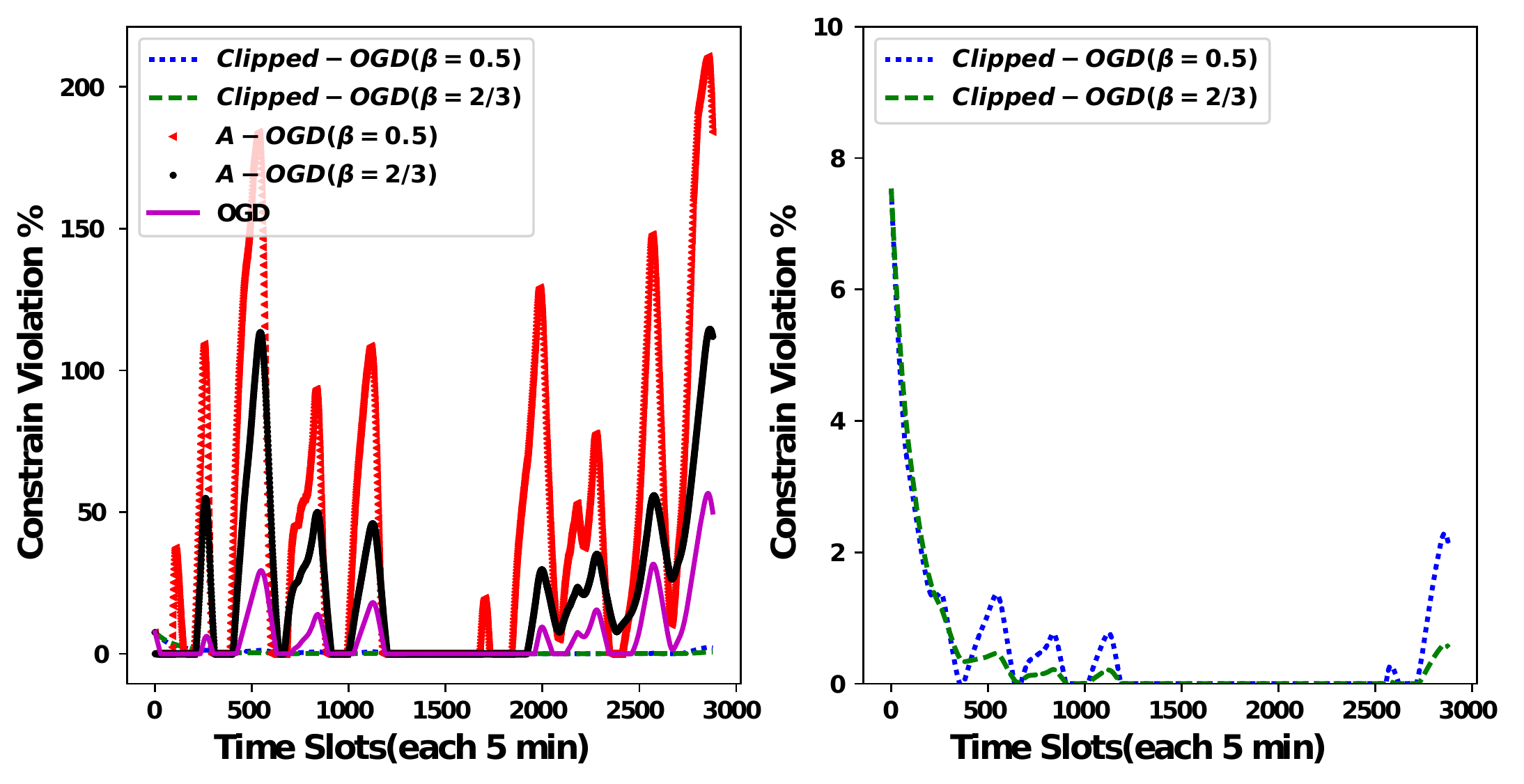}}
  \hspace{.1in}
  \subfigure[]{
    \includegraphics[height=3.1cm]{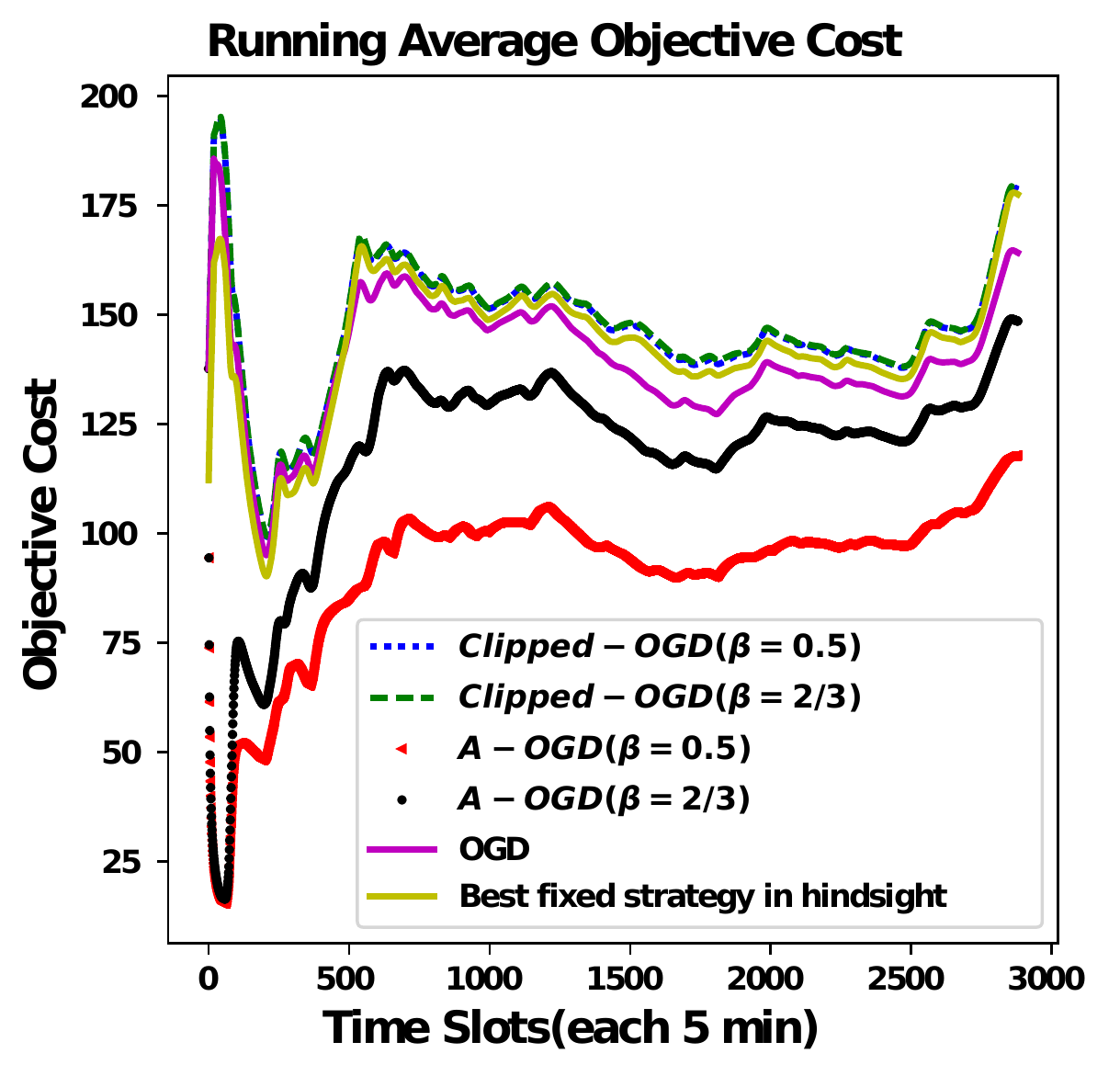}}
  \caption{Economic Dispatch.
           (a): Power Demand Trajectory. 
           (b): Constraint Violation for each time
           step. All of the previous algorithms incurred substantial
           constraint violations. The figure on the right shows the
           violations of our algorithm, which are significantly smaller. 
           (c): Running Average of the Objective Loss}
  \label{fig::ed_obj_con} 
\vskip 0.in
\end{figure}

In this section, we test the performance of the algorithms including
OGD \cite{mahdavi2012trading},
A-OGD \cite{jenatton2016adaptive}, Clipped-OGD (this chapter),
and our proposed algorithm strongly convex case (Our-strong).
Throughout the experiments, our algorithm has the following fixed parameters:
$\alpha = 0.5$, $\sigma =\frac{(m+1)G^2}{2(1-\alpha)}$, $\eta = \frac{1}{T^\beta G\sqrt{R(m+1)}}$.
In order to better show the result of the constraint violation trajectories,
we aggregate all the constraints as a single one
by using $g(\theta_t) = \max_{i\in \{1,...,m\}}g_i(\theta_t)$ as done in \cite{mahdavi2012trading}.

\subsection{A Toy Experiment}

For illustration purposes, we solve the following 2-D toy experiment
with $\theta = [\theta_1,\theta_2]^T$:
\begin{equation*}
\begin{array}{ll}
\min \sum\limits_{t=1}^T c_t^\top\theta,
&s.t.\left|\theta_1\right|+\left|\theta_2\right|-1\le 0.
\end{array}
\end{equation*}
where the constraint is the $\ell_1$-norm constraint.
The vector $c_t$ is generated from a uniform
random vector over $[0,1.2]\times [0,1]$ which is rescaled to have
norm $1$. This leads to slightly average cost on the on the first coordinate.
The offline solutions for different $T$ are obtained by CVXPY \cite{cvxpy}.

All algorithms are run up to $T = 20000$ and are averaged over 10 random sequences of $\{c_t\}_{t=1}^T$.
Since the main goal here is to compare the variables' trajectories generated by different algorithms,
the results for different $T$ are in the Appendix for space purposes.
Fig.~\ref{fig::toy_traj} shows these trajectories for one realization with $T = 8000$.
The blue star is the optimal point's position. 

From Fig.~\ref{fig::toy_traj} we can see that the trajectories generated by Clipped-OGD
follows the boundary very tightly until reaching the optimal point.
This can be explained by the Lemma \ref{lem:bound_step} which shows that 
the constraint violation for single step is also upper bounded.
For the OGD, the trajectory oscillates widely around the boundary of
the true constraint. 
For the A-OGD, its trajectory in Fig.~\ref{fig::toy_traj} violates the
constraint most of the time,
and this violation actually contributes to the lower objective regret shown in the Appendix.

\subsection{Doubly-Stochastic Matrices}

We also test the algorithms for approximation by doubly-stochastic matrices,
as in \cite{jenatton2016adaptive}:
\begin{equation}
\begin{array}{llll}
\min \sum\limits_{t=1}^T \frac{1}{2}\left\|Y_t-X\right\|_F^2
&s.t.\quad X\textbf{1} = \textbf{1},
     & X^T\textbf{1} = \textbf{1},
     & X_{ij}\ge 0.
\end{array}
\end{equation}
where $X\in \mathbb{R}^{d \times d}$ is the matrix variable, $\textbf{1}$ is the vector whose elements are all 1,
and matrix $Y_t$ is the permutation matrix which is randomly generated. 

After changing the equality constraints into inequality ones
(e.g.,$X\textbf{1} = \textbf{1}$ into $X\textbf{1} \ge \textbf{1}$ and $X\textbf{1} \le \textbf{1}$),
we run the algorithms with different T up to $T = 20000$ for 10 different random sequences of $\{Y_t\}_{t=1}^T$.
Since the objective function is strongly convex
with parameter $H_1=1$, we also include our designed strongly convex algorithm as another comparison.
The offline optimal solutions are obtained by CVXPY \cite{cvxpy}.

The mean results for both constraint violation and objective regret
are shown in Fig.~\ref{fig::doubly_obj_con}.
From the result we can see that,
for our designed strongly convex algorithm Our-Strong, its result is around the best ones in
not only the clipped constraint violation, 
but the objective regret. 
For our most-balanced convex case algorithm Clipped-OGD with 
$\beta = 0.5$, although its clipped constraint violation is relatively bigger than A-OGD,
it also becomes quite flat quickly, which means the algorithm quickly converges to a feasible solution.

\subsection{Economic Dispatch in Power Systems}
This example is adapted from \cite{li2018online} and \cite{senthil2010economic},
which considers the problem of power dispatch. That is, at each time step $t$,
we try to minimize the power generation cost $c_i(\theta_{t,i})$ for each generator $i$ while maintaining the
power balance $\sum\limits_{i=1}^n \theta_{t,i} = d_t$, where $d_t$ is the power demand at time $t$.
Also, each power generator produces an emission level $E_i(\theta_{t,i})$.
To bound the emissions, we impose the constraint $\sum\limits_{i=1}^n E_i(\theta_{t,i})\le E_{max}$. 
In addition to requiring this constraint to be satisfied on
average, 
we also require bounded constraint violations at each time step. The
problem is formally stated as:
\begin{equation*}
\begin{array}{lll}
\min \sum\limits_{t=1}^T \Big(\sum\limits_{i=1}^n c_i(\theta_{t,i})+\xi (\sum\limits_{i=1}^n \theta_{t,i}-d_t)^2 \Big),
&s.t. \quad \sum\limits_{i=1}^n E_i(t,i) \le E_{max},
     & 0\le \theta_{t,i} \le \theta_{i,max}.
\end{array}
\end{equation*}
where the second constraint is from the fact that each generator has the power generation limit.

In this example, we use three generators. We define the cost and
emission functions according to \cite{senthil2010economic} and
\cite{li2018online} as $c_i(\theta_{t,i}) = 0.5a_i \theta_{t,i}^2+b_i \theta_{t,i}$,
and $E_i = d_i \theta_{t,i}^2+e_i\theta_{t,i}$, respectively.
The parameters are: $a_1 = 0.2, a_2 = 0.12, a_3 = 0.14$, $b_1 = 1.5, b_2 = 1, b_3 = 0.6$, $d_1 = 0.26, d_2 = 0.38, d_3 = 0.37$,
$E_{max} = 100$, $\xi = 0.5$, and $\theta_{1,max} = 20, \theta_{2,max} = 15, \theta_{3,max} = 18$.
The demand $d_t$ is adapted from real-world 5-minute interval demand data between 04/24/2018 and 05/03/2018
\footnote{https://www.iso-ne.com/isoexpress/web/reports/load-and-demand}, which is shown in Fig.~\ref{fig::ed_obj_con}(a).
The offline optimal solution or best fixed strategy in hindsight is
obtained by an implementation of SAGA \cite{defazio2014saga}. 
The constraint violation for each time step is shown in Fig.~\ref{fig::ed_obj_con}(b),
and the running average objective cost is shown in Fig.~\ref{fig::ed_obj_con}(c).
From these results we can see that our algorithm has very small
constraint violation for each time step, 
which is desired by the requirement. Furthermore, our objective costs
are very close to the best fixed strategy.

\section{Extension to Dynamic OCO with Long-term Constraint}
\label{sec:dynamic-oco-long-term}

In this section, we extend the Algorithm \ref{alg::convex-long-term}
to solve the general time-dependent online resource allocation problems.

Let us use the long-term budget allocation problem solved in \cite{liakopoulos2019cautious} as an example,
and assume that the per time step budget constraint is $g_t(\theta)\le b_t$,
where $b_t$ is the budget at time step $t$.
Since we are usually given the total budget $b$ over $T$ time steps and
have no idea on how to allocate it,
we could set per time step budget constraint being equal to $g_t(\theta)\le b/T$.
The OCO with long-term constraint algorithm
can dynamically allocate the per time step budget usage
and make sure the budget is satisfied on average
as of the result $\sum\limits_{t=1}^T g_t(\theta_t)-b\le o(T)$.
As mentioned in the previous section, 
to solve the problem of the increasing difficulty in satisfying all the constraints
$g_t(\theta^*)-b_t\le 0, t=\{1,2,\dots,T\}$ occurred in previous algorithms,
\cite{liakopoulos2019cautious} used another comparator
$\hat{\theta}=\argmin_{\theta\in \Theta_K}\sum\limits_{t=1}^T f_t(\theta)$,
where $\Theta_K = \{\theta\in\cS_0: \sum\limits_{i=t}^{t+K-1}g_i(\theta)\le 0, 1\le t\le T-K+1\}$,
$\cS_0$ is the fixed convex set, and $K$ is a user-determined parameter.

However, as discussed at the beginning of this chapter,
the constraint set $\Theta_K$ is not appropriate
in other resource allocation problems such as the job scheduling and rates of failure allocation (the constraint violation itself),
since many applications' long-term constraint cannot be simply added together.
In general, there are three types of time-dependent long-term constraint:
\begin{enumerate}
 \item $\sum\limits_{t=1}^Tg_t(\theta_t)$, sum of the constraint functions, ideal for non-causal constraint such as budget one,
which is used in \cite{liakopoulos2019cautious,yu2017online,chen2017online}.
\item
  $Q_T$, where $Q_t = [Q_{t-1} + g_t(\theta_t)]_+$, $t=1,2,\dots, T$
  and $Q_0 = 0$, which considers the causality restriction when adding the constraints.
  For example, the queuing/job scheduling constraint as mentioned before. 
  However, previous works dealing with queuing type long-term constraint such as \cite{yu2017online,chen2017online}
  usually use $\sum\limits_{t=1}^Tg_t(\theta_t)$, which is inappropriate.
\item
  $\sum\limits_{t=1}^T[g_t(\theta_t)]_+$, cumulative constraint, which only considers the violation part
and is ideal for the long-term failure rate constraint like mistake error.
\end{enumerate}

This section's goal is to enable our proposed algorithms to apply to different types of 
time-dependent long-term constraint problems
by bounding the $\sum\limits_{t=1}^T[g_t(\theta_t)]_+$,
since bounding the 3rd type implies the bound for the other two types.

Since the constraint set $\Theta_K$ cannot be used when bounding $\sum\limits_{t=1}^T[g_t(\theta_t)]_+$, 
we need a new way to solve the \emph{loose regret} due to
the problem of
the increasing difficulty in satisfying all the constraints.
As discussed in the previous chapters, another tighter performance metric used in online learning is called \emph{dynamic regret},
which measures the difference of the cumulative loss 
against a comparison sequence, $z_1,\ldots,z_T\in \cS$:
\begin{equation*}
\cR_d = \sum\limits_{t=1}^T f_t(\theta_t) 
- \sum\limits_{t=1}^T f_t(z_t)
\end{equation*}

For the convex $f_t$, $\cR_d\le O(\sqrt{T(1+V)})$ is obtained by \cite{zhang2018adaptive},
while for the strongly convex or exp-concave $f_t$,
$\cR_d\le \max\{O(\log T), O(\sqrt{TV})\}$ is shown in \cite{yuan2019trading},
where $V$ is the comparator sequence's path-length defined as:
\begin{equation*}
  \sum_{t=2}^{T} \|z_{t}-z_{t-1}\| \le V
\end{equation*}

For the purposes of both solving the problem of the loose bound occurred in \emph{static regret}
and mitigating the generalization issue in using the set $\Theta_K$,
we extend the Algorithm \ref{alg::convex-long-term} to bound the dynamic regret $\cR_d$, where the comparator sequence $z_1^T$
is coming from the set
$V_K(z_1^T) = \{z_1^T\in\cS_0: \sum\limits_{t=2}^T\|z_t-z_{t-1}\|\le V, 
\text{the number of feasible $z_i$ (e.g.,$g_i(z_i)\le 0$) is K}\}$.
This generalizes the comparator set $\Theta_K$
by allowing the changes of the comparator sequence as opposed to a fixed one,
which has a much tighter bound compared to the static regret
and is more appropriate under the changing environments.

The assumptions used in this section are the following:
\begin{itemize}
\item The fixed convex set $\cS_0$ is compact with diameter equal to $D$.
\item Both $f_t$ and $g_t$ are Lipschitz continuous with $\partial_\theta f_t(\theta)\le G$,
and $\partial_\theta g_t(\theta)\le G$.
Since $\cS_0$ is compact, without loss of generality, 
we assume $f_t(\theta)\le F$, $g_t(\theta)\le F$, $\forall \theta\in\cS_0$.
\item The comparator sequence $z_1^T$ coming from 
$V_K(z_1^T) = \{z_1^T\in\cS_0: \sum\limits_{t=2}^T\|z_t-z_{t-1}\|\le V, 
\text{the number of feasible $z_i$ (e.g.,$g_i(z_i)\le 0$) is K}\}$
is not empty.
\end{itemize}
where the first two assumptions are ubiquitous in the online convex optimization.
The 3rd one is used to define the dynamic regret used in this chapter,
which is less restrictive compared to
both the $\Theta_K = \{\theta\in\cS_0: \sum\limits_{i=t}^{t+K-1}g_i(\theta)\le 0, 1\le t\le T-K+1\}$ in \cite{liakopoulos2019cautious}
and the $\cap_t\{\theta\in\cS_0: g_t(\theta)\le 0\}$ in \cite{yuan2018online,yu2017online}.
Since $g_t$ can be generated adversarially, 
it is possible to make $\cap_t\{\theta\in\cS_0: g_t(\theta)\le 0\}$ infeasible
by varying the $g_t$ intentionally.

In order to solve the time-changing long-term constraint $g_t(\theta)$,
we modify the Eq.~\eqref{eq::new long term lagrangian} as:
\begin{equation}
\label{eq:gen_lag}
\cL_t(\theta,\lambda) = f_t(\theta) + \lambda[g_t(\theta)]_+ - \frac{\phi_t}{2}\lambda^2
\end{equation}
Although the analysis in the previous section
can be used to deal with time-changing $g_t(\theta)$,
the results only hold true w.r.t. the very loose \emph{static regret}.

\subsection{Convex Case}

Let us first discuss the update rule and the results associated with the case when $f_t(\theta)$ is convex.

We first change the $\cL_t$ in Eq.~\eqref{eq:gen_lag} 
by replacing the time-dependent parameter $\phi_t$ with $\sigma\eta$ as:
\begin{equation}
\label{eq:conv_lag}
\cL_t(\theta,\lambda) = f_t(\theta) + \lambda[g_t(\theta)]_+ - \frac{\sigma\eta}{2}\lambda^2
\end{equation}

The update rule for $t=1,2,\dots,T$ is
\begin{subequations}
\label{eq:convex_update}
\begin{align}
\label{eq:convex_update_lam}
\lambda_t = \frac{[g_t(\theta_t)]_+}{\sigma\eta}
\end{align}
\begin{align}
\label{eq:convex_update_x}
\theta_{t+1} = \Pi_{\cS_0}\Big(\theta_t-\eta\nabla_\theta\cL_t(\theta_t,\lambda_t)\Big)
\end{align}
\end{subequations}
where $\theta_1$ is initialized in $\cS_0$, and we abuse the sub-gradient notation
to denote a single element of the sub-gradient by $\nabla_\theta\cL_t$.

With the update rule in Eq.~\eqref{eq:convex_update}, we can get the following result:
\begin{theorem}
\label{thm:gen_conv}
For any comparator sequence $z_1^T\in V_K(z_1^T)$, by setting $\sigma = 2G^2$
and $\eta = O(\sqrt{\frac{T-K+1+V}{T}})$, we can bound the $\cR_d$ and $\sum\limits_{t=1}^T([g_t(\theta_t)]_+)^2$ as
\begin{subequations}
\begin{align}
\label{eq:gen_conv_obj}
\cR_d = \sum\limits_{t=1}^T\Big(f_t(\theta_t)-f_t(z_t)\Big) \le O(\sqrt{T(T-K+1+V)})
\end{align}
\begin{align}
\label{eq:gen_conv_const}
\sum\limits_{t=1}^T([g_t(\theta_t)]_+)^2 \le O(\sqrt{T(T-K+1+V)})
\end{align}
\end{subequations}

\end{theorem}

Please refer to the Appendix for all the omitted proofs in this section.

Theorem \ref{thm:gen_conv} generalizes the results in the previous section
by both varying the comparator sequence and the constraint feasibility.
More specifically, Theorem \ref{thm:gen_conv} recovers the result in previous section
by setting $V = 0$ and $K = T$.

One direct consequence of the above theorem is:
\begin{corollary}
If $T-K = O(T^{1-\epsilon_1})$, $\epsilon_1\in[0,1]$, and 
$V = O(T^{1-\epsilon_2})$, $\epsilon_2\in[0,1]$, then 
\begin{subequations}
\nonumber
\begin{align}
\cR_d = \sum\limits_{t=1}^T\Big(f_t(\theta_t)-f_t(z_t)\Big) \le \max\{O(T^{1-\epsilon_1/2}), O(T^{1-\epsilon_2/2})\}
\end{align}
\begin{align}
\sum\limits_{t=1}^T([g_t(\theta_t)]_+)^2 \le \max\{O(T^{1-\epsilon_1/2}), O(T^{1-\epsilon_2/2})\}
\end{align}
\begin{align}
\sum\limits_{t=1}^T[g_t(\theta_t)]_+ \le \max\{O(T^{1-\epsilon_1/4}), O(T^{1-\epsilon_2/4})\}
\end{align}
\end{subequations}

\end{corollary}

\begin{proof}
The first two inequalities are due the direct calculation by plugging $T-K = O(T^{1-\epsilon_1})$
and $V = O(T^{1-\epsilon_2})$ into Eq.~\eqref{eq:gen_conv_obj} and \eqref{eq:gen_conv_const}.
The third inequality can be obtained by viewing $[g_t(\theta_t)]_+,t=1,2,\dots,T$ as a vector
and using the vector norm inequality $\|x\|_1\le \sqrt{T}\|x\|$.
\end{proof}

The above Corollary generalizes the result in \cite{liakopoulos2019cautious}
by considering the dynamic regret w.r.t $z_1^T$ and more general long-term constraint bound.

\subsection{Strongly Convex Case}

In this case, we use the Augmented Lagrangian function $\cL_t$ defined in Eq.~\eqref{eq:gen_lag},
which is rewritten here as:
\begin{equation*}
\cL_t(\theta,\lambda) = f_t(\theta) + \lambda[g_t(\theta)]_+ - \frac{\phi_t}{2}\lambda^2
\end{equation*}

The update rule for $t=1,2,\dots,T$ is
\begin{subequations}
\label{eq:s-convex_update}
\begin{align}
\label{eq:s-convex_update_lam}
\lambda_t = \frac{[g_t(\theta_t)]_+}{\phi_t}
\end{align}
\begin{align}
\label{eq:s-convex_update_x}
\theta_{t+1} = \Pi_{\cS_0}\Big(\theta_t-\eta_t\nabla_\theta\cL_t(\theta_t,\lambda_t)\Big)
\end{align}
\end{subequations}
where $\theta_1$ is initialized in $\cS_0$, and we abuse the sub-gradient notation
to denote a single element of the sub-gradient by $\nabla_\theta\cL_t$.

Compared to the update rule in Eq.~\eqref{eq:convex_update}, the one
in strongly convex case has time-dependent parameters like $\phi_t$ and $\eta_t$.
This is aligned with the parameter setup in previous works
like \cite{hazan2007logarithmic,yuan2018online,yuan2019trading}.

The update rule in Eq.~\eqref{eq:s-convex_update} results in the following theorem:
\begin{theorem}
\label{thm:s-conv}
By using $\phi_t = 2G^2\eta_t$, $\eta_t = \frac{1-\gamma}{\ell(1-\gamma^t)}$,
and $\gamma = 1-\frac{1}{2}\sqrt{\frac{\max\{V+T-K,\log^2T/T\}}{(D+1)T}}$,
for $f_t$ with strong convexity parameter $\ell$ and any comparator sequence $z_1^T\in V_K(z_1^T)$, the following results hold:
\begin{subequations}
\begin{align}
\label{thm:s-conv_obj}
\cR_d = \sum\limits_{t=1}^T\Big(f_t(\theta_t)-f_t(z_t)\Big) \le \max\{O(\sqrt{T(T-K+V)}),O(\log T)\}
\end{align}
\begin{align}
\label{thm:gen_conv_const}
\sum\limits_{t=1}^T[g_t(\theta_t)]_+ \le \max\{O(T^{3/4}(T-K+V)^{1/4}),O(\sqrt{T\log T})\}
\end{align}
\end{subequations}

\end{theorem}

Compared to the result in convex case, both the $\cR_d$ and the $\sum\limits_{t=1}^T[g_t(x_t)]_+$
are improved. 
The improvement in terms of the order complexity only happens when
$K = T$ and $V = o(1)$ (e.g., $V = 0$). 
For the other cases, it also reduces the additive value by about $\sqrt{T}$.

\section{Conclusion}
\label{sec:oco-long-term-conclusion}

In this chapter, we propose algorithms for OCO with both convex and strongly convex objective functions.
By applying different update strategies that utilize a modified augmented Lagrangian function,
they can solve OCO with a squared/clipped long-term constraints requirement.
The algorithm for general convex case provides the useful bounds for
both  the long-term constraint violation and the constraint violation
at each time step.
Furthermore, the bounds for the strongly convex case is an improvement compared with the previous efforts in the literature.
Experiments show that our algorithms can follow the constraint boundary tightly and
have relatively smaller clipped long-term constraint violation with reasonably low objective regret.

Furthermore, we extend the algorithms to solve the time-dependent long-term constraint
problem with a variant of \emph{dynamic regret} guarantee,
which can be applied to more general resource allocation problems than the previous algorithms.

\chapter{Conclusion}
\label{chap:conclusion}

Tracking the changes of the environments is a key difference between Online Convex Optimization (OCO) algorithms 
and the batch processing based approaches, 
since the sequential data/observation tends to be shifting over time. 
In this thesis,
we develop different OCO algorithms
for various problems
to enable the decision making on-the-fly with better adaptivity to the changing environments.

One way to have better adaptivity is to examine the proposed algorithms' performance by the notion of the \emph{dynamic regret},
which compares the algorithm's cumulative loss 
against that incurred by a comparison sequence.
For the general exp-concave or strongly convex problems,
we propose discounted Online Newton algorithm to
have dynamic regret guarantee $\cR_d \le\max\{O(\log T),O(\sqrt{TV})\}$,
which is inspired by the forgetting factor used in the Recursive Least Squares algorithms.
Moreover, the trade-off between static and dynamic regret
is analyzed for both Online Least-Squares and its generalization of
strongly convex and smooth objective.
To obtain more computationally efficient algorithms, 
 we also propose a novel gradient descent step size rule for
 strongly convex functions,
 which recovers the dynamic regret bounds
 described above. 

Another way to deal with changing environments is to upper bound the 
notion of \emph{adaptive regret}.
Previous literature has developed algorithms 
for the online convex problems by running a pool of algorithms in parallel,
resulting in
the unwanted increase in both the running time and the implementation complexity.
To avoid these problems,
we propose a new algorithm with same performance guarantee,
which is the exponentiated gradient
descent algorithm with a mixture of fixed-share step.
We show that this algorithm can be applied 
to the online Principal Component Analysis (PCA) and its extension of variance minimization under changing environments.

For the constrained OCO algorithms, a projection operator is almost unavoidable.
When the constrain set is complex, 
such operation is very time-consuming and prevents the algorithms from the true online implementation.
To accelerate the OCO algorithms' update,
our third part of the thesis propose algorithms to replace the true desired projection 
with an approximate closed-form one.
Although the approximation may cause constraint violation for some time steps,
sub-linear cumulative constraint violation is guaranteed to achieve the constraint satisfaction on average.
Furthermore, single step constraint violation is bounded to avoid undesired large step violations.
Finally,
we extend our proposed algorithms' idea to solve the more general time-dependent online resource allocation problems 
with performance guarantee by a variant of \emph{dynamic regret}.

\bibliographystyle{hunsrtMycopy} 
\bibliography{OCO_dynamic}

\begin{thebibliography}{10}

\bibitem{hazan2016introduction}
E. Hazan.
\newblock Introduction to online convex optimization.
\newblock {\em Foundations and Trends{\textregistered} in Optimization},
  2(3-4):157--325, 2016.

\bibitem{blum2004online}
A. Blum, V. Kumar, A. Rudra, and F. Wu.
\newblock Online learning in online auctions.
\newblock {\em Theoretical Computer Science}, 324(2-3):137--146, 2004.

\bibitem{crammer2006online}
K. Crammer, O. Dekel, J. Keshet, S. Shalev-Shwartz, and Y. Singer.
\newblock Online passive-aggressive algorithms.
\newblock {\em Journal of Machine Learning Research}, 7(Mar):551--585, 2006.

\bibitem{hazan2018spectral}
E. Hazan, H. Lee, K. Singh, C. Zhang, and Y. Zhang.
\newblock Spectral filtering for general linear dynamical systems.
\newblock {\em arXiv preprint arXiv:1802.03981}, 2018.

\bibitem{herbster1998tracking}
M. Herbster and M.~K. Warmuth.
\newblock Tracking the best expert.
\newblock {\em Machine learning}, 32(2):151--178, 1998.

\bibitem{cesa2012new}
N. Cesa-Bianchi, P. Gaillard, G. Lugosi, and G. Stoltz.
\newblock A new look at shifting regret.
\newblock {\em arXiv preprint arXiv:1202.3323}, 2012.

\bibitem{cesa2012mirror}
N. Cesa-Bianchi, P. Gaillard, G. Lugosi, and G. Stoltz.
\newblock Mirror descent meets fixed share (and feels no regret).
\newblock In {\em Advances in Neural Information Processing Systems}, pages
  980--988, 2012.

\bibitem{tsuda2005matrix}
K. Tsuda, G. R{\"a}tsch, and M.~K. Warmuth.
\newblock Matrix exponentiated gradient updates for on-line learning and
  bregman projection.
\newblock {\em Journal of Machine Learning Research}, 6(Jun):995--1018, 2005.

\bibitem{warmuth2006online}
M.~K. Warmuth and D. Kuzmin.
\newblock Online variance minimization.
\newblock In {\em International Conference on Computational Learning Theory},
  pages 514--528. Springer, 2006.

\bibitem{warmuth2008randomized}
M.~K. Warmuth and D. Kuzmin.
\newblock Randomized online pca algorithms with regret bounds that are
  logarithmic in the dimension.
\newblock {\em Journal of Machine Learning Research}, 9(Oct):2287--2320, 2008.

\bibitem{niew2016onlinepca}
J. Nie, W. Kotlowski, and M.~K. Warmuth.
\newblock Online pca with optimal regret.
\newblock {\em Journal of Machine Learning Research}, 17(173):1--49, 2016.

\bibitem{yu2017online}
H. Yu, M. Neely, and X. Wei.
\newblock Online convex optimization with stochastic constraints.
\newblock In {\em Advances in Neural Information Processing Systems}, pages
  1427--1437, 2017.

\bibitem{yuan2018online}
J. Yuan and A. Lamperski.
\newblock Online convex optimization for cumulative constraints.
\newblock In {\em Advances in Neural Information Processing Systems}, pages
  6137--6146, 2018.

\bibitem{liakopoulos2019cautious}
N. Liakopoulos, A. Destounis, G. Paschos, T. Spyropoulos, and P. Mertikopoulos.
\newblock Cautious regret minimization: Online optimization with long-term
  budget constraints.
\newblock In {\em International Conference on Machine Learning}, pages
  3944--3952, 2019.

\bibitem{abernethy2008optimal}
J. Abernethy, P.~L. Bartlett, A. Rakhlin, and A. Tewari.
\newblock Optimal strategies and minimax lower bounds for online convex games.
\newblock 2008.

\bibitem{helmbold1998line}
D.~P. Helmbold, R.~E. Schapire, Y. Singer, and M.~K. Warmuth.
\newblock On-line portfolio selection using multiplicative updates.
\newblock {\em Mathematical Finance}, 8(4):325--347, 1998.

\bibitem{das2013online}
P. Das, N. Johnson, and A. Banerjee.
\newblock Online lazy updates for portfolio selection with transaction costs.
\newblock In {\em Twenty-Seventh AAAI Conference on Artificial Intelligence},
  2013.

\bibitem{zinkevich2003online}
M. Zinkevich.
\newblock Online convex programming and generalized infinitesimal gradient
  ascent.
\newblock In {\em Proceedings of the 20th International Conference on Machine
  Learning (ICML-03)}, pages 928--936, 2003.

\bibitem{hazan2009efficient}
E. Hazan and C. Seshadhri.
\newblock Efficient learning algorithms for changing environments.
\newblock In {\em Proceedings of the 26th annual international conference on
  machine learning}, pages 393--400. ACM, 2009.

\bibitem{besbes2015non}
O. Besbes, Y. Gur, and A. Zeevi.
\newblock Non-stationary stochastic optimization.
\newblock {\em Operations research}, 63(5):1227--1244, 2015.

\bibitem{mokhtari2016online}
A. Mokhtari, S. Shahrampour, A. Jadbabaie, and A. Ribeiro.
\newblock Online optimization in dynamic environments: Improved regret rates
  for strongly convex problems.
\newblock In {\em 2016 IEEE 55th Conference on Decision and Control (CDC)},
  pages 7195--7201. IEEE, 2016.

\bibitem{zhang2018adaptive}
L. Zhang, S. Lu, and Z.-H. Zhou.
\newblock Adaptive online learning in dynamic environments.
\newblock In {\em Advances in Neural Information Processing Systems}, pages
  1323--1333, 2018.

\bibitem{cesa2007improved}
N. Cesa-Bianchi, Y. Mansour, and G. Stoltz.
\newblock Improved second-order bounds for prediction with expert advice.
\newblock {\em Machine Learning}, 66(2-3):321--352, 2007.

\bibitem{hazan2007logarithmic}
E. Hazan, A. Agarwal, and S. Kale.
\newblock Logarithmic regret algorithms for online convex optimization.
\newblock {\em Machine Learning}, 69(2):169--192, 2007.

\bibitem{abernethy2009stochastic}
J. Abernethy, A. Agarwal, P.~L. Bartlett, and A. Rakhlin.
\newblock A stochastic view of optimal regret through minimax duality.
\newblock {\em arXiv preprint arXiv:0903.5328}, 2009.

\bibitem{fazel2018global}
M. Fazel, R. Ge, S.~M. Kakade, and M. Mesbahi.
\newblock Global convergence of policy gradient methods for linearized control
  problems.
\newblock {\em arXiv preprint arXiv:1801.05039}, 2018.

\bibitem{hazan2017efficient}
E. Hazan, K. Singh, and C. Zhang.
\newblock Efficient regret minimization in non-convex games.
\newblock In {\em International Conference on Machine Learning}, pages
  1433--1441, 2017.

\bibitem{gao2018online}
X. Gao, X. Li, and S. Zhang.
\newblock Online learning with non-convex losses and non-stationary regret.
\newblock In {\em International Conference on Artificial Intelligence and
  Statistics}, pages 235--243, 2018.

\bibitem{auer2002using}
P. Auer.
\newblock Using confidence bounds for exploitation-exploration trade-offs.
\newblock {\em Journal of Machine Learning Research}, 3(Nov):397--422, 2002.

\bibitem{agarwal2011stochastic}
A. Agarwal, D.~P. Foster, D.~J. Hsu, S.~M. Kakade, and A. Rakhlin.
\newblock Stochastic convex optimization with bandit feedback.
\newblock In {\em Advances in Neural Information Processing Systems}, pages
  1035--1043, 2011.

\bibitem{bubeck2012regret}
S. Bubeck, N. Cesa-Bianchi, et~al.
\newblock Regret analysis of stochastic and nonstochastic multi-armed bandit
  problems.
\newblock {\em Foundations and Trends{\textregistered} in Machine Learning},
  5(1):1--122, 2012.

\bibitem{anava2013online}
O. Anava, E. Hazan, S. Mannor, and O. Shamir.
\newblock Online learning for time series prediction.
\newblock In {\em Conference on learning theory}, pages 172--184, 2013.

\bibitem{yuan2017online}
J. Yuan and A. Lamperski.
\newblock Online control basis selection by a regularized actor critic
  algorithm.
\newblock In {\em American Control Conference (ACC), 2017}, pages 4448--4453.
  IEEE, 2017.

\bibitem{freund1997decision}
Y. Freund and R.~E. Schapire.
\newblock A decision-theoretic generalization of on-line learning and an
  application to boosting.
\newblock {\em Journal of computer and system sciences}, 55(1):119--139, 1997.

\bibitem{cesa2006prediction}
N. Cesa-Bianchi and G. Lugosi.
\newblock {\em Prediction, learning, and games}.
\newblock Cambridge university press, 2006.

\bibitem{beck2003mirror}
A. Beck and M. Teboulle.
\newblock Mirror descent and nonlinear projected subgradient methods for convex
  optimization.
\newblock {\em Operations Research Letters}, 31(3):167--175, 2003.

\bibitem{shalev2012online}
S. Shalev-Shwartz et~al.
\newblock Online learning and online convex optimization.
\newblock {\em Foundations and Trends{\textregistered} in Machine Learning},
  4(2):107--194, 2012.

\bibitem{yuan2019trading}
J. Yuan and A. Lamperski.
\newblock Trading-off static and dynamic regret in online least-squares and
  beyond.
\newblock {\em arXiv preprint arXiv:1909.03118}, 2019.

\bibitem{hall2013dynamical}
E.~C. Hall and R.~M. Willett.
\newblock Dynamical models and tracking regret in online convex programming.
\newblock In {\em Proceedings of the 30th International Conference on
  International Conference on Machine Learning-Volume 28}, pages I--579. JMLR.
  org, 2013.

\bibitem{chiang2012online}
C.-K. Chiang, T. Yang, C.-J. Lee, M. Mahdavi, C.-J. Lu, R. Jin, and S. Zhu.
\newblock Online optimization with gradual variations.
\newblock In {\em Conference on Learning Theory}, pages 6--1, 2012.

\bibitem{kotlowski2015pca}
W. Kot{\l}owski and M.~K. Warmuth.
\newblock Pca with gaussian perturbations.
\newblock {\em arXiv preprint arXiv:1506.04855}, 2015.

\bibitem{mahdavi2012trading}
M. Mahdavi, R. Jin, and T. Yang.
\newblock Trading regret for efficiency: online convex optimization with long
  term constraints.
\newblock {\em Journal of Machine Learning Research}, 13(Sep):2503--2528, 2012.

\bibitem{jenatton2016adaptive}
R. Jenatton, J. Huang, and C. Archambeau.
\newblock Adaptive algorithms for online convex optimization with long-term
  constraints.
\newblock In {\em International Conference on Machine Learning}, pages
  402--411, 2016.

\bibitem{yang2016tracking}
T. Yang, L. Zhang, R. Jin, and J. Yi.
\newblock Tracking slowly moving clairvoyant: Optimal dynamic regret of online
  learning with true and noisy gradient.
\newblock In {\em International Conference on Machine Learning}, pages
  449--457, 2016.

\bibitem{sayed2011adaptive}
A.~H. Sayed.
\newblock {\em Adaptive filters}.
\newblock John Wiley \& Sons, 2011.

\bibitem{guo1993performance}
L. Guo, L. Ljung, and P. Priouret.
\newblock Performance analysis of the forgetting factor rls algorithm.
\newblock {\em International journal of adaptive control and signal
  processing}, 7(6):525--537, 1993.

\bibitem{zhao2019distribution}
P. Zhao, X. Wang, S. Xie, L. Guo, and Z.-H. Zhou.
\newblock Distribution-free one-pass learning.
\newblock {\em IEEE Transactions on Knowledge and Data Engineering}, 2019.

\bibitem{garivier2011upper}
A. Garivier and E. Moulines.
\newblock On upper-confidence bound policies for switching bandit problems.
\newblock In {\em International Conference on Algorithmic Learning Theory},
  pages 174--188. Springer, 2011.

\bibitem{russac2019weighted}
Y. Russac, C. Vernade, and O. Capp{\'e}.
\newblock Weighted linear bandits for non-stationary environments.
\newblock In {\em Advances in Neural Information Processing Systems}, pages
  12017--12026, 2019.

\bibitem{yuan2019online}
J. Yuan and A. Lamperski.
\newblock Online adaptive principal component analysis and its extensions.
\newblock In {\em International Conference on Machine Learning}, pages
  7213--7221, 2019.

\bibitem{bregman1967relaxation}
L.~M. Bregman.
\newblock The relaxation method of finding the common point of convex sets and
  its application to the solution of problems in convex programming.
\newblock {\em USSR computational mathematics and mathematical physics},
  7(3):200--217, 1967.

\bibitem{censor1981iterative}
Y. Censor and A. Lent.
\newblock An iterative row-action method for interval convex programming.
\newblock {\em Journal of Optimization theory and Applications},
  34(3):321--353, 1981.

\bibitem{herbster2001tracking}
M. Herbster and M.~K. Warmuth.
\newblock Tracking the best linear predictor.
\newblock {\em Journal of Machine Learning Research}, 1(Sep):281--309, 2001.

\bibitem{markowitz1952portfolio}
H. Markowitz.
\newblock Portfolio selection.
\newblock {\em The journal of finance}, 7(1):77--91, 1952.

\bibitem{kalai2005efficient}
A. Kalai and S. Vempala.
\newblock Efficient algorithms for online decision problems.
\newblock {\em Journal of Computer and System Sciences}, 71(3):291--307, 2005.

\bibitem{arora2012stochastic}
R. Arora, A. Cotter, K. Livescu, and N. Srebro.
\newblock Stochastic optimization for pca and pls.
\newblock In {\em Communication, Control, and Computing (Allerton), 2012 50th
  Annual Allerton Conference on}, pages 861--868. IEEE, 2012.

\bibitem{duchi2008efficient}
J. Duchi, S. Shalev-Shwartz, Y. Singer, and T. Chandra.
\newblock Efficient projections onto the l 1-ball for learning in high
  dimensions.
\newblock In {\em Proceedings of the 25th international conference on Machine
  learning}, pages 272--279. ACM, 2008.

\bibitem{duchi2010composite}
J.~C. Duchi, S. Shalev-Shwartz, Y. Singer, and A. Tewari.
\newblock Composite objective mirror descent.
\newblock In {\em COLT}, pages 14--26, 2010.

\bibitem{nesterov2005smooth}
Y. Nesterov.
\newblock Smooth minimization of non-smooth functions.
\newblock {\em Mathematical programming}, 103(1):127--152, 2005.

\bibitem{cvxpy}
S. Diamond and S. Boyd.
\newblock {CVXPY}: A {P}ython-embedded modeling language for convex
  optimization.
\newblock {\em Journal of Machine Learning Research}, 17(83):1--5, 2016.

\bibitem{li2018online}
Y. Li, G. Qu, and N. Li.
\newblock Online optimization with predictions and switching costs: Fast
  algorithms and the fundamental limit.
\newblock {\em arXiv preprint arXiv:1801.07780}, 2018.

\bibitem{senthil2010economic}
K. Senthil and K. Manikandan.
\newblock Economic thermal power dispatch with emission constraint and valve
  point effect loading using improved tabu search algorithm.
\newblock {\em International Journal of Computer Applications}, 2010.

\bibitem{defazio2014saga}
A. Defazio, F. Bach, and S. Lacoste-Julien.
\newblock Saga: A fast incremental gradient method with support for
  non-strongly convex composite objectives.
\newblock In {\em Advances in Neural Information Processing Systems}, pages
  1646--1654, 2014.

\bibitem{chen2017online}
T. Chen, Q. Ling, and G.~B. Giannakis.
\newblock An online convex optimization approach to proactive network resource
  allocation.
\newblock {\em IEEE Transactions on Signal Processing}, 65(24):6350--6364,
  2017.

\bibitem{nesterov2013introductory}
Y. Nesterov.
\newblock {\em Introductory lectures on convex optimization: A basic course},
  volume~87.
\newblock Springer Science \& Business Media, 2013.

\end{thebibliography}

\appendix
\chapter{Trading-Off Static and Dynamic Regret in Online Least-Squares
and Beyond}
\label{sec:appdx_tradeoff}

\noindent
\paragraph{Proof of Lemma~\ref{lem:gen_prob_var_path}:}
\begin{proof}

\noindent
The proof follows the analysis in Chapter 2 of \cite{nesterov2013introductory}.

\noindent
From the strong convexity of $f_t(\theta)$, we have 
\small
\begin{equation}
\label{eq::gen_strongly_ineq}
\begin{array}{ll}
f_t(\theta) &\ge f_t(\theta_t)+\nabla f_t(\theta_t)^\top(\theta-\theta_t)+\frac{\ell}{2}\left\|\theta-\theta_t\right\|^2\\
&= f_t(\theta_t)+\nabla f_t(\theta_t)^\top(\theta-\theta_t) +\nabla f_t(\theta_t)^\top(\theta_{t+1}-\theta_t) 
-\nabla f_t(\theta_t)^\top(\theta_{t+1}-\theta_t)+\frac{\ell}{2}\left\|\theta-\theta_t\right\|^2 \\
&= f_t(\theta_t) +\nabla f_t(\theta_t)^\top(\theta_{t+1}-\theta_t) 
+\nabla f_t(\theta_t)^\top(\theta-\theta_{t+1})+\frac{\ell}{2}\left\|\theta-\theta_t\right\|^2
\end{array}
\end{equation}
\normalsize

\noindent
According to the optimality condition of the update rule in Eq.\eqref{eq::gen_prob_update},
we have $\big(\nabla f_t(\theta_t)+\frac{1}{\eta_t}(\theta_{t+1}-\theta_t)\big)^\top(\theta-\theta_{t+1})\ge 0,\forall \theta\in\cS$,
which is $\nabla f_t(\theta_t)^\top(\theta-\theta_{t+1})\ge \frac{1}{\eta_t}(\theta_{t}-\theta_{t+1})^\top(\theta-\theta_{t+1})$.
Then combine with Eq.\eqref{eq::gen_strongly_ineq}, we have 
\begin{equation}
\label{eq::gen_strongly_final}
\begin{array}{ll}
f_t(\theta) &\ge f_t(\theta_t) +\nabla f_t(\theta_t)^\top(\theta_{t+1}-\theta_t) 
+\frac{1}{\eta_t}(\theta_{t}-\theta_{t+1})^\top(\theta-\theta_{t+1}) +\frac{\ell}{2}\left\|\theta-\theta_t\right\|^2
\end{array}
\end{equation}

\noindent
From the smoothness of $f_t(\theta)$, we have 
$f_t(\theta_{t+1}) \le f_t(\theta_t)+\nabla f_t(\theta_t)^\top(\theta_{t+1}-\theta_t)+\frac{u}{2}\left\|\theta_{t+1}-\theta_t\right\|^2$.
Since $\frac{1}{\eta_t} = \frac{\ell(\gamma-\gamma^t)+u(1-\gamma)}{1-\gamma}\ge u$,
we have $f_t(\theta_t)+\nabla f_t(\theta_t)^\top(\theta_{t+1}-\theta_t) 
\ge f_t(\theta_{t+1}) - \frac{1}{2\eta_t}\left\|\theta_{t+1}-\theta_t\right\|^2$.
Then combined with inequality \eqref{eq::gen_strongly_final}, we have 
\begin{equation}
\begin{array}{ll}
f_t(\theta) &\ge f_t(\theta_{t+1})-\frac{1}{2\eta_t}\left\|\theta_{t+1}-\theta_t\right\|^2 
+ \frac{1}{\eta_t}(\theta_{t}-\theta_{t+1})^\top(\theta-\theta_{t+1})
+\frac{\ell}{2}\left\|\theta-\theta_t\right\|^2\\
& = f_t(\theta_{t+1})+\frac{1}{2\eta_t}\left\|\theta_{t+1}-\theta_t\right\|^2 
+ \frac{1}{\eta_t}(\theta_{t}-\theta_{t+1})^\top(\theta-\theta_{t})
+\frac{\ell}{2}\left\|\theta-\theta_t\right\|^2
\end{array}
\end{equation}

\noindent
By setting $\theta = \theta_t^*$ and using the fact $f_t(\theta_t^*)\le f_t(\theta_{t+1})$,
we reformulate the above inequality as:
\begin{equation}
\begin{array}{l}
(\theta_{t}-\theta_{t+1})^\top(\theta_t^*-\theta_{t}) 
\le -\frac{\ell(1-\gamma)}{2\ell(\gamma-\gamma^t)+2u(1-\gamma)}\left\|\theta_t^*-\theta_t\right\|^2
-\frac{1}{2}\left\|\theta_{t+1}-\theta_t\right\|^2
\end{array}
\end{equation}

\noindent
Since $\left\|\theta_{t+1} -\theta_t^*\right\|^2 = \left\|\theta_{t+1}-\theta_t+\theta_t -\theta_t^*\right\|^2$,
we have
\begin{equation}
\begin{array}{ll}
\left\|\theta_{t+1} -\theta_t^*\right\|^2 &= \left\|\theta_{t+1} -\theta_t\right\|^2  + \left\|\theta_{t} -\theta_t^*\right\|^2
+ 2(\theta_{t}-\theta_{t+1})^\top(\theta_t^*-\theta_{t}) \\
& \le \big(1-\frac{\ell(1-\gamma)}{\ell(\gamma-\gamma^t)+u(1-\gamma)}\big)\left\|\theta_{t} -\theta_t^*\right\|^2 \\
& \le \big(1-\frac{\ell(1-\gamma)}{\ell\gamma+u(1-\gamma)}\big)\left\|\theta_{t} -\theta_t^*\right\|^2
\end{array}
\end{equation}

\end{proof}

\noindent
\paragraph{Proof of Theorem~\ref{thm::gen_prob_dynamic_regret}:}
\begin{proof}

\noindent
We use the same steps as in the previous section.
First, according to the Mean Value Theorem, 
we have 
$f_t(\theta_t)-f_t(\theta_t^*) = \nabla f_t(x)^\top(\theta_t-\theta_t^*)
\le \left\|\nabla f_t(x)\right\|\left\|\theta_t-\theta_t^*\right\|$,
where $x\in \{v| v = \delta \theta_t + (1-\delta)\theta_t^*,\delta\in[0,1]\}$.
Due to the assumption on the upper bound of the norm of the gradient, we have 
$f_t(\theta_t)-f_t(\theta_t^*) \le G\left\|\theta_t-\theta_t^*\right\|$.
As a result, 
$\sum\limits_{t=1}^T\big(f_t(\theta_t)-f_t(\theta_t^*)\big)\le G\sum\limits_{t=1}^T\left\|\theta_t-\theta_t^*\right\|$.

\noindent
Now we need to upper bound the term $\sum\limits_{t=1}^T\left\|\theta_t-\theta_t^*\right\|$.
$\sum\limits_{t=1}^T\left\|\theta_t-\theta_t^*\right\|$ is equal to $ \left\|\theta_1-\theta_1^*\right\| 
+ \sum\limits_{t=2}^T\left\|\theta_t-\theta_{t-1}^*+\theta_{t-1}^*-\theta_t^*\right\|$,
which is less than 
$\left\|\theta_1-\theta_1^*\right\| + \sum\limits_{t=1}^{T}\left\|\theta_{t+1}-\theta_{t}^*\right\| 
+ \sum\limits_{t=2}^T\left\|\theta_t^*-\theta_{t-1}^*\right\|$.
According to Lemma \ref{lem:gen_prob_var_path},
we have $\sum\limits_{t=1}^{T}\left\|\theta_{t+1}-\theta_{t}^*\right\| 
\le \rho\sum\limits_{t=1}^T\left\|\theta_{t}-\theta_{t}^*\right\|$, 
where $\rho$ is equal to $\sqrt{1-\frac{l(1-\gamma)}{u(1-\gamma)+l\gamma}}$.
Then we have 
$\sum\limits_{t=1}^T\left\|\theta_t-\theta_t^*\right\| \le
\left\|\theta_1-\theta_1^*\right\| + \rho\sum\limits_{t=1}^{T}\left\|\theta_{t}-\theta_{t}^*\right\| 
+ \sum\limits_{t=2}^T\left\|\theta_t^*-\theta_{t-1}^*\right\|$,
which can be reformulated as 
$\sum\limits_{t=1}^T\left\|\theta_t-\theta_t^*\right\| \le
\frac{1}{1-\rho}(\left\|\theta_1-\theta_1^*\right\| +
+ \sum\limits_{t=2}^T\left\|\theta_t^*-\theta_{t-1}^*\right\|)$.

\noindent
$1-\rho = 1-\sqrt{1-\frac{a_0}{b_0}}= \frac{\sqrt{b_0}-\sqrt{b_0-a_0}}{\sqrt{b_0}}$, 
where $a_0 = \ell$ and $b_0 = \frac{\ell\gamma+u(1-\gamma)}{1-\gamma}$.
Thus, $1/(1-\rho) = \frac{\sqrt{b_0}}{\sqrt{b_0}-\sqrt{b_0-a_0}}
= \frac{\sqrt{b_0}(\sqrt{b_0}+\sqrt{b_0-a_0})}{a_0}$.
After plugging in the expression of $1-\gamma = 1/T^{\beta}$,
$1/(1-\rho) = \frac{\sqrt{\ell(T^{\beta}-1)+u}\big(\sqrt{\ell(T^{\beta}-1)+u}+\sqrt{\ell(T^{\beta}-1)+u-\ell}\big)}{\ell}
\le \frac{2\big(\ell(T^{\beta}-1)+u\big)}{\ell} = 2(T^{\beta}-1)+u/\ell$

\noindent
Then $\mathcal{R}_d = \sum\limits_{t=1}^T\big(f_t(\theta_t)-f_t(\theta_t^*)\big)
\le G\frac{1}{1-\rho}\big(\left\|\theta_1-\theta_1^*\right\| +
+ \sum\limits_{t=2}^T\left\|\theta_t^*-\theta_{t-1}^*\right\|\big)
\le G\big(2(T^{\beta}-1)+u/\ell\big)\big(\left\|\theta_1-\theta_1^*\right\| +
+ \sum\limits_{t=2}^T\left\|\theta_t^*-\theta_{t-1}^*\right\|\big)$.

\end{proof}

\paragraph{Proof of Theorem~\ref{thm::gen_prob_static_regret}:}
\begin{proof}

\noindent
The proof follows the similar steps in the proof of Theorem \ref{thm::quad_static_regret}.

\noindent
According to the non-expansive property of the projection operator and the update rule in Eq.\eqref{eq::gen_prob_update},
we have
\begin{equation*}
\begin{array}{ll}
\left\|\theta_{t+1}-\theta^*\right\|^2& \le \left\|\theta_t-\eta_t\nabla f_t(\theta_t)-\theta^*\right\|^2 \\
& = \left\|\theta_t-\theta^*\right\|^2 -2\eta_t\nabla f_t(\theta_t)^\top(\theta_t-\theta^*)
+\eta_t^2\left\|\nabla f_t(\theta_t)\right\|^2
\end{array}
\end{equation*}
The reformulation gives us
\begin{equation} 
\label{eq::gen_prob_contraction_ineq}
\begin{array}{ll}
\nabla f_t(\theta_t)^\top(\theta_t-\theta^*)
&\le \frac{1}{2\eta_t}\big(\left\|\theta_t-\theta^*\right\|^2 - \left\|\theta_{t+1}-\theta^*\right\|^2\big)
+ \frac{\eta_t}{2}\left\|\nabla f_t(\theta_t)\right\|^2
\end{array}
\end{equation}

\noindent
Moreover, 
$f_t(\theta^*)\ge f_t(\theta_t)+\nabla f_t(\theta_t)^\top(\theta^*-\theta_t)+\frac{\ell}{2}\left\|\theta^*-\theta_t\right\|^2$
due to strong convexity,
which is equivalent to 
$\nabla f_t(\theta_t)^\top(\theta_t-\theta^*)\ge f_t(\theta_t)-f_t(\theta^*)+\frac{\ell}{2}\left\|\theta^*-\theta_t\right\|^2$.
Combined with Eq.\eqref{eq::gen_prob_contraction_ineq}, we have
\begin{equation*}
\begin{array}{ll}
f_t(\theta_t)-f_t(\theta^*) &\le \frac{1}{2\eta_t}\big(\left\|\theta_t-\theta^*\right\|^2 
- \left\|\theta_{t+1}-\theta^*\right\|^2\big)
+ \frac{\eta_t}{2}\left\|\nabla f_t(\theta_t)\right\|^2-\frac{\ell}{2}\left\|\theta^*-\theta_t\right\|^2
\end{array}
\end{equation*}

\noindent
Summing up from $t=1$ to $T$ with $\left\|\nabla f_t(\theta_t)\right\|^2\le G^2$, we get 
\small
\begin{equation}
\label{eq::gen_prob_static_final}
\begin{array}{ll}
\sum\limits_{t=1}^T\big(f_t(\theta_t)-f_t(\theta^*)\big)
&\le \sum\limits_{t=1}^T\frac{1}{2\eta_t}\big(\left\|\theta_t-\theta^*\right\|^2 - \left\|\theta_{t+1}-\theta^*\right\|^2\big)
 + \sum\limits_{t=1}^T\frac{\eta_t}{2}G^2-\sum\limits_{t=1}^T\frac{\ell}{2}\left\|\theta^*-\theta_t\right\|^2 \\
&\le G^2/2\sum\limits_{t=1}^T\eta_t + \frac{1/\eta_1-\ell}{2}\left\|\theta_1-\theta^*\right\|^2 
+ \frac{1}{2}\sum\limits_{t=2}^T\Big[(\frac{1}{\eta_t}-\frac{1}{\eta_{t-1}}-\ell)\left\|\theta^*-\theta_t\right\|^2\Big] 
\end{array}
\end{equation}
\normalsize
Since $\eta_t = \frac{1-\gamma}{\ell(\gamma-\gamma^t)+u(1-\gamma)}$, $1/\eta_1 = u$
and $\frac{1}{\eta_t}-\frac{1}{\eta_{t-1}}-\ell = \frac{\ell(\gamma^{t-1}-1)(1-\gamma)}{1-\gamma}\le 0$.

\noindent
For the term $\sum\limits_{t=1}^T\eta_t = \sum\limits_{t=1}^T\frac{1-\gamma}{\ell(\gamma-\gamma^t)+u(1-\gamma)}$,
it can be reformulated as
$\frac{1}{u}\sum\limits_{t=1}^T\frac{\frac{u(1-\gamma)}{\ell(\gamma-\gamma^t)}}{1+\frac{u(1-\gamma)}{\ell(\gamma-\gamma^t)}}
= \frac{1}{u} + \frac{1}{u}\sum\limits_{t=2}^T\frac{\frac{u(1-\gamma)}{\ell(\gamma-\gamma^t)}}{1+\frac{u(1-\gamma)}{\ell(\gamma-\gamma^t)}}
\le \frac{1}{u}+\frac{1}{u}\sum\limits_{t=2}^T\frac{u(1-\gamma)}{\ell(\gamma-\gamma^t)}
= \frac{1}{u} + \frac{1-\gamma}{\ell\gamma}\sum\limits_{t=2}^T\frac{1}{1-\gamma^{t-1}}
= \frac{1}{u} + \frac{1-\gamma}{\ell\gamma}\sum\limits_{t=1}^{T-1}\frac{1}{1-\gamma^{t}}$.
For $\sum\limits_{t=1}^{T-1}\frac{1}{1-\gamma^{t}}$, 
we know that $\sum\limits_{t=1}^{T-1}\frac{1}{1-\gamma^{t}} \le O(T)$ as shown in the proof of Theorem \ref{thm::quad_static_regret}.
For the term $\frac{1-\gamma}{\ell\gamma}$, $\frac{1-\gamma}{\ell\gamma} = \frac{1}{\ell(T^{\beta}-1)}$.
Combining these two terms' inequalities, 
we get that $\sum\limits_{t=1}^T\eta_t \le O(T^{1-\beta})$.

\noindent
As a result, the inequality \eqref{eq::gen_prob_static_final} can be reduced to
\begin{equation*}
\sum\limits_{t=1}^T\big(f_t(\theta_t)-f_t(\theta^*)\big) \le O(T^{1-\beta})
\end{equation*}

\end{proof}

\paragraph{Proof of Corollary~\ref{corr:strongly_dynamic_regret}:}
\begin{proof}
Since $\gamma = 1-\frac{1}{2}\sqrt{\frac{\max\{V,\log^2 T/T\}}{2DT}}$ and $V\in[0,2DT]$,
$1/2\le\gamma<1$. 

\noindent
Next, we upper bound each term on the right-hand-side of Theorem \ref{thm::strongly_regret} individually. 
$\frac{1}{1-\gamma}V=2\sqrt{\frac{2DT}{\max\{V,\log^2 T/T\}}}V\le O(\sqrt{TV})$.
In order to bound the second term, 
Lemma~\ref{lem:integral} implies that
$(1-\gamma)\sum\limits_{t=1}^T\frac{1}{1-\gamma^t}\le 1+(1-\gamma)(T +
\frac{\ln(1-\gamma)}{\ln\gamma})$.

\noindent
In this case, the logarithm terms can be bounded by:
\begin{equation*}
\begin{array}{ll}
\frac{\ln(1-\gamma)}{\ln\gamma} 
&= \frac{-\ln (\frac{1}{2}\sqrt{\frac{\max\{V,\log^2 T/T\}}{2DT}})}{-\ln (1-\frac{1}{2}\sqrt{\frac{\max\{V,\log^2 T/T\}}{2DT}})}
= \frac{-\ln (\frac{1}{2}\sqrt{\frac{\max\{V,\log^2 T/T\}}
{2DT}})}{\ln\Big(1+\frac{\frac{1}{2}\sqrt{\frac{\max\{V,\log^2 T/T\}}{2DT}}}{1-\frac{1}{2}\sqrt{\frac{\max\{V,\log^2 T/T\}}{2DT}}}\Big)}\\
&\le \ln (2\sqrt{\frac{2DT}{\max\{V,\log^2 T/T\}}})4\sqrt{\frac{2DT}{\max\{V,\log^2 T/T\}}}
\le O(\ln(T/\log T)\frac{T}{\log T} )
\le O(T) 
\end{array}
\end{equation*}
where the first inequality follows by $\ln(1+x)\ge \frac{1}{2}x, x\in[0,1]$, 
and $1-\frac{1}{2}\sqrt{\frac{\max\{V,\log^2 T/T\}}{2DT}}<1$.

\noindent
Thus, $(1-\gamma)\sum\limits_{t=1}^T\frac{1}{1-\gamma^t}\le \max\{O(\log T),O(\sqrt{TV})\}$.
The final result follows by adding the two terms.
\end{proof}

\paragraph{Proof of Lemma~\ref{lem:meta-expert-compare}:}
\begin{proof}
The first part of the proof is the same as the first part of the result in the Proof of Lemma 1 in \cite{zhang2018adaptive},
which follows methods of \cite{cesa2006prediction}.
We define $L_t^\gamma = \sum\limits_{i=1}^tf_i(\theta_i^\gamma)$, 
and $W_t = \sum\limits_{\gamma\in\cH}w_1^\gamma \exp(-\alpha L_t^\gamma)$.

\noindent
The following update is equivalent to the update rule in Algorithm \ref{alg:meta}:
\begin{equation}
\label{eq:expert_reform}
w_t^\gamma = \frac{w_1^\gamma \exp(-\alpha L_{t-1}^\gamma)}
              {\sum\limits_{\mu\in\cH} w_1^\mu \exp(-\alpha L_{t-1}^\mu)}, \quad t\ge 2.
\end{equation} 

\noindent
First, we have 
\begin{equation}
\label{eq:logW_lower}
\begin{array}{ll}
\log W_T &= \log\big(\sum\limits_{\gamma\in\cH}w_1^\gamma\exp(-\alpha L_T^\gamma)\big) \\
&\ge \log\big(\max\limits_{\gamma\in\cH}w_1^\gamma\exp(-\alpha L_T^\gamma)\big)
=-\alpha \min\limits_{\gamma\in\cH}\big(L_T^\gamma+\frac{1}{\alpha}\log\frac{1}{w_1^\gamma}\big).
\end{array}
\end{equation}

Then we bound the quantity $\log(W_t/W_{t-1})$. 
For $t\ge2$, we get
\begin{equation}
\begin{array}{ll}
\log\Big(\frac{W_t}{W_{t-1}}\Big)
&= \log\Big(\frac{\sum_{\gamma\in\cH}w_1^\gamma\exp(-\alpha L_t^\gamma)}
        {\sum_{\gamma\in\cH}w_1^\gamma\exp(-\alpha L_{t-1}^\gamma)}\Big)\\
&= \log\Big(\frac{\sum_{\gamma\in\cH}w_1^\gamma\exp(-\alpha L_{t-1}^\gamma)\exp(-\alpha f_t(\theta_t^\gamma))}
        {\sum_{\gamma\in\cH}w_1^\gamma\exp(-\alpha L_{t-1}^\gamma)}\Big)\\
&=\log\Big(\sum\limits_{\gamma\in\cH}w_t^\gamma\exp(-\alpha f_t(\theta_t^\gamma))\Big)

\end{array}
\end{equation}
where the last equality is due to Eq.\eqref{eq:expert_reform}.

\noindent
When $t=1$, $\log W_1 = \log\Big(\sum\limits_{\gamma\in\cH}w_1^\gamma\exp(-\alpha f_1(\theta_1^\gamma))\Big)$.
Then $\log W_T$ can be expressed as:
\begin{equation}
\label{eq::logW_upper}
\begin{array}{ll}
\log W_T = \log W_1 + \sum\limits_{t=2}^T\log\Big(\frac{W_t}{W_{t-1}}\Big) 
= \sum\limits_{t=1}^T\log\Big(\sum\limits_{\gamma\in\cH}w_t^\gamma\exp(-\alpha f_t(\theta_t^\gamma))\Big).
\end{array}
\end{equation}

\noindent
The rest of the proof is new.

\noindent
Due to the $\alpha$-exp-concavity,
$\exp(-\alpha f_t(\sum_{\gamma\in\cH}w_t^\gamma \theta_t^\gamma))
\ge \sum_{\gamma\in\cH}w_t^\gamma \exp(-\alpha f_t(\theta_t^\gamma))$,
which is equivalent to 
\begin{equation}
\label{eq::exp_concave_ineq}
\begin{array}{ll}
\log\Big(\sum_{\gamma\in\cH}w_t^\gamma \exp(-\alpha f_t(\theta_t^\gamma))\Big) 
\le -\alpha f_t\Big(\sum_{\gamma\in\cH}w_t^\gamma \theta_t^\gamma\Big) 
=-\alpha f_t(\theta_t)
\end{array}
\end{equation}

\noindent
Combining the Inequalities \eqref{eq:logW_lower}, \eqref{eq::logW_upper}, and \eqref{eq::exp_concave_ineq},
we get 
\begin{equation*}
-\alpha \min\limits_{\gamma\in\cH}\big(L_T^\gamma+\frac{1}{\alpha}\log\frac{1}{w_1^\gamma}\big) 
\le -\alpha\sum_{t=1}^T f_t(\theta_t)
\end{equation*}
which can be reformulated as 
\begin{equation*}
\sum_{t=1}^T f_t(\theta_t)\le\min\limits_{\gamma\in\cH}\Big(\sum_{t=1}^Tf_t(\theta_t^\gamma)+\frac{1}{\alpha}\log\frac{1}{w_1^\gamma}\Big)
\end{equation*}

\noindent
Since it holds for the minimum value, it is true for all $\gamma\in\cH$,
which completes the proof.

\end{proof}

\paragraph{Proof of Theorem~\ref{thm:mega_exp-concave}:}
\begin{proof}

When $\gamma = \gamma^* = 
1-\frac{1}{2}\frac{\log T}{T}\sqrt{\frac{\max\{\frac{T}{\log^2 T}V,1\}}{2D}}
=1-\eta^*$, we have $\sum_{t=1}^T (f_t(\theta_t^{\gamma^*})-f_t(z_t))\le \max\{O(\log T), O(\sqrt{TV})\}$
based on the Corollary \ref{cor:logBounds}.

\noindent
Since $0\le V\le 2TD$, $\frac{1}{2}\frac{\log T}{T\sqrt{2D}}\le\eta^*\le \frac{1}{2}$.

\noindent
According to our definition of $\eta_i$, $\min \eta_i = \frac{1}{2}\frac{\log T}{T\sqrt{2D}}$
and $\frac{1}{2}\le \max \eta_i< 1$, 
which means for any value of $V$, there always exists a $\eta_k$ such that
\begin{equation*}
\eta_k = \frac{1}{2}\frac{\log T}{T\sqrt{2D}}2^{k-1}\le \eta^*\le 2\eta_k = \eta_{k+1}
\end{equation*}
where $k = \lfloor \frac{1}{2}\log_2 (\max\{\frac{T}{\log^2 T}V,1\})\rfloor+1$.

\noindent
Now we claim that that running the algorithm with $\gamma_k$ incurs at
most a constant factor increase in the dynamic regret. 

\noindent
Since $0<\eta_k\le \frac{1}{2}$, $\frac{1}{2}\le\gamma_k = 1-\eta_k<1$ and $\gamma_k\ge\gamma^*$.

\noindent
According to Theorem \ref{thm:expConcaveThm}, we have 
  \begin{equation*}
  \begin{array}{ll}
    \sum_{t=1}^T (f_t(\theta_t^{\gamma_k})-f_t(z_t)) &\le -a_1 T \log \gamma_k -a_2\log(1-\gamma_k)
     \quad+ \frac{a_3}{1-\gamma_k} V + a_4.
   \end{array}
  \end{equation*}

\noindent
  Now we bound each term of the regret in terms of the value obtained
  by using $\gamma^*$.
For the first term on the RHS, $-T\log\gamma_k = T\log\frac{1}{\gamma_k}\le T\log\frac{1}{\gamma^*}$.

\noindent
For the second one, $-\log(1-\gamma_k) = -\log\frac{1}{2}(2-2\gamma_k) = -\log\frac{1}{2}2\eta_k$.
Since $1\ge2\eta_k\ge\eta^*$, $\frac{1}{2}2\eta_k\ge\frac{1}{2}\eta^*$,
which leads to $-\log\frac{1}{2}2\eta_k\le -\log\frac{1}{2}\eta^*$
and $-\log(1-\gamma_k)\le -\log\frac{1}{2}\eta^* = \log2 -\log(1-\gamma^*)$.

\noindent
For the third one, $\frac{1}{1-\gamma_k} V 
= \frac{1}{\eta_k}V = \frac{2}{2\eta_k}V\le \frac{2}{\eta^*}V
=\frac{2}{1-\gamma^*}V$.
Thus the claim has been proved.

\noindent
Since using $\gamma_k$ in place of $\gamma^*$ increases the
regret by at most a constant factor, Corollary \ref{cor:logBounds}
implies that:
\begin{equation}
\label{eq:expert_k_regret}
\sum_{t=1}^T (f_t(\theta_t^{\gamma_k})-f_t(z_t))\le \max\{O(\log T), O(\sqrt{TV})\}
\end{equation}

\noindent
Furthermore, from Lemma \ref{lem:meta-expert-compare} we get 
\begin{equation}
\label{eq:expert_k_comparable}
\begin{array}{ll}
\sum\limits_{t=1}^T (f_t(\theta_t) - f_t(\theta_t^{\gamma_k}))
&\le \frac{1}{\alpha}\log\frac{1}{w_1^{\gamma_k}}
\le \frac{1}{\alpha}\log(k(k+1)) \\
&\le 2\frac{1}{\alpha}\log(k+1) 
\le O(\log(\log T))
\end{array}
\end{equation}

\noindent
Combining the above inequalities \eqref{eq:expert_k_regret} and \eqref{eq:expert_k_comparable} completes the proof.
\end{proof}

\paragraph{Proof of Lemma~\ref{lem::strongly_is_exp}:}
\begin{proof}
Let $g(x) = \exp(-\alpha f(x))$. To prove the concavity of $g(x)$,
it is equivalent to show $\langle\nabla g(x)-\nabla g(y),x-y\rangle\le 0,x,y\in\cS$.
Since $\nabla g(x) = \exp(-\alpha f(x))(-\alpha)\nabla f(x)$,
it is equivalent to prove that 
$\langle \exp(-\alpha f(x))\nabla f(x)-\exp(-\alpha f(y))\nabla f(y),x-y\rangle\ge 0$,
which can be reformulated as
\begin{equation}
\label{eq:strong_is_exp_exp}
\exp(-\alpha f(x))\langle \nabla f(x),x-y\rangle \ge \exp(-\alpha f(y))\langle \nabla f(y),x-y\rangle
\end{equation}

\noindent
Without loss of generality, let us assume $f(x)\ge f(y)$.
Due to $\ell$-strong convexity, $f(x)\ge f(y) + \langle \nabla f(y),x-y\rangle+\frac{\ell}{2}\|x-y\|^2$,
which leads to 
\begin{equation}
\label{eq:strong_is_exp_p1}
\langle \nabla f(y),x-y\rangle \le f(x)-f(y)-\frac{\ell}{2}\|x-y\|^2
\end{equation}

\noindent
What's more, $f(y)\ge f(x) + \langle \nabla f(x),y-x\rangle+\frac{\ell}{2}\|x-y\|^2$,
which leads to
\begin{equation}
\label{eq:strong_is_exp_p2}
\langle \nabla f(x),x-y\rangle \ge f(x)-f(y)+\frac{\ell}{2}\|x-y\|^2
\end{equation}

\noindent
The examination of the inequalities \eqref{eq:strong_is_exp_exp}, \eqref{eq:strong_is_exp_p1}, and \eqref{eq:strong_is_exp_p2}
shows that
it is enough to prove
$\exp(-\alpha f(x))(f(x)-f(y)+\frac{\ell}{2}\|x-y\|^2)\ge \exp(-\alpha f(y))(f(x)-f(y)-\frac{\ell}{2}\|x-y\|^2)$,
which can be reformulated as
$\frac{\ell}{2}\|x-y\|^2(\exp(-\alpha f(x))+\exp(-\alpha f(y)))
\ge (f(x)-f(y))(\exp(-\alpha f(y))-\exp(-\alpha f(x)))$.
When $x-y = 0$, it is always true. Let us consider the case when $\|x-y\|>0$.
Then we need to show that
$\frac{\ell}{2}\Big(1+\exp\big(\alpha\big(f(x)-f(y)\big)\big)\Big)
\ge \frac{f(x)-f(y)}{\|x-y\|}\frac{\exp\Big(\alpha\big(f(x)-f(y)\big)\Big)-1}{\|x-y\|}$.
Due to bounded gradient and Mean value theorem,$\frac{f(x)-f(y)}{\|x-y\|}\le G$,
which means it is enough to show that
\small
\begin{equation}
\label{eq:strong_is_exp_p3}
\frac{\ell}{2G}\Big(1+\exp\big(\alpha\big(f(x)-f(y)\big)\big)\Big)
\ge\frac{\exp\Big(\alpha\big(f(x)-f(y)\big)\Big)-1}{\|x-y\|}
\end{equation}
\normalsize

\noindent
According to the Taylor series, 
$\exp\Big(\alpha\big(f(x)-f(y)\big)\Big) 
= 1 + \alpha\big(f(x)-f(y)\big)+\frac{1}{2!}\alpha^2\big(f(x)-f(y)\big)^2
+\dots+\frac{1}{n!}\alpha^n\big(f(x)-f(y)\big)^n,n\to \infty$.
Thus, $\frac{\exp\Big(\alpha\big(f(x)-f(y)\big)\Big)-1}{\|x-y\|}
 = \alpha \frac{f(x)-f(y)}{\|x-y\|}+\frac{1}{2}\alpha^2(f(x)-f(y))\frac{f(x)-f(y)}{\|x-y\|}
 +\dots+\frac{1}{n!}\alpha^n\big(f(x)-f(y)\big)^{n-1}\frac{f(x)-f(y)}{\|x-y\|},n\to \infty$.
 Since $\frac{f(x)-f(y)}{\|x-y\|}\le G$, we have
 \begin{equation}
 \label{eq:strong_is_exp_f1}
 \begin{array}{l}
 \frac{\exp\Big(\alpha\big(f(x)-f(y)\big)\Big)-1}{\|x-y\|}
 \le \alpha G+\frac{1}{2}\alpha^2(f(x)-f(y)) G +\dots
 +\frac{1}{n!}\alpha^n\big(f(x)-f(y)\big)^{n-1}G
 \end{array}
 \end{equation}

\noindent
For the LHS of inequality \eqref{eq:strong_is_exp_p3}, it is equal to
\begin{equation}
\label{eq:strong_is_exp_f2}
\begin{array}{l}
\frac{\ell}{G}+\alpha\frac{\ell}{2G}(f(x)-f(y))
+\frac{1}{2!}\alpha^2\frac{\ell}{2G}(f(x)-f(y))^2
+\dots
+\frac{1}{n!}\alpha^n\frac{\ell}{2G}(f(x)-f(y))^n,n\to \infty
\end{array}
\end{equation}

\noindent
If we compare the coefficients of the RHS from the inequality \eqref{eq:strong_is_exp_f1} with the one in \eqref{eq:strong_is_exp_f2}
and plug in $\alpha = \ell/G^2$, we see that it is always smaller or equal,
which completes the proof.

\end{proof}

\paragraph{Proof of Theorem~\ref{thm:mega_strongly}:}
\begin{proof}
As in the proof of Theorem \ref{thm:mega_exp-concave}, all we need to show is that
there exists an algorithm $A^\gamma$, 
which can bound the regret $\sum_{t=1}^T (f_t(\theta_t^\gamma)-f_t(z_t)) 
\le \max\{O(\log T), O(\sqrt{TV})\}$.

\noindent
When $\gamma = \gamma^* = 
1-\frac{1}{2}\frac{\log T}{T}\sqrt{\frac{\max\{\frac{T}{\log^2 T}V,1\}}{2D}}
=1-\eta^*$, we have $\sum_{t=1}^T (f_t(\theta_t^{\gamma^*})-f_t(z_t))\le \max\{O(\log T), O(\sqrt{TV})\}$
based on the Corollary \ref{corr:strongly_dynamic_regret}.

\noindent
Since $0\le V\le 2TD$, $\frac{1}{2}\frac{\log T}{T\sqrt{2D}}\le\eta^*\le \frac{1}{2}$.

\noindent
According to our definition of $\eta_i$, $\min \eta_i = \frac{1}{2}\frac{\log T}{T\sqrt{2D}}$
and $\frac{1}{2}\le \max \eta_i< 1$, 
which means for any value of $V$, there always exists a $\eta_k$ such that
\begin{equation*}
\eta_k = \frac{1}{2}\frac{\log T}{T\sqrt{2D}}2^{k-1}\le \eta^*\le 2\eta_k = \eta_{k+1}
\end{equation*}
where $k = \lfloor \frac{1}{2}\log_2 (\max\{\frac{T}{\log^2 T}V,1\})\rfloor+1$.

\noindent
Since $0<\eta_k\le \frac{1}{2}$, $\frac{1}{2}\le\gamma_k = 1-\eta_k<1$ and $\gamma_k\ge\gamma^*$.

\noindent
According to Theorem \ref{thm::strongly_regret}, we have 
\begin{equation*}
\sum\limits_{t=1}^T \big(f_t(\theta_t^{\gamma_k})-f_t(z_t)\big) 
\le  \frac{2D\ell}{1-\gamma_k}V + \frac{G^2}{\ell}(1-\gamma_k)\sum\limits_{t=1}^T\frac{1}{1-\gamma_k^t}
\end{equation*}

\noindent
For the first term on the RHS, $\frac{1}{1-\gamma_k} V
= \frac{1}{\eta_k}V = \frac{2}{2\eta_k}V\le \frac{2}{\eta^*}V
=\frac{2}{1-\gamma^*}V$.

\noindent
For the second one, $1-\gamma_k\le 1-\gamma^*$. 
According to the proof in Corollary \ref{corr:strongly_dynamic_regret},
$\sum\limits_{t=1}^T\frac{1}{1-\gamma_k^t} \le \frac{1}{1-\gamma_k}+T + \frac{\log(1-\gamma_k)}{\log\gamma_k}$.
\begin{equation}
\label{eq:coef_strongly1}
\frac{\log(1-\gamma_k)}{\log\gamma_k} = \frac{\log \eta_k}{\log (1-\eta_k)}
= \frac{-\log \eta_k}{-\log (1-\eta_k)}. 
\end{equation}
Since $\eta_k\ge \frac{1}{2}\eta^*$,
$\log\eta_k\ge\log\frac{1}{2}\eta^*$ 
and 
\begin{equation}
\label{eq:coef_strongly2}
0<-\log\eta_k\le-\log\frac{1}{2}\eta^* = \log 2-\log \eta^*.
\end{equation}
Since $\eta_k\ge \frac{1}{2}\eta^*$, $1-\eta_k\le 1-\frac{1}{2}\eta^*$.
Then $\log(1-\eta_k)\le \log(1-\frac{1}{2}\eta^*)$,
which results in 
\begin{equation}
\label{eq:coef_strongly3}
-\log(1-\eta_k)\ge -\log(1-\frac{1}{2}\eta^*)>0.
\end{equation}
Combining inequalities \eqref{eq:coef_strongly2} and \eqref{eq:coef_strongly3} with Eq.\eqref{eq:coef_strongly1},
we get
\begin{equation}
\label{eq:integral_mega}
\begin{array}{ll}
\frac{\log(1-\gamma_k)}{\log\gamma_k}
\le \frac{\log 2-\log \eta^*}{-\log(1-\frac{1}{2}\eta^*)} 
 = \frac{\log 2}{-\log(1-\frac{1}{2}\eta^*)} + \frac{-\log \eta^*}{-\log(1-\frac{1}{2}\eta^*)}
\end{array}
\end{equation}

\noindent
For the first term on the RHS,
\begin{equation*} 
\begin{array}{ll}
-\log(1-\frac{1}{2}\eta^*)&=\log\Big(\frac{1}{1-\frac{1}{4}\sqrt{\frac{\max\{V,\log^2 T/T\}}{2DT}}}\Big)
= \log\Big(1+\frac{\frac{1}{4}\sqrt{\frac{\max\{V,\log^2 T/T\}}{2DT}}}{1-\frac{1}{4}\sqrt{\frac{\max\{V,\log^2 T/T\}}{2DT}}}\Big)\\
&\ge \frac{1}{2}\frac{\frac{1}{4}\sqrt{\frac{\max\{V,\log^2 T/T\}}{2DT}}}{1-\frac{1}{4}\sqrt{\frac{\max\{V,\log^2 T/T\}}{2DT}}}
\ge \frac{1}{8}\sqrt{\frac{\max\{V,\log^2 T/T\}}{2DT}}
\end{array}
\end{equation*}
where the first inequality is due to $\log(1+x)\ge\frac{1}{2}x,x\in[0,1]$ and
the second one is due to $\sqrt{\frac{\max\{V,\log^2 T/T\}}{2DT}}>0$.
As a result, 
\begin{equation*}
\begin{array}{ll}
\frac{\log 2}{-\log(1-\frac{1}{2}\eta^*)}
&\le 8 \sqrt{\frac{2DT}{\max\{V,\log^2 T/T\}}}\log 2
\le 8\frac{T}{\log T}\sqrt{2D}\log 2<O(T).
\end{array}
\end{equation*}

\noindent
For the second term on the RHS of Eq.\eqref{eq:integral_mega}, 
\begin{equation*}
\begin{array}{ll}
-\log \eta^* = \log\Big(2\sqrt{\frac{2DT}{\max\{V,\log^2 T/T\}}}\Big)
\le \log 2 +\frac{1}{2}\log 2D + \frac{1}{2}\log \frac{T}{\log T}
\end{array}
\end{equation*}

\noindent
Combining the inequalities for $-\log \eta^*$ and $-\log(1-\frac{1}{2}\eta^*)$,
we get 
$\frac{-\log \eta^*}{-\log(1-\frac{1}{2}\eta^*)} 
\le (\log 2 +\frac{1}{2}\log 2D + \frac{1}{2}\log \frac{T}{\log T})8\frac{T}{\log T}\sqrt{2D}
\le O(T)$.

\noindent
As a result, $\frac{\log(1-\gamma_k)}{\log\gamma_k}\le O(T)$ and 
$\sum\limits_{t=1}^T\frac{1}{1-\gamma_k^t} \le O(T)$.

\noindent
Since using $\gamma_k$ does not increase the order when used in place of $\gamma^*$, 
we get 
\begin{equation*}
\sum\limits_{t=1}^T \Big(f_t(\theta_t^{\gamma_k})-f_t(z_t)\Big) \le \max\{O(\log T), O(\sqrt{TV})\}
\end{equation*}
which combining with the result of Lemma \ref{lem:meta-expert-compare} completes the proof.

\end{proof}

\paragraph{Online Least-Squares Optimization}

Consider the online least-squares problem with:
\begin{equation}
\label{eq::gen_ls_loss}
f_t(\theta) = \frac{1}{2}\left\|y_t - A_t\theta\right\|^2
\end{equation}
where $A_t\in\mathbb{R}^{m\times n}$, $A_t^\top A_t$ has full rank with $lI\preceq A_t^\top A_t\preceq uI$, 
and $y_t\in\mathbb{R}^m$ comes from a bounded set with
$\left\|y_t\right\|\le D$.

\noindent
In Chapter \ref{chap:trade-off}, we analyzed the dynamic regret of discounted recursive least squares against comparison
sequences $z_1,\ldots,z_T$ with a path length constraint
$\sum_{t=2}^T\|z_t-z_{t-1}\| \le V$. Additionally, we analyzed the trade-off between static and
dynamic regret of a gradient descent rule with comparison sequence
$\theta_t^* = \argmin_{\theta \in \cS} f_t(\theta)$. In this appendix,
we analyze the trade-off between static regret and dynamic regret with
comparison sequence $\theta_t^*$ achieved by discounted recursive
least squares. We will see that the discounted recursive least squares
achieves trade-offs dependent on the condition number, $\delta = u/l$. In
particular, low dynamic regret is only guaranteed for low condition
numbers. 

\noindent
Recall that discounted recursive least squares corresponds to
Algorithm~\ref{alg:discountedNewton} running with a full Newton step and $\eta
= 1$.
In this case, $P_t =
\sum\limits_{i=1}^t\gamma^{i-1}A_{t+1-i}^\top A_{t+1-i} = \gamma
P_{t-1}+A_t^\top A_t$, and the update rule can be written more explicitly
as 
\begin{equation}
\label{eq::org_gen_ls_update}
\theta_{t+1} = \Big(\sum\limits_{i=1}^t\gamma^{i-1}A_{t+1-i}^\top A_{t+1-i}\Big)^{-1}\Big(\sum\limits_{i=1}^t \gamma^{i-1}A_{t+1-i}^\top y_{t+1-i}\Big)
\end{equation}
The above update rule can be reformulated as:
\begin{equation}
\label{eq::gen_ls_update}
\theta_{t+1} = \theta_t - P_t^{-1}\nabla f_t(\theta_t).
\end{equation}

\noindent
Before we analyze dynamic and static regret for the update \eqref{eq::gen_ls_update}, 
we first show some supporting results for $\left\|y_t-A_tx\right\|$ 
and $\left\|\nabla f_t(x)\right\|$, where $x\in \{v| v = \beta \theta_t + (1-\beta)\theta_t^*,\beta\in[0,1]\}$.

\begin{lemma}
  \label{lem:norm_gen_ls_dif}
  {\it
  Let $\theta_t$ be the result of Eq.\eqref{eq::gen_ls_update}, and $\theta_t^* = \argmin f_t(\theta)$.
  For $x\in \{v| v = \beta \theta_t + (1-\beta)\theta_t^*,\beta\in[0,1]\}$,
  If $\left\|y_t\right\|\le D$, then $\left\|y_t-A_tx\right\|\le (u/l +1)D$.
  }
\end{lemma}

\begin{proof}

$\left\|y_t-A_tx\right\| \le \left\|A_t\right\|_2\left\|x\right\|+\left\|y_t\right\|$, and 
$\left\|A_t\right\|_2 = \sqrt{\sigma_1(A_t^\top A_t)}\le \sqrt{u}$. 
For $\left\|x\right\|$, we have
$\left\|x\right\| = \left\|\beta\theta_t+(1-\beta)\theta_t^*\right\|\le \beta\left\|\theta_t\right\|+(1-\beta)\left\|\theta_t^*\right\|$.

\noindent
For the term $\left\|\theta_t\right\|$,
$\left\|\theta_t\right\| = \left\|\Big(\sum\limits_{i=1}^{t-1}\gamma^{i-1}A_{t-i}^\top A_{t-i}\Big)^{-1}
\Big(\sum\limits_{i=1}^{t-1} \gamma^{i-1}A_{t-i}^\top y_{t-i}\Big)\right\|$,
which can be upper bounded by 
$\left\|\Big(\sum\limits_{i=1}^{t-1}\gamma^{i-1}A_{t-i}^\top A_{t-i}\Big)^{-1}\right\|_2
\left\|\Big(\sum\limits_{i=1}^{t-1} \gamma^{i-1}A_{t-i}^\top y_{t-i}\Big)\right\|$.
Then we upper bound these two terms individually.

\noindent
$\left\|\Big(\sum\limits_{i=1}^{t-1}\gamma^{i-1}A_{t-i}^\top A_{t-i}\Big)^{-1}\right\|_2 
= \frac{1}{\sigma_n(\sum\limits_{i=1}^{t-1}\gamma^{i-1}A_{t-i}^\top A_{t-i})}$. 
Since $lI\preceq A_{t-i}^\top A_{t-i}\preceq uI$, 
$\frac{1-\gamma^{t-1}}{1-\gamma}lI\preceq \sum\limits_{i=1}^{t-1}\gamma^{i-1}A_{t-i}^\top A_{t-i})\preceq \frac{1-\gamma^{t-1}}{1-\gamma}uI$.
Thus, $\sigma_n(\sum\limits_{i=1}^{t-1}\gamma^{i-1}A_{t-i}^\top A_{t-i})\ge l\frac{1-\gamma^{t-1}}{1-\gamma}$,
which results in $\left\|\Big(\sum\limits_{i=1}^{t-1}\gamma^{i-1}A_{t-i}^\top A_{t-i}\Big)^{-1}\right\|_2 
\le \frac{1-\gamma}{l(1-\gamma^{t-1})}$.

\noindent
For the term $\left\|\Big(\sum\limits_{i=1}^{t-1} \gamma^{i-1}A_{t-i}^\top y_{t-i}\Big)\right\|$, 
we have $\left\|\Big(\sum\limits_{i=1}^{t-1} \gamma^{i-1}A_{t-i}^\top y_{t-i}\Big)\right\|
\le \sum\limits_{i=1}^{t-1}\gamma^{i-1}\left\|A_{t-i}^\top y_{t-i}\right\|
\le \sum\limits_{i=1}^{t-1}\gamma^{i-1}\left\|A_{t-i}^\top \right\|_2\left\|y_{t-i}\right\|
\le \frac{1-\gamma^{t-1}}{1-\gamma}\sqrt{u}D$.
Then we have $\left\|\theta_t\right\|\le \frac{\sqrt{u}}{l}D$.

\noindent
For $\left\|\theta_t^*\right\|$,
we have $\left\|\theta_t^*\right\| = \left\|(A_t^\top A_t)^{-1}A_t^\top y_t\right\|
\le \left\|(A_t^\top A_t)^{-1}\right\|_2\left\|A_t^\top \right\|_2\left\|y_t\right\|
\le \frac{\sqrt{u}}{l}D$. Thus, $\left\|x\right\| \le \frac{\sqrt{u}}{l}D$
and $\left\|y_t-A_tx\right\| \le \left\|A_t\right\|_2\left\|x\right\|+\left\|y_t\right\|
\le (u/l+1)D$.

\end{proof}

\begin{corollary}
  \label{corol:norm_gen_ls_grad}
  {\it
  Let $\theta_t$ be the result of Eq.\eqref{eq::gen_ls_update} and $\theta_t^* = \argmin f_t(\theta)$.
  For $x\in \{v| v = \beta \theta_t + (1-\beta)\theta_t^*,\beta\in[0,1]\}$, we have 
  $\left\|\nabla f_t(x)\right\| \le \sqrt{u}(u/l+1)D$.
  }
\end{corollary}

\begin{proof}

For $\left\|\nabla f_t(x)\right\|$, we have $\left\|\nabla f_t(x)\right\| = \left\|A_t^\top A_tx-A_t^\top y_t\right\|
\le \left\|A_t^\top \right\|_2\left\|A_tx-y_t\right\|\le \sqrt{u}(u/l+1)D$,
where the second inequality is due to Lemma \ref{lem:norm_gen_ls_dif} and the assumption of $A_t^\top A_t\preceq uI$.

\end{proof}

\noindent
Moreover, we need to obtain the relationship between $\theta_{t+1}-\theta_t^*$ and $\theta_t-\theta_t^*$ 
as another necessary step to get the dynamic regret.

\begin{lemma}
  \label{lem:gen_ls_var_path}
  {\it
  Let $\theta_t^*$ be the solution to $f_t(\theta)$ in Eq.\eqref{eq::gen_ls_loss}.
  When we use the discounted recursive least-squares update in Eq.\eqref{eq::gen_ls_update},
  the following relationship is obtained:
  \begin{equation*}
  \begin{array}{ll}
  \theta_{t+1} -\theta_t^*
   &= \big(I-\gamma^{-1}P_{t-1}^{-1}A_t^\top(I+A_t\gamma^{-1}P_{t-1}^{-1}A_t^\top)^{-1}A_t\big)(\theta_t-\theta_t^*) \\
  &= \big(I+\gamma^{-1}P_{t-1}^{-1}A_t^\top A_t\big)^{-1}(\theta_t-\theta_t^*) 
  \end{array}
  \end{equation*}
  }
\end{lemma}

\begin{proof}

If we set $\Phi_t = \sum\limits_{i=1}^t\gamma^{i-1}A_{t+1-i}^\top y_{t+1-i} = \gamma \Phi_{t-1}+A_t^\top y_t$,
then according to the update of $\theta_{t+1}$ in Eq.\eqref{eq::org_gen_ls_update}, 
we have $\theta_{t+1} = (A_t^\top A_t+\gamma P_{t-1})^{-1}(A_t^\top y_t+\gamma\Phi_{t-1})$,
which by the use of inverse lemma can be further reformulated as:
\begin{equation}
\begin{array}{ll}
\theta_{t+1} = \Big(\gamma^{-1}P_{t-1}^{-1}- \gamma^{-2}P_{t-1}^{-1}A_t^\top(I+
 A_t\gamma^{-1}P_{t-1}^{-1}A_t^\top)^{-1}A_tP_{t-1}^{-1}\Big)\big(A_t^\top y_t+\gamma\Phi_{t-1}\big)
\end{array}
\end{equation}
Then for $\theta_{t+1}-\theta_t^* = \theta_{t+1}-(A_t^\top A_t)^{-1}A_t^\top y_t$, we have:
\begin{equation}
\begin{array}{ll}
\theta_{t+1}-\theta_t^* 
&= \underbrace{\big(I-\gamma^{-1}P_{t-1}^{-1}A_t^\top(I+A_t\gamma^{-1}P_{t-1}^{-1}A_t^\top)^{-1}A_t\big)}_{\circled{1}}\theta_t
+\underbrace{\gamma^{-1}P_{t-1}^{-1}A_t^\top y_t}_{\circled{2.1}} \\
&\quad-\underbrace{\big(\gamma^{-2}P_{t-1}^{-1}A_t^\top(I+A_t\gamma^{-1}P_{t-1}^{-1}A_t^\top)^{-1}A_tP_{t-1}^{-1}
-(A_t^\top A_t)^{-1}\big)A_t^\top y_t}_{\circled{2.2}}
\end{array}
\end{equation}
We want to prove $\circled{2.1}+\circled{2.2} = \circled{1}(-\theta_t^*) = \circled{1}(-(A_t^\top A_t)^{-1}A_t^\top y_t) = \circled{3}$.

\noindent
Since $A(I+BA)^{-1}B = AB(I+AB)^{-1}=(I+AB)^{-1}AB$, for any compatible matrix $A$ and $B$, we have:
\begin{equation}
\begin{array}{ll}
\circled{3}
&= -\big[I-\gamma^{-1}P_{t-1}^{-1}A_t^\top(I+A_t\gamma^{-1}P_{t-1}^{-1}A_t^\top)^{-1}A_t\big](A_t^\top A_t)^{-1}A_t^\top y_t \\
&= -\big[I-(I+\gamma^{-1}P_{t-1}^{-1}A_t^\top A_t)^{-1}\gamma^{-1}P_{t-1}^{-1}A_t^\top A_t\big](A_t^\top A_t)^{-1}A_t^\top y_t \\
&= -\big[(A_t^\top A_t)^{-1} - (I+\gamma^{-1}P_{t-1}^{-1}A_t^\top A_t)^{-1}\gamma^{-1}P_{t-1}^{-1}\big]A_t^\top y_t
\end{array}
\end{equation}
Also, for any compatible $P$, we have $(I+P)^{-1} = I-(I+P)^{-1}P$. 
Then $(I+\gamma^{-1}P_{t-1}^{-1}A_t^\top A_t)^{-1} = I - (I+\gamma^{-1}P_{t-1}^{-1}A_t^\top A_t)^{-1}\gamma^{-1}P_{t-1}^{-1}A_t^\top A_t$.
Then $\circled{3} = -\big[(A_t^\top A_t)^{-1}-\gamma^{-1}P_{t-1}^{-1}
+(I+\gamma^{-1}P_{t-1}^{-1}A_t^\top A_t)^{-1}\gamma^{-2}P_{t-1}^{-1}A_t^\top A_tP_{t-1}^{-1}\big]A_t^\top y_t$.
Compared with $\circled{2.1}+\circled{2.2}$, we need to prove
$(I+\gamma^{-1}P_{t-1}^{-1}A_t^\top A_t)^{-1}\gamma^{-2}P_{t-1}^{-1}A_t^\top A_tP_{t-1}^{-1}$
is equal to
$\gamma^{-2}P_{t-1}^{-1}A_t^\top (I+A_t\gamma^{-1}P_{t-1}^{-1}A_t^\top)^{-1}A_tP_{t-1}^{-1}$, which is always true.

\noindent
As a result, we have $\theta_{t+1} -\theta_t^* 
= \big(I-\gamma^{-1}P_{t-1}^{-1}A_t^\top(I+A_t\gamma^{-1}P_{t-1}^{-1}A_t^\top)^{-1}A_t\big)(\theta_t-\theta_t^*)$,
which can be simplified as $\theta_{t+1} -\theta_t^* 
=\big(I+\gamma^{-1}P_{t-1}^{-1}A_t^\top A_t\big)^{-1}(\theta_t-\theta_t^*)$.

\end{proof}

\begin{corollary}
  \label{corol:gen_ls_var_path_norm}
  {\it
  Let $\theta_t^*$ be the solution to $f_t(\theta)$ in Eq.\eqref{eq::gen_ls_loss}.
  When we use the discounted recursive least-squares update in Eq.\eqref{eq::gen_ls_update},
  the following relation is obtained:
  \begin{equation*}
  \begin{array}{ll}
  \left\|\theta_{t+1} -\theta_t^*\right\| & \le \sqrt{\frac{u}{l}}\frac{u\gamma}{u\gamma+l(1-\gamma)}\left\|\theta_t-\theta_t^*\right\|
  \end{array}
  \end{equation*}
  }
\end{corollary}

\begin{proof}

From Lemma \ref{lem:gen_ls_var_path} we know that
\begin{equation*}
\theta_{t+1} -\theta_t^* = \Big(I+\gamma^{-1}P_{t-1}^{-1}A_t^\top A_t\Big)^{-1}(\theta_t-\theta_t^*) 
\end{equation*}
which can be reformulated as:
\begin{equation*}
\theta_{t+1} -\theta_t^* = P_{t-1}^{-1/2}(I+\gamma^{-1}P_{t-1}^{-1/2}A_t^\top A_tP_{t-1}^{-1/2})^{-1}P_{t-1}^{1/2}(\theta_t-\theta_t^*) 
\end{equation*}
which gives us the following inequality:
\begin{equation*}
\begin{array}{l}
\left\|\theta_{t+1} -\theta_t^*\right\|
\le \left\|P_{t-1}^{-1/2}\right\|_2\left\|(I+\gamma^{-1}P_{t-1}^{-1/2}A_t^\top A_tP_{t-1}^{-1/2})^{-1}\right\|_2
\left\|P_{t-1}^{1/2}\right\|_2\left\|\theta_t-\theta_t^*\right\|
\end{array}
\end{equation*}
Then we will upper bound the terms on the right-hand side individually.

\noindent
Since $lI\preceq A_{t-i}^\top A_{t-i}\preceq uI$, 
$\frac{1-\gamma^{t-1}}{1-\gamma}lI\preceq P_{t-1}= \sum\limits_{i=1}^{t-1}\gamma^{i-1}A_{t-i}^\top A_{t-i}
\preceq \frac{1-\gamma^{t-1}}{1-\gamma}uI$.

\noindent
For the term $\left\|P_{t-1}^{-1/2}\right\|_2$, 
we have $\left\|P_{t-1}^{-1/2}\right\|_2 = \frac{1}{\sqrt{\sigma_n(P_{t-1})}}$. 
Since $\sigma_n(P_{t-1})\ge \frac{1-\gamma^{t-1}}{1-\gamma}l$, 
$\left\|P_{t-1}^{-1/2}\right\|_2\le \frac{1}{\sqrt{l}}\sqrt{\frac{1-\gamma}{1-\gamma^{t-1}}}$.

\noindent
For the term $\left\|P_{t-1}^{1/2}\right\|_2$, 
we have $\left\|P_{t-1}^{1/2}\right\|_2$ $=$ $\sqrt{\sigma_1(P_{t-1})}$.
Since $\sigma_1(P_{t-1})\le \frac{1-\gamma^{t-1}}{1-\gamma}u$,
$\left\|P_{t-1}^{1/2}\right\|_2\le \sqrt{u}\sqrt{\frac{1-\gamma^{t-1}}{1-\gamma}}$.

\noindent
For $\left\|(I+\gamma^{-1}P_{t-1}^{-1/2}A_t^\top A_tP_{t-1}^{-1/2})^{-1}\right\|_2$,
we have $\left\|(I+\gamma^{-1}P_{t-1}^{-1/2}A_t^\top A_tP_{t-1}^{-1/2})^{-1}\right\|_2$\\
\noindent
$= 1/\sigma_n(I$$+$$\gamma^{-1}P_{t-1}^{-1/2}A_t^\top A_tP_{t-1}^{-1/2})$.
For the term
$\sigma_n(I$$+$$\gamma^{-1}P_{t-1}^{-1/2}A_t^\top A_tP_{t-1}^{-1/2})$,
it is equal to 
$1+\sigma_n(\gamma^{-1}P_{t-1}^{-1/2}A_t^\top A_tP_{t-1}^{-1/2})$ $\ge$
$1+\gamma^{-1}\sigma_n(P_{t-1}^{-1/2})\sigma_n(A_t^\top A_t)\sigma_n(P_{t-1}^{-1/2})$.

\noindent
Since $\sigma_n(P_{t-1}^{-1/2}) = \frac{1}{\sqrt{\sigma_1(P_{t-1})}}$
and $\sigma_1(P_{t-1})\le \frac{1-\gamma^{t-1}}{1-\gamma}u$,
$\sigma_n(P_{t-1}^{-1/2})\ge \frac{1}{\sqrt{u}}\sqrt{\frac{1-\gamma}{1-\gamma^{t-1}}}$.
Together with $\sigma_n(A_t^\top A_t)\ge l$, we have
$\sigma_n(P_{t-1}^{-1/2}A_t^\top A_tP_{t-1}^{-1/2}) 
\ge \frac{l}{u}\frac{1-\gamma}{1-\gamma^{t-1}}$,
which results in
$\left\|(I+\gamma^{-1}P_{t-1}^{-1/2}A_t^\top A_tP_{t-1}^{-1/2})^{-1}\right\|_2
\le \frac{1}{1+\gamma^{-1}\frac{l}{u}\frac{1-\gamma}{1-\gamma^{t-1}}}$.

\noindent
Combining the above three terms' inequalities results in
\begin{equation*}
\left\|\theta_{t+1} -\theta_t^*\right\| 
\le \sqrt{\frac{u}{l}}\frac{u(\gamma-\gamma^t)}{u(\gamma-\gamma^t)+l(1-\gamma)}\left\|\theta_t-\theta_t^*\right\|
\le \sqrt{\frac{u}{l}}\frac{u\gamma}{u\gamma+l(1-\gamma)}\left\|\theta_t-\theta_t^*\right\|.
\end{equation*}
\end{proof}

\noindent
Now we are ready to present the dynamic regret for the general recursive least-squares update:
\begin{theorem}
\label{thm::gen_ls_dynamic_regret}
{\it 
Let $\theta_t^*$ be the solution to $f_t(\theta)$ in Eq.\eqref{eq::gen_ls_loss}
and $\delta = u/l\ge 1$ be the condition number.
When using the discounted recursive least-squares update in Eq.\eqref{eq::gen_ls_update} 
with $\gamma <\frac{1}{\delta^{3/2}-\delta+1}$
and $\rho = \sqrt{\frac{u}{l}}\frac{u\gamma}{u\gamma+l(1-\gamma)}<1$,
  we can upper bound the dynamic regret:
  \small
  \begin{equation*}
  \mathcal{R}_d \le \sqrt{u}(u/l+1)D\frac{1}{1-\rho}\big(\left\|\theta_1-\theta_1^*\right\| +
+ \sum\limits_{t=2}^T\left\|\theta_t^*-\theta_{t-1}^*\right\|\big) 
  \end{equation*}
  \normalsize
}
\end{theorem}

\begin{proof}
The proof follows the similar steps in the proof of Theorem \ref{thm::quad_dynamic_regret}.
First, we use the Mean Value Theorem to get 
$f_t(\theta_t)-f_t(\theta_t^*) = \nabla f_t(x)^\top(\theta_t-\theta_t^*)
\le \left\|\nabla f_t(x)\right\|\left\|\theta_t-\theta_t^*\right\|$,
where $x\in \{v| v = \beta \theta_t + (1-\beta)\theta_t^*,\beta\in[0,1]\}$.
According to Corollary \ref{corol:norm_gen_ls_grad},
$\left\|\nabla f_t(x)\right\|\le \sqrt{u}(u/l+1)D$.
As a result, 
$\sum\limits_{t=1}^T\big(f_t(\theta_t)-f_t(\theta_t^*)\big)\le \sqrt{u}(u/l+1)D\sum\limits_{t=1}^T\left\|\theta_t-\theta_t^*\right\|$.

\noindent
Now we need to upper bound the term $\sum\limits_{t=1}^T\left\|\theta_t-\theta_t^*\right\|$.
\begin{equation*}
\begin{array}{ll}
\sum\limits_{t=1}^T\left\|\theta_t-\theta_t^*\right\| 
&= \left\|\theta_1-\theta_1^*\right\| 
+ \sum\limits_{t=2}^T\left\|\theta_t-\theta_{t-1}^*+\theta_{t-1}^*-\theta_t^*\right\| \\
&\le \left\|\theta_1-\theta_1^*\right\| + \sum\limits_{t=1}^{T-1}\left\|\theta_{t+1}-\theta_{t}^*\right\| 
+ \sum\limits_{t=2}^T\left\|\theta_t^*-\theta_{t-1}^*\right\|  \\
&\le \left\|\theta_1-\theta_1^*\right\| + \sum\limits_{t=1}^{T}\left\|\theta_{t+1}-\theta_{t}^*\right\| 
+ \sum\limits_{t=2}^T\left\|\theta_t^*-\theta_{t-1}^*\right\|.
\end{array}
\end{equation*}
According to Corollary \ref{corol:gen_ls_var_path_norm}, 
$\left\|\theta_{t+1} -\theta_t^*\right\| \le \rho\left\|\theta_t-\theta_t^*\right\|$.
Then the above inequality can be reformulated as 
$\sum\limits_{t=1}^T\left\|\theta_t-\theta_t^*\right\| \le
\frac{1}{1-\rho}(\left\|\theta_1-\theta_1^*\right\| +
+ \sum\limits_{t=2}^T\left\|\theta_t^*-\theta_{t-1}^*\right\|)$.
Then $\mathcal{R}_d = \sum\limits_{t=1}^T\big(f_t(\theta_t)-f_t(\theta_t^*)\big)
\le \sqrt{u}(u/l+1)D\frac{1}{1-\rho}(\left\|\theta_1-\theta_1^*\right\| +
+ \sum\limits_{t=2}^T\left\|\theta_t^*-\theta_{t-1}^*\right\|)$.

\end{proof}

\noindent
In the above Theorem \ref{thm::gen_ls_dynamic_regret}, the valid range of $\gamma$ is in $(0,1/(\delta^{3/2}-\delta+1))$.
Let us now examine the requirement of $\gamma$ to achieve the sub-linear static regret:
\begin{theorem}
\label{thm::gen_ls_static_regret}
{\it Let $\theta^*$ be the solution to $\min\sum\limits_{t=1}^T f_t(\theta)$. 
  When using the discounted recursive least-squares update in Eq.\eqref{eq::gen_ls_update} with $1-\gamma = 1/T^{\alpha}, \alpha\in (0,1)$,
  we can upper bound the static regret:
  \begin{equation*}
  \mathcal{R}_s \le O(T^{1-\alpha})
  \end{equation*}
}
\end{theorem}

\begin{proof}

The proof follows the analysis of the online Newton method \cite{hazan2007logarithmic}.
From the update in Eq.\eqref{eq::gen_ls_update}, 
we have $\theta_{t+1}-\theta^* = \theta_t-\theta^*-P_{t}^{-1}\nabla f_t(\theta_t)$
and $P_t(\theta_{t+1}-\theta^*) = P_t(\theta_t-\theta^*)-\nabla f_t(\theta_t)$.
Multiplying the two equalities, we have
$(\theta_{t+1}-\theta^*)^\top P_t(\theta_{t+1}-\theta^*) 
= (\theta_t-\theta^*)^\top P_t(\theta_t-\theta^*)-2\nabla f_t(\theta_t)^\top(\theta_t-\theta^*)
+ \nabla f_t(\theta_t)^\top P_{t}^{-1}\nabla f_t(\theta_t)$.

\noindent
After the reformulation, we have 
$\nabla f_t(\theta_t)^\top(\theta_t-\theta^*) = \frac{1}{2}\nabla f_t(\theta_t)^\top P_{t}^{-1}\nabla f_t(\theta_t)
+\frac{1}{2}(\theta_t-\theta^*)^\top P_t(\theta_t-\theta^*)-\frac{1}{2}(\theta_{t+1}-\theta^*)^\top P_t(\theta_{t+1}-\theta^*)
\le \frac{1}{2}\nabla f_t(\theta_t)^\top P_{t}^{-1}\nabla f_t(\theta_t)
+\frac{1}{2}(\theta_t-\theta^*)^\top P_t(\theta_t-\theta^*)-\frac{1}{2}(\theta_{t+1}-\theta^*)^\top\gamma P_t(\theta_{t+1}-\theta^*)$.

\noindent
Summing the above inequality from $t=1$ to $T$, we have:
\footnotesize
\begin{equation*}
\begin{array}{l}
\sum\limits_{t=1}^T \nabla f_t(\theta_t)^\top(\theta_t-\theta^*) \\
\le \sum\limits_{t=1}^T\frac{1}{2}\nabla f_t(\theta_t)^\top P_{t}^{-1}\nabla f_t(\theta_t)
+ \frac{1}{2}(\theta_1-\theta^*)^\top P_1(\theta_1-\theta^*)
+ \sum\limits_{t=2}^T\frac{1}{2}(\theta_t-\theta^*)^\top(P_t-\gamma P_{t-1})(\theta_t-\theta^*)\\
\quad- \frac{1}{2}(\theta_{T+1}-\theta^*)^\top\gamma P_T(\theta_{T+1}-\theta^*)\\
\le \sum\limits_{t=1}^T\frac{1}{2}\nabla f_t(\theta_t)^\top P_{t}^{-1}\nabla f_t(\theta_t)
+ \frac{1}{2}(\theta_1-\theta^*)^\top(P_1-A_1^\top A_1)(\theta_1-\theta^*)
+ \sum\limits_{t=1}^T\frac{1}{2}(\theta_t-\theta^*)^\top A_t^\top A_t(\theta_t-\theta^*).
\end{array}
\end{equation*}
\normalsize

\noindent
Since $P_1 = A_1^\top A_1$ and $f_t(\theta_t) - f_t(\theta^*) = \nabla f_t(\theta_t)^\top(\theta_t-\theta^*)
-\frac{1}{2}(\theta_t-\theta^*)^\top A_t^\top A_t(\theta_t-\theta^*)$,
we reformulate the above inequality as:
\begin{equation}
\begin{array}{ll}
\sum\limits_{t=1}^T \Big(f_t(\theta_t) - f_t(\theta^*) \Big)
&= \sum\limits_{t=1}^T \Big(\nabla f_t(\theta_t)^\top(\theta_t-\theta^*)
- \frac{1}{2}(\theta_t-\theta^*)^\top A_t^\top A_t(\theta_t-\theta^*)\Big) \\
&\le \sum\limits_{t=1}^T\frac{1}{2}\nabla f_t(\theta_t)^\top P_{t}^{-1}\nabla f_t(\theta_t)\\
&= \sum\limits_{t=1}^T\frac{1}{2}(A_t\theta_t-y_t)^\top A_tP_t^{-1}A_t^\top(A_t\theta_t-y_t) \\
& \le \sum\limits_{t=1}^T \frac{1}{2}\sigma_1(P_t^{-1/2}A_t^\top A_tP_t^{-1/2})\left\|A_t\theta_t-y_t\right\|^2
\end{array}
\end{equation}
Since $\sigma_1(P_t^{-1/2}A_t^\top A_tP_t^{-1/2})\le \sigma_1(P_t^{-1})\sigma_1(A_t^\top A_t)
= \frac{1}{\sigma_n(P_t)}\sigma_1(A_t^\top A_t)$.
Based on the proof of Corollary \ref{corol:gen_ls_var_path_norm},
$\sigma_n(P_t)\ge \frac{1-\gamma^t}{1-\gamma}l$ and $\sigma_1(A_t^TA_t)\le u$.
Then $\sigma_1(P_t^{-1/2}A_t^\top A_tP_t^{-1/2})\le \frac{u}{l}\frac{1-\gamma}{1-\gamma^t}$.
As a result, we have 
\begin{equation}
\begin{array}{ll}
\sum\limits_{t=1}^T \Big(f_t(\theta_t) - f_t(\theta^*) \Big) 
\le \sum\limits_{t=1}^T\frac{1}{2}\frac{u}{l}\frac{1-\gamma}{1-\gamma^t}\left\|A_t\theta_t-y_t\right\|^2 
\le \sum\limits_{t=1}^T\frac{1}{2}\frac{u}{l}\frac{1-\gamma}{1-\gamma^t}(u/l+1)^2D^2 
\le O(T^{1-\alpha})
\end{array}
\end{equation}
where the second inequality is due to Lemma \ref{lem:norm_gen_ls_dif} 
and the third inequality is due to the fact that $\sum\limits_{t=1}^T1/(1-\gamma^t)\le O(T)$
as shown in the proof of Theorem \ref{thm::quad_static_regret}.

\end{proof}

\noindent
Recall that the valid range of $\gamma$ in Theorem \ref{thm::gen_ls_dynamic_regret} is
$(0,1/(\delta^{3/2}-\delta+1))$, 
while having sub-linear static regret requires $\gamma = \frac{T^{\alpha}-1}{T^{\alpha}}$.
Although for some specific $T$, there might be some intersection.
In general, these two are contradictory.
However, as discussed in the main body of the Chapter \ref{chap:trade-off}, more flexible
trade-offs between static and dynamic regret can be achieved via the
gradient descent rule.

\chapter{Online Adaptive Principal Component Analysis and Its extensions}
\label{sec:appdx_adaptivePCA}

\noindent
Before presenting the proofs, we need the following lemma from previous literature:
\begin{lemma}\cite{freund1997decision}
\label{lem::freund_ineq}
Suppose $0\le L\le \tilde{L}$ and $0 < R \le \tilde{R}$. Let $\beta = g(\tilde{L}/\tilde{R})$
where $g(z) = 1/(1+\sqrt{2/z})$. Then
\begin{equation*}
\frac{-L\ln\beta + R}{1-\beta}\le L + \sqrt{2\tilde{L}\tilde{R}} + R
\end{equation*}
\end{lemma}

\noindent
Additionally, we need the following classic bound on traces for
postive semidefinite matrices. See, e.g. \cite{tsuda2005matrix}.
\begin{lemma}
\label{lem::matrix_pos_sym_ineq}
For any positive semi-definite matrix $A$ and any symmetric matrices $B$ and $C$,
$B\preceq C$ implies $\Tr(AB)\le\Tr(AC)$. 
\end{lemma}

\paragraph{Proof of Theorem \ref{thm::adaptive_subset_expert}:}
\begin{proof}
Fix $1\le r\le s \le T$. We set $\mathbf{q_t} = \mathbf{q}\in
\mathcal{B}_{n-k}^n$ for $t=r,\dots,s$ and $0$ elsewhere. Thus, we
have that  $\|\mathbf{q_t}\|_1$ is either $0$ or $1$. 

\noindent
According to Lemma \ref{lem::adaptive_expert_step_ineq}, for both cases of $\mathbf{q_t}$, we have
\begin{equation}
\label{eq::initial_ineq_entropy}
\left\|\mathbf{q_t}\right\|_1 \mathbf{w_t}^\top\mathbf{\ell_t}(1-\exp(-\eta)) - \eta \mathbf{q_t}^\top\mathbf{\ell_t} \le
\sum_{i=1}^n q_{t,i}\ln(\frac{v_{t+1,i}}{\hat{w}_{t,i}})
\end{equation}

\noindent
The analysis for $\sum_{i=1}^n
q_{t,i}\ln(\frac{v_{t+1,i}}{\hat{w}_{t,i}})$ follows the Proof of
Proposition $2$ in \cite{cesa2012new}.  
We describe the steps for completeness, since it is helpful for
understanding the effect of the fixed-share step,
Eq.(\ref{eq::fix_share_expert}). This analysis will be crucial for the
understanding how the fixed-share step can be applied to PCA
problems. 
\footnotesize
\begin{equation}
\label{eq::analysis_entropy_expert}
\begin{array}{ll}
\sum_{i=1}^n q_{t,i}\ln(\frac{v_{t+1,i}}{\hat{w}_{t,i}})
=& \underbrace{\sum_{i=1}^n\Big(q_{t,i}\ln\frac{1}{\hat{w}_{t,i}} - q_{t-1,i}\ln\frac{1}{v_{t,i}}\Big)}_A
 + \underbrace{\sum_{i=1}^n\Big(q_{t-1,i}\ln\frac{1}{v_{t,i}} - q_{t,i}\ln\frac{1}{v_{t+1,i}}\Big)}_B
\end{array}
\end{equation}
\normalsize

\noindent
For the expression of $A$, we have
\footnotesize
\begin{equation*}
\begin{array}{ll}
A = \sum\limits_{i:q_{t,i}\ge q_{t-1,i}}\Big((q_{t,i}-q_{t-1,i})\ln\frac{1}{\hat{w}_{t,i}} +
 q_{t-1,i}\ln\frac{v_{t,i}}{\hat{w}_{t,i}}\Big)
 +
 \sum\limits_{i:q_{t,i}<q_{t-1,i}}\Big(\underbrace{(q_{t,i}-q_{t-1,i})\ln\frac{1}{v_{t,i}}}_{\le 0}+q_{t,i}\ln\frac{v_{t,i}}{\hat{w}_{t,i}}\Big)
\end{array}
\end{equation*}
\normalsize

\noindent
Based on the update in Eq.(\ref{eq::our_expert_update}), we have $1/\hat{w}_{t,i}\le n/\alpha$
and $v_{t,i}/\hat{w}_{t,i}\le 1/(1-\alpha)$. Plugging the bounds into the above equation, we have
\begin{equation*}
\begin{array}{ll}
A \le  \underbrace{\sum\limits_{i:q_{t,i}\ge q_{t-1,i}} (q_{t,i}-q_{t-1,i})}_{= D_{TV}(\mathbf{q_t},\mathbf{q_{t-1}})}\ln\frac{n}{\alpha} 
 + \underbrace{\Big(\sum\limits_{i:q_{t,i}\ge q_{t-1,i}}q_{t-1,i}+
\sum\limits_{i:q_{t,i}<q_{t-1,i}}q_{t,i}\Big)}_{=\left\|\mathbf{q_t}\right\|_1-D_{TV}(\mathbf{q_t},\mathbf{q_{t-1}})}
\ln\frac{1}{1-\alpha} .
\end{array}
\end{equation*}

\noindent
Telescoping the expression of $B$, substituting the above inequality in Eq.(\ref{eq::analysis_entropy_expert}),
and summing over $t=2,\dots,T$, we have 
\begin{equation*}
\sum\limits_{t=2}^T\sum\limits_{i=1}^n q_{t,i}\ln\frac{v_{t+1,i}}{\hat{w}_{t,i}} 
\le 
m(\mathbf{q_{1:T}})\ln\frac{n}{\alpha}+
\Big(\sum\limits_{t=2}^T\left\|\mathbf{q_t}\right\|_1-m(\mathbf{q_{1:T}})\Big)
\ln\frac{1}{1-\alpha}
+ \sum\limits_{i=1}^n q_{1,i}\ln\frac{1}{v_{2,i}}.
\end{equation*}

\noindent
Adding the $t=1$ term to the above inequality, we have
\begin{equation*}
\begin{array}{ll}
\sum\limits_{t=1}^T\sum\limits_{i=1}^n q_{t,i}\ln\frac{v_{t+1,i}}{\hat{w}_{t,i}} 
\le  \left\|\mathbf{q_1}\right\|_1\ln(n)+m(\mathbf{q_{1:T}})\ln\frac{n}{\alpha}
 +\Big(\sum\limits_{t=1}^T\left\|\mathbf{q_t}\right\|_1-m(\mathbf{q_{1:T}})\Big)\ln\frac{1}{1-\alpha}.
\end{array}
\end{equation*}

\noindent
Now we bound the right side, using the choices for $\mathbf{q_t}$
described at the beginning of the proof.  
If $r\ge 2$, $m(\mathbf{q_{1:T}}) = 1$, and $\left\|\mathbf{q_1}\right\|_1 = 0$.
If $r = 1$, $m(\mathbf{q_{1:T}}) = 0$, and $\left\|\mathbf{q_1}\right\|_1 = 1$.
Thus, $m(\mathbf{q_{1:T}}) + \left\|\mathbf{q_1}\right\|_1 = 1$, and the right part can be upper bounded by
$\ln\frac{n}{\alpha}+T\ln\frac{1}{1-\alpha}$.

\noindent
After combining the above inequality with Eq.(\ref{eq::initial_ineq_entropy}),
setting $\mathbf{q_t} = \mathbf{q}\in \mathcal{B}_{\text{n-k}}^\text{n}$ for $t=r,\dots,s$ and $0$ elsewhere,
and multiplying both sides by $n-k$,
we have
\begin{equation*}
\begin{array}{l}
(1-\exp(-\eta))\sum\limits_{t=r}^s (n-k)\mathbf{w_t}^\top\mathbf{\ell_t} - \eta \sum\limits_{t=r}^s (n-k)\mathbf{q}^\top\mathbf{\ell_t} 
\le (n-k)\ln\frac{n}{\alpha}+(n-k)T\ln\frac{1}{1-\alpha}
\end{array}
\end{equation*}
If we set $\alpha = 1/(1+(n-k)T)$, then the right part can be upper bounded by $(n-k)\ln(n(1+(n-k)T))+1$,
which equals to $D$ as defined in the Theorem \ref{thm::adaptive_subset_expert}.
Thus, the above inequality can be reformulated as 
\begin{equation*}
\sum\limits_{t=r}^s (n-k)\mathbf{w_t}^\top\mathbf{\ell_t} \le \frac{\eta \sum\limits_{t=r}^s (n-k)\mathbf{q}^\top\mathbf{\ell_t} + D}{1-\exp(-\eta)}
\end{equation*}

\noindent
Since the above inequality holds for arbitrary $\mathbf{q}\in \mathcal{B}_{\text{n-k}}^\text{n}$, we have
\begin{equation}
\label{eq::expert_ineq_before_final}
\sum\limits_{t=r}^s (n-k)\mathbf{w_t}^\top\mathbf{\ell_t} \le 
\frac{\eta \min\limits_{\mathbf{q}\in\mathcal{B}_{\text{n-k}}^\text{n}}\sum\limits_{t=r}^s (n-k)\mathbf{q}^\top\mathbf{\ell_t} + D}{1-\exp(-\eta)}
\end{equation} 

\noindent
We will apply the inequality in Lemma \ref{lem::freund_ineq} to upper bound the right part in Eq.(\ref{eq::expert_ineq_before_final}).
With $\min\limits_{\mathbf{q}\in\mathcal{B}_{\text{n-k}}^\text{n}}\sum\limits_{t=r}^s (n-k)\mathbf{q}^\top\mathbf{\ell_t} \le L$
and $\eta = \ln(1+\sqrt{2D/L})$, we have 
\begin{equation*}
\sum\limits_{t=r}^s (n-k)\mathbf{w_t}^\top\mathbf{\ell_t} - 
\min\limits_{\mathbf{q}\in\mathcal{B}_{\text{n-k}}^\text{n}}\sum\limits_{t=r}^s (n-k)\mathbf{q}^\top\mathbf{\ell_t} \le \sqrt{2LD}+D
\end{equation*}

\noindent
Since the above inequality always holds for all intervals, $[r,s]$,
the result is proved by maximizing the left side over $[r,s]$.
\end{proof}

\paragraph{Proof of Theorem \ref{thm::adaptive_pca}:}

\begin{proof}
In the proof, we will examine two cases of $Q_t$: $Q_t\in\mathscr{B}_{n-k}^n$, and $Q_t = 0$.

\noindent
We first apply the eigendecomposition to $Q_t$ as $Q_t = \widetilde{D}\diag(\mathbf{q_t})\widetilde{D}^\top$,
where $\widetilde{D} = [\mathbf{\tilde{d}_1},\dots,\mathbf{\tilde{d}_n}]$.
Since in the adaptive setting, $Q_{t-1}$ is either equal to $Q_t$ or $0$,
they share the same eigenvectors and 
$Q_{t-1}$ can be expressed as $Q_{t-1} = \widetilde{D}\diag(\mathbf{q_{t-1}})\widetilde{D}^\top$.

\noindent
According to Lemma \ref{lem::adaptive_pca_step_ineq}, the following inequality is true for both cases of $Q_t$:
\begin{equation}
\label{eq::step_ineq_pca}
\begin{array}{l}
\left\|\mathbf{q_t}\right\|_1\Tr(W_t\mathbf{x_t}\mathbf{x_t}^\top)(1-\exp(-\eta)) -\eta\Tr(Q_t\mathbf{x_t}\mathbf{x_t}^\top)
 \le -\Tr(Q_t\ln \widehat{W}_t) + \Tr(Q_t\ln V_{t+1})
\end{array}
\end{equation}

\noindent
The next steps extend proof of Proposition 2 in \cite{cesa2012new} to the matrix case.

\noindent
We analyze the right part of the above inequality,
which can be expressed as:
\begin{equation}
\label{eq::quantum_split_two_parts}
\begin{array}{l}
-\Tr(Q_t\ln \widehat{W}_t) + \Tr(Q_t\ln V_{t+1}) = \bar{A} + \bar{B}
\end{array}
\end{equation}
where $\bar{A} = -\Tr(Q_t\ln \widehat{W}_t) + \Tr(Q_{t-1}\ln V_t)$, 
and $\bar{B} = - \Tr(Q_{t-1}\ln V_t) + \Tr(Q_t\ln V_{t+1})$.

\noindent
We will first upper bound the $\bar{A}$ term, 
and then telescope the $\bar{B}$ term.

\noindent
$\bar{A}$ can be expressed as:
\footnotesize
\begin{equation*}
\begin{array}{l}
\bar{A} = \sum\limits_{i:q_{t,i}\ge q_{t-1,i}}\Big(
\underbrace{-\Tr\big((q_{t,i}\mathbf{\tilde{d}_i}\mathbf{\tilde{d}_i}^\top-q_{t-1,i}\mathbf{\tilde{d}_i}\mathbf{\tilde{d}_i}^\top)
\ln \widehat{W}_t\big)}_{\circled{1}}
+ \underbrace{\Tr(q_{t-1,i}\mathbf{\tilde{d}_i}\mathbf{\tilde{d}_i}^\top\ln V_t) 
- \Tr(q_{t-1,i}\mathbf{\tilde{d}_i}\mathbf{\tilde{d}_i}^\top\ln\widehat{W}_t)}_{\circled{2}}
\Big) \\
\quad\quad+ \sum\limits_{i:q_{t,i}<q_{t-1,i}}\Big(
\underbrace{-\Tr\big((q_{t,i}\mathbf{\tilde{d}_i}\mathbf{\tilde{d}_i}^\top
-q_{t-1,i}\mathbf{\tilde{d}_i}\mathbf{\tilde{d}_i}^\top)\ln V_t\big)}_{\circled{3}} 
+ \underbrace{\Tr(q_{t,i}\mathbf{\tilde{d}_i}\mathbf{\tilde{d}_i}^\top\ln V_t) 
- \Tr(q_{t,i}\mathbf{\tilde{d}_i}\mathbf{\tilde{d}_i}^\top\ln\widehat{W}_t)}_{\circled{4}}
\Big)
\end{array}
\end{equation*}
\normalsize

\noindent
For $\circled{1}$, it can be expressed as:
\begin{equation*}
\begin{array}{ll}
\circled{1} &= \Tr\big((q_{t,i}\mathbf{\tilde{d}_i}\mathbf{\tilde{d}_i}^\top
- q_{t-1,i}\mathbf{\tilde{d}_i}\mathbf{\tilde{d}_i^T})\ln \widehat{W}_t^{-1}\big)\\
&\le \Tr\big((q_{t,i}\mathbf{\tilde{d}_i}\mathbf{\tilde{d}_i}^\top-q_{t-1,i}\mathbf{\tilde{d}_i}\mathbf{\tilde{d}_i}^\top)\ln\frac{n}{\alpha}\big)
=(q_{t,i}-q_{t-1,i})\ln\frac{n}{\alpha}.
\end{array}
\end{equation*}
The inequality holds because the update in Eq.(\ref{eq::fix_share_pca}) 
implies $\ln\widehat{W}_T^{-1}\preceq I\ln\frac{n}{\alpha}$ and
furthermore,
$(q_{t,i}\mathbf{\tilde{d}_i}\mathbf{\tilde{d}_i}^\top-q_{t-1,i}\mathbf{\tilde{d}_i}\mathbf{\tilde{d}_i}^\top)$
is positive semi-definite. 
Thus, Lemma \ref{lem::matrix_pos_sym_ineq}, gives the result.

\noindent
The expression for $\circled{2}$ can be bounded as 
\begin{equation*}
\begin{array}{ll}
\circled{2} = \Tr(q_{t-1,i}\mathbf{\tilde{d}_i}\mathbf{\tilde{d}_i}^\top\ln(V_t\widehat{W}_t^{-1})) 
  \le q_{t-1,i}\ln\frac{1}{1-\alpha}
\end{array}
\end{equation*}
where the equality is due to the fact that $V_t$ and $\widehat{W}_t$
have the same eigenvectors.
The inequality follows since $\ln(V_t\widehat{W}_t^{-1})\preceq
I\ln\frac{1}{1-\alpha}$, due to the update in
Eq.(\ref{eq::fix_share_pca}),
while $q_{t-1,i}\mathbf{\tilde{d}_i}\mathbf{\tilde{d}_i}^\top$ is
positive semi-definite. Thus Lemma \ref{lem::matrix_pos_sym_ineq} gives the result.

\noindent
The bound $\circled{3}$ can be expressed as:
\begin{equation*}
\circled{3} = \Tr\big((-q_{t,i}\mathbf{\tilde{d}_i}\mathbf{\tilde{d}_i}^\top
+ q_{t-1,i}\mathbf{\tilde{d}_i}\mathbf{\tilde{d}_i}^\top)\ln V_t\big) \le 0
\end{equation*}
Here, the inequality follows since
$\ln V_t \preceq 0$ and 
and
$(-q_{t,i}\mathbf{\tilde{d}_i}\mathbf{\tilde{d}_i}^\top+q_{t-1,i}\mathbf{\tilde{d}_i}\mathbf{\tilde{d}_i}^\top)$
is positive semi-definite.
Thus, Lemma \ref{lem::matrix_pos_sym_ineq} gives the result.

\noindent
For $\circled{4}$, we have $\circled{4} \le q_{t,i}\ln\frac{1}{1-\alpha}$,
which follows the same argument used to bound the term $\circled{2}$.

\noindent
Thus, $\bar{A}$ can be upper bounded as follows:
\begin{equation*}
\begin{array}{ll}
\bar{A} \le  \underbrace{\sum\limits_{i:q_{t,i}\ge q_{t-1,i}} (q_{t,i}-q_{t-1,i})}_{= D_{TV}(\mathbf{q_t},\mathbf{q_{t-1}})}\ln\frac{n}{\alpha}
 + \underbrace{\Big(\sum\limits_{i:q_{t,i}\ge q_{t-1,i}}q_{t-1,i}
+ \sum\limits_{i:q_{t,i}<q_{t-1,i}} q_{t,i}\Big)}_{=\left\|\mathbf{q_t}\right\|_1-D_{TV}(\mathbf{q_t},\mathbf{q_{t-1}})}
\ln\frac{1}{1-\alpha} 
\end{array}
\end{equation*}

\noindent
Then we telescope the $\bar{B}$ term, substitute the above inequality for $\bar{A}$ into Eq.(\ref{eq::quantum_split_two_parts}), 
and sum over $t=2,\dots,T$ to give:
\scriptsize
\begin{equation*}
\begin{array}{l}
\sum\limits_{t=2}^T \Big(-\Tr(Q_t\ln \widehat{W}_t) + \Tr(Q_t\ln V_{t+1})\Big) 
\le m(\mathbf{q_{1:T}})\ln\frac{n}{\alpha}+\Big(\sum\limits_{t=2}^T\left\|\mathbf{q_t}\right\|_1-m(\mathbf{q_{1:T}})\Big)
\ln\frac{1}{1-\alpha}
- \Tr(Q_1\ln V_2)
\end{array}
\end{equation*}
\normalsize

\noindent
Adding the $t=1$ term to the above inequality, we have 
\scriptsize
\begin{equation*}
\begin{array}{l}
\sum\limits_{t=1}^T \Big(-\Tr(Q_t\ln \widehat{W}_t) + \Tr(Q_t\ln V_{t+1})\Big)
\le \left\|\mathbf{q_1}\right\|_1\ln(n)+m(\mathbf{q_{1:T}})\ln\frac{n}{\alpha}
 +\Big(\sum\limits_{t=1}^T\left\|\mathbf{q_t}\right\|_1-m(\mathbf{q_{1:T}})\Big)\ln\frac{1}{1-\alpha}
\end{array}
\end{equation*}
\normalsize

\noindent
For the above inequality, we set $Q_t = Q\in \mathscr{B}_{n-k}^n$ for $t=r,\dots,s$ and $0$ elsewhere,
which makes $\mathbf{q_t} = \mathbf{q}\in \mathcal{B}_{n-k}^n$ for $t=r,\dots,s$ and $0$ elsewhere.
If $r\ge 2$, $m(\mathbf{q_{1:T}}) = 1$, and $\left\|\mathbf{q_1}\right\|_1 = 0$.
If $r = 1$, $m(\mathbf{q_{1:T}}) = 0$, and $\left\|\mathbf{q_1}\right\|_1 = 1$.
Thus, $m(\mathbf{q_{1:T}}) + \left\|\mathbf{q_1}\right\|_1 = 1$, and the right part can be upper bounded by
$\ln\frac{n}{\alpha}+T\ln\frac{1}{1-\alpha}$.

\noindent
The rest of the steps follow exactly the same as in the proof of Theorem \ref{thm::adaptive_subset_expert}.
\end{proof}

\paragraph{Proof of Lemma \ref{lem::mat_lem1_bianchi}:}

\begin{proof}
We first deal with the term $\Tr(Q_t\ln V_{t+1})$.
According to the update in Eq.(\ref{eq::var_unit_v_t+1}), we have
\footnotesize
\begin{equation*}
\begin{array}{l}
\Tr(Q_t\ln V_{t+1}) = \Tr\bigg(Q_t\ln\Big(\frac{\exp(\ln Y_t - \eta C_t)}{\Tr(\exp(\ln Y_t - \eta C_t))}\Big)\bigg) 
 = \Tr\big( Q_t(\ln Y_t -\eta C_t)\big) - \ln \Big(\Tr\big(\exp(\ln Y_t  - \eta C_t)\big)\Big),
\end{array}
\end{equation*}
\normalsize
since $Q_t \in\mathscr{B}_1^n$ and $\Tr(Q_t) = 1$.

\noindent
As a result, we have
$\Tr(Q_t\ln V_{t+1}) - \Tr(Q_t\ln Y_t)$ $=$ $-\eta\Tr(Q_tC_t) - \ln \Big(\Tr\big(\exp(\ln Y_t  - \eta C_t)\big)\Big)$.

\noindent
Thus, to prove the inequality in Lemma \ref{lem::mat_lem1_bianchi},
it is enough to prove the following inequality
\begin{equation*}
\eta\Tr(Y_tC_t) - \frac{\eta^2}{2} + \ln \Big(\Tr\big(\exp(\ln Y_t  - \eta C_t)\big)\Big) \le 0
\end{equation*} 

\noindent
Before we proceed, we need the following lemmas:
\begin{lemma}[Golden-Thompson inequality]
For any symmetric matrices $A$ and $B$, the following inequality holds:
\begin{equation*}
\Tr\big(\exp(A + B)\big) \le \Tr\big(\exp(A)\exp(B)\big)
\end{equation*}
\end{lemma}

\noindent
\begin{lemma}[Lemma 2.1 in \cite{tsuda2005matrix}]
\label{lem::lem21_tsuda}
For any symmetric matrix $A$ such that $0\preceq A\preceq I$ and any $\rho_1,\rho_2 \in\mathbb{R}$,
the following holds:
\begin{equation*}
\exp\big(A\rho_1 + (I-A)\rho_2\big) \preceq A\exp(\rho_1) + (I-A)\exp(\rho_2)
\end{equation*}
\end{lemma}

\noindent
Then we apply the Golden-Thompson inequality to the term $\Tr\big(\exp(\ln Y_t  - \eta C_t)\big)$,
which gives us the inequality below:
\begin{equation*}
\Tr\big(\exp(\ln Y_t  - \eta C_t)\big) \le \Tr(Y_t\exp(-\eta C_t)).
\end{equation*}
For the term $\exp(-\eta C_t)$, 
by applying the Lemma \ref{lem::lem21_tsuda} with $\rho_1 = -\eta$ and $\rho_2 = 0$,
we have the following inequality:
\begin{equation*}
 \exp(-\eta C_t) \preceq I - C_t(1-\exp(-\eta)).
\end{equation*}
Thus, we have
\begin{equation*}
\Tr(Y_t\exp(-\eta C_t)) \le 1 - \Tr(Y_tC_t)(1-\exp(-\eta)),
\end{equation*}
and 
\begin{equation*}
\Tr\big(\exp(\ln Y_t  - \eta C_t)\big) \le 1 - \Tr(Y_tC_t)(1-\exp(-\eta)),
\end{equation*}
since $Y_t \in\mathscr{B}_1^n$ and $\Tr(Y_t) = 1$.

\noindent
Thus, it is enough to prove the following inequality
\begin{equation*}
\eta\Tr(Y_tC_t) - \frac{\eta^2}{2} + \ln \Big(1 - \Tr(Y_tC_t)(1-\exp(-\eta))\Big) \le 0
\end{equation*} 

\noindent
Since $\ln(1-x) \le -x$, we have
\begin{equation*}
\ln \Big(1 - \Tr(Y_tC_t)(1-\exp(-\eta))\Big) \le  -\Tr(Y_tC_t)(1-\exp(-\eta)).
\end{equation*}
Thus, it suffices to prove the following inequality:
\begin{equation*} 
\big(\eta-1+\exp(-\eta)\big)\Tr(Y_tC_t) - \frac{\eta^2}{2} \le 0
\end{equation*}

\noindent
Note that by using convexity
of $\exp(-\eta)$, $\eta-1+\exp(-\eta) \ge 0$.

\noindent
By applying Lemma \ref{lem::matrix_pos_sym_ineq} with $A = Y_t$, $B = C_t$, and $C = I$,
we have $\Tr(Y_tC_t) \le \Tr(Y_t) = 1$.
Thus, when $\eta \ge 0$, it is enough to prove the following inequality 
\begin{equation*}
\eta-1+\exp(-\eta) - \frac{\eta^2}{2} \le 0.
\end{equation*}
This inequality follows from convexity of
$\frac{\eta^2}{2}-\exp(-\eta)$ over $\eta\ge 0$. 
\end{proof}

\paragraph{Proof of Theorem \ref{thm::adaptive_var_simplex}:}

\begin{proof}
First, since $0\preceq C_t\preceq I$, we have $\max_{i,j}|C_t(i,j)|$ $\le 1$.

\noindent
Before we proceed, we need the following lemma from \cite{warmuth2006online}:
\begin{lemma}[Lemma 1 in \cite{warmuth2006online}]
\label{lem::var_simplex_step}
Let $\max_{i,j}|C_t(i,j)|\le\frac{r}{2}$, then
for any $\mathbf{u_t}\in\mathcal{B}_1^n$, any constants $a$ and $b$ such that $0\le a \le \frac{b}{1+rb}$,
and $\eta = \frac{2b}{1+rb}$, we have
\begin{equation*}
a\mathbf{y_t}^\top C_t\mathbf{y_t} - b\mathbf{u_t}^\top C_t\mathbf{u_t} \le d(\mathbf{u_t},\mathbf{y_t}) - d(\mathbf{u_t},\mathbf{v_{t+1}})
\end{equation*}
\end{lemma}

\noindent
Now we apply Lemma \ref{lem::var_simplex_step} 
under the conditions  $r =2$, $a = \frac{b}{2b+1}$, $\eta = 2a$, 
and $b = \frac{c}{2}$.

\noindent
Recall that $d(\mathbf{u_t},\mathbf{y_t}) - d(\mathbf{u_t},\mathbf{v_{t+1}})$ $=$ $\sum_{i}u_{t,i}\ln\Big(\frac{v_{t+1,i}}{y_{t,i}}\Big)$.
Combining this with the inequality in Lemma \ref{lem::var_simplex_step} 
and the fact that $\left\|\mathbf{u_t}\right\|_1 = 1$,
we have 
\begin{equation*}
a\left\|\mathbf{u_t}\right\|_1\mathbf{y_t}^\top C_t\mathbf{y_t} - b\mathbf{u_t}^\top C_t\mathbf{u_t} 
\le \sum_{i}u_{t,i}\ln\Big(\frac{v_{t+1,i}}{y_{t,i}}\Big)
\end{equation*}

\noindent
Note that the above inequality is also true when $\mathbf{u_t} = 0$.

\noindent
Note that the right side of the above inequality 
 is the same as the right part of the Eq.(\ref{eq::initial_ineq_entropy}) 
in the proof of Theorem \ref{thm::adaptive_subset_expert}.

\noindent
As a result, we will use the same steps as in the proof of Theorem \ref{thm::adaptive_subset_expert}.
Then we will set $\mathbf{u_t} = \mathbf{u} = \argmin_{\mathbf{q}\in\mathcal{B}_1^n}\sum\limits_{t=r}^s \mathbf{q}^\top C_t\mathbf{q}$
for $t = r,\dots,s$, and $0$ elsewhere. Summing from $t=1$ up to $T$,
gives the following inequality:
\begin{equation*}
a\big[ \sum_{t=r}^s \mathbf{y_t}^\top C_t\mathbf{y_t} \big] - b\big[\min_{\mathbf{u}\in\mathcal{B}_1^n}\sum_{t=r}^s\mathbf{u}^\top C_t\mathbf{u}\big] 
\le \ln\frac{n}{\alpha} + T\ln\frac{1}{1-\alpha}
\end{equation*}

\noindent
Since $\alpha = 1/(T+1)$, $T\ln\frac{1}{1-\alpha} \le 1$.
Then the above inequality becomes
\begin{equation*}
a\big[ \sum\limits_{t=r}^s \mathbf{y_t}^\top C_t\mathbf{y_t} \big] 
- b\big[\min_{\mathbf{u}\in\mathcal{B}_1^n}\sum\limits_{t=r}^s\mathbf{u}^\top C_t\mathbf{u}\big] 
\le \ln\big((1+T)n\big) + 1
\end{equation*}

\noindent
Plugging in the expressions of $a = c/(2c+2)$, $b = c/2$, and $c = \frac{\sqrt{2\ln\big((1+T)n\big)+2}}{\sqrt{L}}$ we have
\begin{equation*}
\begin{array}{ll}
\sum\limits_{t=r}^s \mathbf{y_t}^\top C_t\mathbf{y_t} - \min_{\mathbf{u}\in\mathcal{B}_1^n}\sum\limits_{t=r}^s\mathbf{u}^\top C_t\mathbf{u}
&\le c\Big[\min_{\mathbf{u}\in\mathcal{B}_1^n}\mathbf{u}^\top C_t\mathbf{u}\Big] + 2\frac{c+1}{c}\big(\ln\big((1+T)n\big) + 1\big) \\
&\le cL + 2\frac{c+1}{c}\big(\ln\big((1+T)n\big) + 1\big) \\
&= 2\sqrt{2L\Big(\ln\big((1+T)n\big)+1\Big)} + 2\ln\big((1+T)n\big)
\end{array}
\end{equation*}

\noindent
Since the inequality holds for any $1\le r\le s \le T$, the proof is
concluded by maximizing over $[r,s]$ on the left.
\end{proof}

\chapter{Online Convex Optimization for Cumulative Constraints}
\label{sec:appdx_oco-long-term}

\section{Toy Example Results}
\noindent
The results including different $T$ up to $20000$ are shown in Fig.\ref{fig::toy_obj_con}, 
whose results are averaged over 10 random sequences of $\{c_t\}_{t=1}^T$.
Since the standard deviations are small, we only plot the mean results.

\noindent
From Fig.\ref{fig::toy_traj} we can see that the trajectories generated by $Clipped-OGD$
follows the boundary very tightly until reaching the optimal point.
which is also reflected by the Fig.~\ref{fig::toy_obj_con}(a) of the clipped long-term constraint violation.
For the $OGD$, its trajectory oscillates a lot around the boundary of
the actual constraint. 
And if we examine the clipped and non-clipped constraint violation in Fig.~\ref{fig::toy_obj_con}(b),
we find that although the clipped constraint violation is very high, its non-clipped one is very small.
This verifies the statement we make in the beginning that the big constraint violation at one time step is canceled out 
by the strictly feasible constraint at the other time step.
For the $A-OGD$, its trajectory in Fig.\ref{fig::toy_traj} violates the constraint most of the time,
and this violation actually contributes to the lower objective regret shown in Fig.\ref{fig::toy_obj_con}.

\begin{figure}
\vskip 0.0in
  \centering
  \subfigure[]{
    \includegraphics[height=3.9cm]{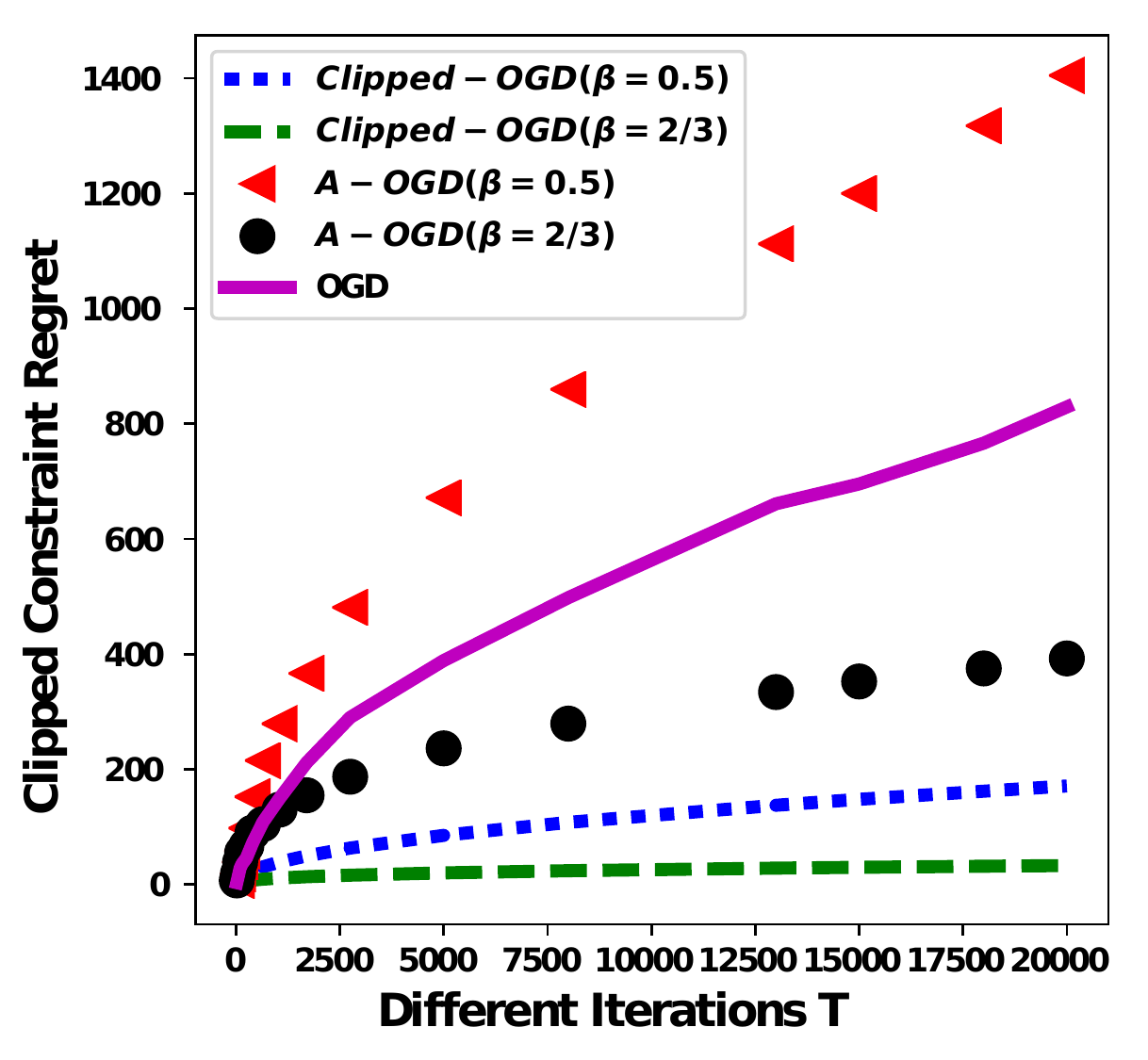}}
  \hspace{.1in}
  \subfigure[]{
    \includegraphics[height=3.9cm]{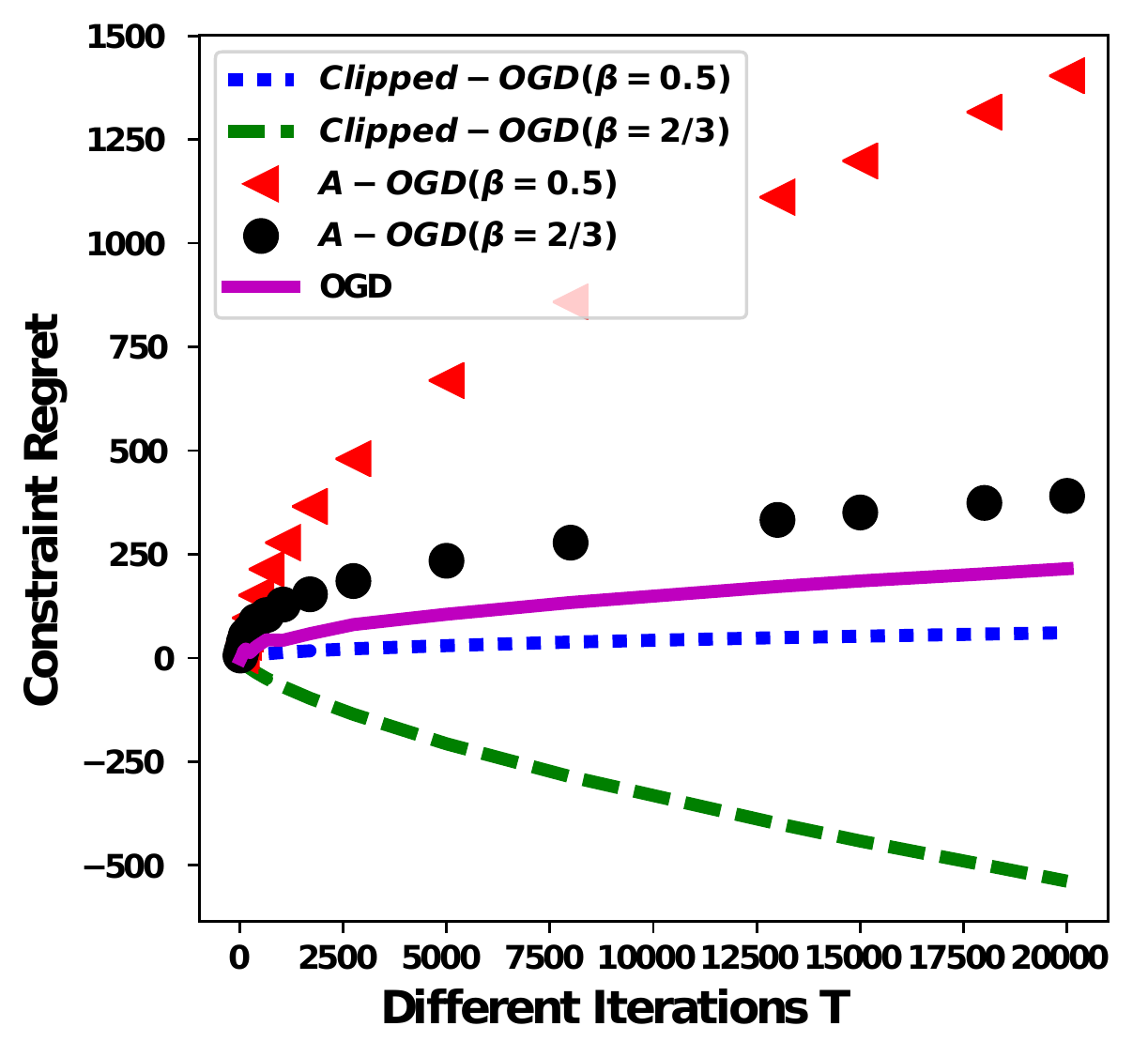}}
  \hspace{.1in}
  \subfigure[]{
    \includegraphics[height=3.9cm]{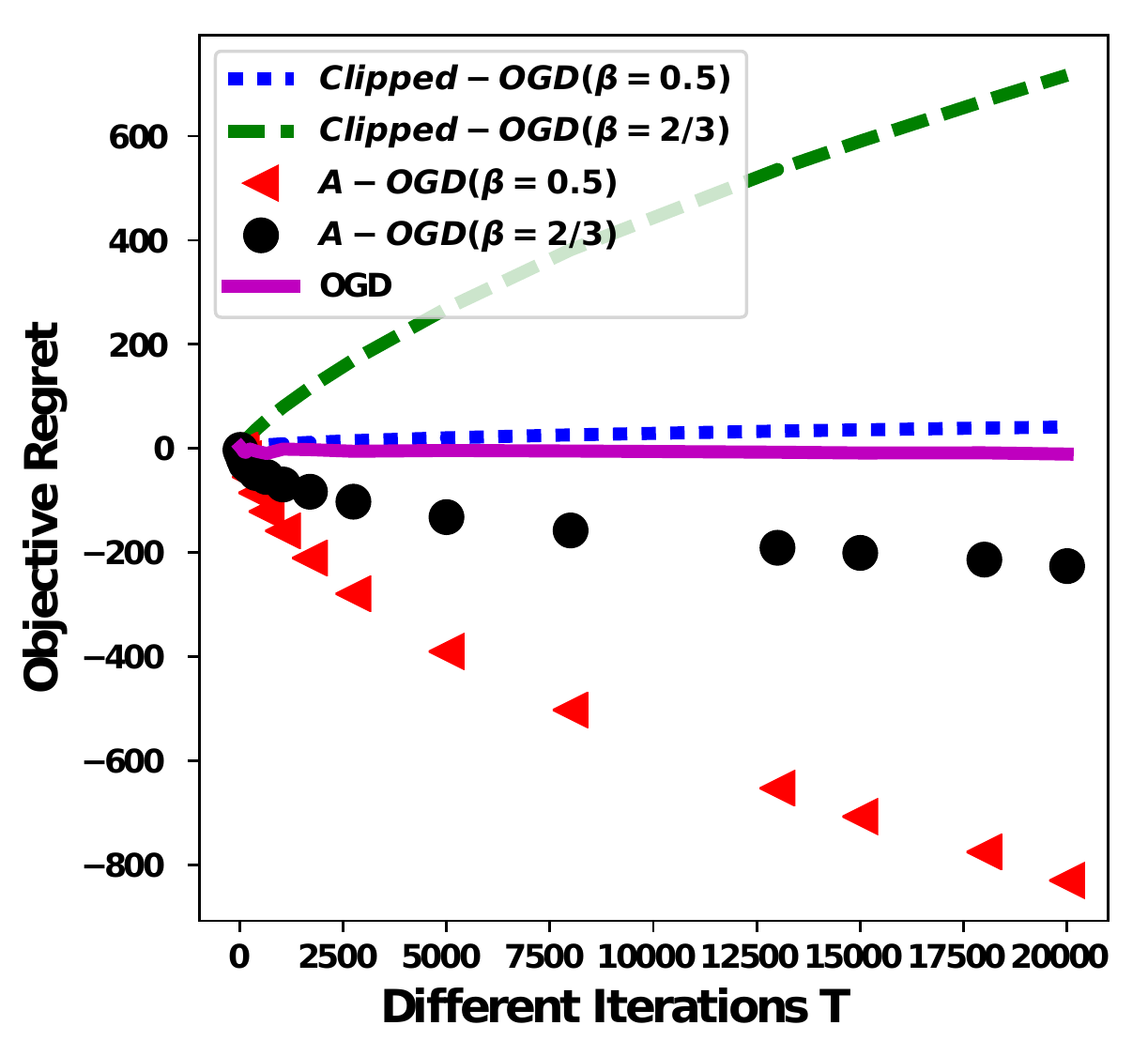}}
  \caption{Toy Example Results: 
           (a): Clipped Long-term Constraint Violation. 
           (b): Long-term Constraint Violation.
           (c): Cumulative Regret of the Loss function}
  \label{fig::toy_obj_con} 
\vskip 0.in
\end{figure}

\section{Proofs}

\paragraph{Proof of Lemma \ref{lem:bound_step}}

\begin{proof}
Recall that the update for $\theta_{t+1}$ is
\begin{equation}
\label{eq::update_x_lemm1}
\theta_{t+1} = \Pi_{\mathcal{B}}\Big(\theta_t-\eta \partial_\theta f_t(\theta_t) - \frac{[g(\theta_t)]_+}{\sigma} \partial_\theta ([g(\theta_t)]_+)\Big)
\end{equation}
Let $y_t=\theta_t-\eta \partial_\theta f_t(\theta_t) - \frac{[g(\theta_t)]_+}{\sigma} \partial_\theta ([g(\theta_t)]_+)$.

\noindent
We first need to show that $g(\theta_{t+1})\le g(y_t)$.
Without loss of generality, let us assume that $y_t$ is not in the set $\mathcal{B}$.
From convexity we have $g(y_t)\ge g(\theta_{t+1})+\nabla_\theta g(\theta_{t+1})^\top(y_t-\theta_{t+1})$.
From non-expansiveness of the projection operator, we have that 
$(y_t-\theta_{t+1})^\top(\theta-\theta_{t+1})\le 0$ for $\theta\in \mathcal{B}$. 
Let $\theta = \theta_{t+1} - \epsilon_0\nabla_\theta g(\theta_{t+1})$ with $\epsilon_0$ small enough to make $\theta\in \mathcal{B}$.
We have $-\epsilon_0 (y_t-\theta_{t+1})^\top\nabla_\theta g(\theta_{t+1}) \le 0$. 
Then we have $g(\theta_{t+1})\le g(y_t)$.

\noindent
As a result, if $g(y_t)$ is upper bounded,  then so is $g(\theta_{t+1})$,
where $\theta_{t+1} = \Pi_{\mathcal{B}}(y_t)$. 
If $T$ is large enough,
$\eta\left\|\partial_\theta f_t(\theta_t)\right\|$ would be very small. 
Thus, we can use $0$-order Taylor expansion for differentiable $g(\theta)$ as below:
\begin{equation}
  \label{eq:simpleTaylor}
\begin{array}{ll}
g(y_t) = g\Big(\theta_t-\eta\partial_\theta
           f_t(\theta_t)-\frac{[g(\theta_t)]_+}{\sigma} \partial_\theta
           ([g(\theta_t)]_+)\Big)
  
  \le g\Big(\theta_t-\frac{[g(\theta_t)]_+}{\sigma} \partial_\theta
    ([g(\theta_t)]_+)\Big)+C \eta 
\end{array}
\end{equation}
where $C$ is a constant determined by the Taylor expansion remainder, as well as
the bound $\|\partial_\theta [g(\theta_t)]_+\| \|\partial_\theta f(\theta_t)\| \le G^2$.

\noindent
Set $\epsilon =(2 C \sigma R^2\eta)^{1/3}=O(\frac{1}{T^{1/6}})$. We will show that if $g(\theta_t) <
\epsilon$, then $g(\theta_{t+1}) \le \epsilon + O(1/\sqrt{T}) =
O(\frac{1}{T^{1/6}})$. We will also show that if $g(\theta_t)\ge \epsilon$,
then $g(\theta_{t+1})\le g(\theta_t)$. It follows then by induction that if
$g(\theta_1)<\epsilon$, then $g(\theta_t) = O(\frac{1}{T^{1/6}})$ for all
$t$. We prove these
inequalities in three cases. Since $g(\theta_{t+1})\le g(y_t)$, it suffices
to bound $g(y_t)$. 

\noindent
\textbf{Case 1: $g(\theta_t) \le 0$.} In this case, the inequality for
$g(y_t)$, \eqref{eq:simpleTaylor}, becomes 
\begin{equation*}
g(y_t) \le g(\theta_t)+C\eta \le C\eta = O(\frac{1}{\sqrt{T}})
\end{equation*}

\noindent
\textbf{Case 2: $0<g(\theta_t) < \epsilon$.} Since $[g(\theta_t)]_+=g(\theta_t)$, the bound on $g(y_t)$ becomes
\begin{equation}
  \label{eq:smallPosBound}
g(y_t) \le g\Big(\theta_t-\frac{g(\theta_t)}{\sigma} \nabla_\theta g(\theta_t)\Big) + C\eta
\end{equation}
We will bound the right using standard methods from gradient descent proofs.
Since $g$ is convex and $\nabla_\theta g(\theta)$ has Lipschitz constant, $L$,
we have the inequality:
\begin{equation}
  \label{eq:convexLipschitzBound}
  g(y)\le g(x) + \nabla_x g(x)^\top (y-x)+\frac{L}{2}\|y-x\|^2  
\end{equation}
for all $x$ and $y$ \cite{nesterov2013introductory}.

\noindent
Recall that
$\epsilon = O(\frac{1}{T^{1/6}})$. Assume that $T$ is sufficiently
large so that $\frac{Lg(\theta_t)}{2\sigma}<\frac{L\epsilon}{2\sigma} < 1$. 
Applying
\eqref{eq:convexLipschitzBound} with $x=\theta_t$ and
$y=\theta_t-\frac{g(\theta_t)}{\sigma}\nabla_\theta g(\theta_t)$ gives
\begin{align}
  g(y_t) & \le  g\Big(\theta_t-\frac{[g(\theta_t)]_+}{\sigma} \partial_\theta
           (g(\theta_t))\Big) + C\eta \\ 
  \label{eq:specialLipschitzBound}
  &\le
                                                               g(\theta_t)
                                                               -
                                                               \frac{g(\theta_t)}{\sigma}(1-\frac{Lg(\theta_t)}{2\sigma})\left\|\nabla_\theta
                                                               g(\theta_t)
                                                               \right\|^2
  + C\eta \\
  &\le g(\theta_t) + C\eta = O(\frac{1}{T^{1/6}}). 
\end{align}
where the third bound follows since $1-\frac{Lg(\theta_t)}{2\sigma} >
0$.

\noindent
\textbf{Case 3: $g(\theta_t)\ge \epsilon$. }
A case can arise such that $g(\theta_{t-1})<\epsilon$ but an additive term of order
$O(\frac{1}{T^{1/2}})$ leads to $\epsilon \le g(\theta_{t})\le \epsilon +
C\eta = O(\frac{1}{T^{1/6}})$. We will now show that no further
increases are possible by bounding the final two terms of
\eqref{eq:specialLipschitzBound} as
\begin{equation}
  \label{eq:equivBound}
                                         -
                                                               \frac{g(\theta_t)}{\sigma}(1-\frac{Lg(\theta_t)}{2\sigma})\left\|\nabla_\theta
                                                               g(\theta_t)
                                                               \right\|^2
  + C\eta \le 0 \iff C\eta \le  \frac{g(\theta_t)}{\sigma}(1-\frac{Lg(\theta_t)}{2\sigma})\left\|\nabla_\theta
                                                               g(\theta_t)
                                                               \right\|^2.
\end{equation}
Now, we
lower-bound the terms on the right of \eqref{eq:equivBound}. Since $\epsilon +
C\eta=O(\frac{1}{T^{1/6}})$, we have that for sufficiently large $T$,
$1-\frac{Lg(\theta_t)}{2\sigma}\ge 1-\frac{L(\epsilon+C\eta)}{2\sigma}\ge
\frac{1}{2}$.  Further note that by convexity, $g(0)\ge
g(\theta_t)-\nabla_\theta g(\theta_t)^\top \theta_t$. Since we assume that $0$ is feasible,
we have that
\begin{equation*}
  \epsilon \le g(\theta_t) \le \nabla_\theta g(\theta_t)^\top \theta_t \le \|\nabla_\theta
  g(\theta_t)\| \|\theta_t\| \le \|\nabla_\theta g(\theta_t)\| R.
\end{equation*}
The final inequality follows since $\theta_t\in\mathcal{B}$. Thus, we have
the following bound for the right of \eqref{eq:equivBound}:
\begin{equation*}
  \frac{g(\theta_t)}{\sigma}(1-\frac{Lg(\theta_t)}{2\sigma})\left\|\nabla_\theta
    g(\theta_t)\right\|^2
    \ge \frac{\epsilon^3}{2\sigma R^2} = C\eta.
  \end{equation*}
  The final equality follows by the definition of $\epsilon$. 
\end{proof}

\noindent
\paragraph{Proof of the Proposition \ref{prop::bound_step_max}:}
\begin{proof}
From the construction of $\bar{g}(\theta)$, we have the 
$\bar{g}(\theta)\ge \max\limits_{i}g_i(\theta)$. Thus, if we can upper bound the $\bar{g}(\theta)$,
$g_i(\theta)$ will automatically be upper bounded.
In order to use Lemma \ref{lem:bound_step}, we need to make sure
the following conditions are satisfied:
\begin{itemize}
\item $\bar{g}(\theta)$ is convex and differentiable. 
\item $\left\|\nabla_\theta \bar{g}(\theta)\right\|$ is upper bounded. 
\item $\left\|\nabla_\theta^{\prime\prime} \bar{g}(\theta)\right\|_2$ is upper bounded, 
where $\nabla_\theta^{\prime\prime} \bar{g}(\theta)$
is the Hessian matrix.
\end{itemize}
The first condition is satisfied due to the formula of $\bar{g}(\theta)$.
To examine the second one, we have 
\begin{equation*}
\nabla_\theta \bar{g}(\theta) = \frac{1}{\sum\limits_{i=1}^m \exp g_i(\theta)}\Bigg[\sum\limits_{i=1}^m \exp g_i(\theta) \nabla_\theta g_i(\theta)\Bigg]
\end{equation*}
\begin{equation*}
\begin{array}{ll}
\left\|\nabla_\theta \bar{g}(\theta)\right\|^2 = \frac{1}{\big(\sum\limits_{i=1}^m \exp g_i(\theta)\big)^2}
\left\|\sum\limits_{i=1}^m \exp g_i(\theta) \nabla_\theta g_i(\theta)\right\|^2
\le \frac{m\sum\limits_{i=1}^m (\exp g_i(\theta))^2\left\|\nabla_\theta g_i(\theta)\right\|^2 }{\big(\sum\limits_{i=1}^m \exp g_i(\theta)\big)^2} 
\le m G^2
\end{array}
\end{equation*}
Thus, $\left\|\nabla_\theta \bar{g}(\theta)\right\| \le \sqrt{m}G$ and the second condition is satisfied.

\noindent
For $\left\|\nabla_\theta^{\prime\prime} \bar{g}(\theta)\right\|_2$, we have 
\begin{equation*}
\begin{array}{ll}
\nabla_\theta^{\prime\prime} \bar{g}(\theta) =& \underbrace{\frac{1}{\sum\limits_{i=1}^m \exp g_i(\theta)}
\Bigg[\sum\limits_{i=1}^m \exp g_i(\theta) \nabla_\theta^{\prime\prime} g_i(\theta) + \exp g_i(\theta)\nabla_\theta g_i(\theta)\nabla_\theta g_i(\theta)^\top\Bigg]}_A \\
& - \underbrace{\frac{1}{\sum\limits_{i=1}^m \exp g_i(\theta)}\Big(\sum\limits_{i=1}^m \exp g_i(\theta)\nabla_\theta g_i(\theta)\Big)\Big(\sum\limits_{i=1}^m \exp g_i(\theta)\nabla_\theta g_i(\theta)^\top\Big)}_B
\end{array}
\end{equation*}
To upper bound $\left\|\nabla_\theta^{\prime\prime} \bar{g}(\theta)\right\|_2$, which is 
\begin{equation*}
\max\limits_{u^\top u = 1} u^\top\nabla_\theta^{\prime\prime} \bar{g}(\theta)u = \max\limits_{u^\top u=1}u^\top Au 
- u^\top Bu\le \max\limits_{u^\top u=1}u^\top Au
\end{equation*}
where the inequality is due to the fact that $B\succeq 0$.

\noindent
Thus, we have $\left\|\nabla_\theta^{\prime\prime} \bar{g}(\theta)\right\|_2\le \left\|A\right\|_2$.
For the $\left\|A\right\|_2$, we have 
\begin{equation*}
\begin{array}{ll}
\left\|A\right\|_2 &= \max\limits_{u^\top u = 1} u^\top Au 
\le \frac{1}{\sum\limits_{i=1}^m \exp g_i(\theta)}\Big(\sum\limits_{i=1}^m\max\limits_{u^\top u = 1}
\exp g_i(\theta)u^\top \nabla_\theta^{\prime\prime}g_i(\theta)u\Big) \\
&\quad\quad+\frac{1}{\sum\limits_{i=1}^m \exp g_i(\theta)}\Big(\sum\limits_{i=1}^m\max\limits_{u^\top u = 1}
\exp g_i(\theta)\left\|\nabla_\theta g_i(\theta)^\top u\right\|^2\Big) \\
&\le \frac{1}{\sum\limits_{i=1}^m \exp g_i(\theta)}\Big(\sum\limits_{i=1}^m \exp g_i(\theta)(L_i + \left\|\nabla_\theta g_i(\theta)\right\|^2)\Big)\\
&\le \frac{1}{\sum\limits_{i=1}^m \exp g_i(\theta)}\Big(\sum\limits_{i=1}^m \exp g_i(\theta)\Big)(\bar{L} + G^2) = \bar{L} + G^2
\end{array}
\end{equation*}
where the first inequality comes from the optimality definition, the second inequality comes from the upper bound for each 
$\left\|\nabla_\theta^{\prime\prime}g_i(\theta)\right\|_2$ and the Cauchy - Schwarz inequality,
and the last inequality comes from the fact that $\bar{L} = \max L_i$ and $\left\|\nabla_\theta g_i(\theta)\right\|$ is upper bounded by $G$.
Thus, the last condition is also satisfied.

\end{proof}

\noindent
\paragraph{Proof of the Proposition \ref{prop::similarResultTo2012}:}
\begin{proof}
From Theorem \ref{thm::sumOfSquareLongterm}, we know that $\sum\limits_{t=1}^T\Big([g_i(\theta_t)]_+\Big)^2 \le O(\sqrt{T})$.
By using the inequality $(y_1+y_2+...+y_n)^2\le n(y_1^2+y_2^2+...+y_n^2)$,
setting $y_i$ being equal to $[g_i(\theta_t)]_+$, and $n = T$, we have 
$\Big(\sum\limits_{t=1}^T[g_i(\theta_t)]_+\Big)^2 \le T\sum\limits_{t=1}^T\Big([g_i(\theta_t)]_+\Big)^2 \le O(T^{3/2})$.
Then we obtain that $\sum\limits_{t=1}^T[g_i(\theta_t)]_+ \le O(T^{3/4})$. 
Because $g_i(\theta_t)\le [g_i(\theta_t)]_+$, we also have $g_i(\theta_t) \le O(T^{3/4})$.
\end{proof}

\noindent
\paragraph{Proof of the Proposition \ref{prop::tradeOffLossAndConstraint}:}
\begin{proof}

Since we only change the step size for Algorithm \ref{alg::convex-long-term}, 
the previous result in Lemma \ref{lem:sumOfLagfunction} and part of the proof 
up to Eq.(\ref{eq::UpperBoundOfsumOfObjAndLongTermConstrain}) in Theorem \ref{thm::sumOfSquareLongterm}
can be used without any changes.

\noindent
First, let us rewrite the Eq.(\ref{eq::UpperBoundOfsumOfObjAndLongTermConstrain}):
\begin{equation}
\label{eq::key_our}
\begin{array}{rl}
\sum\limits_{t=1}^T\Big(f_t(\theta_t) - f_t(\theta^*)\Big) & +
\sum\limits_{i=1}^m\sum\limits_{t=1}^T\frac{([g_i(\theta_t)]_+)^2}{\sigma\eta}\Big(1-\frac{(m+1)G^2}{2\sigma}\Big)
\le \frac{R^2}{2\eta}+\frac{\eta T}{2}(m+1)G^2
\end{array}
\end{equation}

\noindent
By plugging in the definition of $\alpha$, $\eta$, and that $\frac{([g_i(\theta_t)]_+)^2}{\sigma\eta}\alpha\ge 0$,
we have
\begin{equation*}
\begin{array}{ll}
\sum\limits_{t=1}^T\Big(f_t(\theta_t)-f_t(\theta^*)\Big)&\le \frac{R^2}{2}T^{\beta}+\frac{(m+1)G^2}{2}T^{1-\beta}
                                              =O(T^{max\{\beta,1-\beta\}}) 
\end{array}
\end{equation*}

\noindent
As argued in the proof of Theorem \ref{thm::sumOfSquareLongterm},
we have the following inequalities with the help of
$\sum\limits_{t=1}^T\Big(f_t(\theta_t) - f_t(\theta^*)\Big)\ge -FT$:
\begin{equation}
\label{eq::constrain_ineq_our}
\begin{array}{l}
\sum\limits_{i=1}^m\sum\limits_{t=1}^T\frac{([g_i(\theta_t)]_+)^2}{\sigma\eta}\alpha 
\le \frac{R^2}{2}T^{\beta} + \frac{(m+1)G^2}{2}T^{1-\beta}+FT \\
\sum\limits_{t=1}^T([g_i(\theta_t)]_+)^2\le \frac{\sigma}{\alpha}(\frac{R^2}{2}+\frac{(m+1)G^2}{2}T^{1-2\beta}+FT^{1-\beta})
\end{array}
\end{equation}

\noindent
Then we have 
\small
\begin{equation*}
\begin{array}{l}
\sum\limits_{t=1}^T[g_i(\theta_t)]_+\le \sqrt{T\sum\limits_{t=1}^T\Big([g_i(\theta_t)]_+\Big)^2} 
\le\sqrt{\frac{T\sigma}{\alpha}\Big(\frac{R^2}{2}+\frac{(m+1)G^2}{2}T^{1-2\beta}+FT^{1-\beta}\Big)}
=O(T^{1-\beta/2})
\end{array}
\end{equation*}
\normalsize

\end{proof}

\noindent
It is also interesting to figure out why \cite{mahdavi2012trading} 
cannot have this user-defined trade-off benefit.
From \cite{mahdavi2012trading}, the key inequality in obtaining their conclusions is:
\begin{equation}
\label{eq::key_2012}
\begin{array}{l}
\sum\limits_{t=1}^T \Big(f_t(x_t)-f_t(x^*)\Big)+
\sum\limits_{i=1}^m\frac{\Big[\sum\limits_{t=1}^Tg_i(x_t)\Big]_+^2}{2(\sigma\eta T+m/\eta)}\\
\le \frac{R^2}{2\eta} + \frac{\eta T}{2}\Big((m+1)G^2+2mD^2\Big)
\end{array}
\end{equation} 
The main difference between Eq.(\ref{eq::key_2012}) and Eq.(\ref{eq::key_our})
is in the denominator of $\frac{\Big[\sum\limits_{t=1}^Tg_i(x_t)\Big]_+^2}{2(\sigma\eta T+m/\eta)}$.
Eq.(\ref{eq::key_2012}) has the form $(\sigma\eta T+m/\eta)$, 
while Eq.(\ref{eq::key_our}) has the form $(\sigma\eta)$. 
The coupled $\eta$ and $1/\eta$ prevents Eq.(\ref{eq::key_2012}) from arriving this user-defined trade-off.

\noindent
The next proofs of the Proposition \ref{prop::true_violation_bound_2011} and \ref{prop::true_violation_bound_2016}
show how we can use our proposed Lagrangian function in Eq.(\ref{eq::new long term lagrangian})
to make the algorithms in \cite{mahdavi2012trading} and \cite{jenatton2016adaptive}
to have the clipped long-term constraint violation bounds.

\noindent
\paragraph{Proof of the Proposition \ref{prop::true_violation_bound_2011}:}
\begin{sproof}

If we look into the proof of Lemma 2 and Proposition 3 in \cite{mahdavi2012trading}, 
the new Lagrangian formula does not lead to any difference,
which means that the $\mathcal{L}_t(\theta,\lambda)$ defined in Eq.~(\ref{eq::new long term lagrangian})
is also valid for the drawn conclusions.
Then in the proof of Theorem 4 in \cite{mahdavi2012trading}, we can change $g_i(\theta_t)$
to $[g_i(\theta_t)]_+$. 
The maximization for $\lambda$ over the range $[0,+\infty)$ is also valid, 
since $[g_i(\theta_t)]_+$ automatically satisfies this requirement.
Thus, the claimed bounds hold.
\end{sproof}

\noindent
\paragraph{Proof of the Proposition \ref{prop::true_violation_bound_2016}:}
\begin{sproof}

The previous augmented Lagrangian formula $\mathcal{L}_t(\theta,\lambda)$ used in \cite{jenatton2016adaptive} is:
\begin{equation*}
\label{eq::old_lag_2016}
\mathcal{L}_t(\theta,\lambda) = f_t(\theta)+\lambda g(\theta) - \frac{\phi_t}{2}\lambda^2
\end{equation*}
The Lemma 1 in \cite{jenatton2016adaptive} is the upper bound of $\mathcal{L}_t(\theta_t,\lambda) - \mathcal{L}_t(\theta_t,\lambda_t)$.
The proof does not make any difference between formula (\ref{eq::old_lag_2016}) and (\ref{eq::new_lag_2016}).
So we can still have the same conclusion of Lemma 1.
The Lemma 2 in \cite{jenatton2016adaptive} is the lower bound of $\mathcal{L}_t(\theta_t,\lambda) - \mathcal{L}_t(\theta^*,\lambda_t)$.
Since it only uses the fact that $g(\theta^*)\le 0$, which is also true for $[g(\theta^*)]_+$,
we can have the same result with $g(\theta_t)$ being replaced with $[g(\theta_t)]_+$.
The Lemma 3 in \cite{jenatton2016adaptive} is free of $\mathcal{L}_t(\theta,\lambda)$ formula, so it is also true for the new formula.
The Lemma 4 in \cite{jenatton2016adaptive} is the result of Lemma 1-3, 
so it is also valid if we change $g(\theta_t)$ to $[g(\theta_t)]_+$.
Then the conclusion of Theorem 1 in \cite{jenatton2016adaptive} is valid for $[g(\theta_t)]_+$ as well.
\end{sproof}

\paragraph{Proof of Theorem~\ref{thm:gen_conv}:}
\begin{proof}
Due to the non-expansiveness of the projection in Eq.~\eqref{eq:convex_update_x},
we can get 
\begin{equation*}
\begin{array}{ll}
\|z_t-\theta_{t+1}\|^2 & \le\|z_t-\theta_t+\eta\nabla_\theta\cL_t(\theta_t,\lambda_t)\|^2 \\
&=\|z_t-\theta_t\|^2+\eta^2\|\nabla_\theta\cL_t(\theta_t,\lambda_t)\|^2+2\eta\langle z_t-\theta_t,\nabla_\theta\cL_t(\theta_t,\lambda_t)\rangle
\end{array}
\end{equation*}
which can be reformulated as
\begin{equation}
\label{eq:proj_ineq_conv}
\langle \theta_t-z_t,\nabla_\theta\cL_t(\theta_t,\lambda_t)\rangle \le \frac{1}{2\eta}(\|z_t-\theta_t\|^2-\|z_t-\theta_{t+1}\|^2)
+\frac{\eta}{2}\|\nabla_\theta\cL_t(\theta_t,\lambda_t)\|^2
\end{equation}

\noindent
Due to the convexity of $\cL_t(\theta,\lambda)$ in terms of $\theta$, we have
\begin{equation}
\label{eq:convexity_ineq}
\cL_t(\theta_t,\lambda_t)-\cL_t(z_t,\lambda_t) \le \langle \theta_t-z_t,\nabla_\theta\cL_t(\theta_t,\lambda_t)\rangle
\end{equation}

\noindent
Plugging Eq.~\eqref{eq:convexity_ineq} into Eq.~\eqref{eq:proj_ineq_conv}, we get
\begin{equation}
\label{eq:con_update_ineq}
\cL_t(\theta_t,\lambda_t)-\cL_t(z_t,\lambda_t) \le \frac{1}{2\eta}(\|z_t-\theta_t\|^2-\|z_t-\theta_{t+1}\|^2)
+\frac{\eta}{2}\|\nabla_\theta\cL_t(\theta_t,\lambda_t)\|^2
\end{equation}

\noindent
Next, we analyze the term $\|z_t-\theta_{t+1}\|^2$:
\begin{equation}
\label{eq:path_length_ineq}
\begin{array}{ll}
\|z_t-\theta_{t+1}\|^2 &= \|z_t-z_{t+1}+z_{t+1}-\theta_{t+1}\|^2 \\
&=\|z_t-z_{t+1}\|^2+2\langle z_{t+1}-\theta_{t+1},z_t-z_{t+1}\rangle+\|z_{t+1}-\theta_{t+1}\|^2
\end{array}
\end{equation}

\noindent
Since both $\theta_1^T$ and $z_1^T$ are in $\cS_0$, $\|\theta_{t+1}-z_{t+1}\|\le D$
and $\langle z_{t+1}-\theta_{t+1},z_t-z_{t+1}\rangle \ge -D\|z_{t+1}-z_t\|$
due to the assumption that the diameter of $\cS_0$ is $D$ and Cauchy-Schwarz inequality.
Thus, Eq.~\eqref{eq:path_length_ineq} can be lower bounded as
\begin{equation}
\label{eq:z_t-x_t+1-lower}
\|z_t-\theta_{t+1}\|^2\ge -2D\|z_{t+1}-z_t\| + \|z_{t+1}-\theta_{t+1}\|^2
\end{equation}

\noindent
Plug the above inequality into Eq.~\eqref{eq:con_update_ineq} gives
\begin{equation*}
\cL_t(\theta_t,\lambda_t)-\cL_t(z_t,\lambda_t) \le \frac{D}{\eta}\|z_{t+1}-z_t\|
 +\frac{\eta}{2}\|\nabla_\theta\cL_t(\theta_t,\lambda_t)\|^2
 +\frac{1}{2\eta}(\|z_t-\theta_t\|^2-\|z_{t+1}-\theta_{t+1}\|^2)
\end{equation*}

\noindent
For $\|\nabla_\theta\cL_t(\theta_t,\lambda_t)\|^2$ $=$ $\|\nabla_\theta f_t(\theta_t)+\lambda_t\nabla_\theta[g_t(\theta_t)]_+\|^2$,
it can be upper bounded by $2G^2+2G^2\lambda_t^2$ 
due to the fact that $\|a+b\|^2\le 2\|a\|^2+2\|b\|^2$
and both $\|\nabla_\theta f_t(\theta_t)\|$ and $\|\nabla_\theta[g_t(\theta_t)]_+\|$ are upper bounded by $G$.
Plugging this upper bound into the above inequality gives
\begin{equation*}
\cL_t(\theta_t,\lambda_t)-\cL_t(z_t,\lambda_t) \le 
\frac{1}{2\eta}(\|z_t-\theta_t\|^2-\|z_{t+1}-\theta_{t+1}\|^2) + \frac{D}{\eta}\|z_{t+1}-z_t\|
+\eta G^2\lambda_t^2 + \eta G^2
\end{equation*}

\noindent
Expanding the left part of the above inequality and using $\lambda_t = \frac{[g_t(\theta_t)]_+}{\sigma\eta}$
results in
\begin{equation*}
\begin{array}{l}
f_t(\theta_t)-f_t(z_t)+\frac{([g_t(\theta_t)]_+)^2}{\sigma\eta}(1-\frac{G^2}{\sigma})\\
\le \frac{1}{2\eta}(\|z_t-\theta_t\|^2-\|z_{t+1}-\theta_{t+1}\|^2) + \frac{D}{\eta}\|z_{t+1}-z_t\|
+ \eta G^2+\frac{[g_t(\theta_t)]_+[g_t(z_t)]_+}{\sigma\eta}
\end{array}
\end{equation*}

\noindent
Since $\sigma =2G^2$, and $[g_t(\theta_t)]_+$ is upper bounded by $F$,
we can get:
\small
\begin{equation*}
f_t(\theta_t)-f_t(z_t)+\frac{([g_t(\theta_t)]_+)^2}{4G^2\eta} \le
 \frac{1}{2\eta}(\|z_t-\theta_t\|^2-\|z_{t+1}-\theta_{t+1}\|^2) + \frac{D}{\eta}\|z_{t+1}-z_t\|
+ \eta G^2+\frac{F[g_t(z_t)]_+}{2G^2\eta}
\end{equation*}
\normalsize

\noindent
Summing over $t=1$ to $T$, setting $z_{T+1} = z_T$, and using  $\|z_1-x_1\|^2\le D^2$ gives
\begin{equation*}
\sum\limits_{t=1}^T\Big(f_t(\theta_t)-f_t(z_t)\Big) + \frac{1}{4G^2\eta}\sum\limits_{t=1}^T([g_t(\theta_t)]_+)^2
\le \frac{D^2}{2\eta}+\frac{DV}{\eta}+G^2\eta T+\frac{F}{2G^2\eta}\sum\limits_{t=1}^T[g_t(z_t)]_+
\end{equation*}

\noindent
According to the definition of the sequence $z_1^T$, the number of $[g_t(z_t)]_+ = 0$ is $K$.
With the bound of $[g_t(z_t)]_+\le F$, 
$\sum\limits_{t=1}^T[g_t(z_t)]_+ \le (T-K)F$.

\noindent
Thus, the above inequality can be reformulated as
\begin{equation*}
\sum\limits_{t=1}^T\Big(f_t(\theta_t)-f_t(z_t)\Big) + \frac{1}{4G^2\eta}\sum\limits_{t=1}^T([g_t(\theta_t)]_+)^2
\le \frac{D^2}{2\eta}+\frac{DV}{\eta}+G^2\eta T+\frac{F^2}{2G^2\eta}(T-K)
\end{equation*}

\noindent
If we plug in the expression of $\eta = O(\sqrt{\frac{T-K+1+V}{T}})$, we get
\begin{equation*}
\sum\limits_{t=1}^T\Big(f_t(\theta_t)-f_t(z_t)\Big) + \frac{1}{4G^2\eta}\sum\limits_{t=1}^T([g_t(\theta_t)]_+)^2
\le O(\sqrt{T(T-K+1+V)})
\end{equation*}

\noindent
Since $ \frac{1}{4G^2\eta}\sum\limits_{t=1}^T([g_t(\theta_t)]_+)^2\ge 0$,
$\sum\limits_{t=1}^T\Big(f_t(\theta_t)-f_t(z_t)\Big)\le O(\sqrt{T(T-K+1+V)})$.

\noindent
Also, $\Big(f_t(\theta_t)-f_t(z_t)\Big)\ge -FT$.
Then $ \frac{1}{4G^2\eta}\sum\limits_{t=1}^T([g_t(\theta_t)]_+)^2\le O(T)$,
which results in $\sum\limits_{t=1}^T([g_t(\theta_t)]_+)^2\le O(\sqrt{T(T-K+1+V)})$.

\end{proof}

\paragraph{Proof of Theorem~\ref{thm:s-conv}:}
\begin{proof}
Since the update is the same as the one in Eq.~\eqref{eq:convex_update} except the time-dependent parameters,
we have the same inequality as in Eq.~\eqref{eq:proj_ineq_conv}:
\begin{equation}
\label{eq:s-conv_proj}
\langle \theta_t-z_t,\nabla_\theta\cL_t(\theta_t,\lambda_t)\rangle \le \frac{1}{2\eta_t}(\|z_t-\theta_t\|^2-\|z_t-\theta_{t+1}\|^2)
+\frac{\eta_t}{2}\|\nabla_\theta\cL_t(\theta_t,\lambda_t)\|^2
\end{equation}

\noindent
Since $f_t$ is $\ell$ strongly convex, $\cL_t(\theta,\lambda)$ is also $\ell$ strongly convex as:
\begin{equation*}
\langle \theta_t-z_t,\nabla_\theta\cL_t(\theta_t,\lambda_t)\rangle 
\ge \cL_t(\theta_t,\lambda_t) -\cL_t(z_t,\lambda_t) +\frac{\ell}{2}\|z_t-\theta_t\|^2
\end{equation*}

\noindent
Plugging the above inequality into Eq.~\eqref{eq:s-conv_proj} gives:
\begin{equation*}
\cL_t(\theta_t,\lambda_t)-\cL_t(z_t,\lambda_t) \le \frac{1}{2\eta_t}(\|z_t-\theta_t\|^2-\|z_t-\theta_{t+1}\|^2)
+\frac{\eta_t}{2}\|\nabla_\theta\cL_t(\theta_t,\lambda_t)\|^2 - \frac{\ell}{2}\|z_t-\theta_t\|^2
\end{equation*}

\noindent
Due to the lower bound for $\|z_t-\theta_{t+1}\|^2$ as in Eq.~\eqref{eq:z_t-x_t+1-lower}
and the upper bound for $\|\nabla_\theta\cL_t(\theta_t,\lambda_t)\|^2$ as $2G^2+2G^2\lambda_t^2$
as in the proof of Theorem \ref{thm:gen_conv},
we get
\scriptsize
\begin{equation*}
\cL_t(\theta_t,\lambda_t)-\cL_t(z_t,\lambda_t) \le \frac{D}{\eta_t}\|z_{t+1}-z_t\|
+\frac{1}{2\eta_t}(\|z_t-\theta_t\|^2-\|z_{t+1}-\theta_{t+1}\|^2)
- \frac{\ell}{2}\|z_t-\theta_t\|^2 + \eta_tG^2+\eta_tG^2\lambda_t^2
\end{equation*}
\normalsize

\noindent
Due to the concavity of $\cL_t(\theta,\lambda)$ in terms of $\lambda$,
we have
\begin{equation*}
\cL_t(\theta_t,\lambda)-\cL_t(\theta_t,\lambda_t) 
\le \langle \nabla_\lambda\cL_t(\theta_t,\lambda_t),\lambda-\lambda_t\rangle = 0
\end{equation*}
where the equality is due to the update of $\lambda_t$.

\noindent
Adding the above two inequalities and reformulating with $\phi_t = 2G^2\eta_t$ and $g_t(\theta_t)\le F$ gives
\begin{equation*}
\begin{array}{ll}
f_t(\theta_t)-f_t(z_t)+\lambda[g_t(\theta_t)]_+-\frac{\phi_t}{2}\lambda^2 &\le \frac{D}{\eta_t}\|z_{t+1}-z_t\|
+\frac{1}{2\eta_t}(\|z_t-\theta_t\|^2-\|z_{t+1}-\theta_{t+1}\|^2) \\
&\quad- \frac{\ell}{2}\|z_t-\theta_t\|^2 + \eta_tG^2+ \frac{F}{\phi_t}[g_t(z_t)]_+
\end{array}
\end{equation*}

\noindent
Summing over from $t=1$ to $T$ gives
\footnotesize
\begin{equation*}
\begin{array}{ll}
\sum\limits_{t=1}^T\Big(f_t(\theta_t)-f_t(z_t)\Big)+\lambda\sum\limits_{t=1}^T[g_t(\theta_t)]_+
-\frac{\lambda^2}{2}\sum\limits_{t=1}^T\phi_t \\
\le D\sum\limits_{t=1}^T\frac{\|z_{t+1}-z_t\|}{\eta_t}
+(\frac{1}{2\eta_1}-\frac{\ell}{2})\|z_1-\theta_1\|^2 +G^2\sum\limits_{t=1}^T\eta_t
+\sum\limits_{t=2}^T(\frac{1}{2\eta_t}-\frac{1}{2\eta_{t-1}}-\frac{\ell}{2})\|z_t-\theta_t\|^2
 + F\sum\limits_{t=1}^T\frac{[g_t(z_t)]_+}{\phi_t}
\end{array}
\end{equation*}
\normalsize

\noindent
Since $\eta_t = \frac{1-\gamma}{\ell(1-\gamma^t)}$ and $\phi_t = 2G^2\eta_t$,
$\frac{1}{2\eta_1}-\frac{\ell}{2} = 0$
and
$\frac{1}{2\eta_t}-\frac{1}{2\eta_{t-1}}-\frac{\ell}{2} \le 0$.
Also, $\frac{1}{\eta_t}\le \frac{\ell}{1-\gamma}$
and $\frac{1}{\phi_t}\le \frac{1}{2G^2}\frac{\ell}{1-\gamma}$.

\noindent
Thus, by setting $z_{T+1}=z_T$ and using the fact that $z_1^T\in V_K(z_1^T)$, the above inequality can be simplified as 
\begin{equation*}
\begin{array}{ll}
\sum\limits_{t=1}^T\Big(f_t(\theta_t)-f_t(z_t)\Big)+\lambda\sum\limits_{t=1}^T[g_t(\theta_t)]_+
-\frac{\lambda^2}{2}\sum\limits_{t=1}^T\phi_t \\
\le \frac{D\ell}{1-\gamma}V
+\frac{G^2(1-\gamma)}{\ell}\sum\limits_{t=1}^T\frac{1}{1-\gamma^t}
 + \frac{F^2\ell}{2G^2(1-\gamma)}(T-K)
\end{array}
\end{equation*}

\noindent
Maximizing the LHS over $\lambda$, we get $\lambda = \frac{\sum\limits_{t=1}^T[g_t(\theta_t)]_+}{\sum\limits_{t=1}^T\phi_t}$,
which gives
\small
\begin{equation*}
\sum\limits_{t=1}^T\Big(f_t(\theta_t)-f_t(z_t)\Big)+\frac{1}{2}\frac{\Big(\sum\limits_{t=1}^T[g_t(\theta_t)]_+\Big)^2}{\sum\limits_{t=1}^T\phi_t}
\le \Big(D\ell V+ \frac{\ell F^2(T-K)}{2G^2}\Big)\frac{1}{1-\gamma}+\frac{G^2(1-\gamma)}{\ell}\sum\limits_{t=1}^T\frac{1}{1-\gamma^t}
\end{equation*}
\normalsize

\noindent
For the first term on the RHS, 
since $\frac{1}{1-\gamma} = 2\sqrt{\frac{(D+1)T}{\max\{V+(T-K),\log^2T/T\}}} \le 2\sqrt{\frac{(D+1)T}{V+(T-K)}}$,
it can be upper bounded by $O(\sqrt{T(V+T-K)})$.

\noindent
For the second term on the RHS,
since $\sum\limits_{t=1}^T\frac{1}{1-\gamma^t}\le O(T)$ according to the proof of Corollary 3 in \cite{yuan2019trading},
it can be upper bounded by $\max\{O(\sqrt{T(V+T-K)}),O(\log T)\}$.

\noindent
Thus,
\begin{equation*}
\sum\limits_{t=1}^T\Big(f_t(\theta_t)-f_t(z_t)\Big)
+\frac{1}{2}\frac{\Big(\sum\limits_{t=1}^T[g_t(\theta_t)]_+\Big)^2}{\sum\limits_{t=1}^T\phi_t} 
\le \max\{O(\sqrt{T(V+T-K)}),O(\log T)\}
\end{equation*}

\noindent
As a result, $\sum\limits_{t=1}^T\Big(f_t(\theta_t)-f_t(z_t)\Big)\le \max\{O(\sqrt{T(V+T-K)}),O(\log T)\}$.

\noindent
Since $\sum\limits_{t=1}^T\Big(f_t(\theta_t)-f_t(z_t)\Big)\ge -FT$,
the upper bound for
$\sum\limits_{t=1}^T[g_t(\theta_t)]_+$ is obtained by using the definition of $\phi_t$ and the upper bound of it.

\end{proof}


\end{document}